\documentclass[nohyperref]{article}
\usepackage[table]{xcolor}

\usepackage[accepted]{icml2024}
\usepackage{natbib}

\usepackage[normalem]{ulem}

\usepackage{dsfont}
\usepackage{mathtools}
\usepackage{amssymb,bm}
\usepackage{cancel}
\usepackage[outline]{contour}
\usepackage[normalem]{ulem}
\usepackage{amsthm}
\usepackage{amsmath}
\usepackage[colorlinks = true,
            linkcolor = blue,
            urlcolor  = blue,
            citecolor = blue]{hyperref}
\usepackage{thmtools}
\usepackage{thm-restate}
\usepackage{caption}
\usepackage{pifont}
\usepackage{subcaption}
\usepackage{comment}
\usepackage{hyperref}
\usepackage{booktabs}
\usepackage{multirow}

\DeclarePairedDelimiter\abs{\lvert}{\rvert}%
\newcommand{\norm}[1]{\left\lVert#1\right\rVert}

\DeclareMathOperator*{\argmin}{arg\,min}

\theoremstyle{plain}
\newtheorem{theorem}{Theorem}[section]
\newtheorem{proposition}[theorem]{Proposition}
\newtheorem{lemma}[theorem]{Lemma}

\theoremstyle{definition}
\newtheorem{definition}[theorem]{Definition}
\newtheorem{assumption}{Assumption}
\theoremstyle{remark}
\newtheorem{remark}[theorem]{Remark}

\usepackage{ulem}
\usepackage{mathtools}
\usepackage{amssymb}

\newcommand{\R}{{\mathbb{R}}}

\usepackage[capitalize,noabbrev]{cleveref}
\PassOptionsToPackage{table}{xcolor}

\begin{document}

%

%

\twocolumn[

\icmltitle{Dynamic Survival Analysis with Controlled Latent States
}

\icmlsetsymbol{equal}{*}

\begin{icmlauthorlist}
\icmlauthor{Linus Bleistein}{equal,inria,crc,evry}
\icmlauthor{Van-Tuan Nguyen}{equal,inria,califrais,lpsm}
\icmlauthor{Adeline Fermanian}{califrais}
\icmlauthor{Agathe Guilloux}{inria,crc}
\end{icmlauthorlist}

\icmlaffiliation{inria}{Inria Paris, F-75015 Paris, France}
\icmlaffiliation{crc}{Centre de Recherche des Cordeliers, INSERM, Université de Paris, Sorbonne Université, F-75006 Paris, France}
\icmlaffiliation{evry}{LaMME, UEVE and UMR 8071, Paris Saclay University, F-91042, Evry, France}
\icmlaffiliation{califrais}{LOPF, Califrais’ Machine Learning Lab, Paris, France}
\icmlaffiliation{lpsm}{Laboratoire de Probabilités, Statistique et Modélisation, LPSM, Univ. Paris Cité, F-75005, Paris, France}

\icmlcorrespondingauthor{Linus Bleistein}{linus.bleistein@inria.fr}

\icmlkeywords{Signature Methods, Time series, Dynamical Systems, Applications to Healthcare.}

\vskip 0.3in
]



\printAffiliationsAndNotice{\icmlEqualContribution} 

\begin{abstract}
We consider the task of learning individual-specific intensities of counting processes from a set of static variables and irregularly sampled time series. We introduce a novel modelization approach in which the intensity is the solution to a controlled differential equation. We first design a neural estimator by building on neural controlled differential equations. In a second time, we show that our model can be linearized in the signature space under sufficient regularity conditions, yielding a signature-based estimator which we call CoxSig. We provide theoretical learning guarantees for both estimators, before showcasing the performance of our models on a vast array of simulated and real-world datasets from finance, predictive maintenance and food supply chain management.
\end{abstract}

\section{Introduction}

Time-to-event data is ubiquitous in numerous fields such as meteorology, economics, healthcare and finance. We typically want to predict when an event - which can be a catastrophic earthquake, the burst of a housing bubble, the onset of a disease or a financial crash - will occur by using some prior historical information \citep{ogata1988statistical, bacry2015hawkes, bussy2019comparison}. This general problem encompasses many settings and in particular survival analysis, where every individual experiences at most one event \citep{cox1972regression}.

For an individual $i$, we have access to several event times $T^i_1 < T^i_{2} < \dots $ and features $\mathbf{W}^i \in \mathbb{R}^s$ measured at time $0$. For instance, in neurology, one might consider the onset times of a series of seizures \citep{rasheed2020machine} and $\mathbf{W}^i$ summarizes unchanging characteristics of the individual (age, gender, ethnicity, \dots). The physician's goal is to determine whether an individual has a high probability to experience a seizure at time $t$ given their characteristics. Such a task is most often addressed by modelling the individual specific intensity of a counting process of the form $\sum_{j \geq 1} \mathds{1}_{T^i_j \leq t}$,
using for instance Cox models \citep{cox1972regression,aalen2008survival, kvamme2019time} or Hawkes processes in the case of self-exciting processes \citep{bacry2015hawkes}. Recent advances in the field have also enriched these models using deep architectures \citep{mei2017neural,kvamme2019time,omi2019fully,chen2020neural,groha2020general,shchur2021neural,de2022predicting, tang2022soden}. Once learnt, the intensity of the process can be used to predict occurrence times of future events or rank individuals based on their relative risks.    

\paragraph{Learning with Time-dependent Data.} More realistically, in addition to the static features $\mathbf{W}^{i}$, we often have access to time-dependent features along with their sampling times 
\[\mathbf{X}^i =: \{(\mathbf{X}^i(t_1),t_1),\dots, (\mathbf{X}^i(t_K),t_K)\} \in \mathbb{R}^{d\times K},\] 
where $D= \{t_1,\dots,t_K\} \subset [0, \tau]$ is a set of measurement times and $\tau$ the end of study. Taking again the example of seizure prediction, the time-dependent features may represent some measurements made by a wearable device, as done for instance by \citet{dumanis2017seizure}. Taking both the static and time-dependent information into account is crucial when making predictions. This setting calls for highly flexible models of the intensity which take into account the stream of information carried by the longitudinal features. 

\paragraph{From joint models to ODE-based methods.} This problem has been tackled by the bio-statistics community, in particular using joint models that concurrently fit parametric models to the trajectory of the longitudinal features and the intensity of the counting process \citep{ibrahim2010basic,crowther2013joint,proust2014joint,long2018joint}. Popular implementations include \texttt{JMBayses} \citep{rizopoulos2014r}. While being highly interpretable, they do not scale to high-dimensional and frequently measured data, despite some recent algorithmic advances~\citep{hickey2016joint,murray2022fast,rustand2024fast} adapted to moderate dimension (up to $\simeq 5$ longitudinal features).

Modern deep methods, that can encode complex and meaningful patterns from complex data in latent states, offer a particularly attractive alternative for this problem. However, the literature bridging the gap between deep learning and survival analysis is scarce. Notably, \citet{lee2019dynamic} tackle this problem by embedding the time-dependent data through a recurrent neural network combined with an attention mechanism. They then use this embedding in a discrete-time setting to maximize the likelihood of dying in a given time-frame conditional on having survived until this time. \citet{moon2022survlatent} combine a probabilistic model with a continuous-time neural network, namely the ODE-RNNS of \citet{rubanova2019latent} in a similar setup.

\paragraph{Modelling Time Series with Controlled Latent States.} Building on the increasing momentum of differential equation-based methods for learning \citep{chen2018neural,de2019gru,rubanova2019latent,chen2020neural,moon2022survlatent,marion2022scaling}, we propose a novel modelling framework in which the unknown intensity of the counting process is parameterized by a latent state driven by a controlled differential equation (CDE). Formally, we let the unknown intensity of the counting process of individual $i$ depend on their covariates $\mathbf{W}^i$ and an unobserved process $x^i: [0, \tau] \to \mathbb{R}^d$ that is the continuous unobserved counterpart of the time series $\mathbf{X}^i$, i.e., $(\mathbf{X}^i(t),t) = x^i(t)$ for all $t \in D$. We model the intensity (i.e. the instantaneous probability of experiencing an event --- see Section \ref{section:intensity}) by setting
\begin{equation}
    \label{eq:true_intensity}
    \lambda_\star^i \big(t\,|\, \mathbf{W}^i,(x^i(s) )_{s \leq t} \big) = \exp\big( z_\star^i(t) + \beta^\top_\star \mathbf{W}^i\big),
\end{equation}
where the dynamical latent state $z^i_\star(t) \in \mathbb{R}$ is the solution to the CDE
\begin{align}
\label{eq:true_CDE}
dz^i_\star(t) = \mathbf{G}_\star \big(z^i_\star(t) \big)^\top dx^i(t)
\end{align}
with initial condition $z_\star^i(0) = 0$ driven by $x^i$. Here, the vector field $\mathbf{G}_\star:\mathbb{R} \to \mathbb{R}^{ d}$ and $\beta_\star \in \mathbb{R}^s$ are both unknown. This means that the latent dynamics are common between individuals, but are driven by individual-specific data, yielding individual-specific intensities. Such a modelling strategy is reminiscent of state space models, which embed times series through linear controlled latent differential equations \citep{gu2021efficiently, cirone2024theoretical}. Our framework is introduced in more detail later.

\paragraph{Contributions.} In an effort to provide scalable and efficient models for event-data analysis, we propose two novel estimators. We first leverage neural CDEs \citep{kidger2020neural}, which directly approximate the vector field $\mathbf{G}_\star$ with a neural vector field $\mathbf{G}_\psi$. In a second time, following \citet{fermanian2021framing} and \citet{bleistein2023learning}, we propose to linearize the unknown dynamic latent state $z^i_\star(\cdot)$ in the signature space. Informally, this means that at any time $t$, we have the simplified expression
\[
z_\star^i(t) \approx \alpha_{\star,N}^\top \mathbf{S}_N(x^i_{[0,t]})
\]
where $\alpha_{\star,N}$ is an unknown finite-dimensional vector and $\mathbf{S}_N(x^i_{[0,t]})$ is a deterministic transformation of the time series $x^i$ observed up to time $t$ called the \textit{signature transform}. Notice that in this form, the vector $\alpha_{\star,N}$ does not depend on $t$ and can hence be learned at any observation time. We obtain theoretical guarantees for both models ; for the second model in particular, we state a precise decomposition of the variance and the discretization bias of our estimator, which crucially depends on the coarseness of the sampling grid $D$. Finally, we benchmark both methods on simulated and real-world datasets from finance, healthcare and digital food retail, in a survival analysis setting. Our signature-based estimator provides state-of-the-art results.

\paragraph{Summary.} Section \ref{section:models_assumptions} details our theoretical framework. In Section \ref{section:theoretical_signature}, we state theoretical guarantees for our model. Lastly, we conduct a series of experiments in Section \ref{section:experiments} that displays the strong performances of our models against an array of benchmarks. All proofs are given in the appendix. The code is available at \url{https://github.com/LinusBleistein/signature_survival}.

\section{Modelling Point Processes with Controlled Latent States}
\label{section:models_assumptions}

\subsection{The Data}
\label{subsection:data}

In practice, an individual can be censored (for example after dropping out from a study) or cannot experience more than a given number of events. To take this into account, we introduce $Y^i:[0,\tau] \to \{0,1\}$ the at-risk indicator function, which equals $1$ when the individual $i$ is still at risk of experiencing an event. Together with $Y^i$, we define $$N^i(t) := \sum\limits_{j \geq 1} \mathds{1}_{T^i_j \leq t} Y^i(T^i_j)$$ as the stochastic process counting the number of events experienced by individual $i$ up to time $t$ and while $Y^i(T^i_j)=1$. Our dataset 
\[\mathcal{D}_n := \{\mathbf{X}^i, \mathbf{W}^i,Y^i(t), N^i(t),\, 0 \leq t \leq \tau\}
\] 
consists of $n$ i.i.d. historical observations up to time $\tau$. Our setup can be extended to individual-dependent grids $(D^i)_{i=1}^n$, but we choose to focus on the former setting for the sake of clarity. The individual specific time series are only observed as long as the individual is at risk. We first make an assumption on the time series. 

\begin{assumption}
\label{assumption:continuous_path}
    For every individual $i=1,\dots,n$, there exists a continuous path of bounded variation $x^i:[0,\tau] \to \mathbb{R}^d$ satisfying, for all $0 \leq s < t \leq \tau$,
    \[
    \norm{x^i}_{\textnormal{1-var},[s,t]} := \sup_D \sum\limits_{k} \big\|x^i({t_{k+1}})-x^i(t_k)\big\| \leq L_x\abs{t-s}
    \]
    where $\norm{\cdot}$ is the Euclidean norm and the supremum is taken over all finite dissections $D = \left\{s = t_1 < \dots < t_K = t \right\}$. 
    The time series $\mathbf{X}^i$ is a discretization of $x^i$ on the grid $D$.
\end{assumption}
Remark that this assumption implies that the paths are $L_x$-Lipschitz. We now state a supplementary assumption on the static features.

\begin{assumption}
\label{assumption:bounded_static_features}
    There exists a constant $B_\mathbf{W} > 0$ such that for every $i=1,\dots,n$, $\norm{\mathbf{W}^i}_2 \leq B_\mathbf{W}$.
\end{assumption}




\subsection{Modelling Intensities with Controlled Differential Equations}
\label{section:intensity}
\paragraph{Intensity of a counting process.} We define the individual-specific intensity $\lambda_\star^i \big(t\,|\, \mathbf{W}^i,x^i_{[0,t]} \big)$ of the underlying counting process, which we will simply write $\lambda^i_\star(t)$ in the following, as
\begin{align*}
    \lambda^i_\star(t):= \lim\limits_{h \to 0^+} \frac{1}{h} \mathbb{E}\big(N^i(t+h)-N^i(t)\,|\, \mathcal{F}^i_t\big)
\end{align*}
where $\mathcal{F}^i_t$ is the past information at time $t$ which includes $\mathbf{W}^i$ and $x^i_{[0,t]}$ \citep{aalen2008survival}.   

\paragraph{Controlled Dynamics.} Controlled differential equations are a theoretical framework that allows to generalize ODEs beyond the non-autonomous regime \cite{lyons2007differential}. Recall that a non-autonomous ODE is the solution to 
\[
    dz(t) = \mathbf{F}(z(t),t)dt
\]
with a given initial value $z(0) = z_0 \in \mathbb{R}^p$. Here, the vector field $\mathbf{F}:\mathbb{R}^p \times [0,+\infty[ \to \mathbb{R}^p$ depends explicitly on $t \geq 0$, allowing for time-varying dynamics unlike autonomous ODEs whose dynamics remain unchanged through time. Controlled differential equations can be seen as a generalization of non-autonomous ODEs. They allow for the vector field to depend explicitly on the values of another $\mathbb{R}^d$-valued function $x:[0,1] \to \mathbb{R}^d$ through 
\begin{align*}
    dz(t) =  \mathbf{\tilde F}(z(t), x(t))dt 
\end{align*}
thus encoding even richer dynamics. Formally, a CDE writes
\begin{align*}
& dz(t) = \mathbf{G}\big(z(t)\big)dx(t)\\
& z(0) = z_0 \in \mathbb{R}^p
\end{align*}
where $\mathbf{G}$ is a $\mathbb{R}^{p \times d}$-valued vector field. Existence and uniqueness of the solution is ensured under regularity conditions on $\mathbf{G}$ and $x$ by the Picard-Lindelhöf Theorem (see Theorem \ref{thm:pl_theorem}). The following assumption is needed in order to ensure that the function
\begin{equation*}
    \lambda^i_\star(t) = \exp\big( z_\star^i(t) + \beta^\top_\star \mathbf{W}^i\big),
\end{equation*}
where the dynamical latent state $z^i_\star(t) \in \mathbb{R}$ is the solution to the CDE
\begin{align*}
dz^i_\star(t) = \mathbf{G}_\star \big(z^i_\star(t) \big)^\top dx^i(t)
\end{align*}
with initial condition $z_\star^i(0) = 0$ driven by $x^i$ is well-defined. 
\begin{assumption}
\label{assumption:true_vf_bounded}
    The vector field $\mathbf{G}_\star:\mathbb{R} \to \mathbb{R}^{ d}$ defining $\lambda_\star^i$ in Equation \eqref{eq:true_CDE} is $L_{\mathbf{G}_\star}$-Lipschitz; $\beta_\star$ is such that $\norm{\beta_\star}_2 \leq B_{\beta,2}$, $\norm{\beta_\star}_1 \leq B_{\beta,1}$ and $\norm{\beta_\star}_0 \leq B_{\beta,0}$, where $B_{\beta,2},B_{\beta,1}, B_{\beta,0} >0$ are constants.
\end{assumption}

Under these assumptions, the intensity is bounded at all times.

\begin{lemma}[A bound on the intensity]\label{lemma:intensity_bounded} 
For every individual $i=1,\dots,n$ and all $t \in [0,\tau]$, the log intensity $\log \lambda_\star^i(t)$ is upper bounded by
    \begin{align*}
        B_{\beta,2} B_\mathbf{W} +  \norm{\mathbf{G}_\star(0)}_{\textnormal{op}} L_xt \exp\big(L_{\mathbf{G}_\star}L_xt \big)
    \end{align*}
    almost surely. 
\end{lemma}

This is a direct consequence of Lemma 3.3 in \citet{bleistein2023generalization}. Remark that $\norm{\mathbf{G}_\star(0)}_{\textnormal{op}} < \infty$ since the vector field is Lipschitz and hence continuous.
\begin{remark}
    By differentiation, one can see that the intensity itself satisfies a so-called controlled Volterra differential equation \citep{lin2020controlled}. Indeed, differentiating the intensity $\lambda^i_\star$ yields the CDE
    \begin{align*}
        d\lambda^i_\star(s) = \lambda^i_\star(s) \mathbf{G}_\star(z^i_\star(s))dx^i(s)
    \end{align*}
    with initial condition $\lambda^i_\star(0) = \exp(\beta_\star^\top \mathbf{W}^i)$. Note that this CDE is path dependent, i.e., its vector field depends on the path $z_\star^i:[0,\tau] \to \mathbb{R}$.   
\end{remark}

\begin{remark}
    This model enforces continuity of the intensity: indeed, the solution of a CDE inherits the regularity of its driving path. A possible solution to accommodate discontinuous intensity functions is to add a jump term to the generative CDE, which could then be learnt using neural jump ODEs \citep{jia2019neural}. 
\end{remark}
\subsection{Neural Controlled Differential Equations}
\label{sub:NCDEs}
Following the ideas of continuous time models, our first approach to learning the dynamics is to fit a parameterized intensity to this model by setting 
\begin{align*}
    \lambda^i_\theta(s) = \exp(\alpha^\top z^i_\theta(s) + \beta^\top \mathbf{W}^i),
\end{align*}
where $z^i_\theta(s) \in \mathbb{R}^p$ is an embedding of the time series $\mathbf{X}^i$ parameterized by $\theta \in \mathbb{R}^v$ and $\alpha \in \mathbb{R}^p$ is a learnable parameter. We propose to use Neural Controlled Differential Equations (NCDEs), a powerful tool for embedding irregular time series introduced by \citet{kidger2020neural}. NCDEs work by first embedding a time series $\mathbf{X}^i$ in the space of functions of bounded variation, yielding $x^{i, D}:[0,\tau] \to \mathbb{R}^d$, before defining a representation of the data through 
\begin{align*}
    dz_\theta(t) = \mathbf{G}_\psi \big( z_\theta(s) \big)dx^{i, D}(s)
\end{align*}
with initial condition $z_\theta(0) = \mathbf{0}$. It is common practice to set $\mathbf{G}_\psi:\mathbb{R}^{p} \to \mathbb{R}^{p \times d}$ to be a small feed-forward neural network parameterized by $\psi$. The learnable parameters of this model are thus $\theta = (\alpha,\psi,\beta)$. In our setup, the embedding must be carefully chosen in order not to leak information from the future observations. Hence natural cubic splines, used in the original paper by \citet{kidger2020neural}, cannot be used and we resort to the piecewise constant interpolation scheme proposed by \citet{morrill2021neural} and defined as $x^{i,D}(s) = (\mathbf{X}^i(t_k), s) \textnormal{ for all } s \in [t_k,t_{k+1}[$. This yields a discretely updated latent state equal to 
\begin{align*}
    z^{i,D}_\theta(t_k) = z^{i,D}_\theta(t_{k-1}) + \mathbf{G}_\psi(z^{i,D}_\theta(t_{k-1}))\Delta \mathbf{X}^i(t_{k})
\end{align*}
where $\Delta \mathbf{X}^i(t_{k}) =\mathbf{X}^i(t_{k})-\mathbf{X}^i(t_{k-1})$. This architecture has been studied under the name of \textit{controlled ResNet} because of its resemblance with the popular ResNet \cite{cirone2023neural, bleistein2023generalization}.

In order to provide theoretical guarantees, we restrict ourselves to a bounded set of NCDEs i.e. we consider a set of NCDE predictors 
\begin{align*}
    \Theta_{1} = \{\theta \in \mathbb{R}^v \textnormal{s.t.} \norm{\alpha}_2 \leq B_\alpha, \norm{\psi} \leq B_\psi, \norm{\beta}_2 \leq B_{\beta,2}\}
\end{align*}
where the norm on $\psi$ refers to the sum of $\ell_2$ norms of the weights and biases of the neural vector field $\mathbf{G}_\psi$.
This restriction is fairly classical in statistical learning theory \cite{bach2021learning}. 

\subsection{Linearizing CDEs in the Signature Space}  

\paragraph{The Signature Transform.} While neural controlled differential equations allow for great flexibility in representation of the time series, they are difficult to train and require significant computational resources. The signature is a promising and theoretically well-grounded tool from stochastic analysis, that allows for a parameter-free embedding of the time series. Mathematically, the signature coefficient of a function 
$$x:t\in [0,\tau] \mapsto \big(x^{(1)}(t),\dots, x^{(d)}(t) \big)$$ 
associated to a word $I=(i_1,\dots,i_k) \in \{1,\dots, d\}^k$ of size $k$ is the function 
\begin{align*}
    \mathbf{S}^I(x_{[0,t]}):= \int_{0 < u_1 < \dots < u_k < t} dx ^{(i_1)}(u_1) \dots dx ^{(i_k)}(u_k)
\end{align*}
which maps $[0,\tau]$ to $\mathbb{R}$. The integral is to be understood as the Riemann-Stieltjes integral. While the definition of the signature is technical, it can simply be seen as a feature extraction step. We refer to Figure \ref{fig:example_sig} for an illustration. The truncated signature of order $N \geq 1$, which we write $\mathbf{S}_N(x_{[0,t]})$, is equal to the collection of all signature coefficients associated to words of size $k \leq N$ sorted by lexicographical order. Finally, the infinite signature is the sequence defined through 
\begin{align*}
    \mathbf{S}(x_{[0,t]}) = \lim\limits_{N \to + \infty}\mathbf{S}_N(x_{[0,t]}).
\end{align*}

\paragraph{Learning with Signatures.} Signatures are a prominent tool in stochastic analysis since the pioneering work of \citet{chen1958integration} and \citet{lyons2007differential}. They have recently found successful applications in statistics and machine learning as a feature representation for irregular time series \citep{kidger2019deep,morrill2020generalised,fermanian2021embedding,salvi2021higher, fermanian2022functional, lyons2022signature,bleistein2023learning, horvath2023optimal} and a tool for analyzing residual neural networks in the infinite depth limit \cite{fermanian2021framing}.

\begin{figure}
    \centering
    \includegraphics[width=\columnwidth]{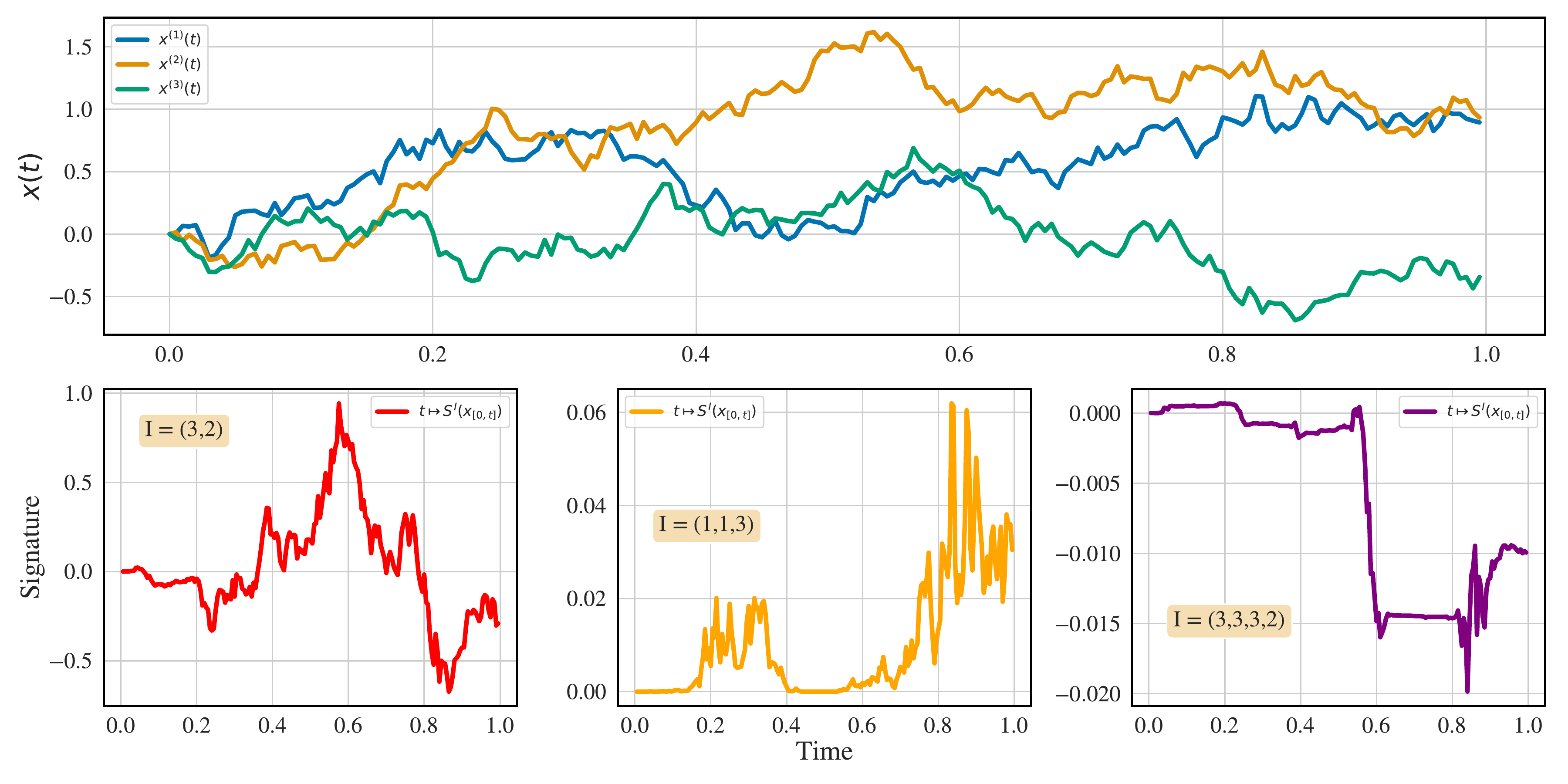}
    \caption{Sample path $x(t)$ of a $3$-dimensional fractional Brownian motion on \textbf{top}, and three signature coefficients $\mathbf{S}^I(x_{[0,t]})$ associated to different words on the \textbf{bottom}.}
    \label{fig:example_sig}
\end{figure}

\paragraph{Signatures and CDEs.} An appealing feature of signatures is their connection to controlled differential equations. Indeed, under sufficient regularity assumptions \citep{friz2010multidimensional,fermanian2021framing,bleistein2023learning, cirone2023neural}, the generative CDE \eqref{eq:true_intensity} can be linearized in the signature space. Informally, this means that there exists a sequence $\alpha_\star$ such that for all $t \in [0,\tau]$ we have
\begin{align*}
    z^i_\star(t) = \alpha_\star^\top \mathbf{S}(x^i_{[0,t]}). 
\end{align*}
The mathematical definition of $\alpha_\star$ is technical and we refer to Appendix \ref{appendix:signture_linearization} for a formal statement and a discussion of the regularity assumptions. 
Hence, under the corresponding regularity conditions, the true intensity for individual $i$ writes 
\begin{align*}
  \lambda_\star^i(t) = \exp \big(\alpha_\star^\top \mathbf{S}(x^i_{[0,t]})+ \beta_\star^\top \mathbf{W}^i \big).
\end{align*}

This motivates the use of the signature-based estimator
\begin{align*}
    \lambda_\theta^{i,D}(t) :=  \exp \big(\alpha^\top \mathbf{S}_N(x^{i,D}_{[0,t]}) + \beta^\top \mathbf{W}^i \big),
\end{align*}
where $\theta = (\alpha,\beta) \in \mathbb{R}^{q}\times \mathbb{R}^s$, $N \geq 1$ is treated as a hyperparameter and  $x^{i,D}$
corresponds to the piecewise constant embedding of the observed time series $\mathbf{X}^{i}$ described previously.
The integer $q=\frac{d^{N-1}-1}{d-1}$ is the size of the signature truncated at depth $N \geq 1$. The superscript in $D$ emphasises the dependence of this estimator on the observation grid $D$.

Similarly to the NCDE-based estimator, we restrict ourselves to the bounded set of estimators 
\begin{align*}
    \Theta_2 = \{\theta \textnormal{ s.t. } \norm{\alpha} \leq B_\alpha, \norm{\beta} \leq B_{\beta,2}\}. 
\end{align*}

\subsection{Connections to Cox Models with Time-Varying Covariates}

Cox models with time-varying covariates are the classical class of models 
\citep{therneau2000modeling,aalen2008survival,zhang2018time}, where the individual specific hazard rate has the form $\lambda_\theta^i(t) = \lambda_{0}(t)\exp(\alpha_\star^\top \mathbf X^i(t)+ \beta_\star^\top \mathbf{W}^i)$, where
 $\lambda_0:[0,\tau] \to \mathbb{R}_+$ is called the \textit{baseline hazard}.

 For signature-based embeddings, recall that we compute the signature of a time-embedded time series $\mathbf{X}^i=\{(\mathbf{X}^i(t_1),t_1),\dots, (\mathbf{X}^i(t_k),t_k)\}$. In fact, this amounts to
\begin{align*}
   &\alpha^\top \mathbf{S}_N(x^{i,D}_{[0,t]})\\
   & \quad = \underbrace{\sum\limits_{k = 0}^N \alpha_k t^k}_{= \log \lambda_0(t)} +\underbrace{ \alpha_{I_1}^\top \mathbf{X}^i(t)+ \sum\limits_{I \in I_2} \alpha_I \mathbf{S}^I(x^{i,D}_{[0,t]})}_{=\log \textnormal{ of individual specific hazard rate}} 
\end{align*}

where $\alpha_{I_1}$ is a subvector of $\alpha$ and $I_2 \subset \prod_{k=2}^N \{1,\dots,d\}^k$. Hence our model can be interpreted as a generalized version of Cox models with time-varying covariates. A similar interpretation holds for NCDEs. We detail this link in Appendix \ref{appendix:cox_connections}. 

\section{Theoretical Guarantees}
\label{section:theoretical_signature}

\subsection{The Learning Problem} 

For both models, the parameter $\theta$ can be fitted by likelihood maximization by solving 
\begin{align}
\label{eq:minimization_problem}
    \Hat{\theta} \in \argmin\limits_{\theta \in \Theta} \ell^D_n(\theta) + \textnormal{pen}(\theta),
\end{align}
where $\Theta \in \{\Theta_1,\Theta_2\}$ depending on whether one uses signature or NCDE-based embeddings,  $\textnormal{pen}:\Theta \to \mathbb{R}_+$ is a penalty and $\ell^D_n(\theta)$ is equal to the negative log-likelihood of the sample $\mathcal{D}_n$ evaluated at $\theta$. 

Unless specified other, the following statements hold for both NCDEs and signature-based embeddings (up to different constants given explicitly in the proofs). Following \citet{aalen2008survival}, the negative log likelihood $ \ell^D_n(\theta)$ of the sample writes
\begin{align*}
    \frac{1}{n}\sum\limits_{i=1}^n \int_0^\tau \lambda^{i,D}_\theta(s) Y^i(s) ds  -\int_0^\tau \log \lambda_\theta^{i,D}(s) dN^i(s),
\end{align*}
and we let 
\begin{align*}
    \ell_n^\star = \frac{1}{n}\sum\limits_{i=1}^n \int \lambda_\star^i(s) Y^i(s) ds -\int \log \lambda_\star^i(s)  dN^i(s)
\end{align*}
be the true likelihood of the data. Our goal, in this section, is to obtain a bias-variance decomposition of the difference
\[
\ell^D_n(\Hat{\theta}) - \ell^\star_n
\]
between the true likelihood and the likelihood of the learnt model. 
\subsection{A Risk Bound}

\begin{theorem}[Informal Risk Bound for the Signature Model]\label{theo:completeriskbound} Consider the signature-based embedding. Let $\Hat{\theta}$ be the solution of \eqref{eq:minimization_problem} with $\textnormal{pen}(\theta) = \eta_1\norm{\alpha}_1 + \eta_2\norm{\beta}_1$. For any $N \geq 1$, we have with high probability and an appropriate choice of $\eta_1,\eta_2$ that
\begin{align*}
    \ell^D_n(\Hat{\theta}) - \ell^\star_n  \leq & \textnormal{ Discretization bias} + \textnormal{Approximation bias} \\
    & + \mathcal{O}\Bigg(\sqrt{\frac{ \log N d^{N}}{n}}\Bigg) + \mathcal{O}\Bigg(\sqrt{\frac{\log s}{n}}\Bigg).
\end{align*}
\end{theorem}

For a formal statement, see Appendix \ref{appendix:proof_full_risk_bound}. We make a series of comments on this result. 

\begin{enumerate}
    \item This full risk bound can only be obtained for the signature-based model. It can also be extended to other types of penalty such as Ridge or Group Lasso (see for instance \citet{nardi2008asymptotic}) For NCDEs, we are able to give precise guarantees on the bias following \citet{bleistein2023generalization}, but a precise control of the variance term is out of reach.
    \item The discretization bias is proportional to $|D| := \max\limits_{i=1,\dots,K} |t_i-t_{i-1}|$ and hence vanishes as sampling gets finer. 
    \item The approximation bias crucially depends on the regularity of the unknown tensor field $\mathbf{G}_\star$, and more precisely on the speed of decay of its derivatives, which can be seen as a measure of smoothness of the target function. 
    \item The regularity assumptions made on $\mathbf{G}_\star$ are not necessary to bound the approximation bias of the NCDE model: in this case, this bias term depends on the approximation capacities of the neural tensor field.
    \item Remarkably, we obtain classical rates in $n^{-1/2}$ for the variance term. For signature based methods, fast rates in $n^{-1}$ are yet to be obtained.
\end{enumerate}

\section{Experimental Evaluation}
\label{section:experiments}

We now focus on the survival analysis setup. We hence let $T^i$ be the unique time-of-event, which may eventually be censored, of individual $i$. $\Delta^i$ is the censorship indicator, equal to $1$ if the individual experiences the event and to $0$ otherwise. 

\subsection{Training Setup}

We train on a dataset $\mathcal{D}_n$ of the same structure than described in Section \ref{subsection:data} and learn the parameter $\Hat{\theta}$ by solving the optimization problem \eqref{eq:minimization_problem}. NCDEs are trained without penalization, while we use a mixture of elastic-net penalties
\[
\textnormal{pen}(\theta) := \eta_1\textnormal{pen}_{\textnormal{EN}}(\alpha) + \eta_2\textnormal{pen}_{\textnormal{EN}}(\beta)
\]
for training the signature-based model, where 
$\textnormal{pen}_{\textnormal{EN}}(\cdot) = \gamma\norm{\cdot}_1 + (1-\gamma)\norm{\cdot}_2$. 
The hyperparameters $(\eta_1,\eta_2,N)$ are chosen by cross-validation of a mixed metric equal to the difference between the C-index and the Brier score (see below) and we set $\gamma = 0.1$. 
We refer to Appendix \ref{appendix:baselines} for a detailed description of the training procedures. We evaluate our model's capacity to predict events in $[t,t+\delta t]$ by leveraging values of the longitudinal features up to $t$ (see Figure \ref{fig:score_graph}) through a ranking metric and a calibration metric. This evaluation procedure is standard \citep{lee2019dynamic}.

\begin{figure}
    \centering
    \includegraphics[width=\columnwidth]{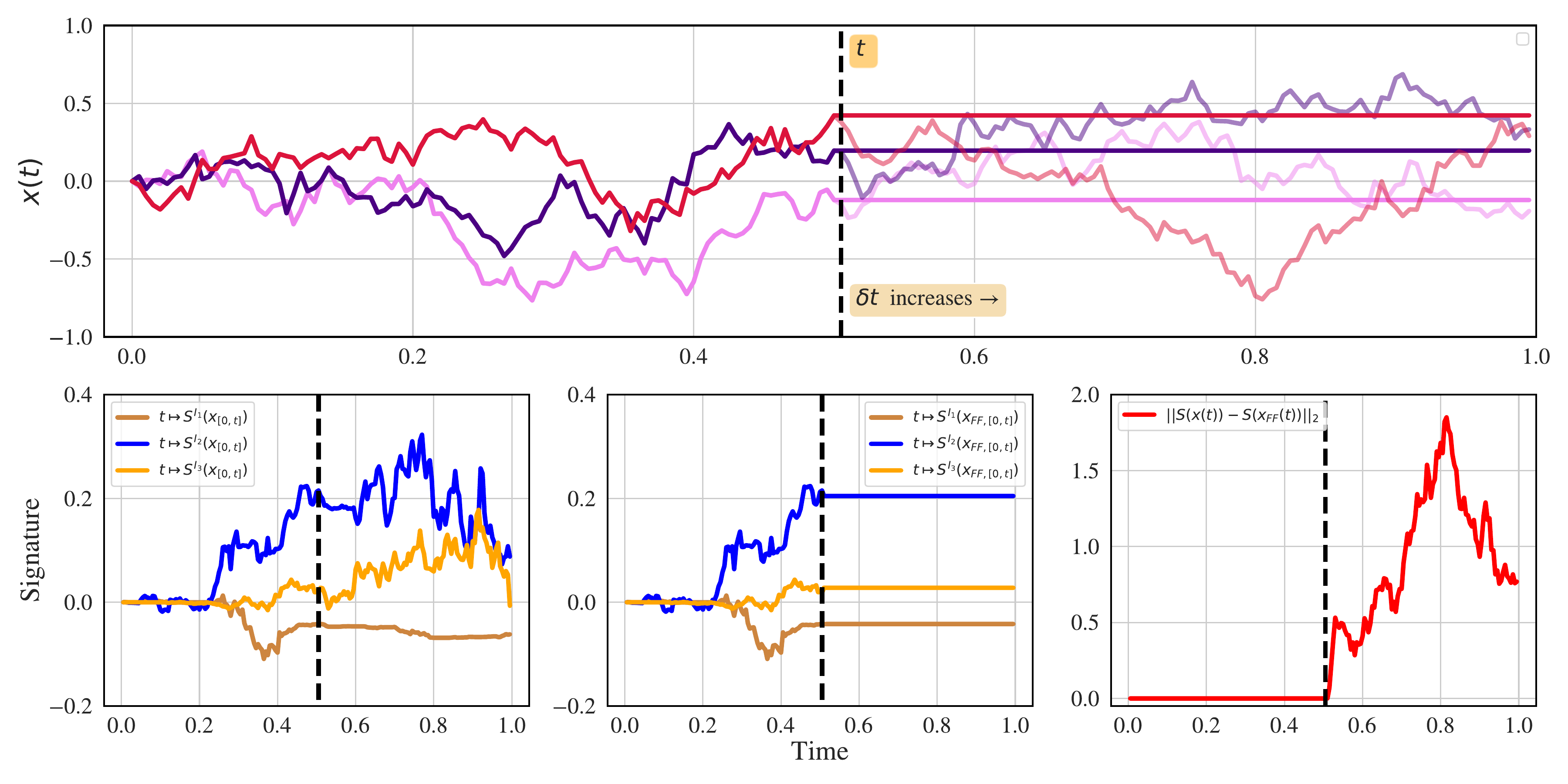}
    \caption{On the \textbf{top}, observed time series up to time $t$ in bold colors and true time series in faded colors. When evaluating our models, we fill-forward the last observed value from $t$ on. On the \textbf{bottom}, signatures of the true path (\textbf{left}), of the observed path (\textbf{center}) and difference in $\ell_2$ norm (\textbf{right}) --- $x_{FF}(t)$ denotes the filled-forward time series.}
    \label{fig:score_graph}
\end{figure}

\subsection{Metrics}

We compute four metrics using the individual specific survival functions as estimated by our model with parameters $\theta$. At time $t + \delta t$ for $\delta t >0$ conditional on survival up to time $t$, and on observation of the longitudinal features up to time $t$, it is defined as  
\[
r^i_\theta(t,\delta t) = \mathbb{P}\Big(T^i > t + \delta t\,|\, T^i > t, (\mathbf{X}^{i}(s))_{\substack{ s \leq t \\ s \in D}}, \mathbf{W}^i\Big).
\]
We describe its detailed computation in Appendix \ref{appendix:metrics}.

\paragraph{Time-dependent Concordance Index.} Following \citet{lee2019dynamic}, we measure the discriminative power of our models by using a time-dependent concordance index $C(t,\delta t)$ that captures our models ability to correctly rank individuals on their predicted probability of survival. The concordance index $C(t,\delta t)$ is then finally computed as 
\begin{align*}
\frac{\sum\limits_{j=1}^n \sum\limits_{i=1}^n\mathds{1}_{r_\theta^i(t,\delta t) > r^j_\theta(t,\delta t)} \mathds{1}_{T^i > T^j,\, T^j \in [t,t + \delta t], \Delta^j = 1}}{\sum\limits_{j=1}^n \sum\limits_{i=1}^n \mathds{1}_{T^i > T^j,\, T^j \in [t,t + \delta t], \Delta^j = 1}}.
\end{align*}
This metric captures the capacity of our model to discriminate between $j$ and another individual $i$ through the conditional probability of survival. 

\paragraph{Brier Score.} While the concordance index is a ranking-based measure, the Brier Score measures the accuracy in predictions by comparing the estimated survival function and the survival indicator function \citep{lee2019dynamic,kvamme2019time,kvamme2023brier}. Formally, we define the Brier score $\textnormal{BS}(t,\delta t)$ as
\begin{align*}
    &\frac{1}{n}\sum\limits_{i=1}^n \mathds{1}_{T^i \leq t + \delta t, \Delta^i=1} r^i_\theta(t,\delta t)^2 \\
    +&  \frac{1}{n}\sum\limits_{i=1}^n \mathds{1}_{T^i > t + \delta t} (1-r^i_\theta(t,\delta t))^2.
\end{align*}
Contrarily to the C-index, the Brier score is a measure of calibration of the predictions: it measures the distance between the estimated survival function and the indicator function of survival on the interval $[t,t+\delta t]$.

\paragraph{Averaged performance.} Additionally, we evaluate the average prediction performance of our models over a set of different prediction times. The averaged C-index and Brier score on the interval $[t_1, t_2]$ along with the window time $\delta t$ are defined respectively as 
\[\frac{1}{t_2 - t_1}\int_{t_1}^{t_2} C(s,\delta t) ds \quad \text{and} \quad \frac{1}{t_2 - t_1}\int_{t_1}^{t_2} BS(s,\delta t) ds. \]

\paragraph{Comparison with static metrics.} A crucial difference with static survival analysis metrics is that our metric only compares the individuals who experienced the event in this time window to all the ones who are still at risk at time t. This can lead to a C-index below $0.5$ and Brier scores above $0.25$ without the model being worse than random.

\paragraph{Additional metrics.} We furthermore report AUC and weighted Brier score in Appendix \ref{appendix:more_results}.

\begin{figure}
    \centering
    \includegraphics[width=1\columnwidth]{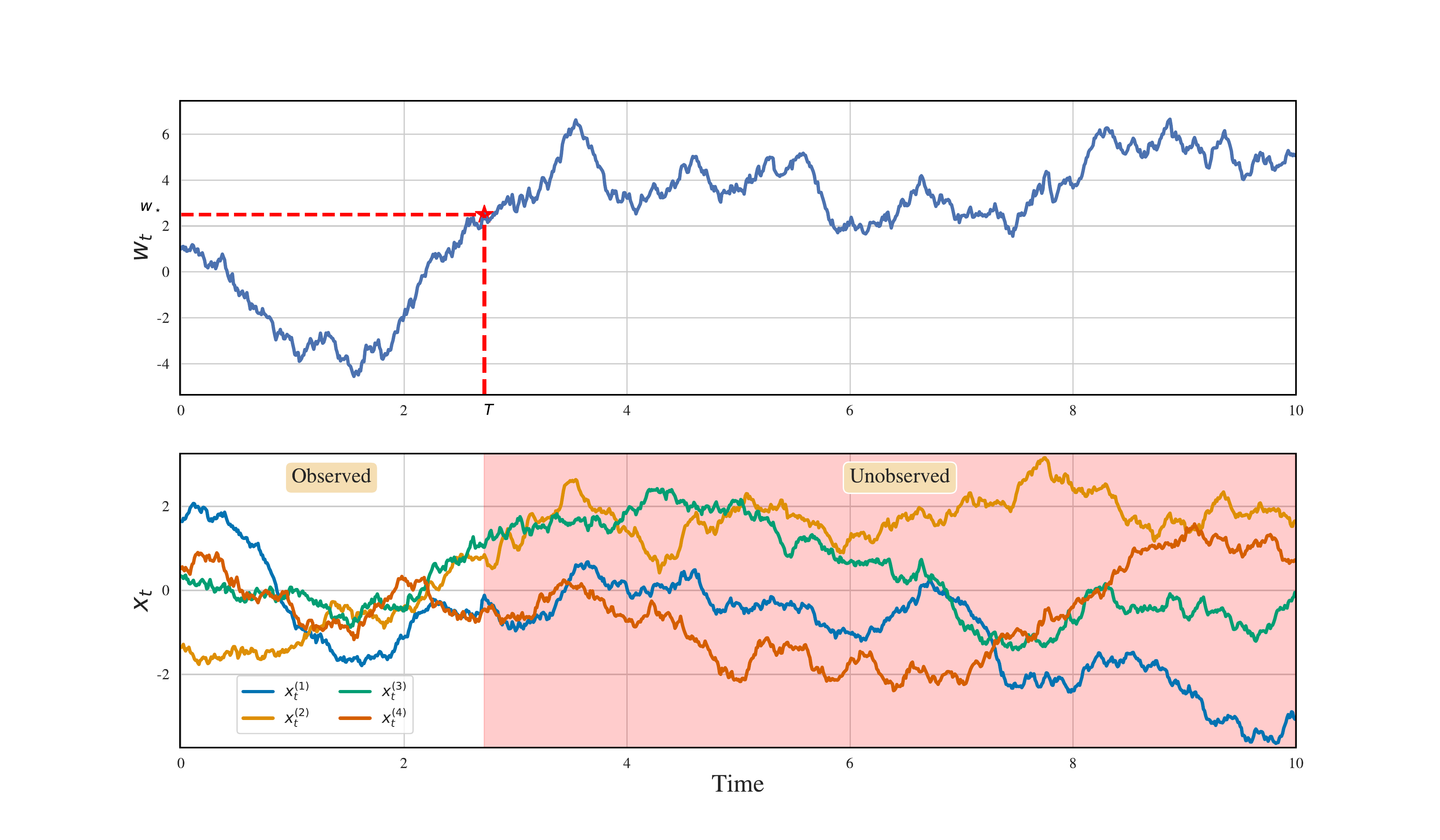}
    \caption{\footnotesize Time series $\mathbf{X}^i$ of a randomly picked individual on \textbf{bottom} and unobserved SDE $w^i(t)$ on the \textbf{top}. The red star indicates the first hitting time of the threshold value $w_\star = 2.5$.}
    \label{fig:path_rnd_individual}
\end{figure}

\subsection{Methods}

We propose three distinct methods. In addition to the signature-based model, which we call \textbf{CoxSig}, we also consider \textbf{CoxSig+} which adds the first value of the time series to the static features. This is motivated by the translation invariance of signatures (see discussion below). Our last method is the \textbf{NCDE} embedding of the longitudinal features.  We benchmark our three models against a set of competing methods. All methods are detailed in Appendix \ref{appendix:baselines}.

\paragraph{Time-Independent Cox Model.} As a sanity check, we implement a simple Cox model with elastic-net penalty which uses the parameterized intensity $\lambda^i_\theta(t) = \lambda_0(t)\exp(\beta^\top \mathbf{W}^i)$ using \texttt{scikit-survival} \citep{sksurv}. This baseline allows to check whether our proposed methods can make use of the supplementary time-dependent information. If no static features are available, we use the first observed value of the time series, i.e., $\mathbf{W}^i = \mathbf{X}^i(0)$. 

\paragraph{Random Survival Forest (RSF).} We use RSF \citep{ishwaran2008random} with static features $\mathbf{W}^i$ as the only input. Similarly to our implementation of the Cox model, we use the first value of the time series as static features if no other features are available.

\paragraph{Dynamic DeepHit \citep{lee2019dynamic}.} DDH is a state-of-the-art method for dynamical survival analysis, that combines an RNN with an attention mechanism and uses both time dependent and static features.

\paragraph{SurvLatent ODE \citep{moon2022survlatent}.} SLODE is a recent deep learning framework for survival analysis that leverages an ODE-RNN architecture \citep{rubanova2019latent} to handle the time dependent features. 

\subsection{Synthetic Experiments}

\begin{table}
\scriptsize
    \centering
    \begin{tabular}{lccccc}
        \toprule
         \textbf{Name} & $n$ & $d$ & \textbf{Censoring} & \textbf{Avg. Times} & \\
         \midrule
         Hitting time & 500 & 5 &  Terminal ($3.2 \%$) & 177 \\
         Tumor Growth & 500 & 2 &  Terminal ($8.4 \%$) & 250 \\
         Maintenance & 200 & 17  &  Online ($50 \%$) & 167 \\
        Churn & 1043 & 14 &  Terminal $(38.4 \%)$ & 25\\
        \bottomrule
    \end{tabular}
    \caption{Description of the 4 datasets we consider. The integer $d$ is the dimension of the time series including the time channel. \textit{Terminal} censoring means that the individuals are censored at the end of the overall observation period $[0,\tau]$ if they have not experienced any event. It is opposed to \textit{online} censoring that can happen at any time in $[0,\tau]$. The reported percentage indicates the censoring level i.e. the share of the population that does not experience the event. The last column reports the average number of observations times over individuals.} 
    \label{tab:data_summary}
\end{table}

\paragraph{Hitting time of a partially observed SDE.} Predicting hitting times is a crucial problem in finance --- for instance, when pricing so-called catastrophe bounds triggering a payment to the holder in case of an event \citep{cheridito2015pricing, corcuera2016pricing}. Their relation to survival analysis is well documented, see e.g.~\cite{lee2006threshold}. Building on this problem, we consider the Ornstein-Uhlenbeck SDE $$dw^i(t) = - \omega (w^i(t)-\mu)dt + \sum\limits_{j=1}^d dx^{(i,j)}(t) + \sigma dB^i(t)$$
where $d=5$, $\sigma=1$, $\mu=0.1$ and $\omega=0.1$ are fixed parameters. $x^i(t) = (x^{(i,1)}(t),\dots, x^{(i,d-1)}(t))$ is a sample path of a fractional Brownian motion with Hurst parameter $H=0.6$, and $B^i(t)$ is a Brownian noise term. In this setup, our data consists of $\mathbf{X}^i$ which is a downsampled version of $x^i$ and the Brownian part is unobserved. Our goal is to predict the first hitting time $\min\{t> 0 \,|\, w_t \geq w_\star \}$
of a threshold value $w_\star = 2.5$. We train on $n=500$ individuals. Figure \ref{fig:path_rnd_individual} shows the sample paths and SDE of a randomly selected individual. This setup is close to a well-specified model since signatures linearize controlled differential equations.

\begin{figure}
    \centering
    \includegraphics[width=1\columnwidth]{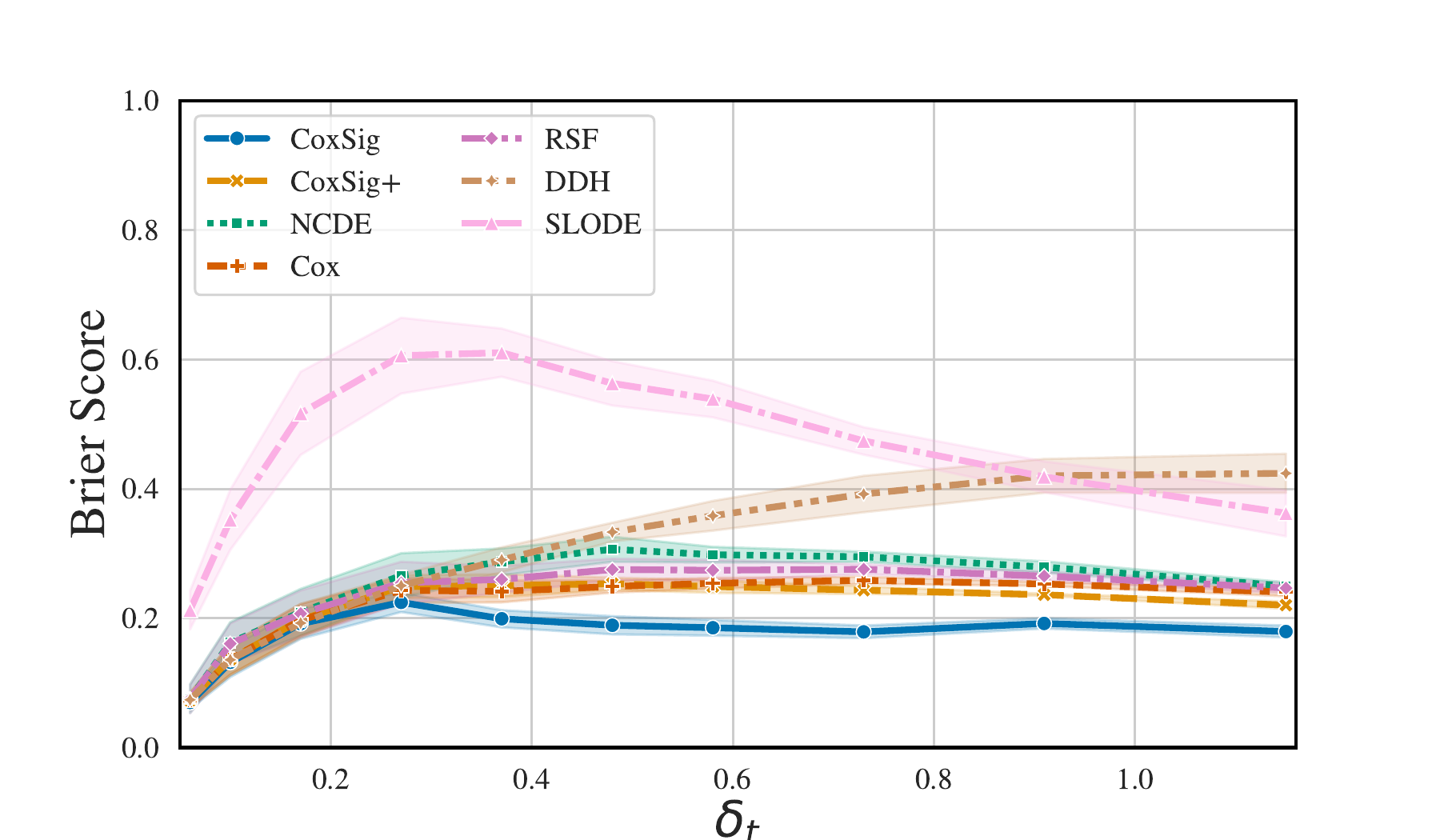}
    \caption{\footnotesize Brier score $\delta t \mapsto \textnormal{BS}(t,\delta t)$, evaluated at $t=0.23$, for the partially observed SDE experiment. Confidence intervals indicate 1 standard deviation. }
    \label{fig:evolution_bs}
\end{figure}

\paragraph{Tumor Growth.} We similarly aim at predicting the hitting time of a stochastic process modelling the growth of a tumor \citep{simeoni2004predictive}, where $x^i$ represents a drug-intake. In this experiment, the time series $\mathbf{X}^i$ is very-low dimensional ($d=2$, which includes the time channel).

\subsection{Real-World Datasets}

\paragraph{Predictive Maintenance.} \citep{saxena2008damage} This dataset collects simulations of measurements of sensors placed on aircraft gas turbine engines run until a threshold value is reached. In this context, the time-to-event is the failure time. This dataset features a small sample size, considerable censoring rates and a high number of time channels. 
 
\paragraph{Churn prediction.} We use a private dataset provided by Califrais, a food supply chain company that delivers fresh products from Rungis to food professionals. The company has access to a variety of features observed through time for every customer. Its goal is for example to predict when the customer will churn. The time series in this setup are high dimensional but sampled at a low frequency.

Further details on all datasets are provided in Appendix \ref{appendix:more_results}. Overall, our datasets are diverse in terms of sample size, size and length of the time series and censoring type.

\subsection{Results}

\begin{table}
\footnotesize
\begin{center}
\begin{tabular}{cccc}

\toprule
\multicolumn{2}{c}{\textbf{Algorithms}} & \textbf{Avg. C-Index} $\uparrow$ & \textbf{IBS} $\downarrow$\\[0.8ex]

\midrule
\multirow{4}{0.3cm}{\rotatebox{90}{OU}}
&  CoxSig & \contour{black}{0.857±0.01}  & \contour{black}{0.091±0.01} \\
& CoxSig+ & 0.857±0.01 & 0.095±0.01  \\
& NCDE & 0.517±0.04 & 0.103±0.01  \\
& DDH & 0.545±0.02 & 0.094±0.01  \\
& SLODE & 0.621±0.05 & 0.253±0.03 \\
\midrule
\multirow{4}{0.3cm}{\rotatebox{90}{Tumor}}
& CoxSig & 0.696±0.02 & 0.138±0.01 \\
& CoxSig+ & 0.797±0.03 & 0.137±0.01 \\
& NCDE & 0.827±0.02 & \contour{black}{0.130±0.01} \\
& DDH & \contour{black}{0.941±0.05}  & 0.133±0.01 \\
& SLODE & 0.601±0.07 & 0.136±0.01 \\
\midrule
\multirow{4}{0.3cm}{\rotatebox{90}{NASA}}
& CoxSig & 0.858±0.04 & 0.154±0.03 \\
& CoxSig+ & \contour{black}{0.867±0.04}  & 0.154±0.03 \\
& NCDE & 0.541±0.09 & 0.178±0.04  \\
& DDH & 0.813±0.06 & 0.156±0.02  \\
& SLODE & 0.438±0.14 & \contour{black}{0.145±0.02} \\
\midrule
\multirow{4}{0.3cm}{\rotatebox{90}{Califrais}}
& CoxSig & 0.741±0.01 & 0.130±0.01 \\
& CoxSig+ & \contour{black}{0.751±0.01}  & \contour{black}{0.129±0.01}  \\
& NCDE & 0.529±0.05 & 0.152±0.01 \\
& DDH & 0.570±0.03 & 0.139±0.01\\
& SLODE & 0.542±0.03 & 0.193±0.03 \\
\bottomrule
\end{tabular}
\end{center}
\caption{Averaged value of our metrics for 4 considered dataset over set of 10 different values of $t$ chosen from the $5$ to the $50$th percentile of the distribution of event times.
The values of $\delta t$ for each dataset is chosen to be the same as that shown in Figure \ref{fig:c_index}.}
\label{tab:avg_over_t}
\end{table}

\begin{figure*}
    \centering
    \includegraphics[width=0.33\textwidth]{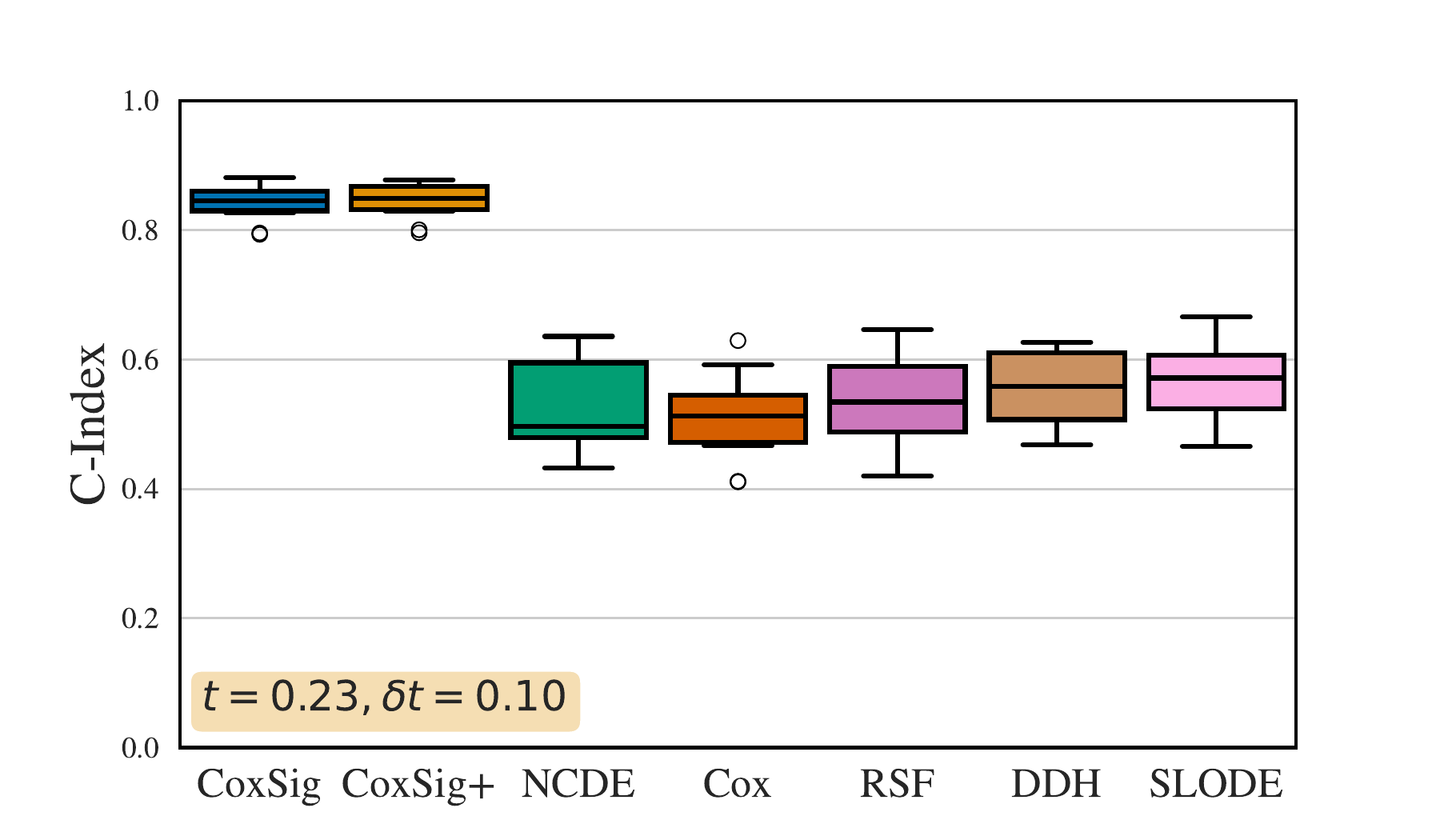}
    \includegraphics[width=0.33\textwidth]{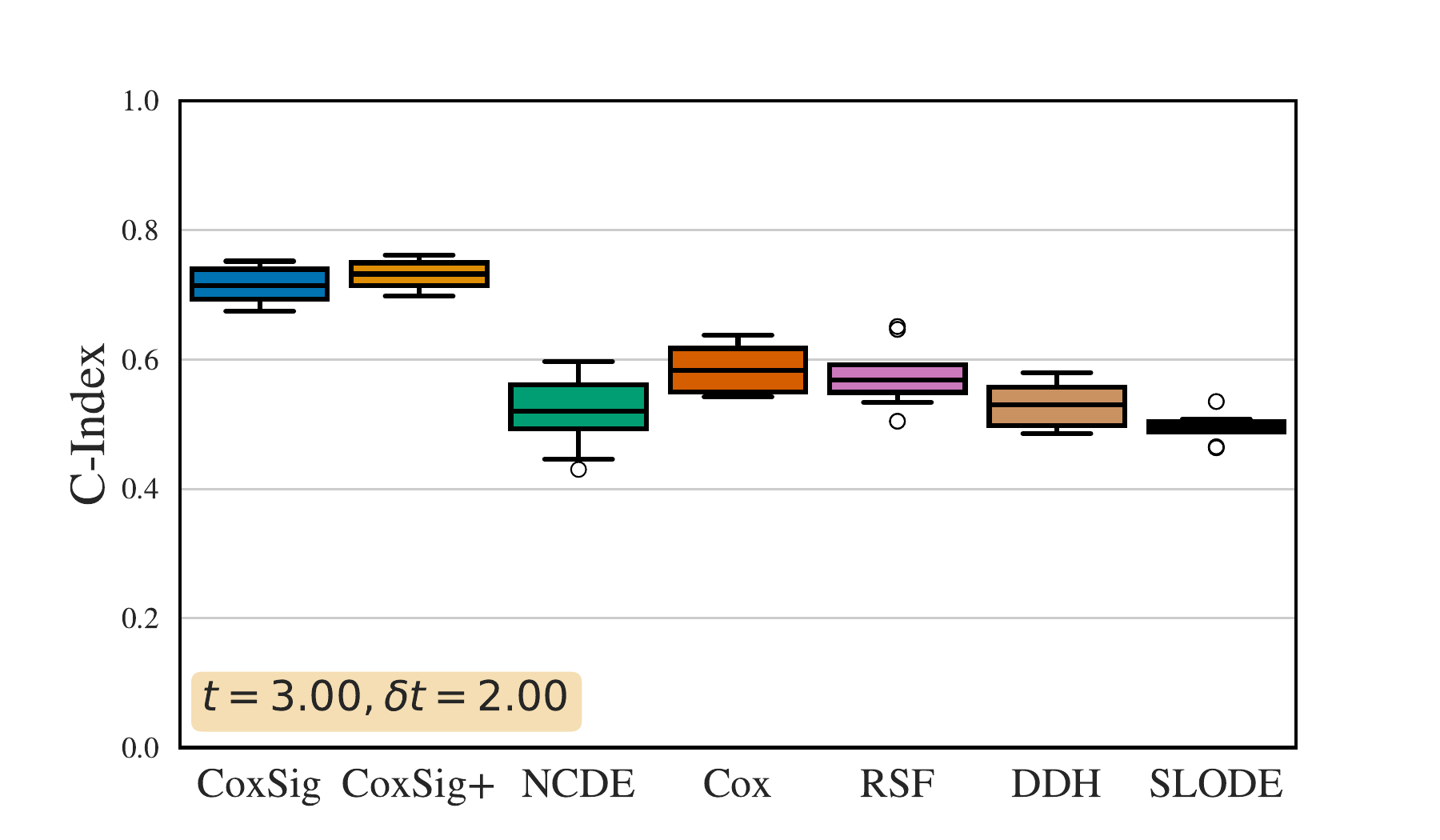}
    \includegraphics[width=0.33\textwidth]{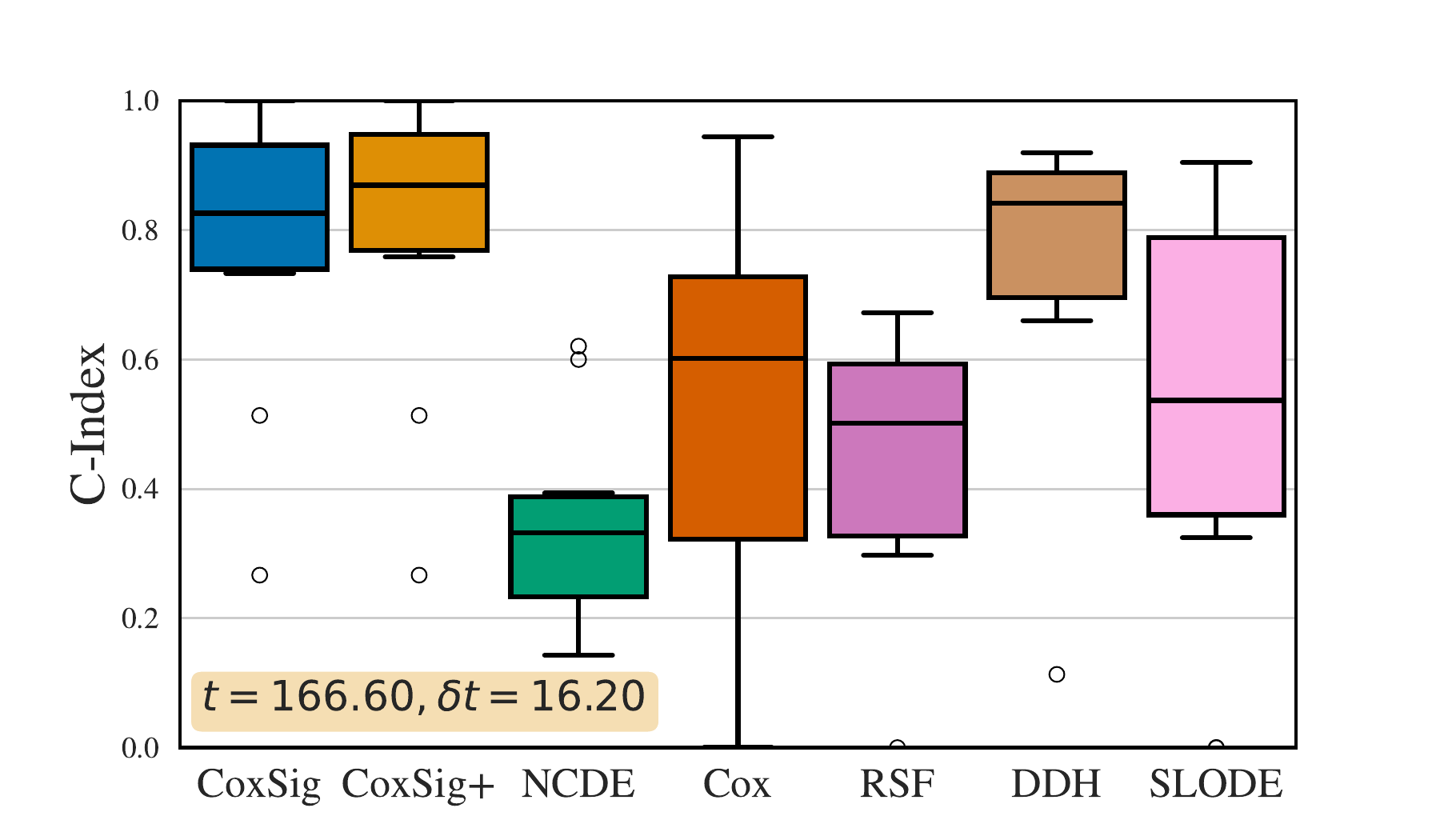}
    \includegraphics[width=0.33\textwidth]{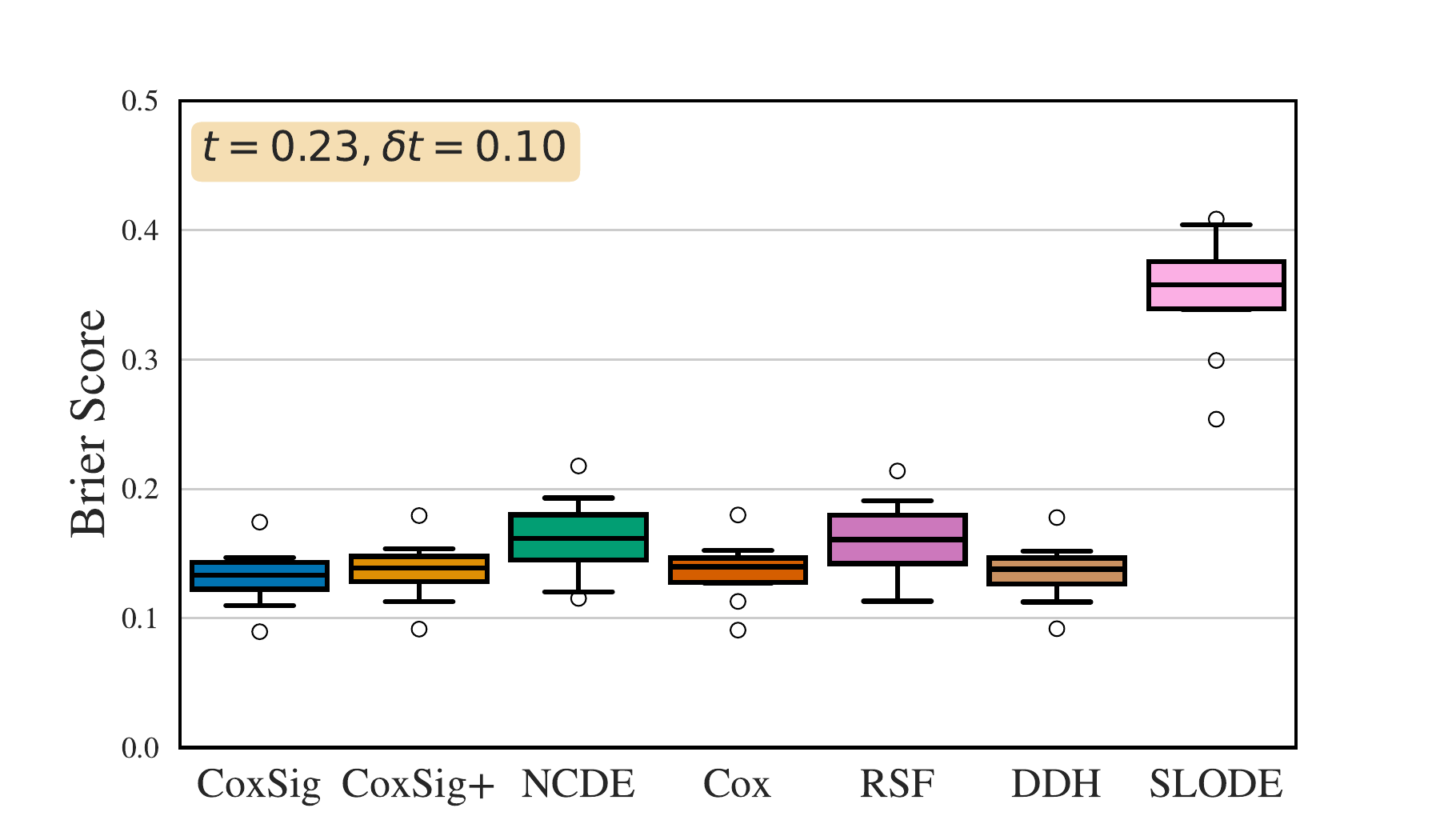}
    \includegraphics[width=0.33\textwidth]{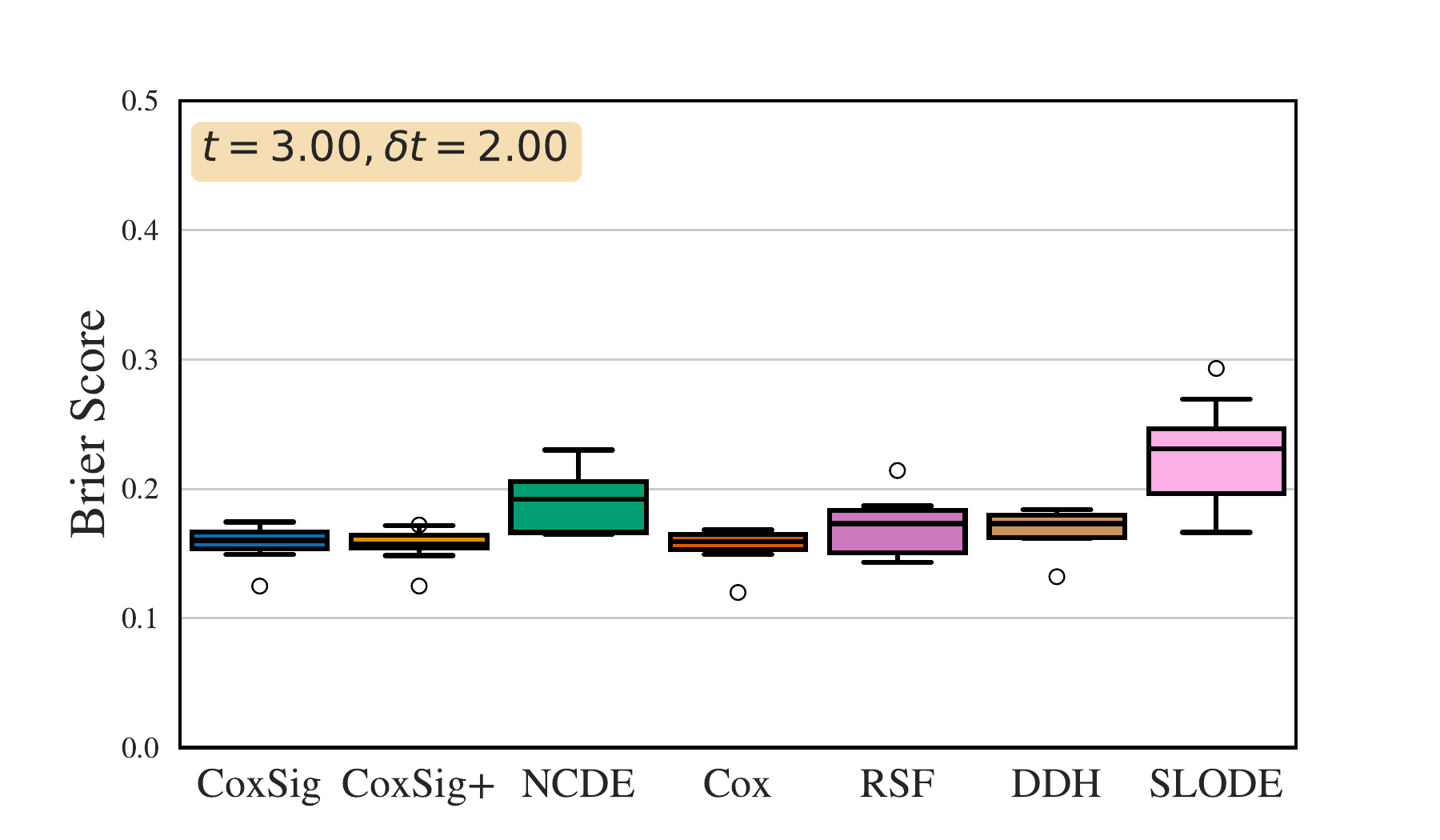}
    \includegraphics[width=0.33\textwidth]{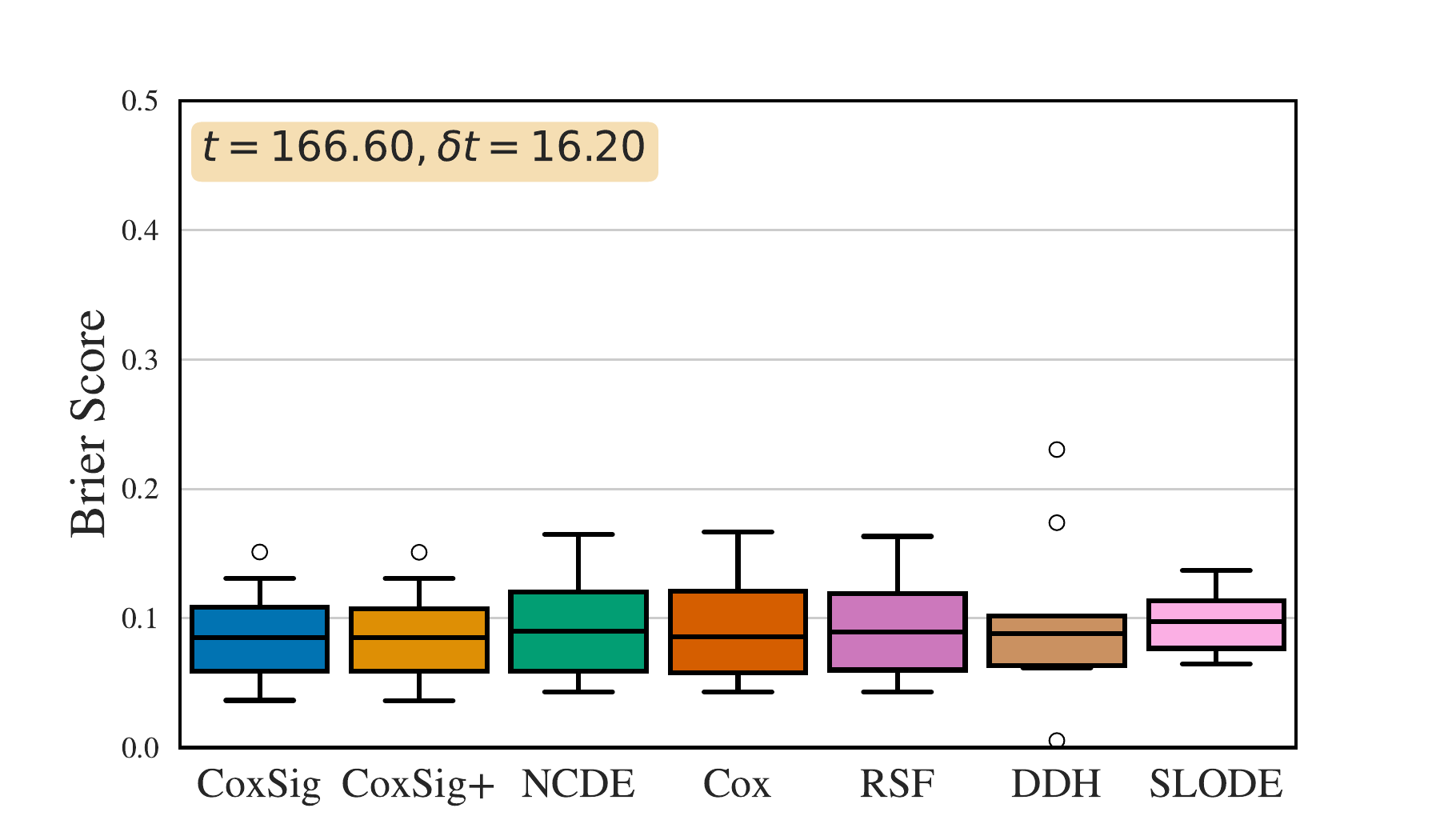}
    \caption{ C-Index (\textit{higher} is better) on \textbf{top} and Brier score (\textit{lower} is better) on \textbf{bottom} for hitting time of a partially observed SDE (\textbf{left}), churn prediction (\textbf{center}) and predictive maintenance (\textbf{right}) evaluated at chosen points $(t,\delta t)$. $t$ is chosen as the first decile of the event times i.e. $90 \%$ of the events occur after $t$.  Hollow dots indicate outliers, and error bars indicate $80 \%$ of the interquartile range. We report detailed results for numerous points $(t,\delta t)$ in Appendix \ref{appendix:more_results}.}
    \label{fig:c_index}
\end{figure*}

\paragraph{General performance of CoxSig.} Overall, the signature-based estimators outperform competing methods. We observe that CoxSig and CoxSig+ improve over the strongest baselines in terms of Brier scores. Contrarily to the strong baseline DDH, this improvement is consistent over larger prediction windows $[t,t+ \delta t]$ as $\delta t$ increases (see Figure \ref{fig:evolution_bs}). They provide even stronger improvements in terms of C-indexes (see Figure \ref{fig:c_index} and Appendix \ref{appendix:more_results}). This suggests that they are particularly well-tailored for ranking tasks, such as identifying the most-at-risk individual. Including the first observed value of the time series generally improves CoxSig's performance: this is possibly due to the fact that signatures are invariant by translation (i.e. the signature of $x:t\mapsto x(t)$ is equal to the signature of $x:t\mapsto x(t) + a$), and hence including the first value of the time series provides non-redundant information.

\paragraph{Performance on low-dimensional data.} A notable exception is the tumor growth simulation, in which CoxSig is generally outperformed (see Figures \ref{fig:c_index_tumor_growth} and \ref{fig:bs_tumor_growth} in the appendix). The competitive performance of signatures for moderate to high dimensional data streams and its below average performance on low dimensional data is a well-studied feature (see \citet{fermanian2021embedding} for an empirical study). A possible solution to handle low-dimensional data is to use embeddings before computing signatures to make them more informative \citep{morrill2020generalised}.    

\paragraph{NCDEs.} On the other side, NCDEs generally tie or perform worse than competing methods. Notably, when considering C-indexes, they even perform worse than random on the predictive maintenance dataset. This stands in stark contrast to their good performances on classification or regression tasks \citep{kidger2020neural,morrill2021neural,vanderschueren2023accounting}.  

\paragraph{Running times.} Finally, we observe that our methods run in similar times than DDH, while including cross-validation (see Figure \ref{fig:running_times} in the appendix). Models that do not use time dependent features (RSF and Cox) are 2 orders of magnitude faster to train.

\section{Conclusion}

We have designed and analyzed a model for generic counting processes driven by a controlled latent state, which can be readily estimated using either NCDE or signature-based estimators. CoxSig in particular offers a parsimonious alternative to deep models and yields excellent performance for survival analysis. Future research efforts will be targeted at extending our model to competing risks and multimodal data. 

\paragraph{Limitations.} While our model shows competitive performance on moderate to high-dimensional data, one central limitation is its below average performance on low dimensional data. We also stress that the extension to very high dimensional time series is computationally prohibitive since the signature scales exponentially with the dimension of the time series. Finally, our experimental setup is limited to the survival analysis case: we plan on extending it to general counting processes in future work.

\section*{Acknowledgements} For this work, AG has benefited from the support of the National Agency for Research under the France 2030 program with the reference ANR-22-PESN-0016. LB and AG thank Killian Bakong-Epouné, who performed preliminary analysis of the NCDE model during an internship at Inria. LB, VTN and AF thank the Sorbonne Center for Artificial Intelligence (SCAI) and its team for their support. LB thanks Claire Boyer and Gérard Biau for supervising a previous internship that sparked his interest in signatures.   

\section*{Impact statement} This paper presents work whose goal is to rank individuals or to predict their risk of event by relying on past information. This is a common goal in many diverse fields, including healthcare (which patient is most at risk of dying ?), insurance policy pricing (given a person's history, what is her probability of experiencing an event covered by her policy ?) and human resources management (which employee is most likely to quit ?). 

Our proposed method can be used in any of these settings to personalize predictions, and hence target interventions. As any prediction algorithm, it may cause some disadvantaged groups to suffer from biased decisions.    

\bibliography{refs}

\begin{thebibliography}{72}
\providecommand{\natexlab}[1]{#1}
\providecommand{\url}[1]{\texttt{#1}}
\expandafter\ifx\csname urlstyle\endcsname\relax
  \providecommand{\doi}[1]{doi: #1}\else
  \providecommand{\doi}{doi: \begingroup \urlstyle{rm}\Url}\fi

\bibitem[Aalen et~al.(2008)Aalen, Borgan, and Gjessing]{aalen2008survival}
Aalen, O., Borgan, O., and Gjessing, H.
\newblock \emph{Survival and event history analysis: a process point of view}.
\newblock Springer Science \& Business Media, 2008.

\bibitem[Bach(2010)]{bach2010self}
Bach, F.
\newblock {Self-concordant analysis for logistic regression}.
\newblock \emph{Electronic Journal of Statistics}, 4:\penalty0 384 -- 414, 2010.

\bibitem[Bach(2021)]{bach2021learning}
Bach, F.
\newblock Learning theory from first principles.
\newblock \emph{Draft of a book, version of Sept}, 6:\penalty0 2021, 2021.

\bibitem[Bacry et~al.(2015)Bacry, Mastromatteo, and Muzy]{bacry2015hawkes}
Bacry, E., Mastromatteo, I., and Muzy, J.-F.
\newblock Hawkes processes in finance.
\newblock \emph{Market Microstructure and Liquidity}, 1\penalty0 (01):\penalty0 1550005, 2015.

\bibitem[Bleistein \& Guilloux(2024)Bleistein and Guilloux]{bleistein2023generalization}
Bleistein, L. and Guilloux, A.
\newblock On the generalization capacities of neural controlled differential equations.
\newblock In \emph{International Conference on Learning Representations}, 2024.

\bibitem[Bleistein et~al.(2023)Bleistein, Fermanian, Jannot, and Guilloux]{bleistein2023learning}
Bleistein, L., Fermanian, A., Jannot, A.-S., and Guilloux, A.
\newblock Learning the dynamics of sparsely observed interacting systems.
\newblock In Krause, A., Brunskill, E., Cho, K., Engelhardt, B., Sabato, S., and Scarlett, J. (eds.), \emph{Proceedings of the 40th International Conference on Machine Learning}, volume 202, pp.\  2603--2640. PMLR, 2023.

\bibitem[Boyd \& Vandenberghe(2004)Boyd and Vandenberghe]{boyd2004convex}
Boyd, S.~P. and Vandenberghe, L.
\newblock \emph{Convex optimization}.
\newblock Cambridge university press, 2004.

\bibitem[Bussy et~al.(2019)Bussy, Veil, Looten, Burgun, Ga{\"\i}ffas, Guilloux, Ranque, and Jannot]{bussy2019comparison}
Bussy, S., Veil, R., Looten, V., Burgun, A., Ga{\"\i}ffas, S., Guilloux, A., Ranque, B., and Jannot, A.-S.
\newblock Comparison of methods for early-readmission prediction in a high-dimensional heterogeneous covariates and time-to-event outcome framework.
\newblock \emph{BMC medical research methodology}, 19:\penalty0 1--9, 2019.

\bibitem[Chen(1958)]{chen1958integration}
Chen, K.-T.
\newblock Integration of paths—a faithful representation of paths by non-commutative formal power series.
\newblock \emph{Transactions of the American Mathematical Society}, 89:\penalty0 395--407, 1958.

\bibitem[Chen et~al.(2018)Chen, Rubanova, Bettencourt, and Duvenaud]{chen2018neural}
Chen, R. T.~Q., Rubanova, Y., Bettencourt, J., and Duvenaud, D.~K.
\newblock Neural ordinary differential equations.
\newblock In Bengio, S., Wallach, H., Larochelle, H., Grauman, K., Cesa-Bianchi, N., and Garnett, R. (eds.), \emph{Advances in Neural Information Processing Systems}, volume~31, pp.\  6572--6583. Curran Associates, Inc., 2018.

\bibitem[Chen et~al.(2021)Chen, Amos, and Nickel]{chen2020neural}
Chen, R. T.~Q., Amos, B., and Nickel, M.
\newblock Neural spatio-temporal point processes.
\newblock In \emph{International Conference on Learning Representations}, 2021.

\bibitem[Cheridito \& Xu(2015)Cheridito and Xu]{cheridito2015pricing}
Cheridito, P. and Xu, Z.
\newblock Pricing and hedging cocos.
\newblock \emph{Available at SSRN 2201364}, 2015.

\bibitem[Cirone et~al.(2023)Cirone, Lemercier, and Salvi]{cirone2023neural}
Cirone, N.~M., Lemercier, M., and Salvi, C.
\newblock Neural signature kernels as infinite-width-depth-limits of controlled resnets.
\newblock In Krause, A., Brunskill, E., Cho, K., Engelhardt, B., Sabato, S., and Scarlett, J. (eds.), \emph{Proceedings of the 40th International Conference on Machine Learning}, pp.\  25358--25425. PMLR, 2023.

\bibitem[Cirone et~al.(2024)Cirone, Orvieto, Walker, Salvi, and Lyons]{cirone2024theoretical}
Cirone, N.~M., Orvieto, A., Walker, B., Salvi, C., and Lyons, T.
\newblock Theoretical foundations of deep selective state-space models.
\newblock \emph{arXiv preprint arXiv:2402.19047}, 2024.

\bibitem[Corcuera \& Valdivia(2016)Corcuera and Valdivia]{corcuera2016pricing}
Corcuera, J.~M. and Valdivia, A.
\newblock Pricing cocos with a market trigger.
\newblock In Benth, F.~E. and Di~Nunno, G. (eds.), \emph{Stochastics of Environmental and Financial Economics}, pp.\  179--209. Springer International Publishing, 2016.

\bibitem[Cox(1972)]{cox1972regression}
Cox, D.~R.
\newblock Regression models and life-tables.
\newblock \emph{Journal of the Royal Statistical Society: Series B (Methodological)}, 34\penalty0 (2):\penalty0 187--202, 1972.

\bibitem[Crowther et~al.(2013)Crowther, Abrams, and Lambert]{crowther2013joint}
Crowther, M.~J., Abrams, K.~R., and Lambert, P.~C.
\newblock Joint modeling of longitudinal and survival data.
\newblock \emph{The Stata Journal}, 13\penalty0 (1):\penalty0 165--184, 2013.

\bibitem[De~Brouwer et~al.(2019)De~Brouwer, Simm, Arany, and Moreau]{de2019gru}
De~Brouwer, E., Simm, J., Arany, A., and Moreau, Y.
\newblock {GRU}-{ODE}-{B}ayes: {C}ontinuous modeling of sporadically-observed time series.
\newblock In Wallach, H., Larochelle, H., Beygelzimer, A., d\textquotesingle Alch\'{e}-Buc, F., Fox, E., and Garnett, R. (eds.), \emph{Advances in Neural Information Processing Systems}, volume~32, pp.\  7379--7390. Curran Associates, Inc., 2019.

\bibitem[De~Brouwer et~al.(2022)De~Brouwer, Gonzalez, and Hyland]{de2022predicting}
De~Brouwer, E., Gonzalez, J., and Hyland, S.
\newblock Predicting the impact of treatments over time with uncertainty aware neural differential equations.
\newblock In Camps-Valls, G., Ruiz, F. J.~R., and Valera, I. (eds.), \emph{Proceedings of The 25th International Conference on Artificial Intelligence and Statistics}, volume 151, pp.\  4705--4722. PMLR, 2022.

\bibitem[Dumanis et~al.(2017)Dumanis, French, Bernard, Worrell, and Fureman]{dumanis2017seizure}
Dumanis, S.~B., French, J.~A., Bernard, C., Worrell, G.~A., and Fureman, B.~E.
\newblock Seizure forecasting from idea to reality. outcomes of the my seizure gauge epilepsy innovation institute workshop.
\newblock \emph{Eneuro}, 4\penalty0 (6), 2017.

\bibitem[Fermanian(2021)]{fermanian2021embedding}
Fermanian, A.
\newblock Embedding and learning with signatures.
\newblock \emph{Computational Statistics \& Data Analysis}, 157:\penalty0 107148, 2021.

\bibitem[Fermanian(2022)]{fermanian2022functional}
Fermanian, A.
\newblock Functional linear regression with truncated signatures.
\newblock \emph{Journal of Multivariate Analysis}, 192:\penalty0 105031, 2022.

\bibitem[Fermanian et~al.(2021)Fermanian, Marion, Vert, and Biau]{fermanian2021framing}
Fermanian, A., Marion, P., Vert, J.-P., and Biau, G.
\newblock Framing {RNN} as a kernel method: {A} neural {ODE} approach.
\newblock In Ranzato, M., Beygelzimer, A., Dauphin, Y., Liang, P., and Vaughan, J.~W. (eds.), \emph{Advances in Neural Information Processing Systems}, volume~34, pp.\  3121--3134. Curran Associates, Inc., 2021.

\bibitem[Friz \& Victoir(2010)Friz and Victoir]{friz2010multidimensional}
Friz, P.~K. and Victoir, N.~B.
\newblock \emph{Multidimensional Stochastic Processes as Rough Paths: Theory and Applications}, volume 120 of \emph{Cambridge Studies in Advanced Mathematics}.
\newblock Cambridge University Press, Cambridge, 2010.

\bibitem[Groha et~al.(2020)Groha, Schmon, and Gusev]{groha2020general}
Groha, S., Schmon, S.~M., and Gusev, A.
\newblock A general framework for survival analysis and multi-state modelling.
\newblock \emph{arXiv preprint arXiv:2006.04893}, 2020.

\bibitem[Gu et~al.(2022)Gu, Goel, and R{\'e}]{gu2021efficiently}
Gu, A., Goel, K., and R{\'e}, C.
\newblock Efficiently modeling long sequences with structured state spaces.
\newblock In \emph{International Conference on Learning Representations}, 2022.

\bibitem[Hickey et~al.(2016)Hickey, Philipson, Jorgensen, and Kolamunnage-Dona]{hickey2016joint}
Hickey, G.~L., Philipson, P., Jorgensen, A., and Kolamunnage-Dona, R.
\newblock Joint modelling of time-to-event and multivariate longitudinal outcomes: recent developments and issues.
\newblock \emph{BMC medical research methodology}, 16:\penalty0 1--15, 2016.

\bibitem[Horvath et~al.(2023)Horvath, Lemercier, Liu, Lyons, and Salvi]{horvath2023optimal}
Horvath, B., Lemercier, M., Liu, C., Lyons, T., and Salvi, C.
\newblock Optimal stopping via distribution regression: a higher rank signature approach.
\newblock \emph{arXiv preprint arXiv:2304.01479}, 2023.

\bibitem[Ibrahim et~al.(2010)Ibrahim, Chu, and Chen]{ibrahim2010basic}
Ibrahim, J.~G., Chu, H., and Chen, L.~M.
\newblock Basic concepts and methods for joint models of longitudinal and survival data.
\newblock \emph{Journal of Clinical Oncology}, 28\penalty0 (16):\penalty0 2796, 2010.

\bibitem[Ishwaran et~al.(2008)Ishwaran, Kogalur, Blackstone, and Lauer]{ishwaran2008random}
Ishwaran, H., Kogalur, U.~B., Blackstone, E.~H., and Lauer, M.~S.
\newblock {Random survival forests}.
\newblock \emph{The Annals of Applied Statistics}, 2\penalty0 (3):\penalty0 841 -- 860, 2008.

\bibitem[Jia \& Benson(2019)Jia and Benson]{jia2019neural}
Jia, J. and Benson, A.~R.
\newblock Neural jump stochastic differential equations.
\newblock In Wallach, H., Larochelle, H., Beygelzimer, A., d\textquotesingle Alch\'{e}-Buc, F., Fox, E., and Garnett, R. (eds.), \emph{Advances in Neural Information Processing Systems}, volume~32, 2019.

\bibitem[Kidger \& Lyons(2020)Kidger and Lyons]{kidger2020signatory}
Kidger, P. and Lyons, T.
\newblock Signatory: differentiable computations of the signature and logsignature transforms, on both cpu and gpu.
\newblock In \emph{International Conference on Learning Representations}, 2020.

\bibitem[Kidger et~al.(2019)Kidger, Bonnier, Perez~Arribas, Salvi, and Lyons]{kidger2019deep}
Kidger, P., Bonnier, P., Perez~Arribas, I., Salvi, C., and Lyons, T.
\newblock {Deep signature transforms}.
\newblock In Wallach, H., Larochelle, H., Beygelzimer, A., d\textquotesingle Alch\'{e}-Buc, F., Fox, E., and Garnett, R. (eds.), \emph{Advances in Neural Information Processing Systems}, volume~32. Curran Associates, Inc., 2019.

\bibitem[Kidger et~al.(2020)Kidger, Morrill, Foster, and Lyons]{kidger2020neural}
Kidger, P., Morrill, J., Foster, J., and Lyons, T.
\newblock Neural controlled differential equations for irregular time series.
\newblock In Larochelle, H., Ranzato, M., Hadsell, R., Balcan, M.~F., and Lin, H. (eds.), \emph{Advances in Neural Information Processing Systems}, volume~33, pp.\  6696--6707. Curran Associates, Inc., 2020.

\bibitem[Kingma \& Ba(2015)Kingma and Ba]{kingma2014adam}
Kingma, D.~P. and Ba, J.
\newblock Adam: {{A Method}} for {{Stochastic Optimization}}.
\newblock In \emph{International Conference on Learning Representations}, 2015.

\bibitem[Kingma \& Welling(2013)Kingma and Welling]{kingma2013auto}
Kingma, D.~P. and Welling, M.
\newblock Auto-encoding variational bayes.
\newblock \emph{arXiv preprint arXiv:1312.6114}, 2013.

\bibitem[Kvamme \& Borgan(2023)Kvamme and Borgan]{kvamme2023brier}
Kvamme, H. and Borgan, {\O}.
\newblock The brier score under administrative censoring: Problems and a solution.
\newblock \emph{Journal of Machine Learning Research}, 24\penalty0 (2):\penalty0 1--26, 2023.

\bibitem[Kvamme et~al.(2019)Kvamme, {{\O}}rnulf Borgan, and Scheel]{kvamme2019time}
Kvamme, H., {{\O}}rnulf Borgan, and Scheel, I.
\newblock Time-to-event prediction with neural networks and cox regression.
\newblock \emph{Journal of Machine Learning Research}, 20\penalty0 (129):\penalty0 1--30, 2019.

\bibitem[Lee et~al.(2019)Lee, Yoon, and Van Der~Schaar]{lee2019dynamic}
Lee, C., Yoon, J., and Van Der~Schaar, M.
\newblock Dynamic-deephit: A deep learning approach for dynamic survival analysis with competing risks based on longitudinal data.
\newblock \emph{IEEE Transactions on Biomedical Engineering}, 67\penalty0 (1):\penalty0 122--133, 2019.

\bibitem[Lee \& Whitmore(2006)Lee and Whitmore]{lee2006threshold}
Lee, M.-L.~T. and Whitmore, G.~A.
\newblock Threshold regression for survival analysis: modeling event times by a stochastic process reaching a boundary.
\newblock 2006.

\bibitem[Lemler(2016)]{lemler2016oracle}
Lemler, S.
\newblock {Oracle inequalities for the Lasso in the high-dimensional Aalen multiplicative intensity model}.
\newblock \emph{Annales de l'Institut Henri Poincaré, Probabilités et Statistiques}, 52\penalty0 (2):\penalty0 981 -- 1008, 2016.

\bibitem[Lin \& Yong(2020)Lin and Yong]{lin2020controlled}
Lin, P. and Yong, J.
\newblock Controlled singular volterra integral equations and pontryagin maximum principle.
\newblock \emph{SIAM Journal on Control and Optimization}, 58\penalty0 (1):\penalty0 136--164, 2020.

\bibitem[Long \& Mills(2018)Long and Mills]{long2018joint}
Long, J.~D. and Mills, J.~A.
\newblock Joint modeling of multivariate longitudinal data and survival data in several observational studies of huntington’s disease.
\newblock \emph{BMC medical research methodology}, 18\penalty0 (1):\penalty0 1--15, 2018.

\bibitem[Lyons \& McLeod(2022)Lyons and McLeod]{lyons2022signature}
Lyons, T. and McLeod, A.~D.
\newblock Signature methods in machine learning.
\newblock \emph{arXiv preprint arXiv:2206.14674}, 2022.

\bibitem[Lyons et~al.(2007)Lyons, Caruana, and L{\'e}vy]{lyons2007differential}
Lyons, T., Caruana, M., and L{\'e}vy, T.
\newblock \emph{Differential Equations driven by Rough Paths}, volume 1908 of \emph{Lecture Notes in Mathematics}.
\newblock Springer, Berlin, 2007.

\bibitem[Marion et~al.(2022)Marion, Fermanian, Biau, and Vert]{marion2022scaling}
Marion, P., Fermanian, A., Biau, G., and Vert, J.-P.
\newblock Scaling resnets in the large-depth regime.
\newblock \emph{arXiv preprint arXiv:2206.06929}, 2022.

\bibitem[Mei \& Eisner(2017)Mei and Eisner]{mei2017neural}
Mei, H. and Eisner, J.~M.
\newblock The neural hawkes process: A neurally self-modulating multivariate point process.
\newblock In Guyon, I., Luxburg, U.~V., Bengio, S., Wallach, H., Fergus, R., Vishwanathan, S., and Garnett, R. (eds.), \emph{Advances in Neural Information Processing Systems}, volume~30, 2017.

\bibitem[Moon et~al.(2022)Moon, Groha, and Gusev]{moon2022survlatent}
Moon, I., Groha, S., and Gusev, A.
\newblock Survlatent ode : A neural ode based time-to-event model with competing risks for longitudinal data improves cancer-associated venous thromboembolism (vte) prediction.
\newblock In Lipton, Z., Ranganath, R., Sendak, M., Sjoding, M., and Yeung, S. (eds.), \emph{Proceedings of the 7th Machine Learning for Healthcare Conference}, volume 182, pp.\  800--827. PMLR, 2022.

\bibitem[Morrill et~al.(2020)Morrill, Fermanian, Kidger, and Lyons]{morrill2020generalised}
Morrill, J., Fermanian, A., Kidger, P., and Lyons, T.
\newblock A generalised signature method for multivariate time series feature extraction.
\newblock \emph{arXiv preprint arXiv:2006.00873}, 2020.

\bibitem[Morrill et~al.(2021)Morrill, Kidger, Yang, and Lyons]{morrill2021neural}
Morrill, J., Kidger, P., Yang, L., and Lyons, T.
\newblock Neural controlled differential equations for online prediction tasks.
\newblock \emph{arXiv preprint arXiv:2106.11028}, 2021.

\bibitem[Murray \& Philipson(2022)Murray and Philipson]{murray2022fast}
Murray, J. and Philipson, P.
\newblock A fast approximate em algorithm for joint models of survival and multivariate longitudinal data.
\newblock \emph{Computational Statistics \& Data Analysis}, 170:\penalty0 107438, 2022.

\bibitem[Nardi \& Rinaldo(2008)Nardi and Rinaldo]{nardi2008asymptotic}
Nardi, Y. and Rinaldo, A.
\newblock On the asymptotic properties of the group lasso estimator for linear models.
\newblock \emph{Electronic Journal of Statistics}, 0, 01 2008.
\newblock \doi{10.1214/08-EJS200}.

\bibitem[Ogata(1988)]{ogata1988statistical}
Ogata, Y.
\newblock Statistical models for earthquake occurrences and residual analysis for point processes.
\newblock \emph{Journal of the American Statistical association}, pp.\  9--27, 1988.

\bibitem[Omi et~al.(2019)Omi, ueda, and Aihara]{omi2019fully}
Omi, T., ueda, n., and Aihara, K.
\newblock Fully neural network based model for general temporal point processes.
\newblock In Wallach, H., Larochelle, H., Beygelzimer, A., d\textquotesingle Alch\'{e}-Buc, F., Fox, E., and Garnett, R. (eds.), \emph{Advances in Neural Information Processing Systems}, volume~32, 2019.

\bibitem[P{\"o}lsterl(2020)]{sksurv}
P{\"o}lsterl, S.
\newblock scikit-survival: A library for time-to-event analysis built on top of scikit-learn.
\newblock \emph{Journal of Machine Learning Research}, 21\penalty0 (212):\penalty0 1--6, 2020.

\bibitem[Proust-Lima et~al.(2014)Proust-Lima, S{\'e}ne, Taylor, and Jacqmin-Gadda]{proust2014joint}
Proust-Lima, C., S{\'e}ne, M., Taylor, J.~M., and Jacqmin-Gadda, H.
\newblock Joint latent class models for longitudinal and time-to-event data: a review.
\newblock \emph{Statistical methods in medical research}, 23\penalty0 (1):\penalty0 74--90, 2014.

\bibitem[Rasheed et~al.(2020)Rasheed, Qayyum, Qadir, Sivathamboo, Kwan, Kuhlmann, O’Brien, and Razi]{rasheed2020machine}
Rasheed, K., Qayyum, A., Qadir, J., Sivathamboo, S., Kwan, P., Kuhlmann, L., O’Brien, T., and Razi, A.
\newblock Machine learning for predicting epileptic seizures using eeg signals: A review.
\newblock \emph{IEEE Reviews in Biomedical Engineering}, 14:\penalty0 139--155, 2020.

\bibitem[Reizenstein \& Graham(2020)Reizenstein and Graham]{reizenstein2018iisignature}
Reizenstein, J.~F. and Graham, B.
\newblock Algorithm 1004: The iisignature library: Efficient calculation of iterated-integral signatures and log signatures.
\newblock \emph{ACM Transactions on Mathematical Software}, 46\penalty0 (1):\penalty0 1--21, 2020.

\bibitem[Rizopoulos(2016)]{rizopoulos2014r}
Rizopoulos, D.
\newblock The r package jmbayes for fitting joint models for longitudinal and time-to-event data using mcmc.
\newblock \emph{Journal of Statistical Software}, 72\penalty0 (7):\penalty0 1–46, 2016.

\bibitem[Rubanova et~al.(2019)Rubanova, Chen, and Duvenaud]{rubanova2019latent}
Rubanova, Y., Chen, R. T.~Q., and Duvenaud, D.~K.
\newblock Latent ordinary differential equations for irregularly-sampled time series.
\newblock In Wallach, H., Larochelle, H., Beygelzimer, A., d\textquotesingle Alch\'{e}-Buc, F., Fox, E., and Garnett, R. (eds.), \emph{Advances in Neural Information Processing Systems}, volume~32, pp.\  5320--5330. Curran Associates, Inc., 2019.

\bibitem[Rustand et~al.(2024)Rustand, Van~Niekerk, Krainski, Rue, and Proust-Lima]{rustand2024fast}
Rustand, D., Van~Niekerk, J., Krainski, E.~T., Rue, H., and Proust-Lima, C.
\newblock Fast and flexible inference for joint models of multivariate longitudinal and survival data using integrated nested laplace approximations.
\newblock \emph{Biostatistics}, 25\penalty0 (2):\penalty0 429--448, 2024.

\bibitem[Salvi et~al.(2021)Salvi, Lemercier, Liu, Horvath, Damoulas, and Lyons]{salvi2021higher}
Salvi, C., Lemercier, M., Liu, C., Horvath, B., Damoulas, T., and Lyons, T.
\newblock Higher order kernel mean embeddings to capture filtrations of stochastic processes.
\newblock In Ranzato, M., Beygelzimer, A., Dauphin, Y., Liang, P., and Vaughan, J.~W. (eds.), \emph{Advances in Neural Information Processing Systems}, volume~34, pp.\  16635--16647, 2021.

\bibitem[Saxena et~al.(2008)Saxena, Goebel, Simon, and Eklund]{saxena2008damage}
Saxena, A., Goebel, K., Simon, D., and Eklund, N.
\newblock Damage propagation modeling for aircraft engine run-to-failure simulation.
\newblock In \emph{2008 International Conference on Prognostics and Health Management}, pp.\  1--9, 2008.

\bibitem[Shchur et~al.(2021)Shchur, Türkmen, Januschowski, and Günnemann]{shchur2021neural}
Shchur, O., Türkmen, A.~C., Januschowski, T., and Günnemann, S.
\newblock Neural temporal point processes: A review.
\newblock In Zhou, Z.-H. (ed.), \emph{Proceedings of the Thirtieth International Joint Conference on Artificial Intelligence, {IJCAI-21}}, pp.\  4585--4593, 2021.

\bibitem[Shorack \& Wellner(2009)Shorack and Wellner]{shorack2009empirical}
Shorack, G.~R. and Wellner, J.~A.
\newblock \emph{Empirical processes with applications to statistics}.
\newblock Society for Industrial and Applied Mathematics, 2009.

\bibitem[Simeoni et~al.(2004)Simeoni, Magni, Cammia, De~Nicolao, Croci, Pesenti, Germani, Poggesi, and Rocchetti]{simeoni2004predictive}
Simeoni, M., Magni, P., Cammia, C., De~Nicolao, G., Croci, V., Pesenti, E., Germani, M., Poggesi, I., and Rocchetti, M.
\newblock Predictive pharmacokinetic-pharmacodynamic modeling of tumor growth kinetics in xenograft models after administration of anticancer agents.
\newblock \emph{Cancer research}, 64\penalty0 (3):\penalty0 1094--1101, 2004.

\bibitem[Tang et~al.(2022)Tang, Ma, Mei, and Zhu]{tang2022soden}
Tang, W., Ma, J., Mei, Q., and Zhu, J.
\newblock Soden: A scalable continuous-time survival model through ordinary differential equation networks.
\newblock \emph{Journal of Machine Learning Research}, 23\penalty0 (34):\penalty0 1--29, 2022.

\bibitem[Therneau \& Grambsch(2000)Therneau and Grambsch]{therneau2000modeling}
Therneau, T.~M. and Grambsch, P.~M.
\newblock \emph{Modeling survival data: extending the Cox model}.
\newblock Springer, New-York, 2000.

\bibitem[Van~de Geer(1995)]{van1995exponential}
Van~de Geer, S.
\newblock Exponential inequalities for martingales, with application to maximum likelihood estimation for counting processes.
\newblock \emph{The Annals of Statistics}, pp.\  1779--1801, 1995.

\bibitem[Vanderschueren et~al.(2023)Vanderschueren, Curth, Verbeke, and Van Der~Schaar]{vanderschueren2023accounting}
Vanderschueren, T., Curth, A., Verbeke, W., and Van Der~Schaar, M.
\newblock Accounting for informative sampling when learning to forecast treatment outcomes over time.
\newblock In Krause, A., Brunskill, E., Cho, K., Engelhardt, B., Sabato, S., and Scarlett, J. (eds.), \emph{Proceedings of the 40th International Conference on Machine Learning}, volume 202, pp.\  34855--34874. PMLR, 2023.

\bibitem[Virmaux \& Scaman(2018)Virmaux and Scaman]{virmaux2018lipschitz}
Virmaux, A. and Scaman, K.
\newblock Lipschitz regularity of deep neural networks: analysis and efficient estimation.
\newblock In Bengio, S., Wallach, H., Larochelle, H., Grauman, K., Cesa-Bianchi, N., and Garnett, R. (eds.), \emph{Advances in Neural Information Processing Systems}, volume~31, 2018.

\bibitem[Zhang et~al.(2018)Zhang, Reinikainen, Adeleke, Pieterse, and Groothuis-Oudshoorn]{zhang2018time}
Zhang, Z., Reinikainen, J., Adeleke, K.~A., Pieterse, M.~E., and Groothuis-Oudshoorn, C.~G.
\newblock Time-varying covariates and coefficients in cox regression models.
\newblock \emph{Annals of Translational Medicine}, 6\penalty0 (7), 2018.

\end{thebibliography}
\bibliographystyle{icml2024}

\clearpage
\appendix
\onecolumn 

\begin{center}
    \Large \textbf{Supplementary Material} 
\end{center}
\bigskip

\section{Supplementary Mathematical Elements}

\subsection{Supplementary elements on survival analysis}
\label{appendix:formal_filtration}

The counting process associated with the observation of $T^i_1 < T^i_{2} < \dots $ is denoted by $\tilde N^i$. The observed counting process is $t \to N^i(t)= \int_0^t Y^i(s)d\tilde N^i(s)$. The integral against the counting process $N^i$ is to be understood is to be understood as a Stieltjes integral, i.e., $\int_0^t \lambda_\star^i(s) dN^i(s) = \sum_{T_i \leq t} \lambda_\star^i(T_i)$  --- see \citet[p.55-56]{aalen2008survival}. Its intensity writes $\lambda_\star^i(t\,|\, \mathbf{W}^i,(x^i(s))_{s \leq t}) Y^i(t)$, which we simply write $\lambda^i_\star(t)Y^i(t)$ to alleviate notations.

To the observations, we associate the filtration $\mathcal F$,  with all $\sigma$-fields at $0 \leq t \leq \tau$ defined as 
\[
    \mathcal{F}_t = \bigcup_{i = 1, \ldots,n} \mathcal{F}^i_t
\]
where $\mathcal{F}^i_t = \sigma\big( x^i(s), \mathbf W^i, N^i(s), Y^i(s), 0 \leq s\leq t \big)$. We assume in addition that $Y^i$ is $\mathcal F^i$-predictable. 

Using the Doob-Meyer decomposition of counting processes - see \citet[p. 52-60]{aalen2008survival} - we have
\begin{equation}
\label{eq:doob_meyer}
 N^i(t) = \int_0^t \lambda_\star^i(s) Y^i(s) ds+ M^i(t)   
\end{equation}
where $M^i$ is local square integrable martingale with respect to $\mathcal{F}^i$.

\subsection{Picard-Lindelhöf Theorem}

\begin{theorem}\label{thm:pl_theorem}
    Let $x: [0, \tau] \to \R^d$ be a continuous path of bounded variation, and assume that $\mathbf{G}:\mathbb{R}^p \to \mathbb{R}^{p \times d}$ is Lipschitz continuous. Then the CDE
    \[
    dz(t) = \mathbf{G}(z(t))dx(t)
    \]
    with initial condition $z_0 \in \mathbb{R}^p$ has a unique solution on $[0, \tau]$. 
\end{theorem}

A full proof can be found in \citet[][Theorem 4]{fermanian2021framing}. Remark that since in our setting, NCDEs are Lipschitz since typical neural vector fields, such as feed-forward neural networks, are Lipschitz \citep{virmaux2018lipschitz}. This ensures that the solutions to NCDEs are well defined.

\subsection{Continuity of the Flow of CDEs}

We state a result on the continuity of the flow adapted from \citet{bleistein2023generalization}, Theorem B.5.

\begin{theorem}
\label{thm:flow_continuity}
    Let $\mathbf{F},\mathbf{G}:\mathbb{R}^p \to \mathbb{R}^{p \times d}$ be  two Lipschitz vector fields with Lipschitz constants  $L_\mathbf{F},L_\mathbf{G} > 0$. Let $x,r:[0,\tau]\to \mathbb{R}^d$ be either continuous or piecewise constant paths of total variations bounded by $L_x$ and $L_r$.
    Consider the controlled differential equations 
    \[
    dw(t) = \mathbf{F}(w(t))dx(t)\,\, \text{ and } \,\,dv(t) = \mathbf{G}(v(t))dr(t)
    \]
    with initial conditions $w(0)=v(0) = 0$ respectively. It then follows that for any $t \in [0, \tau]$
    \begin{align*}
    \norm{w(t)-v(t)} & \leq \Bigg(\norm{x-r}_{\infty,[0,t]} \big(1 + L_\mathbf{F}L_r \mathcal{K}\big)+ \max_{v \in \Omega } \norm{\mathbf{F}(v)-\mathbf{G}(v)}_{\textnormal{op}} L_r\Bigg) \exp(L_\mathbf{F} L_x),
\end{align*}
where 
\begin{align*}
    \mathcal{K} = \Bigg[L_\mathbf{F}\big(\norm{\mathbf{F}(0)}_{\textnormal{op}} L_x \big) \exp(L_\mathbf{F} L_x) +\norm{\mathbf{F}(0)}_{\textnormal{op}} \Bigg]\exp(L_\mathbf{F} L_x)
\end{align*}
and 
\begin{align*}
    \Omega = \big\{u \in \mathbb{R}^p \,|\, \norm{u} \leq (\norm{\mathbf{G}(0)}_\textnormal{op}L_r)\exp(L_\mathbf{G}L_r) \big\}.
\end{align*}
\end{theorem}


\subsection{Linearization in the Signature Space}
\label{appendix:signture_linearization}

\subsubsection{General Result}

In this section, we give additional details on the linearization of CDEs in the signature space. We first define the differential product. 

\begin{definition}
    Let $F,G:\mathbb{R}^p \to \mathbb{R}^p$ be two $\mathcal{C}^\infty$ vector fields and let $J(\cdot)$ be the Jacobian matrix. Their differential product $F \star G:\mathbb{R}^p\to \mathbb{R}^p$ is the smooth vector field defined for every $h \in \mathbb{R}^p$ by 
    \begin{align*}
        (F\star G )(h) = \sum\limits_{j=1}^e \frac{\partial G}{\partial h_j}(h)F_j(h) = J(G)(h)F(h).
    \end{align*}
\end{definition}

We now consider a tensor field $\mathbf{F}:\mathbb{R}^p \to \mathbb{R}^{p \times d}$ which we write 
\begin{align*}
\mathbf{F} = 
\begin{bmatrix}
    \lvert & \dots & \lvert \\
    F^1 & \dots & F^d \\
     \lvert & \dots & \lvert
\end{bmatrix},
\end{align*}
where for every $1 \leq i \leq d$, $F^i:\mathbb{R}^p \to \mathbb{R}^p$, and we define 
\begin{align*}
    \Gamma_k(\mathbf{F}) := \sup\limits_{ \norm{h} \leq M,\, i_1 \leq \dots \leq i_k \leq d } \norm{F^{i_1}\star \dots \star F^{i_k} (h)}_2. 
\end{align*}

Consider the solution $z:[0,\tau] \to \mathbb{R}^p$ to the CDE 
\begin{align}
\label{eq:cde_appendix}
    & dz(t) = \mathbf{F}(z(t))dx(t)\\
    & z(0) = \mathbf{0} \in \mathbb{R}^p \nonumber
\end{align}
where $x:[0,\tau] \to \mathbb{R}^d$ is a continuous path of finite total variation bounded by $L_x \tau >0$. We recall the following result from \citet{fermanian2021framing}, Proposition 4.

\begin{proposition}[\citet{fermanian2021framing}, Proposition 4.] 
\label{prop:truncation_bias_sig}
We have 
\begin{align*}
    \norm{z_\star^i(t) - \alpha_{\star,N}^\top \mathbf{S}_N(x_{[0,t]})} \leq \frac{(d L_x t)^{N+1}}{(N+1)!}\Gamma_{N+1}(\mathbf{F})
\end{align*}
    
\end{proposition}

As a consequence, we have the following theorem.

\begin{theorem}
\label{thm:linearization}
    Let $\mathbf{F}:\mathbb{R}^p \to \mathbb{R}^{p \times d}$ be a $\mathcal{C}^\infty$ tensor field. If 
    \begin{align*}
        \frac{(d L_x t )^{N+1}}{(N+1)!}\Gamma_{N+1}(\mathbf{F}) \to 0
    \end{align*}
    as $N \to + \infty$, then the solution $z$ to the CDE \eqref{eq:cde_appendix} can be written as 
    \begin{align*}
        z(t) = \sum\limits_{k \geq 1} \, \, \, \sum\limits_{I \in \{1,\dots,d\}^k} \mathbf{S}^{I}(x_{[0,t]}) F^{i_1}\star \dots \star F^{i_k}(\mathbf{0}).
    \end{align*}
\end{theorem}

\subsubsection{Application to our Model}

Recall that we have defined our generative model through the CDE 
\begin{align*}
        dz_\star^i(t) = \mathbf{G}_\star(z_\star^i(t))dx^i(t)
    \end{align*}
with initial condition $z_\star^i(0) = 0$, where $\mathbf{G}_\star:\mathbb{R} \to \mathbb{R}^p$ is a $L_{\mathbf{G}_\star}$-Lipschitz vector field. Since in our case, the vector field $\mathbf{G}_\star$ maps $\mathbb{R}$ to $\mathbb{R}^{d}$, it can be written as 
\begin{align*}
    \mathbf{G}_\star(z) = \begin{bmatrix}
        G_\star^1(z) & \dots & G_\star^d(z),
    \end{bmatrix},
\end{align*}
where for every $1 \leq i \leq d$, $G^i_\star: \mathbb{R} \to \mathbb{R}$. In this setup, for $1 \leq i_1,i_2\leq d$ the differential product collapses to 
\begin{align*}
    (G^{i_1}_\star \star G^{i_2}_\star) (h) =  (G^{i_2}_\star)'(h) \times G^{i_1}_\star(h) \in \mathbb{R}.
\end{align*}
For $1 \leq i_1,i_2,i_3 \leq d$, it writes 
\begin{align*}
     (G^{i_1}_\star \star G^{i_2}_\star \star G^{i_3}_\star) (h) &  = (G^{i_2}_\star \star G^{i_3}_\star)'(h)\times G^{i_1}_\star(h)  \\
     & = \big((G^{i_3}_\star)'(h) \times G^{i_2}_\star (h) \big)' \times G^{i_1}_\star(h) \\
     & = \big( (G^{i_3}_\star)^{(2)}(h)\times G^{i_2}_\star (h) + (G^{i_3}_\star)'(h)\times (G^{i_2}_\star)' (h)\big) \times G^{i_1}_\star(h) \in \mathbb{R}.
\end{align*}
One can derive similar expression for $1\leq i_1,\dots,i_k \leq d$. In line with Theorem \ref{thm:linearization}, we make the following Assumption on the vector field $\mathbf{G}_\star$.

\begin{assumption}
\label{assumption:decaying_derivatives}
    The vector field $\mathbf{G}_\star$ satisfies 
    \begin{align*}
        \frac{(L_x \tau d)^{N+1}}{(N+1)!}\Gamma_{N+1}(\mathbf{G}_\star) \to 0
    \end{align*}
    as $N \to \infty$.
\end{assumption}

We can write the $\ell_2$ and $\ell_1$ norms of $\alpha_{\star,N}$ as functions of the differential product of $\mathbf{G}_\star$.

\begin{lemma}
\label{lemma:alpha_bound}
    We have that  
    \begin{align*}
        \norm{\alpha_{\star,N}}_2 \leq \Big(\sum\limits_{k=1}^N d^k \Gamma_k(\mathbf{G}_\star)^2\Big)^{1/2}
    \end{align*}
    and 
    \begin{align*}
        \norm{\alpha_{\star,N}}_1 \leq\sum\limits_{k=1}^N d^k \Gamma_k(\mathbf{G}_\star).
    \end{align*}
\end{lemma}

\begin{proof}
    Starting with the $\ell_2$ norm, one has 
    \begin{align*}
        \norm{\alpha_{\star,N}}_2 & = \Big(\sum\limits_{k=1}^N \,\, \sum\limits_{1 \leq i_1,i_2,\dots,i_k \leq d} G_\star^{i_1}\star \dots \star G_\star^{i_k}(\mathbf{0})^2\Big)^{1/2}\\
    & \leq \Big(\sum\limits_{k=1}^N \,\, d^k\,\, \max\limits_{1 \leq i_1,i_2,\dots,i_k \leq d}\,\, \abs{G_\star^{i_1}\star \dots \star G_\star^{i_k}(\mathbf{0})}^2\Big)^{1/2}\\
    & \leq \Big(\sum\limits_{k=1}^N d^k\Gamma_{k}(\mathbf{G}_\star)^2\Big)^{1/2}.
    \end{align*}
    Moving on to the $\ell_1$ norm, we similarly obtain
    \begin{align*}
        \norm{\alpha_{\star,N}}_1 & = \Big(\sum\limits_{k=1}^N \,\, \sum\limits_{1 \leq i_1,i_2,\dots,i_k \leq d} \abs{G_\star^{i_1}\star \dots \star G_\star^{i_k}(\mathbf{0})} \Big)\\
    & \leq \Big(\sum\limits_{k=1}^N \,\, d^k\,\, \max\limits_{1 \leq i_1,i_2,\dots,i_k \leq d}\,\, \abs{G_\star^{i_1}\star \dots \star G_\star^{i_k}(\mathbf{0})}\Big)\\
    & \leq\sum\limits_{k=1}^N d^k\Gamma_{k}(\mathbf{G}_\star).
    \end{align*}
\end{proof}

\subsection{Signature of a Discretized Path}

We recall the following result from \citet{bleistein2023learning}.
\begin{theorem}
\label{thm:discretization_signatures}
    Let $x:[0,\tau] \to \R^d$ be a path satisfying Assumption \ref{assumption:continuous_path}. Let $D = \{t_1,\dots,t_K\} \subset [0,\tau]$ be a grid of sampling points, and $x^{D}$ the piecewise constant interpolation of the path $x$ sampled on the grid $D$.  For all $\alpha \in \mathbb{R}^q$, where $q := \frac{d^{N}-1}{d-1}$, we have
    \begin{align*}
        \abs{\alpha^\top\big(\mathbf{S}_N(x_{[0,t]})-\mathbf{S}_N(x^{D}_{[0,t]})\big)} \leq c_3(N) \norm{\alpha} \abs{D},
    \end{align*}
    where 
    \begin{align*}
        c_3(N) = 2e\frac{(L_xt)^{N-1}-1}{L_xt-1}L_x.
    \end{align*}
\end{theorem}

\subsection{The Cox Connection}
\label{appendix:cox_connections}

\paragraph{Signature-based embeddings.} Consider a continuous path of bounded variation $x:t \mapsto (x(t),t) \in \mathbb{R}^d$. First, remark that for every word of size one $I \in \{1,\dots,d\}$, the signature writes 
\begin{align*}
    \mathbf{S}^I(x_{[0,t]}) = \int_{0 < u_1 < t} dx^{(I)}(s) = x^{(I)}(t).
\end{align*}
Furthermore, for any word $I = (d,\dots,d)$ of size $k$ made only of the letter $d$, i.e., words that only include the time channel, we have 
\begin{align*}
    \mathbf{S}^I(x_{[0,t]}) = \int_{0 < u_1 < \dots < u_k < t} du_1 \dots du_k = \frac{1}{k !}t^k.
\end{align*}
This shows that for $x = x^{i,D}$
\begin{align*}
    \alpha^\top \mathbf{S}_N(x^{i,D}_{[0,t]})= \alpha_{I_1}^\top (1,t,t^2,\dots,t^N) +\alpha_{I_2}^\top \mathbf{X}^i(t)+ \sum\limits_{I \in I_3} \alpha_I \mathbf{S}^I(x^{i,D}_{[0,t]})
\end{align*}
where $\alpha_{I_1}$ is the subvector of $\alpha$ collecting all coefficients associated to the words $\{d\}, \{d,d\},\dots,\{d,\dots,d\}$ containing only the letter $d$, $\alpha_{I_2}$ is the subvector collecting all coefficients associated to the $d-1$ words $\{1\}, \{2\},\dots,\{d-1\}$ of size $1$, and $\alpha_{I_3}$ collects the remaining coefficients.

\textbf{NCDEs.} For any $N \geq 1$, consider the augmented vector field 
\begin{align*}
    \Tilde{\mathbf{G}}_\psi(z) = \begin{bmatrix}
         \mathbf{G}_\psi(z)& \mathbf{0}_{p\times (N-1)} \\
        \mathbf{0}_{(N -1) \times d}& \mathbf{I}_{(N -1)\times (N-1)}
    \end{bmatrix} \in \mathbb{R}^{(N-1+p) \times (N-1+d)}, \quad z \in \R^{p},
\end{align*}
and an embedding of the time series $\mathbf{X}^{i}$ of the form $\tilde{x}^{i, D}(s) = (\mathbf{X}^{i}(t_k), s, s^2, \dots, s^N)\in \mathbb{R}^{d+ N-1}$ for $s \in [t_k, t_{k+1}[$. 
The latent state of the NCDE model is now updated as 
\begin{align*}
    \Tilde{z}^{i,D}_\theta(t_{k+1}) & = \Tilde{z}^{i,D}_\theta(t_{k}) +  \Tilde{\mathbf{G}}_\psi(\Tilde{z}^{i,D}_\theta(t_{k}))\Delta \Tilde{\mathbf{X}}^{i}({t_{k+1}})\\
    & =\Tilde{z}^{i,D}_\theta(t_{k})  +  \begin{bmatrix}
        \mathbf{G}_\psi(z_\theta^{i,D})\Delta \mathbf{X}^i(t_{k+1})\\
        \Delta t_{k+1}\\
        \vdots \\
        \Delta t^N_{k+1}
    \end{bmatrix}
     = \begin{bmatrix}
        z^{i,D}_\theta(t_{k}) + \mathbf{G}_\psi(z_\theta^{i,D})\Delta \mathbf{X}^i(t_{k+1})\\
        t_{k+1}\\
        \vdots \\
        t^N_{k+1}
    \end{bmatrix}.
\end{align*}
This proves that in the NCDE based model, the intensity can similarly be written as 
\begin{align*}
    \alpha^\top \Tilde{z}^{i,D}_\theta(t) = \alpha_{I_1}^\top {z}^{i,D}_\theta(t)+ \alpha_{I_2}^\top (1,t,t^2,\dots,t^N)
\end{align*}
where $\alpha = (\alpha_{I_1},\alpha_{I_2})$, and $\alpha_1 \in \mathbb{R}^p$ and $\alpha_2 \in \mathbb{R}^N$. 

\subsection{Self-concordance}

We now state a self-concordance bound, which can be found along with its proof in \citet{bach2010self}.

\begin{lemma}
\label{lemma:bach_sc}
    Let $g:\mathbb{R}\to \mathbb{R}$ be a convex, three times differentiable function such that
    \begin{align*}
        \abs{g^{(3)}(x)} \leq M g^{(2)}(x)
    \end{align*}
    for all $x \in \mathbb{R}$ and for some $M \geq 0$. Then it follows that 
\begin{align*}
    \frac{g^{(2)}(0)}{M^2}\Phi(-Mt) \leq g(t)-g(0) - tg'(0) \leq \frac{g^{(2)}(0)}{M^2}\Phi(Mt)
\end{align*}
for all $t \geq 0$, where 
\begin{align*}
    \Phi:t\mapsto \exp(t)-t-1.
\end{align*}
\end{lemma}

\subsection{Decomposition of the difference in likelihoods}

We first define the empirical KL-divergence between the true and parameterized intensity associated to the sample $\mathcal{D}_n$ as
\begin{align*}
    \textnormal{KL}_n(\lambda_\star,\lambda^D_\theta)  &:=  \frac{1}{n} \sum\limits_{i=1}^n  \int_0^\tau \log \frac{\lambda_\star^i(s)}{\lambda^{i}_\theta(s)}\lambda_\star^i(s) Y^i(s) ds -  \frac{1}{n} \sum\limits_{i=1}^n \int_0^\tau \big( \lambda_\star^i(s) - \lambda^{i}_\theta(s)\big)Y^i(s)ds.
\end{align*}

This definition is classical for intensities of counting processes \citep{aalen2008survival,lemler2016oracle}. We now show that minimizing the empirical KL-divergence between the true and the parameterized intensity amounts to minimizing the empirical log likelihood, ignoring a noise term that will be canceled by setting the penalty accordingly.

\begin{proposition}
\label{prop:doob_meyer}
    For every $\theta \in \Theta$, the difference in likelihoods $\ell^D_n(\theta) - \ell^\star_n$ decomposes as 
    \[
     \textnormal{KL}_n(\lambda_\star,\lambda^D_\theta) - \frac{1}{n} \sum\limits_{i=1}^n \int \log \frac{\lambda^{i,D}_\theta(s)}{\lambda^i_\star(s)}dM^i(s), 
    \]
    where $M^i:[0,\tau] \to \mathbb{R}$ is a local square integrable martingale. 
\end{proposition}

This proposition is a consequence of the Doob-Meyer decomposition
$
N^i(t) = \int_0^t \lambda_\star^i(s) Y^i(s) ds+ M^i(t) 
$
of the counting process \citep{aalen2008survival}. We now furthermore define the total variation divergence as 
\begin{align*}
    \textnormal{TV}_n(\lambda_\star,\lambda^D_\theta):= \frac{1}{n} \sum\limits_{i=1}^n \int_0^\tau \abs{\lambda^i_\star(s)-\lambda_\theta^{i,D}(s)}Y^i(s)ds
\end{align*}
and the quadratic log divergence $\textnormal{D}^2_n(\lambda_\star,\lambda^D_\theta)$ as 
\begin{align*}
     \frac{1}{n} \sum\limits_{i=1}^n\int_0^\tau \big(\log \lambda^{i,D}_\theta(s)-\log \lambda^i_\star(s)\big)^2\lambda^i_\star(s)Y^i(s)ds.
\end{align*}

\begin{proposition}
\label{prop:double_inequality}
    There exist two constants $c_1,c_2 > 0$ such that 
    \begin{align*}
        c_1 \textnormal{TV}_n(\lambda_\star,\lambda^D_\theta)^2 \leq \textnormal{KL}_n(\lambda_\star,\lambda^D_\theta) \leq c_2 \textnormal{D}^2_n(\lambda_\star,\lambda^D_\theta).
    \end{align*}
    More precisely, the constants $c_1,c_2$ are functions of $\Theta,L_x,\tau$ and $L_{\mathbf{G}_\star}$ and are given explicitly in Appendix \ref{appendix:proof_double_inequality}.
\end{proposition}

This bound is obtained by combining a Pinsker-type inequality (Proposition \ref{prop:pinsker}) and a self-concordance bound (Proposition \ref{prop:self_concordance}). It is informative in two ways. First, it shows that minimizing the negative empirical log-likelihood and hence the KL-divergence between the true and parameterized intensity will lead to a minimization of the total variation between the two intensities. Secondly, it shows that the KL-divergence is upper bounded by a term involving the difference of the logarithms of the intensities. We make use of this second bound to obtain a bias-variance decomposition.

\newpage

\section{Proofs}

\subsection{Proof of Proposition \ref{prop:doob_meyer}}
\begin{proof}

Thanks to the Doob-Meyer decomposition of Equation~\eqref{eq:doob_meyer}, the log-likelihood associated to individual $i$ is 
\begin{align*}
    \ell^D_i(\theta) & = \int_0^\tau \lambda^{i,D}_\theta(s) Y^i(s) ds  -\int_0^\tau \log \lambda_\theta^{i,D}(s) dN^i(s)\\
    & = \int_0^\tau \Big( \lambda^{i,D}_\theta(s) - \log \lambda_\theta^{i,D}(s) \lambda_\star^i(s)\Big) Y^i(s)ds - \int_0^\tau \log \lambda^{i,D}_\theta(s)dM^i(s). 
\end{align*}
Similarly, the log-likelihood associated to the true intensity $\lambda_\star^i$ writes 
\begin{align*}
    \ell_n^\star & = \int_0^\tau \lambda^{i}_\star(s) Y^i(s) ds  -\int_0^\tau \log \lambda_\star^{i}(s) dN^i(s)\\
& = \int_0^\tau \Big( \lambda^{i}_\star(s) - \log \lambda_\star^{i}(s) \lambda_\star^i(s)\Big) Y^i(s)ds - \int_0^\tau \log \lambda^{i}_\star(s)dM^i(s). 
\end{align*}
Hence, we get
\begin{align*}
     & \ell^D_n(\theta) - \ell^\star_n \\
     & = \frac{1}{n} \sum\limits_{i=1}^n \int \Big[\lambda^{i,D}_\theta(s)-\lambda_\star^i(s)- \lambda_\star^i(s) \log \frac{\lambda^{i,D}_\theta(s)}{\lambda_\star^i(s)}\Big] Y^i(s) ds  -  \frac{1}{n} \sum\limits_{i=1}^n \int \log \frac{\lambda^{i,D}_\theta(s)}{\lambda_\star^i(s)}dM^i(s)\\
    & =\textnormal{KL}_n(\lambda_\star,\lambda^D_\theta)- \frac{1}{n} \sum\limits_{i=1}^n \int \log \frac{\lambda^{i,D}_\theta(s)}{\lambda_\star^i(s)}dM^i(s).
\end{align*}
This concludes the proof.

\end{proof}

\subsection{Proof of Proposition \ref{prop:double_inequality}}
\label{appendix:proof_double_inequality}
To prove Proposition \ref{prop:double_inequality}, we essentially combine a Pinsker-type inequality and a self-concordance bound. We prove these two bounds separatly bellow. Combining them yields the double inequality of Proposition \ref{prop:double_inequality}. In the following, for all $t \in [0,\tau]$, we let 
\begin{align*}
    \Lambda_\theta^{i,D}(t) = \int_0^t \lambda_\theta^{i,D}(s)Y^i(s)ds
\end{align*}
and 
\begin{align*}
    \Lambda_\star^i(t) = \int_0^t \lambda_\star^{i}(s)Y^i(s)ds
\end{align*}
be the cumulative hazard functions. 

\begin{proposition}[Pinsker's inequality.]
\label{prop:pinsker}
    Let $\lambda_\star$ be the true intensity defined in Equations~\eqref{eq:true_intensity} and~\eqref{eq:true_CDE}, and $\lambda_\theta$ be an intensity parameterized by $\theta \in \Theta$. Under Assumptions \ref{assumption:continuous_path}, \ref{assumption:bounded_static_features},  and \ref{assumption:true_vf_bounded}, we have that 
\begin{align*}
    c_1 \textnormal{TV}_n(\lambda_\star,\lambda^D_\theta)^2 \leq \textnormal{KL}_n(\lambda_\star,\lambda^D_\theta),
\end{align*}
with 
\begin{align*}
        c_1 := \frac{1}{\tau} \exp(-B_{\beta,2} B_\mathbf{W})\Big[ \frac{4}{3}\exp\Big(\norm{\mathbf{G}_\star(0)}_{\textnormal{op}} L_x \tau \exp\big(L_{\mathbf{G}_\star}L_x \tau \big) \Big)+\frac{2}{3}\exp\big(B_\alpha \exp(L_x \tau) \big)\Big]^{-1}
\end{align*}
for signature-based embeddings and 
\begin{align*}
    c_1 :=  \frac{1}{\tau} \exp(-B_{\beta,2} B_\mathbf{W})\Big[ \frac{4}{3}\exp \Big(\norm{\mathbf{G}_\star(0)}_{\textnormal{op}} L_x \tau \exp\big(L_{\mathbf{G}_\star}L_x \tau \big) \Big)+\frac{2}{3}\exp\Big( \norm{\mathbf{G}_\psi(0)}_{\textnormal{op}} L_x \tau \exp\big(L_{\mathbf{G}_\psi}L_x \tau)\Big)\Big]^{-1}
\end{align*}
for NCDEs.
\end{proposition}

\begin{proof}
    We have 
    \begin{align*}
        \textnormal{TV}_n(\lambda_\star,\lambda^D_\theta) = \frac{1}{n} \sum\limits_{i=1}^n \int_0^\tau | \lambda^{i}_\star(s)-\lambda_\theta^{i,D}(s)| Y^i(s)ds & = \frac{1}{n} \sum\limits_{i=1}^n \int_0^\tau \bigg|\frac{\lambda^i_\star(s)}{\lambda_\theta^{i,D}(s)}-1\bigg|\lambda_\theta^{i,D}(s)Y^i(s)ds\\
        & = \int_0^\tau \sqrt{\Bigg(\frac{\lambda^i_\star(s)}{\lambda_\theta^{i,D}(s)}-1\Bigg)^2}\lambda_\theta^{i,D}(s)Y^i(s)ds,
    \end{align*}
    where we have used that $\abs{x} = \sqrt{x^2}$. Note that by definition, $\lambda_\theta^{i,D}(s) = \exp\big(z_\theta^{i,D}(s) + \beta^\top \mathbf{W}^i\big) > 0$ for all $s \in [0,\tau]$: we can thus safely divide by this term. 
    Now, since for all $x\geq 0$
    \begin{align*}
        (x-1)^2 \leq \Big(\frac{4}{3}+\frac{2}{3}x\Big)\xi(x)
    \end{align*}
    with 
    \begin{align*}
        \xi(x) := x \log x -x + 1,
    \end{align*}
 one obtains 
    \begin{align*}
        \textnormal{TV}_n(\lambda_\star,\lambda^D_\theta) \leq\frac{1}{n} \sum\limits_{i=1}^n \int_0^\tau \sqrt{ \Big(\frac{4}{3}+\frac{2}{3}\frac{\lambda^i_\star(s)}{\lambda_\theta^{i,D}(s)}\Big)\xi\Big(\frac{\lambda^i_\star(s)}{\lambda_\theta^{i,D}(s)}\Big)}\lambda_\theta^{i,D}(s)Y^i(s)ds. 
    \end{align*}
    Using the Cauchy-Schwarz inequality yields 
    \begin{align*}
    \textnormal{TV}_n(\lambda_\star,\lambda_\theta)
        & \leq \frac{1}{n}\sum\limits_{i=1}^n \Bigg[\int_0^\tau \Big(\frac{4}{3}+\frac{2}{3}\frac{\lambda^i_\star(s)}{\lambda_\theta^{i, D}(s)}\Big) \lambda_\theta^{i, D}(s)Y^i(s)ds\Bigg]^{1/2} \Bigg[\int_0^\tau \xi\Big(\frac{\lambda^i_\star(s)}{\lambda_\theta^{i, D}(s)}\Big) \lambda_\theta^{i, D}(s)Y^i(s)ds\Bigg]^{1/2} \\
        & \leq \frac{1}{n}\sum\limits_{i=1}^n \Bigg[\frac{4}{3}\Lambda_\theta^{i,D}(\tau) + \frac{2}{3}\Lambda_\star^i(\tau)\Bigg]^{1/2}\Bigg[\int_0^\tau\Big( \lambda_\star^i(s)\log \frac{\lambda_\star^i(s)}{\lambda_\theta^{i, D}(s)} +  \lambda_\theta^{i, D}(s) - \lambda_\star^i(s)\Big)Y^i(s)ds \Bigg]^{1/2}\\
        & \leq \max\limits_{i=1,\dots,n}\Bigg[\frac{4}{3}\Lambda_\theta^{i,D}(\tau) + \frac{2}{3}\Lambda_\star^i(\tau)\Bigg]^{1/2} \sqrt{\textnormal{KL}_n(\lambda_\star,\lambda_\theta^{D})}.
    \end{align*}

    Taking the square on both sides yields  
    \begin{align*}
        \textnormal{TV}_n(\lambda_\star,\lambda^D_\theta)^2 \leq \max\limits_{i=1,\dots,n}\Bigg[\frac{4}{3}\Lambda_\theta^{i,D}(\tau) + \frac{2}{3}\Lambda_\star^i(\tau)\Bigg] \textnormal{KL}_n(\lambda_\star,\lambda_\theta^D). 
    \end{align*}
    \paragraph{Bounding the true cumulative hazard function.}
    We have
    \begin{align*}
        \max\limits_{i=1,\dots,n}\Bigg[\frac{4}{3}\Lambda_\theta^{i,D}(\tau) + \frac{2}{3}\Lambda_\star^i(\tau)\Bigg]  \leq \frac{4}{3}\max\limits_{i=1,\dots,n} \Lambda_\theta^{i,D}(\tau)  + \frac{2}{3}\max\limits_{i=1,\dots,n} \Lambda_\star^i(\tau).
    \end{align*}
    Moreover, Lemma \ref{lemma:intensity_bounded} yeilds that, for all $s \in [0,\tau]$,
    \begin{align*}
        \log \lambda_\star^i(s) & \leq B_{\beta,2} B_\mathbf{W} +  \norm{\mathbf{G}_\star(0)}_{\textnormal{op}} L_xs \exp\big(L_{\mathbf{G}_\star}L_xs \big)\\
        & \leq B_{\beta,2} B_\mathbf{W} +  \norm{\mathbf{G}_\star(0)}_{\textnormal{op}} L_x \tau \exp\big(L_{\mathbf{G}_\star}L_x \tau \big).
    \end{align*}
    Hence for all $i=1,\dots,n$, it holds that
    \begin{align*}
        \Lambda_\star^i(\tau) = \int_0^\tau \lambda_\star^i(s) Y^i(s) ds & \leq \tau \sup\limits_{s \in [0,\tau]} \lambda_\star^i(s)\leq \tau \exp\Big(B_{\beta,2} B_\mathbf{W} +  \norm{\mathbf{G}_\star(0)}_{\textnormal{op}} L_x \tau \exp\big(L_{\mathbf{G}_\star}L_x \tau \big)\Big).
    \end{align*}
    Since this last bound does not depend on $i$, this gives us that 
    \begin{align*}
         \frac{2}{3}\max\limits_{i=1,\dots,n} \Lambda_\star^i(\tau) \leq \frac{2\tau}{3}\exp\Big(B_{\beta,2} B_\mathbf{W} +  \norm{\mathbf{G}_\star(0)}_{\textnormal{op}} L_x \tau \exp\big(L_{\mathbf{G}_\star}L_x \tau \big)\Big).
    \end{align*}
    Similarly, one obtains that 
    \begin{align*}
        \frac{4}{3}\max\limits_{i=1,\dots,n} \Lambda_\theta^{i,D}(\tau)  \leq \frac{4\tau}{3} \max\limits_{i=1,\dots,n} \sup\limits_{s \in [0,\tau]} \lambda^i_\theta(s).
    \end{align*}
    \paragraph{Signature-based embeddings.}
    For signature-based embeddings, we have 
    \begin{align*}
        \sup\limits_{s \in [0,\tau]} \lambda^i_\theta(s) \leq \exp(B_{\beta,2} B_\mathbf{W})\Big[\exp(B_\alpha \exp(L_x \tau))\Big].
    \end{align*}
    \paragraph{NCDEs.}
    For NCDEs, one obtains
    \begin{align*}
         \sup\limits_{s \in [0,\tau]} \lambda^i_\theta(s) \leq \exp\big(B_{\beta,2} B_\mathbf{W})\exp\Big[ \norm{\mathbf{G}_\psi(0)}_{\textnormal{op}} L_x \tau \exp\big(L_{\mathbf{G}_\psi}L_x \tau)\Big]. 
    \end{align*}

    \paragraph{Final Bound.}
    Putting everything together, one finally has that 
    \begin{align*}
        c_1  \textnormal{TV}_n(\lambda_\star,\lambda^D_\theta)^2 \leq \textnormal{KL}_n(\lambda_\star,\lambda_\theta), 
    \end{align*}
    where 
    \begin{align*}
        c_1 = \frac{1}{\tau} \exp(-B_{\beta,2} B_\mathbf{W})\Big[ \frac{4}{3}\exp(\norm{\mathbf{G}_\star(0)}_{\textnormal{op}} L_x \tau \exp\big(L_{\mathbf{G}_\star}L_x \tau \big))+\frac{2}{3}\exp(B_\alpha \exp(L_x \tau))\Big]^{-1}
    \end{align*}
    when using signatures and 
    \begin{align*}
       c_1:= \frac{1}{\tau} \exp(-B_{\beta,2} B_\mathbf{W})\Big[ \frac{4}{3}\exp(\norm{\mathbf{G}_\star(0)}_{\textnormal{op}} L_x \tau \exp\big(L_{\mathbf{G}_\star}L_x \tau \big))+\frac{2}{3}\exp\Big[ \norm{\mathbf{G}_\psi(0)}_{\textnormal{op}} L_x \tau \exp\big(L_{\mathbf{G}_\psi}L_x \tau)\Big]\Big]^{-1}
    \end{align*}
    when using NCDEs.
    
\end{proof}

We also prove the following self-concordance bound, a close result can be found in~\citet{lemler2016oracle}.

\begin{proposition}[A self-concordance bound.]
    \label{prop:self_concordance}
   Let $\lambda_\star$ be the true intensity defined in Equations~\eqref{eq:true_intensity} and ~\eqref{eq:true_CDE}, and $\lambda_\theta$ be a intensity parameterized by $\theta \in \Theta$. Under Assumptions \ref{assumption:continuous_path}, \ref{assumption:bounded_static_features}, and \ref{assumption:true_vf_bounded}, it holds that
    \begin{align*}
    \textnormal{KL}_n(\lambda_\star,\lambda_\theta) \leq c_2 \textnormal{D}^2_n(\lambda_\star,\lambda_\theta^D),
    \end{align*}
    where 
    \begin{align*}
        c_2 := \frac{\exp{M}-M-1}{M}, 
    \end{align*}
    and
    \begin{align*}
        M:= \exp(B_{\beta,2} B_\mathbf{W}) \bigg[\exp\Big(\norm{\mathbf{G}_\star(0)}_{\textnormal{op}} L_x \tau \exp\big(L_{\mathbf{G}_\star}L_x \tau \big)\Big) + \exp\Big(B_\alpha \exp(L_x \tau)\Big)\bigg]
    \end{align*}
    when using signatures and 
    \begin{align*}
        M:= \exp(B_{\beta,2} B_\mathbf{W}) \bigg[\exp\Big(\norm{\mathbf{G}_\star(0)}_{\textnormal{op}} L_x \tau \exp\big(L_{\mathbf{G}_\star}L_x \tau \big)\Big) + \exp\Big( \norm{\mathbf{G}_\psi(0)}_{\textnormal{op}} L_x \tau \exp\big(L_{\mathbf{G}_\psi}L_x \tau)\Big) \bigg]
    \end{align*}
    when using NCDEs.
\end{proposition}

\begin{proof}
    Define 
    \begin{align*}
        g:t \mapsto \frac{1}{n} \sum\limits_{i=1}^n\int_0^\tau \Big[\exp\big(t(\log \lambda^{i,D}_\theta(s)-\log \lambda^i_\star(s))\big)-t \log \frac{\lambda^{i,D}_\theta(s)}{\lambda^i_\star(s)}-1\Big]\lambda^i_\star(s)Y^i(s)ds.
    \end{align*}
    This function satisfies all assumptions needed in Lemma \ref{lemma:bach_sc}. The function $f:t \mapsto \exp( \gamma t) - \gamma t-1$ is convex for all $\gamma \in \mathbb{R}$. Convexity is preserved by integration against a positive function: indeed, if $f(t,s)$ is convex in $t$ for all $s$ and $h(s) \geq 0$ for all $s\in \mathcal{A}$, then 
    \begin{align*}
        \int_\mathcal{A} f(t,s)h(s)ds 
    \end{align*}
    is convex in $t$ \citep[see][Page 79]{boyd2004convex}. The function is also clearly $C^\infty$. Finally, remark that by differentiating the integral, one obtains 
    \begin{align*}
        & g'(t) = \frac{1}{n} \sum\limits_{i=1}^n\int_0^\tau \Big[ \big(\log \lambda^{i,D}_\theta(s)-\log \lambda^i_\star(s) \big)\exp\big(t(\log \lambda^{i,D}_\theta(s)-\log \lambda^i_\star(s))\big)- \log \frac{\lambda^{i,D}_\theta(s)}{\lambda^i_\star(s)}\Big]\lambda^i_\star(s)Y^i(s)ds,\\
        &  g^{(2)}(t) = \frac{1}{n} \sum\limits_{i=1}^n\int_0^\tau \big(\log \lambda^{i,D}_\theta(s)-\log \lambda^i_\star(s) \big)^2\exp\big(t(\log \lambda^{i,D}_\theta(s)-\log \lambda^i_\star(s))\big)\lambda^i_\star(s)Y^i(s)ds,\\
        & g^{(3)}(t) = \frac{1}{n} \sum\limits_{i=1}^n\int_0^\tau \big(\log \lambda^{i,D}_\theta(s)-\log \lambda^i_\star(s) \big)^3\exp\big(t(\log \lambda^{i,D}_\theta(s)-\log \lambda^i_\star(s))\big)\lambda^i_\star(s)Y^i(s)ds.
    \end{align*}
    Hence
    \begin{align*}
         \abs{g^{(3)}(t)} & \leq \frac{1}{n} \sum\limits_{i=1}^n \int_0^\tau \norm{\lambda^{i,D}_\theta-\lambda^i_\star}_\infty (\log \lambda^{i,D}_\theta(s)-\log \lambda^i_\star(s))^2\exp\big(t(\log \lambda^{i,D}_\theta(s)-\log \lambda^i_\star(s))\big)\lambda^i_\star(s)Y^i(s)ds\\
        & = M g^{(2)}(t)
    \end{align*}
    for all $t \in \mathbb{R}$ with 
    \begin{align*}
        M := \max\limits_{i=1,\dots,n} \norm{\lambda^{i,D}_\theta-\lambda^i_\star}_\infty.
    \end{align*}
    Using now Lemma \ref{lemma:bach_sc}, we get at $t=1$
    \begin{align*}
       \frac{g^{(2)}(0)}{M^2} \Phi(-M)  \leq g(1)\underbrace{-g(0)-g'(0)}_{=0} \leq \frac{g^{(2)}(0)}{M^2} \Phi(M). 
    \end{align*}
    We have  
    \begin{align*}
        g^{(2)}(0) = \frac{1}{n} \sum\limits_{i=1}^n\int_0^\tau \big(\log \lambda^{i,D}_\theta(s)-\log \lambda^i_\star(s) \big)^2\lambda^i_\star(s)Y^i(s)ds 
    \end{align*} 
    Finally, remark that 
    \begin{align*}
        g(1) = \textnormal{KL}_n(\lambda_\star,\lambda_\theta),
    \end{align*}
    and hence 
    \begin{align*}
    \textnormal{KL}_n(\lambda_\star,\lambda_\theta) \leq \frac{\exp{M}-M-1}{nM} \sum\limits_{i=1}^n\int_0^\tau \big(\log \lambda^{i,D}_\theta(s)-\log \lambda^i_\star(s) \big)^2\lambda^i_\star(s)Y^i(s)ds. 
    \end{align*}
    Turning to the constant $M$, remark that 
    \begin{align}
        M := \max\limits_{i=1,\dots,n} \norm{\lambda^i_\theta-\lambda^i_\star}_\infty \leq \max\limits_{i= 1,\dots,n} \norm{\lambda^i_\theta}_\infty + \max\limits_{i= 1,\dots,n} \norm{\lambda^i_\star}_\infty.
    \end{align}
    Similarly to what is done in the proof of Proposition \ref{prop:pinsker}, we have
    \begin{align*}
        \max\limits_{i= 1,\dots,n} \norm{\lambda^i_\star}_\infty \leq \exp\bigg(B_{\beta,2} B_\mathbf{W} +  \norm{\mathbf{G}_\star(0)}_{\textnormal{op}} L_x \tau \exp\big(L_{\mathbf{G}_\star}L_x \tau \big)\bigg), 
    \end{align*}
    and 
    \begin{align*}
        \max\limits_{i= 1,\dots,n} \norm{\lambda^{i,D}_\theta}_\infty \leq \exp(B_{\beta,2} B_\mathbf{W})\Big[\exp\big(B_\alpha \exp(L_x \tau) \big)\Big]
    \end{align*}
    when using signatures and 
    \begin{align*}
         \max\limits_{i= 1,\dots,n} \norm{\lambda^{i,D}_\theta}_\infty \leq \exp\big(B_{\beta,2} B_\mathbf{W})\exp\Big[ \norm{\mathbf{G}_\psi(0)}_{\textnormal{op}} L_x \tau \exp\big(L_{\mathbf{G}_\psi}L_x \tau)\Big]
    \end{align*}
    when using NCDEs. 
\end{proof}

\subsection{Formal Statement and Proof of the Risk Decomposition}
\label{appendix:proof_risk_decomposition}

\begin{proposition}\label{prop:bias_majoration}
     Let $\lambda_\star$ be the true intensity defined in Equations~\ref{eq:true_intensity}-~\ref{eq:true_CDE} and $\lambda^D_\theta$ be the intensity parameterized by $\theta \in \Theta$. Under Assumptions \ref{assumption:continuous_path}, \ref{assumption:bounded_static_features}, \ref{assumption:true_vf_bounded} and \ref{assumption:decaying_derivatives}, for signature-based embeddings and $\theta = (\alpha_{\star,N},\beta_\star)$, it holds that
    \begin{align*}
        \textnormal{D}^2_n(\lambda_\star,\lambda_{ \theta}^D) & \leq 2 \tau \exp\bigg(B_{\beta,2} B_\mathbf{W} +  \norm{\mathbf{G}_\star(0)}_{\textnormal{op}} L_x \tau \exp\big(L_{\mathbf{G}_\star}L_x \tau \big)\bigg)\underbrace{\Bigg[\frac{(d L_x)^{N+1}}{(N+1)!}\Lambda_{N+1}(\mathbf{G}_\star)\Bigg]^2 \frac{\tau^{2N+3}}{2N + 3}}_{\textnormal{Approximation bias}}\\
    & \quad + 2 \tau \exp\bigg(B_{\beta,2} B_\mathbf{W} +  \norm{\mathbf{G}_\star(0)}_{\textnormal{op}} L_x \tau \exp\big(L_{\mathbf{G}_\star}L_x \tau \big)\bigg)\underbrace{c_3(N)^2 \sum\limits_{k=1}^N d^k\Gamma_{k}(\mathbf{G}_\star)^2 \abs{D}^2}_{\textnormal{Discretization bias}}
    \end{align*}
    whereas for NCDEs, for $\theta = (\alpha,\psi, \beta_\star)$, it holds that
    \begin{align*}
        \textnormal{D}^2_n(\lambda_\star,\lambda_{ \theta}^D) & \leq 2 \tau \exp(2L_{\mathbf{G}_\psi}L_x)\bigg[\underbrace{\Big(B_\alpha \big(1+L_{\mathbf{G}_\psi}L_x \mathcal{K}\big) L_x\abs{D}\Big)^2}_{\textnormal{Discretization bias}} + \underbrace{\Big(\max_{v \in \Omega } \norm{\alpha^\top \mathbf{G}_\psi(v)-\mathbf{G}_\star(v)}\Big)^2}_{\textnormal{Approximation bias}}\bigg] \\
        & \quad \times \exp\bigg(B_{\beta,2} B_\mathbf{W} +  \norm{\mathbf{G}_\star(0)}_{\textnormal{op}} L_x \tau \exp\big(L_{\mathbf{G}_\star}L_x \tau \big)\bigg).
    \end{align*}
\end{proposition}

\begin{proof} We have the following.
    \paragraph{NCDEs.} We first consider the case where the individual time series are embedded using NCDEs. Considering a general $\theta \in \Theta$, we then have 
    \begin{align*}
        \textnormal{D}^2_n(\lambda_\star,\lambda^D_\theta) & = \frac{1}{n} \sum\limits_{i=1}^n\int_0^\tau \big(\alpha^\top z^{i,D}_\theta(s) + \beta^\top \mathbf{W}^i - z_\star^i(s) + \beta_\star^\top \mathbf{W}^i\big)^2\lambda^i_\star(s)Y^i(s)ds\\
        & = \frac{1}{n} \sum\limits_{i=1}^n\int_0^\tau \big(\alpha^\top z^{i,D}_\theta(s) - z_\star^i(s) + (\beta-\beta_\star)^\top \mathbf{W}^i \big)^2\lambda^i_\star(s)Y^i(s)ds\\
        & \leq \frac{2}{n} \sum\limits_{i=1}^n\int_0^\tau \big(\alpha^\top z^{i,D}_\theta(s) - z_\star^i(s) \big)^2 \lambda^i_\star(s)Y^i(s)ds + \frac{2}{n} \sum\limits_{i=1}^n\int_0^\tau \big((\beta-\beta_\star)^\top \mathbf{W}^i\big)^2\lambda^i_\star(s)Y^i(s)ds,
    \end{align*}
    where the last inequality is obtained using $(a+b)^2 \leq 2a^2 + 2b^2$. Since this is true for all $\theta \in \Theta$, we first chose $\beta = \beta_\star$, hence cancelling the second term. Turning to the remaining term, we have 
    \begin{align*}
        \alpha^\top z^{i,D}_\theta(s) - z_\star^i(s) = \alpha^\top z^{i,D}_\theta(s) - \alpha^\top z^{i}_\theta(s) + \alpha^\top z^{i}_\theta(s) - z_\star^i(s),
    \end{align*}
    where $z^i_\theta(s)$ is the solution to the CDE driven by continuous unobserved path  
    \begin{align*}
        dz^i_\theta(t) = \mathbf{G}_\psi (z^i_\theta(t))dx^i(t)
    \end{align*}
    with initial condition $z^i(0) = 0 \in \mathbb{R}$. Using the continuity of the flow (Theorem \ref{thm:flow_continuity}), one obtains 
    \begin{align*}
        \abs{\alpha^\top z^{i,D}_\theta(s) - \alpha^\top z^{i}_\theta(s)} & \leq \norm{\alpha} \big\|z^{i,D}_\theta(s) - z^{i}_\theta(s)\big\| \\
        & \leq B_\alpha \exp(L_{\mathbf{G}_\psi}L_x)\big(1+L_{\mathbf{G}_\psi}L_x \mathcal{K}\big) \norm{x^i-x^{i,D}}_\infty.
    \end{align*}
    Using the fact that 
    \begin{align*}
        \norm{x^i-x^{i,D}} \leq L_x \abs{D},
    \end{align*}
    one finally obtains 
    \begin{align*}
        \abs{\alpha^\top z^{i,D}_\theta(s) - \alpha^\top z^{i}_\theta(s)} \leq B_\alpha \exp(L_{\mathbf{G}_\psi}L_x)\big(1+L_{\mathbf{G}_\psi}L_x \mathcal{K}\big) L_x\abs{D}.
    \end{align*}
    Using again the continuity of the flow, one also has that 
    \begin{align*}
        \abs{\alpha^\top z^{i}_\theta(s) - z_\star^i(s)} \leq \exp(L_{\mathbf{G}_\psi}L_x) \max_{v \in \Omega } \norm{\alpha^\top \mathbf{G}_\psi(v)-\mathbf{G}_\star(v)}
    \end{align*}
    since $\alpha^\top z_\theta^i(s)$ is the solution to the CDE 
    \begin{align*}
        du^i(t) = \alpha^\top \mathbf{G}_\psi(u^i(t))dx^i(t) 
    \end{align*}
    with initial condition $u^i(0)= 0$. Putting everything together, one obtains
    \begin{align*}
        (\alpha^\top z^{i,D}_\theta(s) - z_\star^i(s))^2 \leq  2 \exp(2L_{\mathbf{G}_\psi}L_x)\bigg[\Big(B_\alpha \big(1+L_{\mathbf{G}_\psi}L_x \mathcal{K}\big) L_x\abs{D}\Big)^2 + \Big(\max_{v \in \Omega } \norm{\alpha^\top \mathbf{G}_\psi(v)-\mathbf{G}_\star(v)}\Big)^2\bigg].
    \end{align*}
    Plugging this result in the original inequality on the squared log-divergence yields
    \begin{align*}
        \textnormal{D}^2_n(\lambda_\star,\lambda^D_\theta) \leq & 2 \tau   \exp(2L_{\mathbf{G}_\psi}L_x)\Bigg[\Big(B_\alpha \big(1+L_{\mathbf{G}_\psi}L_x \mathcal{K}\big) L_x\abs{D}\Big)^2 + \Big(\max_{v \in \Omega } \norm{\alpha^\top \mathbf{G}_\psi(v)-\mathbf{G}_\star(v)}\Big)^2\Bigg] \\
        & \times \exp\Bigg(B_{\beta,2} B_\mathbf{W} +  \norm{\mathbf{G}_\star(0)}_{\textnormal{op}} L_x \tau \exp\big(L_{\mathbf{G}_\star}L_x \tau \big)\Bigg).
    \end{align*}
    \paragraph{Signature-based embeddings.} We proceed in a similar fashion. We have for any $\theta \in \Theta$ that
    \begin{align*}
    \big(\log \lambda^{i,D}_\theta(s)-\log \lambda^i_\star(s)\big)^2 & = \Big( \alpha^\top \mathbf{S}_N(x^{i,D}_{[0,s]}) + \beta^\top \mathbf{W}^i - \alpha_\star^\top \mathbf{S}(x^{i}_{[0,s]}) - \beta_\star^\top \mathbf{W}^i\Big)^2\\
    & = \Big( \alpha^\top \mathbf{S}_N(x^{i,D}_{[0,s]}) - \alpha_\star^\top \mathbf{S}(x^{i}_{[0,s]}) + (\beta-\beta_\star) ^\top \mathbf{W}^i\Big)^2.
    \end{align*}
    Now, we furthermore have 
    \begin{align*}
        \alpha^\top\mathbf{S}_N(x^{i,D}_{[0,s]}) - \alpha_\star^\top \mathbf{S}(x^i_{[0,s]})& =  \alpha^\top\mathbf{S}_N(x^{i,D}_{[0,s]}) - \alpha_{\star,N}^\top\mathbf{S}_N(x^{i}_{[0,s]}) + \alpha_{\star,N}^\top\mathbf{S}_N(x^{i}_{[0,s]}) - \alpha_\star^\top \mathbf{S}(x^i_{[0,s]}).
    \end{align*}
    In particular, taking $\alpha = \alpha_{\star,N}$, we have
    \begin{align*}
        \alpha_{\star,N}^\top\mathbf{S}_N(x^{i,D}_{[0,s]}) - \alpha_\star^\top \mathbf{S}(x^i_{[0,s]}) & =  \alpha_{\star,N}^\top \Big(\mathbf{S}_N(x^{i,D}_{[0,s]})-\mathbf{S}_N(x^{i}_{[0,s]})\Big) + \alpha_{\star,N}^\top\mathbf{S}_N(x^{i}_{[0,s]}) - \alpha_\star^\top \mathbf{S}(x^i_{[0,s]}).
    \end{align*}
    Using Proposition \ref{prop:truncation_bias_sig}, one obtains
    \begin{align*}
    \abs{\alpha_{\star,N}^\top\mathbf{S}_N(x^{i}_{[0,s]}) - \alpha_\star^\top \mathbf{S}(x^i_{[0,s]})} \leq \frac{(d L_x s)^{N+1}}{(N+1)!}\Lambda_{N+1}(\mathbf{G}_\star). 
    \end{align*}
    Additionally, using Theorem \ref{thm:discretization_signatures}, we have 
    \begin{align*}
        \big|\alpha_{\star,N}^\top \big(\mathbf{S}_N(x^{i,D}_{[0,s]})-\mathbf{S}_N(x^{i}_{[0,s]})\big)\big| \leq \norm{\alpha_{\star,N}} c_3(N) \abs{D}.
    \end{align*}
    We obtain for $\theta = (\alpha_{\star,N},\beta_\star)$ that
    \begin{align*}
        \big(\log \lambda^{i,D}_{\theta}(s)-\log \lambda^i_\star(s)\big)^2  \leq 2 \Big[\frac{(d L_x s)^{N+1}}{(N+1)!}\Lambda_{N+1}(\mathbf{G}_\star)\Big]^2 + 2\norm{\alpha_{\star,N}}^2 c_3(N)^2 \abs{D}^2
    \end{align*}
    using the fact that $(a+b)^2 \leq 2 a^2+2b^2$. Finally, integrating yields
\begin{align*}
    \textnormal{D}^2_n(\lambda_\star,\lambda^D_\theta) & \leq \frac{1}{n} \sum\limits_{i=1}^n \exp\bigg(B_{\beta,2} B_\mathbf{W} +  \norm{\mathbf{G}_\star(0)}_{\textnormal{op}} L_x \tau \exp\big(L_{\mathbf{G}_\star}L_x \tau \big)\bigg) 2\tau \Big[\frac{(d L_x)^{N+1}}{(N+1)!}\Lambda_{N+1}(\mathbf{G}_\star)\Big]^2 \int_0^\tau  s^{2(N+1)}ds \\
    & + \frac{1}{n} \sum\limits_{i=1}^n \exp\bigg(B_{\beta,2} B_\mathbf{W} +  \norm{\mathbf{G}_\star(0)}_{\textnormal{op}} L_x \tau \exp\big(L_{\mathbf{G}_\star}L_x \tau \big)\bigg) 2 \tau \norm{\alpha_{\star,N}}^2 c_3(N)^2 \abs{D}^2\\
    & \leq 2 \tau \exp\bigg(B_{\beta,2} B_\mathbf{W} +  \norm{\mathbf{G}_\star(0)}_{\textnormal{op}} L_x \tau \exp\big(L_{\mathbf{G}_\star}L_x \tau \big)\bigg)\Bigg[\frac{(d L_x)^{N+1}}{(N+1)!}\Lambda_{N+1}(\mathbf{G}_\star)\Bigg]^2 \frac{\tau^{2N+3}}{2N + 3}\\
    & + 2 \tau \exp\bigg(B_{\beta,2} B_\mathbf{W} +  \norm{\mathbf{G}_\star(0)}_{\textnormal{op}} L_x \tau \exp\big(L_{\mathbf{G}_\star}L_x \tau \big)\bigg)\norm{\alpha_{\star,N}}^2 c_3(N)^2 \abs{D}^2.
\end{align*}
Using Lemma \ref{lemma:alpha_bound}, we can furthermore simplify the bound to 
\begin{align*}
     \textnormal{D}^2_n(\lambda_\star,\lambda^D_\theta) & \leq 2 \tau \exp\Bigg(B_{\beta,2} B_\mathbf{W} +  \norm{\mathbf{G}_\star(0)}_{\textnormal{op}} L_x \tau \exp\big(L_{\mathbf{G}_\star}L_x \tau \big)\Bigg)\Bigg[\frac{(d L_x)^{N+1}}{(N+1)!}\Gamma_{N+1}(\mathbf{G}_\star)\Bigg]^2 \frac{\tau^{2N+3}}{2N + 3}\\
    & + 2 \tau \exp\Bigg(B_{\beta,2} B_\mathbf{W} +  \norm{\mathbf{G}_\star(0)}_{\textnormal{op}} L_x \tau \exp\big(L_{\mathbf{G}_\star}L_x \tau \big)\Bigg)c_3(N)^2 \sum\limits_{k=1}^N d^k\Gamma_{k}(\mathbf{G}_\star)^2 \abs{D}^2.
\end{align*}

\end{proof}

\subsection{Proof of Theorem~\ref{theo:completeriskbound}
    }
\label{appendix:proof_full_risk_bound}
\begin{theorem}[Formal statement of Theorem~\ref{theo:completeriskbound}]
    Consider the signature-based embedding. Let $\Hat{\theta} = (\Hat{\alpha},\Hat{\beta})$ be the solution of \eqref{eq:minimization_problem} with $\textnormal{pen}(\theta) = \eta_1\norm{\alpha}_1 + \eta_2\norm{\beta}_1$. Under Assumptions \ref{assumption:continuous_path}, \ref{assumption:bounded_static_features}, \ref{assumption:true_vf_bounded}  and \ref{assumption:decaying_derivatives}, and writing $\beta_\star = (\beta_\star^{(1)},\dots,\beta_\star^{(s)} ) $, we have for any $N \geq 1$ and any $x > 0$ that with probability greater than $1-4e ^{-x}$
    \begin{align*}
        \ell^D_n(\Hat{\theta}) - \ell_n^\star \leq & 2 \tau \exp\bigg(B_{\beta,2} B_\mathbf{W} +  \norm{\mathbf{G}_\star(0)}_{\textnormal{op}} L_x \tau \exp\big(L_{\mathbf{G}_\star}L_x \tau \big)\bigg)\underbrace{\Bigg[\frac{(d L_x)^{N+1}}{(N+1)!}\Gamma_{N+1}(\mathbf{G}_\star)\Bigg]^2 \frac{\tau^{2N+3}}{2N + 3}}_{\textnormal{Approximation bias}}\\
    & + 2 \tau \exp\bigg(B_{\beta,2} B_\mathbf{W} +  \norm{\mathbf{G}_\star(0)}_{\textnormal{op}} L_x \tau \exp\big(L_{\mathbf{G}_\star}L_x \tau \big)\bigg)\underbrace{c_3(N)^2 \sum\limits_{k=1}^N d^k\Gamma_{k}(\mathbf{G}_\star)^2 \abs{D}^2}_{\textnormal{Discretization bias}}\\
    & + 4\frac{(L_x \tau)^{k^\star} }{k^{\star}! } \sqrt{\frac{2 \big(x + \log(Nd^N) \big) \lambda_\infty}{n} }\sum\limits_{k=1}^N d^k\Gamma_{k}(\mathbf{G}_\star) + 4 B_{\beta,0} \sup\limits_{k=1,\dots,s} \abs{\beta_\star^{(k)}} B_{\mathbf W} \sqrt{\frac{2 \tau \lambda_\infty (x + \log s)}{n}}. 
    \end{align*}
\end{theorem}

\begin{proof}

By optimality of $\Hat{\theta}$, we have for all $\theta \in \Theta$ that 
    \begin{align*}
        \ell_n^D(\Hat{\theta}) + \textnormal{pen}(\Hat{\theta})-\ell_n^\star \leq  \ell_n^D(\theta) + \textnormal{pen}(\theta) - \ell_n^\star
    \end{align*}
    and hence, using Proposition \ref{prop:doob_meyer}, one obtains
    \begin{align*}
        \textnormal{KL}_n(\lambda_\star,\lambda^D_{\Hat{\theta}}) \leq \textnormal{KL}_n(\lambda_\star,\lambda^D_{\theta}) + \textnormal{pen}(\theta)-\textnormal{pen}(\Hat{\theta}) + \frac{1}{n} \sum\limits_{i=1}^n \int \log \frac{\lambda^{i,D}_{\Hat{\theta}}(s)}{\lambda^{i,D}_\theta(s)}dM^i(s). 
    \end{align*}
    Using Proposition \ref{prop:double_inequality}, the KL-divergence on the right hand side can be bounded by the squared log divergence, yielding
    \begin{align}\label{eqn:bias_variance_decomposition}
    \textnormal{KL}_n(\lambda_\star,\lambda^D_{\Hat{\theta}}) \leq c_2 \textnormal{D}^2_n(\lambda_\star,\lambda^D_\theta) +  \textnormal{pen}(\theta)-\textnormal{pen}(\Hat{\theta}) + \frac{1}{n} \sum\limits_{i=1}^n \int \log \frac{\lambda^{i,D}_{\Hat{\theta}}(s)}{\lambda^{i,D}_\theta(s)}dM^i(s). 
    \end{align}
    The ``bias term'' $c_2 \textnormal{D}^2_n(\lambda_\star,\lambda^D_\theta)$ can be bounded thanks to Proposition~\ref{prop:bias_majoration}. We shall now derive a bound for the term
\begin{equation*}
     \textnormal{pen}(\theta)-\textnormal{pen}(\Hat{\theta}) + \frac{1}{n} \sum\limits_{i=1}^n \int \log \frac{\lambda^{i,D}_{\Hat{\theta}}(s)}{\lambda^{i,D}_\theta(s)}dM^i(s). 
\end{equation*}
Since Equation~\eqref{eqn:bias_variance_decomposition} holds true for all $\theta \in \Theta$, we can set $\theta = (\alpha_{\star,N},\beta_\star)$. We now study the term 
\begin{align}\label{eqn:empirical_process}
    \frac{1}{n} \sum\limits_{i=1}^n \int_0^\tau \log \frac{\lambda^{i,D}_{\hat{\theta}}(s)}{\lambda^{i,D}_{\theta^\star_N}(s)}dM^i(s) &= \frac{1}{n} \sum\limits_{i=1}^n \int_0^\tau \big(\hat \alpha^\top \mathbf{S}_N(x^{i,D}_{[0,s]}) + \hat \beta^\top \mathbf{W}^i - \alpha_{\star,N}^\top \mathbf{S}(x^i_{[0,t]})- \beta_\star^\top \mathbf{W}^i  \big)dM^i(s)\\&
    =( \hat \alpha - \alpha_{\star,N})^\top\frac{1}{n} \sum\limits_{i=1}^n \int_0^\tau  \mathbf{S}_N(x^{i,D}_{[0,s]}) dM^i(s) + (\hat \beta - \beta_\star)^\top 
    \frac{1}{n} \sum\limits_{i=1}^n \mathbf{W}^i M^i(t).
\end{align}
which appears in Proposition~\ref{prop:doob_meyer} when considering signature based embeddings. We make a repeated use of the following lemmas in our derivation of a bound for this term.

\begin{lemma}\label{lemma:bound_signaturecoef}
    Let $\mathbf{S}_{[k],j}(x^{i,D}_{[0,t]}) \in \mathbb{R}$ be the signature coefficient associated to the $j$-th word of the $k$-th signature layer of a time series $x^{i,D}_{[0,t]}$ evaluated at time $t \leq \tau$. Then we have 
    \begin{align*}
    \norm{\mathbf{S}_{[k],j}(x^{i, D}_{[0,\cdot]})}_{\infty,[0,t]} = \max\limits_{s \in [0,t]}  | \mathbf{S}_{[k],j}(x^{i,D}_{[0,s]}) | \leq \frac{(L_x t)^k }{k! } \leq \frac{(L_x \tau)^k }{k! }.
    \end{align*}
\end{lemma}
\begin{proof}
    For all $s \in [0,t]$, we have 
    \begin{align*}
         | \mathbf{S}_{[k],j}(x^{i,D}_{[0,s]}) | \leq \norm{\mathbf{S}_{[k]}(x^{i,D}_{[0,s]})} \leq \frac{\norm{x^{i,D}_{[0,t]}}_{\textnormal{1-var}}^k}{k !} 
    \end{align*}
    where $\mathbf{S}_{[k]}(x^{i,D}_{[0,s]})$ refers to the full signature layer of depth $k$, and the last inequality can be found in \citet{fermanian2021embedding}. Using Assumption \ref{assumption:continuous_path}, we have  
    \begin{align*}
         \frac{\norm{x^{i,D}_{[0,t]}}_{\textnormal{1-var}}^k}{k !} \leq   \frac{\norm{x^i_{[0,t]}}_{\textnormal{1-var}}^k}{k !} \leq \frac{(L_x t)^k}{k !} \leq \frac{(L_x \tau)^k}{k !}. 
    \end{align*}
\end{proof}

The following deviation inequality is a direct consequence from the one in~\citet{van1995exponential} and derives from inequalities for general martingales that can be found in~\citet{shorack2009empirical} for instance.

\begin{lemma}[Deviation inequality for a martingale]\label{lemma:deviation_martingale}
Let $\Upsilon$ be a locally square integrable martingale .Then, for any $x > 0$ and $t \geq 0$, the following holds true for 
\begin{equation}\label{eqn:martingale_deviation}
\mathbb P \Big( \big| \Upsilon(t) \big| \geq \sqrt{2 v(t) x} + \frac{B(t) x}{3}, \langle \Upsilon(t) \rangle  \leq v(t), \sup_{s \in [0,t]} |\Delta \Upsilon(s)| \leq B(t)  \Big) \leq 2 e^{-x},
\end{equation}
where $\langle \Upsilon(t) \rangle$ is the predictable variation of $\Upsilon$ and $\Delta \Upsilon(t)$ its jump at time $t$.
\end{lemma}



Going back to Equation~\eqref{eqn:empirical_process} and decomposing on the signature layers, we can write
\begin{align*}
  \Big| ( \hat \alpha - \alpha_{\star,N})^\top\frac{1}{n} \sum\limits_{i=1}^n \int_0^\tau  \mathbf{S}_N(x^{i,D}_{[0,s]}) dM^i(s)\Big| &= \Big|\sum_{k=1}^N ( \hat \alpha_{[k]} - \alpha_{\star,[k]})^\top\frac{1}{n} \sum\limits_{i=1}^n \int_0^\tau  \mathbf{S}_{[k]}(x^{i,D}_{[0,s]}) dM^i(s)\Big| \\&\leq \sum_{k=1}^N \big\| \hat \alpha_{[k]} - \alpha_{\star,[k]} \big\|_{1} \sup_{ 1 \leq j \leq d^k}\Big| \frac1n\sum\limits_{i=1}^n \int_0^\tau  \mathbf{S}_{[k],j}(x^{i,D}_{[0,s]}) dM^i(s)\Big|.
\end{align*}
Because the martingales $M^i$ are independent, and the signature coefficients bounded, the term 
\begin{equation*}
    \chi_n(\tau) = \frac1n\sum\limits_{i=1}^n \int_0^\tau  \mathbf{S}_{[k],j}(x^{i,D}_{[0,s]}) dM^i(s)
\end{equation*} 
is itself a martingale. Moreover, since each $M^i$ comes from a counting process via a Doob Meier decomposition, its jumps are bounded by $1$ and, at a given time, there is (almost surely) at most one $M^i$ that jumps. As a consequence, we get the following bound  with jumps bounded by
\begin{equation*}
  \sup_{t \in [0,\tau]} \big| \Delta \chi_n(t) \big|  \leq \frac1n \sup_{i = 1 , \ldots, n} \big\|\mathbf{S}_{[k],j}(x^{i, D}_{[0,\cdot]})\big\|_{\infty,[0,\tau]} \leq \frac{(L_x \tau)^k }{nk! },
\end{equation*}
and quadratic variation given at time $t \in [0, \tau]$ by
\begin{align}\label{eqn:question_sur_deviation}
   \langle \chi_n(t) \rangle &=  \Big\langle \frac1n\sum\limits_{i=1}^n \int_0^t  \mathbf{S}_{[k],j}(x^{i,D}_{[0,s]}) dM^i(s) \Big\rangle = \frac{1}{n^2} \sum_{ i = 1}^n \int_0^t \mathbf{S}_{[k],j}(x^{i,D}_{[0,s]})^2 \lambda^i_\star(s) Y^i(s) ds \nonumber\\& \leq \nonumber \sup_{i = 1 , \ldots, n} \|\mathbf{S}_{[k],j}(x^{i, D}_{[0,\cdot]})\|^2_{\infty,[0,\tau]} \frac1{n^2} \sum_{i=1}^n \Lambda^i_\star(t) \\
   &\leq \nonumber \frac1{n}  \sup_{i = 1 , \ldots, n} \|\mathbf{S}_{[k],j}(x^{i, D}_{[0,\cdot]})\|^2_{\infty,[0,\tau]} \sup_{i = 1 , \ldots, n}  \Lambda^i_\star(t) \\&\leq \frac{L_x^{2k} \tau^{2k+1} \lambda_\infty}{n(k!)^2} ,
\end{align} 
where we have used Lemma~\ref{lemma:bound_signaturecoef} and the fact that
\begin{equation*}
    \Lambda^i_\star(t) = \int_0^t  \lambda^i_\star(s)Y^{i}(s)ds \leq t \sup_{s \in [0,t]}  \lambda^i_\star(s) \leq \tau \sup_{s \in [0,\tau]}  \lambda^i_\star(s) \leq \tau  \lambda_\infty,
\end{equation*}
with
\begin{equation*}
    \lambda_\infty = \exp\big(B_{\beta,2} B_\mathbf{W}) \exp\big(\norm{\mathbf{G}_\star(0)}_{\textnormal{op}} L_x \tau \exp\big(L_{\mathbf{G}_\star}L_x \tau \big) \big)
\end{equation*} according to Lemma~\ref{lemma:intensity_bounded}.
Lemma~\ref{lemma:deviation_martingale} now warrants that for any $\varepsilon > 0$ with a probability greater than $1 - 2 e^{-\varepsilon}$
\begin{align*}
  \Big| \frac1n\sum\limits_{i=1}^n \int_0^\tau  \mathbf{S}_{[k],j}(x^{i,D}_{[0,s]}) dM^i(s)\Big| \leq \sqrt{\frac{2 \varepsilon L_x^{2k} \tau^{2k+2} \lambda_\infty}{n(k!)^2 } } + \frac{\varepsilon(L_x \tau)^k }{3nk! } \leq \frac{(L_x \tau)^{k^\star} }{k^{\star}! } \Big(\sqrt{\frac{2 \varepsilon \tau^2\lambda_\infty}{n} }+  \frac{\varepsilon }{3n }\Big)
\end{align*}
where 
\begin{equation*}
    k^\star = \text{argmax}_{k \geq 1 }\frac{(L_x \tau)^{k} }{k! }.
\end{equation*}
 A double union bound on the signature layers and the signature coefficients within each layer ensures that  for any $\varepsilon > 0$ with a probability greater than $1 - 2 e^{-\varepsilon}$
\begin{align*}
    \sup_{1 \leq k \leq N} \sup_{1 \leq j \leq d^k} \Big| \frac1n\sum\limits_{i=1}^n \int_0^\tau  \mathbf{S}_{[k],j}(x^{i,D}_{[0,s]}) dM^i(s)\Big| \leq \frac{(L_x \tau)^{k^\star} }{k^{\star}! } \Big(\sqrt{\frac{2 \big(\varepsilon + \log(Nd^N) \big) \tau^2\lambda_\infty}{n} }+  \frac{\varepsilon + \log(Nd^N) }{3n }\Big).
\end{align*}
As a consequence
\begin{equation*}
    \Big| (\hat \alpha - \alpha_{\star,N})^\top\frac{1}{n} \sum\limits_{i=1}^n \int_0^\tau  \mathbf{S}_N(x^{i,D}_{[0,s]}) dM^i(s) \Big| \leq \| \hat \alpha - \alpha_{\star,N} \|_1 \frac{(L_x \tau)^{k^\star} }{k^{\star}! } \Big(\sqrt{\frac{2 \big(\varepsilon + \log(Nd^N) \big)\tau^2 \lambda_\infty}{n} }+  \frac{\varepsilon + \log(Nd^N)}{3n }\Big)
\end{equation*}
 for any $\varepsilon > 0$ with a probability greater than $1 - 2 e^{-\varepsilon}$.  

We apply the same line of reasoning to
\begin{equation*}
   \Big| (\hat \beta - \beta_\star)^\top 
    \frac{1}{n} \sum\limits_{i=1}^n \mathbf{W}^i  M^i(t) \Big| \leq \|\hat \beta - \beta_\star\|_1 \sup_{1 \leq m \leq s} \Big| \frac{1}{n}\sum\limits_{i=1}^n W^i_mM^i(t)\Big|.
\end{equation*}
For each $m$, the term $\sum\limits_{i=1}^n W^i_mM^i(\cdot)$ is a martingale with predictable variation less than
\begin{equation*}
    \sum\limits_{i=1}^n (W^i_m)^2 \Lambda^i_\star(t) \leq B^2_{\mathbf W} \tau \lambda_\infty.
\end{equation*}
Its jumps are bounded by $B_{\mathbf W}$, and therefore Lemma~\ref{lemma:deviation_martingale} applies. Via an union bound, we deduce that for any $\varepsilon>0$ and with a probability greater that $1 - 2 e^{-\varepsilon}$
\begin{equation*}
  \Big| (\hat \beta - \beta_\star)^\top 
    \frac{1}{n} \sum\limits_{i=1}^n \mathbf{W}^i  M^i(t) \Big| \leq \|\hat \beta - \beta_\star\|_1 
    \sqrt{\frac{2 B^2_{\mathbf W} \tau \lambda_\infty (\varepsilon + \log s)}{n}} + \frac{B_{\mathbf W} (\varepsilon + \log s)}{3n}.
\end{equation*}

Now defining the penalty
\begin{align*}
    \textnormal{pen}(\theta) &= \|\alpha \|_1 \frac{(L_x \tau)^{k^\star} }{k^{\star}! } \Big(\sqrt{\frac{2 \big(\varepsilon + \log(Nd^N) \big) \tau^2\lambda_\infty}{n} }+  \frac{\varepsilon + \log(Nd^N)}{3n }\Big) \\
    & \quad + \| \beta\|_1 \Big( \sqrt{\frac{2 B^2_{\mathbf W} \tau \lambda_\infty (\varepsilon + \log s)}{n}} + \frac{B_{\mathbf W} (\varepsilon + \log s)}{3n}\Big),
\end{align*}
we obtain
\begin{align*}
     &\textnormal{pen}(\theta_{\star,N})-\textnormal{pen}(\Hat{\theta}) + \frac{1}{n} \sum\limits_{i=1}^n \int \log \frac{\lambda^{i,D}_{\Hat{\theta}}(s)}{\lambda^{i,D}_\theta(s)}dM^i(s) \\
     & \quad \leq
     2\|\alpha_{\star,N} \|_1 \frac{(L_x \tau)^{k^\star} }{k^{\star}! } \Big(\sqrt{\frac{2 \big(x + \log(Nd^N) \big)\tau^2 \lambda_\infty}{n} }+  \frac{x + \log(Nd^N)}{3n }\Big) \\
     & \qquad +2 \| \beta_\star\|_1 \Big( \sqrt{\frac{2 B^2_{\mathbf W} \tau \lambda_\infty (x + \log s)}{n}} + \frac{B_{\mathbf W} (x + \log s)}{3n}\Big)
\end{align*}
with a probability greater that $1 - 4 e^{-\varepsilon}$ for any $\varepsilon > 0 $. For $n$ large enough, we can write
\begin{align*}
     &2\|\alpha_{\star,N} \|_1 \frac{(L_x \tau)^{k^\star} }{k^{\star}! } \Big(\sqrt{\frac{2 \big(\varepsilon + \log(Nd^N) \big) \tau^2\lambda_\infty}{n} }+  \frac{\varepsilon + \log(Nd^N)}{3n }\Big) \\&+2 \| \beta_\star\|_1 \Big( B_{\mathbf W}  \sqrt{\frac{2 \tau \lambda_\infty (\varepsilon + \log s)}{n}} + \frac{B_{\mathbf W} (\varepsilon + \log s)}{3n} \Big)\\ & \quad \leq 4 \|\alpha_{\star,N} \|_1 \frac{(L_x \tau)^{k^\star} }{k^{\star}! } \sqrt{\frac{2 \big(\varepsilon + \log(Nd^N) \big) \tau^2\lambda_\infty}{n} } + 4 \| \beta_\star\|_1  B_{\mathbf W}\sqrt{\frac{2  \tau \lambda_\infty (\varepsilon + \log s)}{n}}. 
\end{align*}
Finally, using Lemma \ref{lemma:alpha_bound} and Assumption \ref{assumption:true_vf_bounded}, we can bound the $\ell_1$ norms of $\alpha_{\star,N}$ and $\beta_{\star,N}$ and obtain that for large $n$, for any $\varepsilon>0$ we have with probability greater than $1-4e^{-\varepsilon}$ that 
\begin{align*}
    & 4 \|\alpha_{\star,N} \|_1 \frac{(L_x \tau)^{k^\star} }{k^{\star}! } \sqrt{\frac{2 \big(\varepsilon + \log(Nd^N) \big) \tau^2\lambda_\infty}{n} } + 4 \| \beta_\star\|_1 B_{\mathbf W} \sqrt{\frac{2  \tau \lambda_\infty (\varepsilon + \log s)}{n}} \\
    & \quad \leq 4\frac{(L_x \tau)^{k^\star} }{k^{\star}! } \sqrt{\frac{2 \big(\varepsilon + \log(Nd^N) \big) \tau^2 \lambda_\infty}{n} }\sum\limits_{k=1}^N d^k\Gamma_{k}(\mathbf{G}_\star) + 4 B_{\beta,0} \sup\limits_{k=1,\dots,s} \abs{\beta_\star^{(k)}} B_{\mathbf W} \sqrt{\frac{2 \tau \lambda_\infty (\varepsilon + \log s)}{n}}.
\end{align*}

\end{proof}

\newpage

\section{Algorithmic and Implementation Details}
In this Section, we provide extra information about learning algorithms described in the main paper and their hyperparameters optimization by gridsearch.

\subsection{Description of Competing Methods}
\label{appendix:baselines}

\subsubsection{CoxSig and CoxSig+}

\paragraph{Implementation.} We use \texttt{iisignature} \citep{reizenstein2018iisignature} to compute signatures. Alternatives for computing signatures include the \texttt{signatory} library \citep{kidger2020signatory}.

\paragraph{Training.} We minimize the penalized negative log-likelihood (defined in \ref{eq:minimization_problem} in the main paper) using a vanilla proximal point algorithm \citep{boyd2004convex}.

\paragraph{Hyperparameters.}
The initial learning rate of the proximal gradient algorithm is set to $e^{-3}$ and the learning rate for each iteration is chosen by back tracking linesearch method \citep{boyd2004convex}. The hyperparameters of penalization strength $(\eta_1,\eta_2)$ and truncation depth $N$ are chosen by 1-fold cross-validation of a mixed metric equal to the difference between the C-index and the Brier score. We select the best hyperparameters that minimize the average of this mixed metric on the validation set. We list the hyperparameters search space of this algorithm below.
\begin{itemize}
    \item $\eta_1$: \{1, $e^{-1}$, $e^{-2}$, $e^{-3}$, $e^{-4}$, $e^{-5}$\};
    \item $\eta_2$:  \{1, $e^{-1}$, $e^{-2}$, $e^{-3}$, $e^{-4}$, $e^{-5}$\};
    \item $N$: \{2, 3\}. Larger values were considered in the beginning of experiments but were removed from the cross-validation grid because they yielded bad performance and numerical instabilities.
\end{itemize}

\subsubsection{NCDE}

\paragraph{Implementation.} We implement the fill-forward discrete update of NCDEs in \texttt{Pytorch}.

\paragraph{Structure.} The neural vector field is a feed-forward network composed of two fully connected hidden layers whose hidden dimension is set to 128. We choose to represent the latent state in 4 dimensions---the number of nodes in the input layer is therefore set to 4. The dimension of the output layer is equal to the multiplication of the dimension of the hidden layer (128) and the dimension of the sample paths of a given data set. \texttt{tanh} is set to be the activation function for all the nodes in the network.

\paragraph{Training.} The model was trained for 50 epochs using the Adam optimizer \citep{kingma2014adam} with a batch size of 32 and cross-validated learning rate set to $e^{-4}$.

\subsubsection{Cox Model}

\paragraph{Implementation and Training.} We use a classical Cox model with elastic-net penalty as a baseline, which is given either the first measured value of the individual time series or the static features if they are available. The intensity of this model has then the form
\begin{equation*}
    \lambda^i_\theta(t) = \lambda_0(t)\exp(\beta^\top \mathbf{W}^i),
\end{equation*}
where $\mathbf{W}^i = \mathbf{X}^i(0)$ if no static features are available. We use the implementation provided in the Python package \texttt{scikit-survival} and called \texttt{CoxnetSurvivalAnalysis} \citep{sksurv}.

\paragraph{Hyperparameters.} The ElasticNet mixing parameter $\gamma$ is set to $0.1$. The hyperparameter of penalization strength $\eta$ is chosen by cross-validation as described above. We crossvalidate over the set \{1, $e^{-1}$, $e^{-2}$, $e^{-3}$, $e^{-4}$, $e^{-5}$\} to select the best value.

\subsubsection{Random Survival Forest}
\paragraph{Implementation.} We use the implementation of RSF \citep{ishwaran2008random} provided in the Python package \texttt{scikit-survival} \citep{sksurv}.

\paragraph{Training.} We train this model with static features $\mathbf{W}^i$ as the only input. Similarly to our implementation of the Cox model, we use the first value of the time series as static features if no other features are available.

\paragraph{Hyperparameters.} We cross-validate two hyperparameters on the following grids.
\begin{itemize}
    \item \texttt{max\_features}: \{\texttt{None}, \texttt{sqrt}\};
    \item \texttt{min\_samples\_leaf}: \{1, 5, 10\};
\end{itemize}

\subsubsection{Dynamic Deep-Hit \citep{lee2019dynamic}}

DDH is a dynamical survival analysis algorithm that frames dynamical survival analysis as a classification problem. It divides the considered time period $[0,\tau]$ into a set of contiguous time intervals. The network is then trained to predict a time interval of event for every subject, which is a multiclass classification task.  

\paragraph{Network Architecture.} Being adapted to competing events, Dynamic Deep-Hit combines a shared network with a cause-specific network. The \textit{shared network} is a combination of a RNN-like network that processes the longitudinal data and an attention mechanism, which helps the network decide which part of the history of the measurements is important. The \textit{cause-specific network} is a feed-forward network taking as an input the history of embedded measurements and learning a cause-specific representation. See Figure \ref{fig:ddh} for a graphical representation of the network's structure. 
\begin{figure}[h!]
    \centering
    \includegraphics[width=0.5\textwidth]{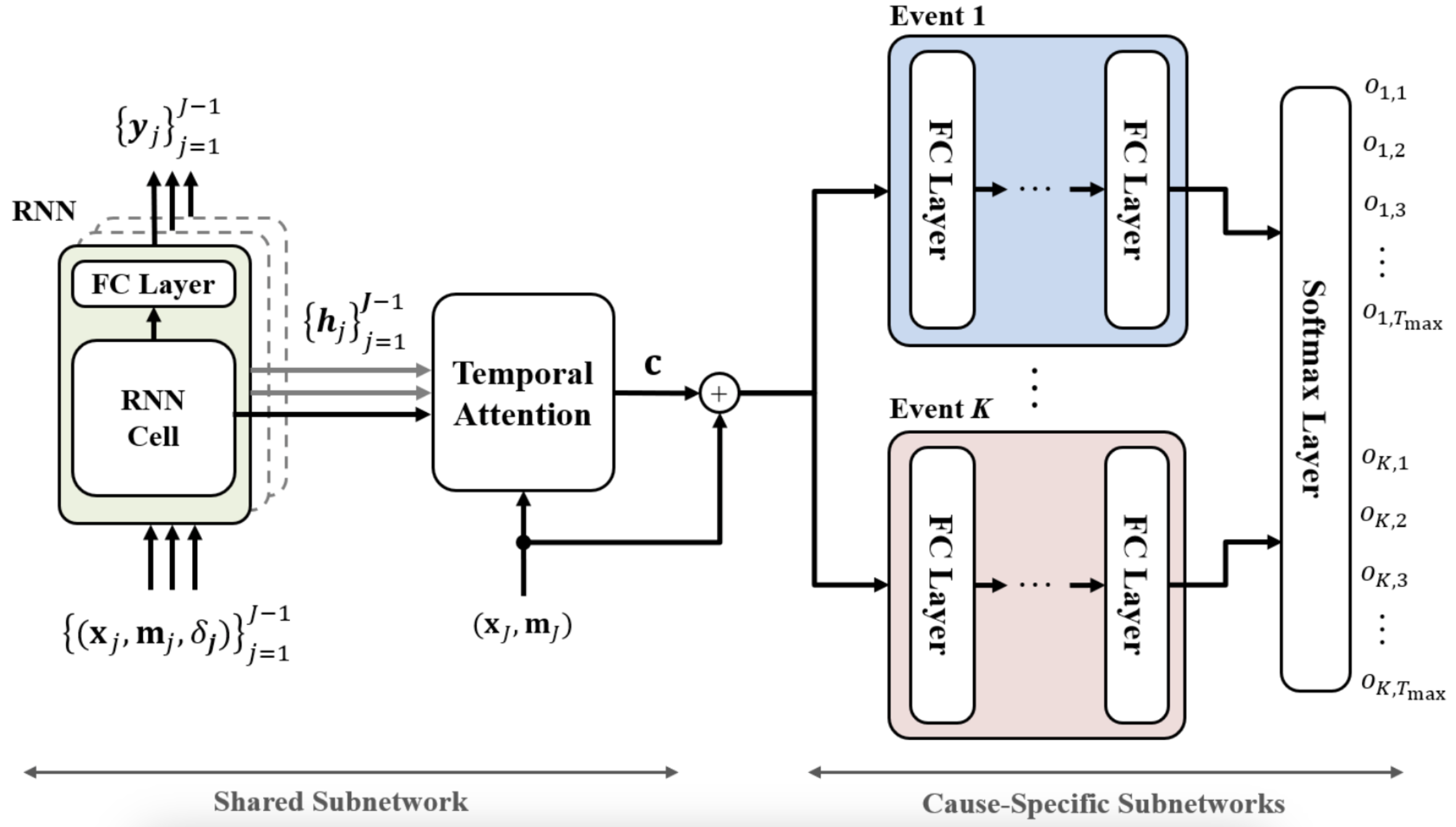}
    \caption{Network structure of Dynamic DeepHit. Figure is taken from \citet{lee2019dynamic}.}
    \label{fig:ddh}
\end{figure}

\paragraph{Loss Function.} The loss function of DDH is a sum of three loss functions 
\[
\ell_{\textnormal{Dynamic DeepHit}} = \ell_{\textnormal{log-likelihood}} + \ell_{\textnormal{ranking}} + \ell_{\textnormal{prediction}}. 
\]

The first loss maximizes the conditional likelihood of dying in the interval $[t_{k},t_{k+1}[$ given that the individual has survived up to time $t_k$. On a side note, we notice that the claim of \citet{lee2019dynamic} that this loss corresponds to ``the negative log-likelihood of the joint distribution of the first hitting
time and corresponding event considering the right-censoring" of the data is hence inexact. This might explain the results observed in Figure \ref{fig:evolution_bs}: DDH's performance, in terms of Brier score, strongly degrades as $\delta t$ increases because the model is only trained to predict one step ahead, instead of maximizing the full likelihood.  

The second loss favors correct rankings among at risk individuals: an individual experiencing an event at time $T^i$ should have a higher risk score at time $t < T^i$ than an individual $j$ for which $T^j > T^i$. 

The third loss is a prediction loss, which measures the difference between the value of the time-dependent features and a prediction of this value made by the shared network. The loss is minimized using Adam \citep{kingma2014adam}.  

\paragraph{Hyperparameters.} In our setting, we use the network in its original structure. The learning rate is set to $e^{-4}$ and the number of epochs to 300. 

\subsubsection{SurvLatent ODE \citep{moon2022survlatent}}

\paragraph{Network Architecture.}
SurvLatent ODE is a variational autoencoder architecture~\citep{kingma2013auto}. The encoder embeds the entire longitudinal features into an initial latent state, and the decoder uses this latent state to drive the latent trajectory and to estimate the distribution of event time. In this framework, the encoder is an ODE-RNN architecture~\citep{rubanova2019latent}, which handles the longitudinal features sequentially backward in time and outputs the posterior over the initial latent state. The decoder, which is adapted to competing events, consists of an ODE model and cause-specific decoder modules. The latent trajectory derived from the ODE model is shared across cause-specific decoder modules to estimate the cause-specific discrete hazard functions. See Figure \ref{fig:slode} for a graphical representation of the network's structure. 
\begin{figure}[h!]
    \centering
    \includegraphics[width=0.7\textwidth]{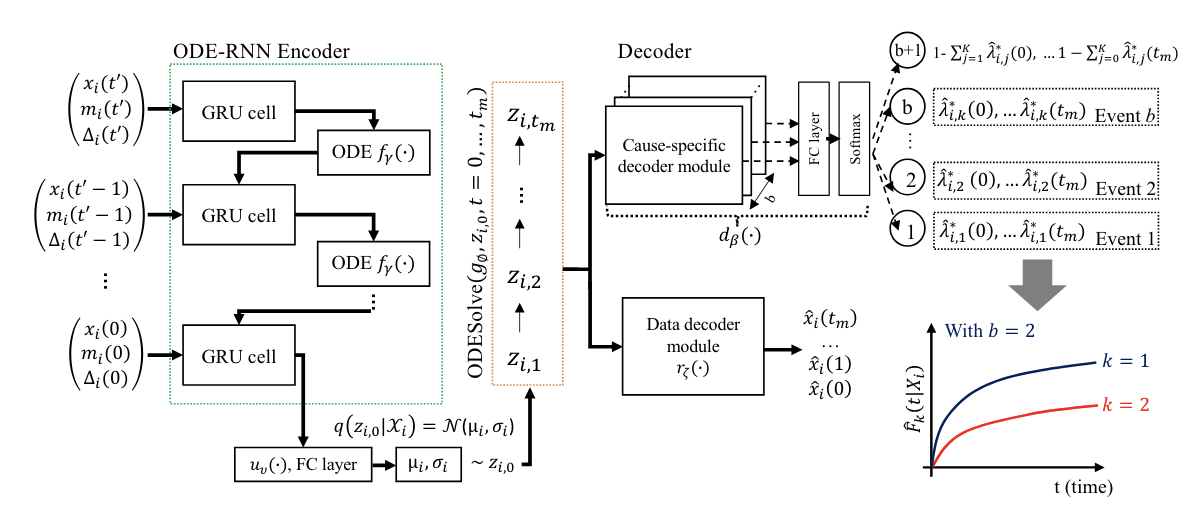}
    \caption{Network structure of SurvLatent ODE. Figure taken from \citet{moon2022survlatent}.}
    \label{fig:slode}
\end{figure}

\paragraph{Loss Function. }

The loss function is a combination of the log-likelihood and the Kullback-Leibler divergence between the approximate and the true posterior over the initial latent state. 
\paragraph{Hyperparameters.} In our setting, we use the network in its original structure. The learning rate is set to $e^{-2}$ and the number of epochs to 15, as in the original paper. The training of this framework cannot use subjects whose last longitudinal measurement time is equal to the event time, which is not the case for our proposed methods as well as other competing methods. In order to avoid this problem, we then stop observing the longitudinal measurement before the time-to-event for a period equal to 80 \% of the event time of these subjects when training the model of this framework.

\subsection{Computation of the Different Metrics}
\label{appendix:metrics}

The following lemma details the computation of the conditional survival function.

\begin{lemma}
For any $i \in \{1, \dots, n\}$,
\[
r^i_\theta(t,\delta t) = \exp \Big(-\int_t^{t + \delta t} \lambda_\theta^i(u,x^{i, D}_{[0,u \wedge t]}) du \Big),
\]
where
$r^i_\theta(t,\delta t) = \mathbb{P}\big(T^i > t + \delta t\,|\, T^i > t, x^{i,D}_{[0,t]}\big)$ is the survival function of individual $i$, as estimated by the model with parameters $\theta$, at time $t + \delta t$ for $\delta t >0$ conditional on survival up to time $t$, and on observation of the longitudinal features up to time $t$, and the notation $\lambda_\theta^i(u,x^{i, D}_{[0,u \wedge t]})$ means that the intensity at time $u$ is computed by using the longitudinal features up to time $u \wedge t = \min(u,t)$. 
\end{lemma}

\begin{proof}
    Since Bayes rule gives
\[
r^i_\theta(t,\delta t) = \mathbb{P}\Big(T^i > t + \delta t\,|\, T^i > t, x^{i,D}_{[0,t]}, \mathbf{W}^i\Big) = \frac{\mathbb{P}\Big(T^i > t + \delta t \,|\, x^{i,D}_{[0,t]}, \mathbf{W}^i\Big)}{\mathbb{P}\Big(T^i > t\,|\, x^{i,D}_{[0,t]}, \mathbf{W}^i\Big)},
\]
we can compute this score by using the fact that 
\begin{align*}
    \mathbb{P}\Big(T^i > t\,|\, x^{i,D}_{[0,t]},\mathbf{W}^i\Big) = \exp(-\Lambda_\theta^{i,D}(t)),
\end{align*}
where we recall that $\Lambda_\theta^{i,D}(t)$ is the cumulative hazard function 
\begin{align*}
    \Lambda_\theta^{i,D}(t) := \int_0^t \lambda_\theta^{i,D}(s)Y ^i(s) ds. 
\end{align*}
We refer the reader unfamiliar with survival analysis to \citet[Chapter 1, p. 6]{aalen2008survival} for a proof of this expression of the survival function. This then yields
\begin{align*}
 r^i_\theta(t,\delta t) & = \frac{\exp(-\int_0^{t + \delta t} \lambda_\theta^i(u,x^{i, D}_{[0,u \wedge t]}) du)}{\exp(-\int_0^{t} \lambda_\theta^i(u, x^{i, D}_{[0,u \wedge t]})du)}\\
& = \exp(-\int_t^{t + \delta t} \lambda_\theta^i(u,x^{i, D}_{[0,u \wedge t]}) du).
\end{align*}
\end{proof}

Beside the two metrics described in the main paper, we report our results in term of two more metrics namely the weighted Brier Score and the area under the receiver operating characteristic curve (AUC). The details of these metrics are given below.

\paragraph{Weighted Brier Score.} 

The weighted version of the Brier score, which we write $\textnormal{WBS}(t,\delta t)$, is defined as
\begin{align*}
    \sum\limits_{i=1}^n \mathds{1}_{T^i \leq t, \, \Delta^i=1} \frac{r^i_\theta(t,\delta t))^2}{\Hat{G}(T^i)} + \mathds{1}_{T^i \geq t} \frac{(1-r^i_\theta(t,\delta t))^2}{\Hat{G}(t)}, 
\end{align*}
where $\Hat{G}(\cdot)$ is the probability of censoring weight, estimated by the Kaplan-Meier estimator.

\paragraph{AUC.} We define the area under the receiver operating characteristic curve $\textnormal{AUC}(t,\delta t)$ as
\begin{align*}
\frac{\sum\limits_{i=1}^n \sum\limits_{j=1}^n\mathds{1}_{r_\theta^i(t,\delta t) > r^j_\theta(t,\delta t)} \mathds{1}_{T^i > t + \delta t,\, T^j \in [t,t + \delta t]}w_j}{(\sum\limits_{i=1}^n \mathds{1}_{T^i > t + \delta t}) (\sum\limits_{i=1}^n  \mathds{1}_{ T^i \in [t,t + \delta t]} w_i)},
\end{align*}
where $w_i$ are inverse probability of censoring weights, estimated by the Kaplan-Meier estimator.

\section{Details of Experiments and Datasets}
\label{appendix:more_results}

The main characteristics of the datasets used in the paper are summarized in Table \ref{tab:data_summary} and we provide more detailed information of these datasets in subsections below. For the experiments, each dataset is randomly divided into a training set (80$\%$) and test set (20$\%$). Hyperparameter optimization is performed as follows. We split the training set, using 4/5 for training and 1/5 for validation. We then re-fit on the whole training set with the best hyperparameters and report the results on the test set for 10 runs. Note that the performance is evaluated at numerous points $(t, \delta_t)$, where $t$ is set to the 5th, 10th, and 20th percentile of the distribution of event times.

\begin{table}[h!]
\small
    \centering
    \begin{tabular}{lcccccc}
        \toprule
         \textbf{Name} & $n$ & $d$ & \textbf{Static Features} & \textbf{Censoring} & \textbf{Avg. Observation Times} & \textbf{Source}\\
         \midrule
         Hitting time & 500 & 5 & \ding{55} & Terminal ($3.2 \%$) & 177 & Simulation \\
         Tumor Growth & 500 & 2 & \ding{55} & Terminal ($8.4 \%$) & 250 & \citet{simeoni2004predictive} \\
        Predictive Maintenance & 200 & 17  & \ding{55} & Online ($50 \%$) & 167 & \citet{saxena2008damage}\\
        Churn & 1043 & 14 & \ding{55} & Terminal $(38.4 \%)$ & 25 & Private dataset \\
        \bottomrule
    \end{tabular}
    \caption{Description of the 4 datasets we consider. The integer $d$ is the dimension of the time series including the time channel. \textit{Terminal} censoring means that the individuals are censored at the end of the overall observation period $[0,\tau]$ if they have not experienced any event. It is opposed to \textit{online} censoring that can happen at any time in $[0,\tau]$. The reported percentage indicates the censoring level i.e. the share of the population that does not experience the event.} 
    \label{tab:data_summary}
\end{table}

\subsection{Hitting Time of a partially observed SDE}

\paragraph{Time series.} The paths $x_t = (x^{(1)}_t,\dots, x^{(d-1)}_t)$ are $(d-1)$-dimensional sample paths of a fractional Brownian motion with Hurst parameter $H=0.6$, and $B^i(t)$ is a Brownian noise term. We set $d=5$. The paths are sampled at $1000$ times over the time interval $[0,10]$. All simulations are done using the \texttt{stochastic} package\footnote{Available at \href{https://github.com/crflynn/stochastic}{https://github.com/crflynn/stochastic}}. The time series $\mathbf{X}^i$ are identical, up to observation time, to the ones used for simulations.  

\textbf{Event definition} We consider the stochastic differential equation
\begin{equation*}
    \label{eq:experiment_1_cde}
    dw_t = - \omega (w_t-\mu)dt + \sum\limits_{i=1}^d dx^{(i)}_t + \sigma dB_t,
\end{equation*}
where $w_t$ is trajectory of each individual with $(\sigma, \mu, \omega) \in \mathbb{R}^3$ are fixed parameters. In our experiment, the parameters are chosen to be $\sigma=1$, $\mu=0.1$ and $\omega=0.1$.  We then define the time-of-event as the time when trajectory cross the threshold $w_\star \in \mathbb{R}$ during the observation period $[t_0 \, \, t_N]$, which is
\begin{equation*}
    T^\star = \min \{ t_0 \leq t \leq t_N  \,|\, w_t \geq w_\star\}.
\end{equation*}
In our experiments, we use the threshold value $w_\star=2.5$. The target SDE is simulated using an Euler discretization. We train on $n=500$ individuals.

\textbf{Censorship} We censor individuals whose trajectory does not cross the threshold during the observation period. This means that individuals are never censored during the observation period, but only at the end. The simulated censoring level is 3.2$\%$.

\paragraph{Supplementary Figures.} Figure \ref{fig:OU_appendix_path_hist} provides an example of the full sample path of an individual and the distribution of the event times of the whole population. We add additional results on the test set in Figures \ref{fig:c_index_OU}, \ref{fig:bs_OU}, \ref{fig:wbs_OU}, \ref{fig:auc_OU} and \ref{fig:running_times}.
\begin{figure*}
    \centering
    \includegraphics[width=0.49\textwidth]{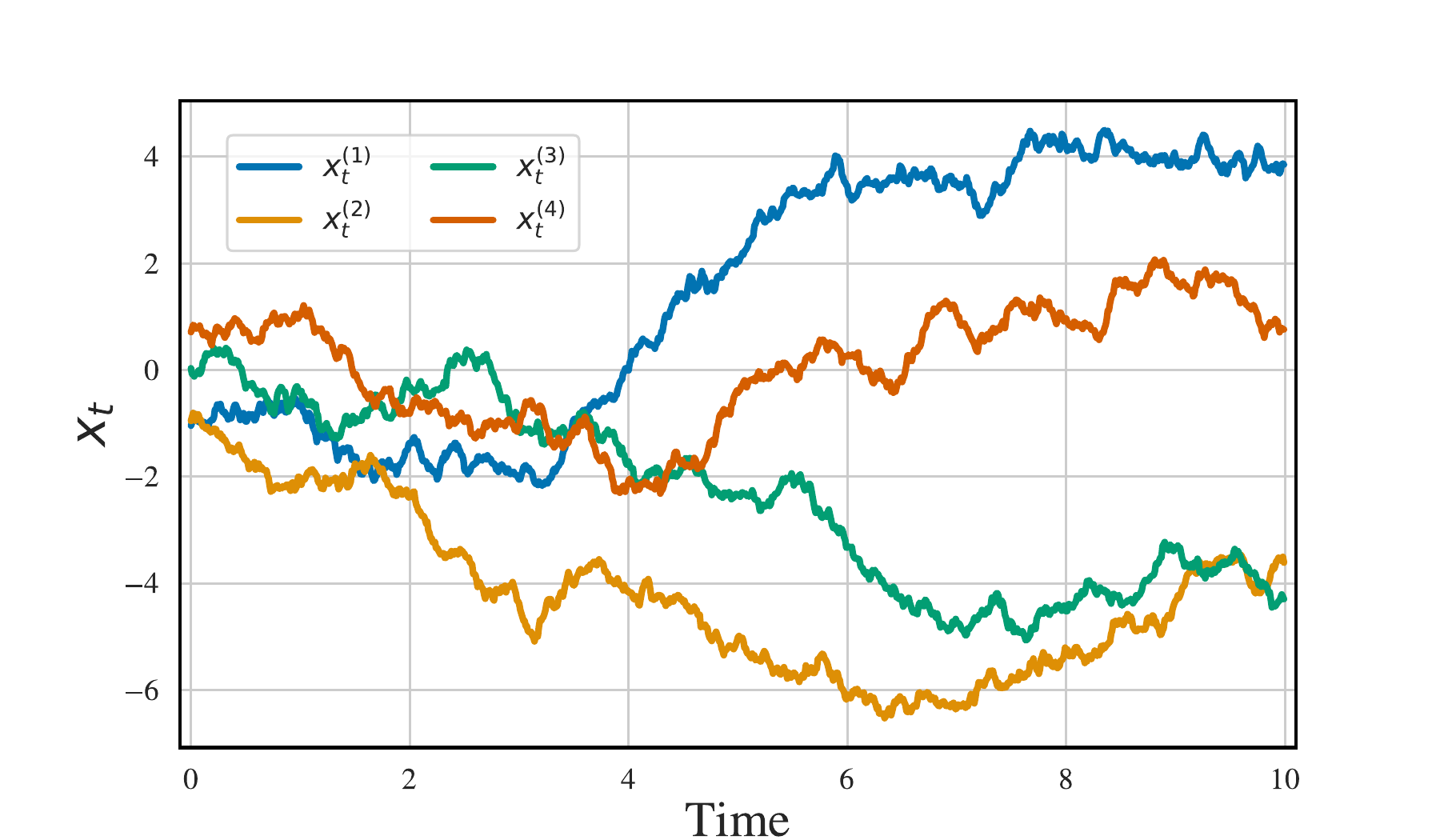}
    \includegraphics[width=0.49\textwidth]{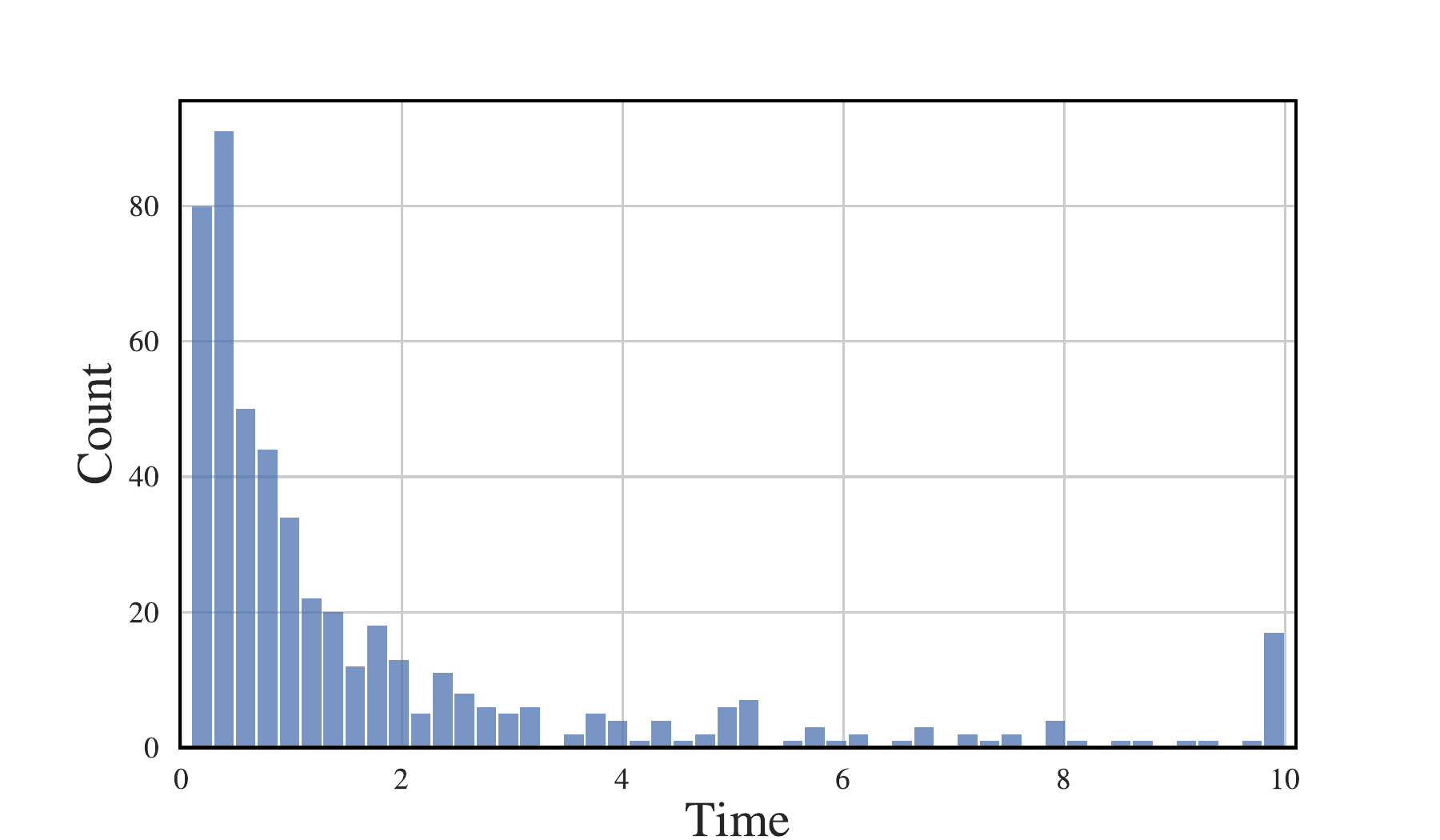}
    
    \caption{\footnotesize Full sample path of an individual (\textbf{left}) and distribution of the event times (\textbf{left}) for the partially observed SDE experiment. The surge in events at the terminal time indicates terminal censorship.}
    \label{fig:OU_appendix_path_hist}
\end{figure*}

\begin{figure*}[h!]
    \centering
    \includegraphics[width=0.33\textwidth]{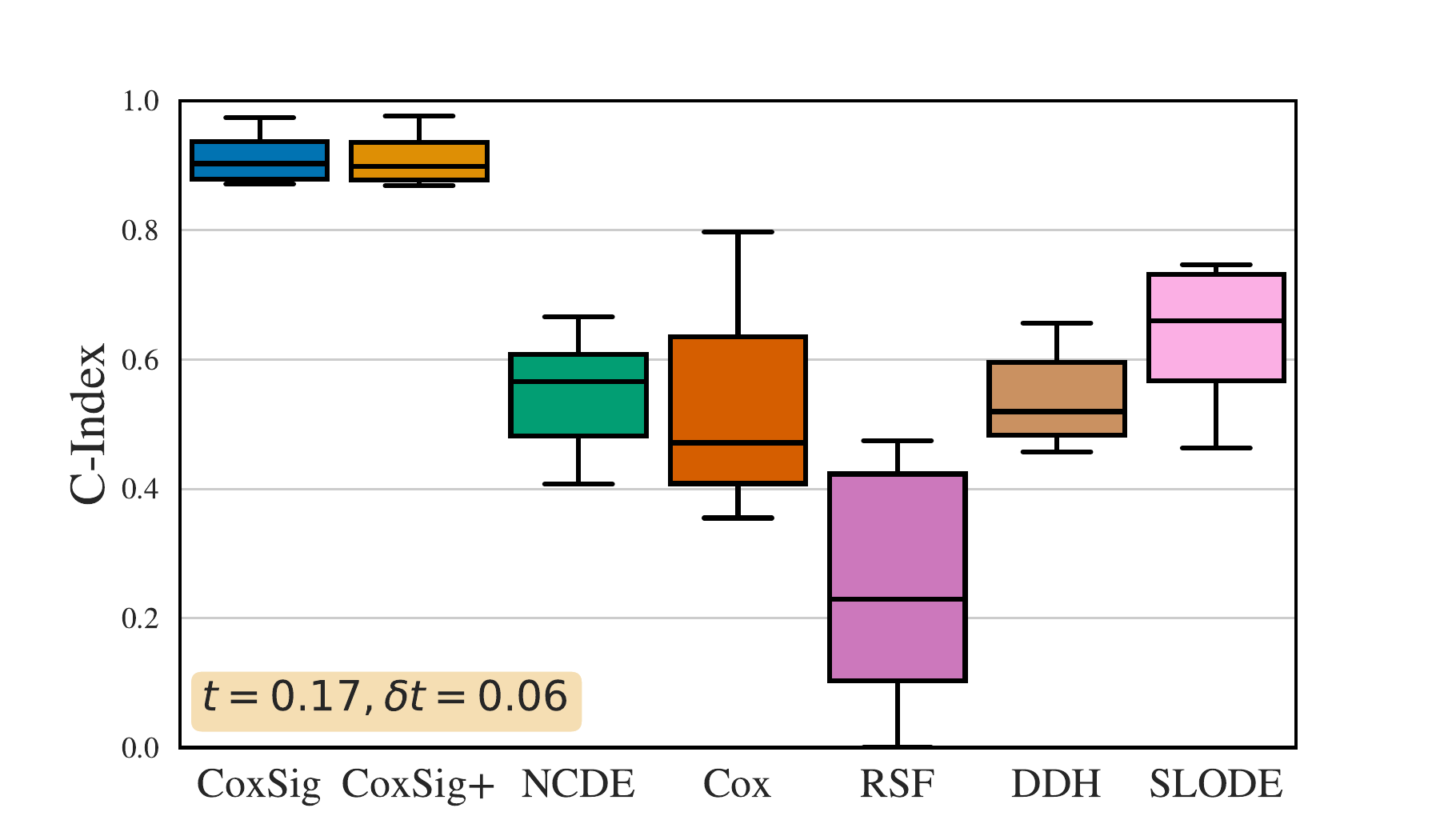}
    \includegraphics[width=0.33\textwidth]{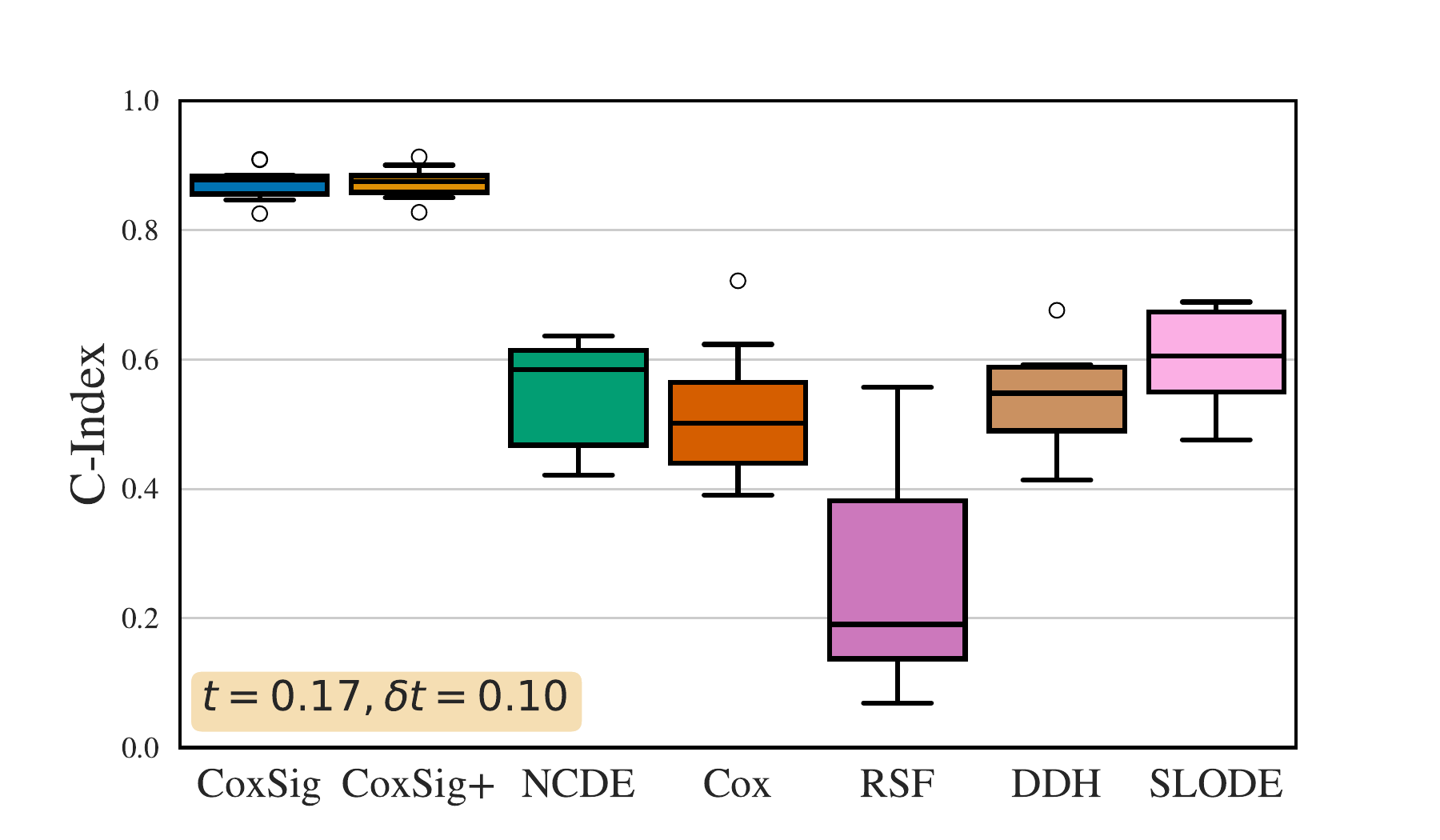}
    \includegraphics[width=0.33\textwidth]{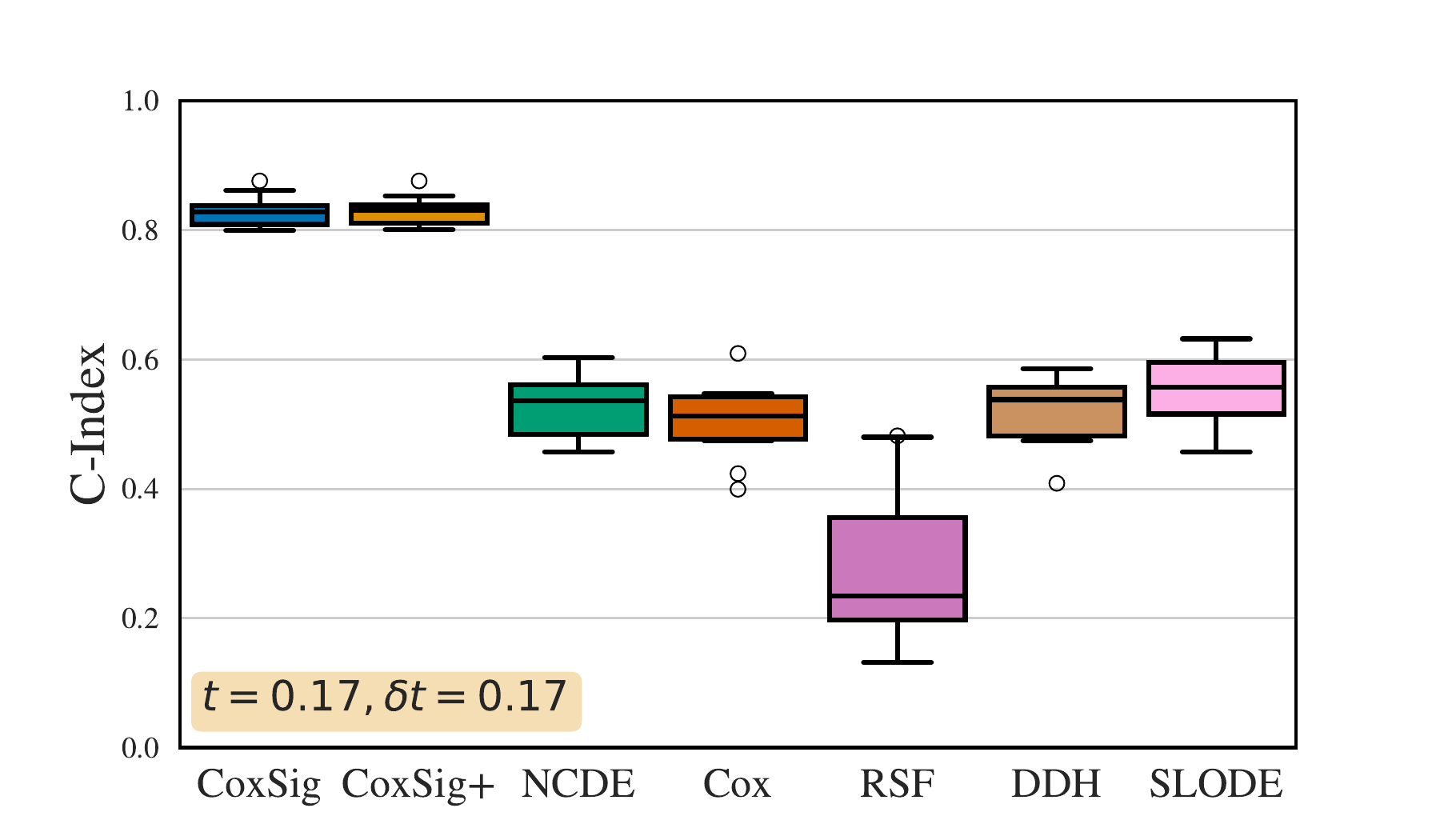}
    \includegraphics[width=0.33\textwidth]{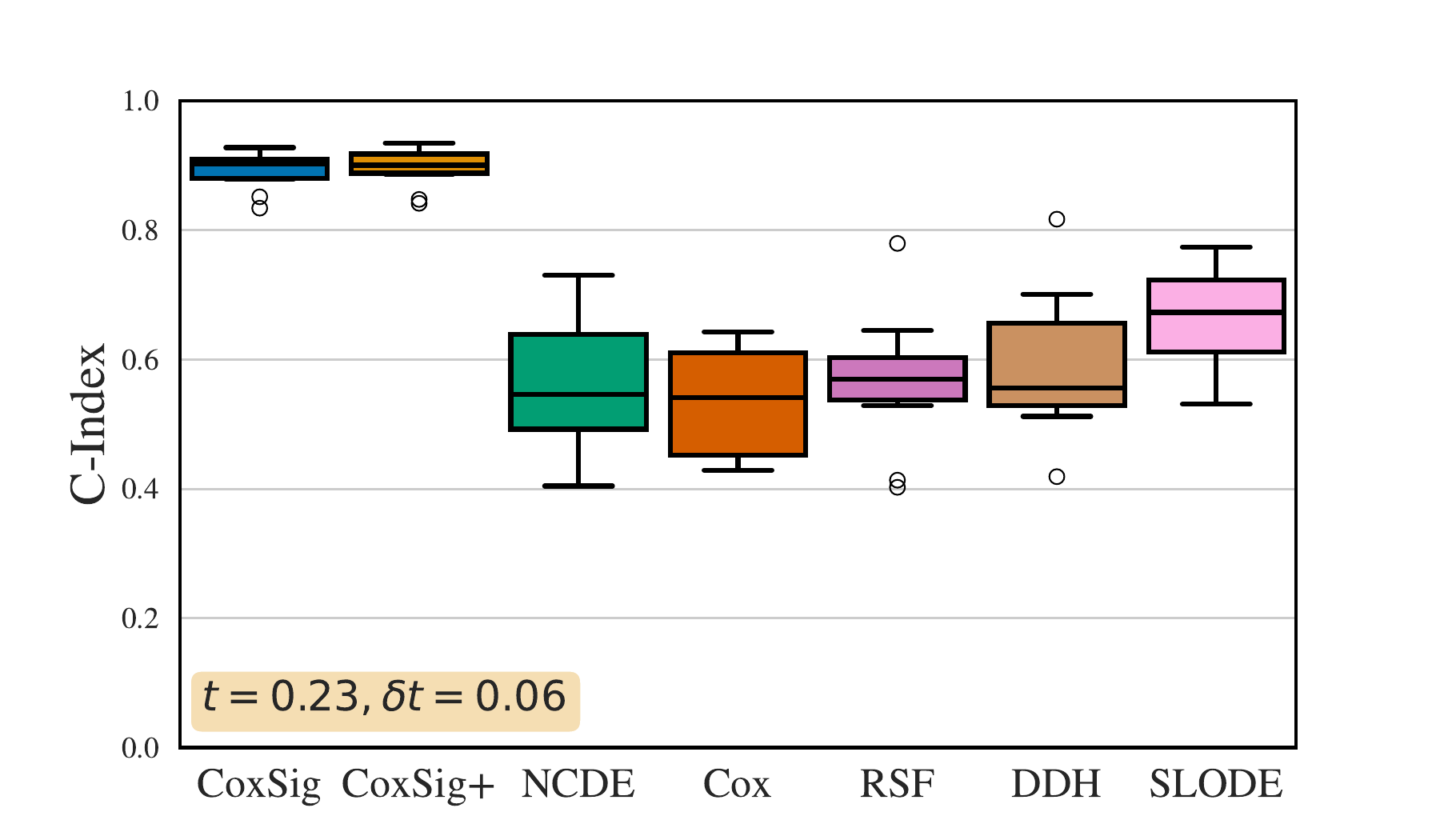}
    \includegraphics[width=0.33\textwidth]{figures/updated_figures/OU/OU_C-Index_pt=0.23,dt=0.10.pdf}
    \includegraphics[width=0.33\textwidth]{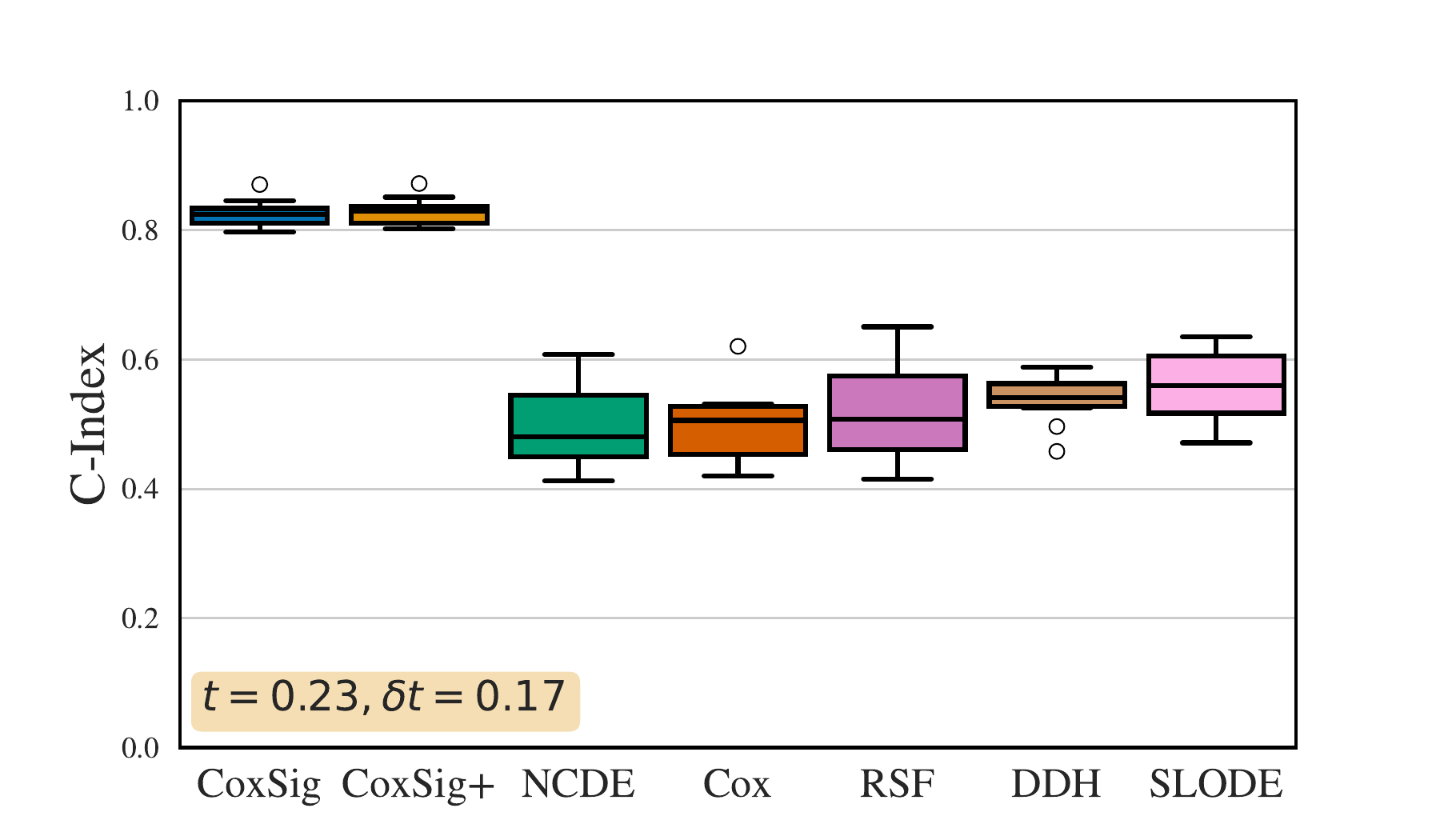}
    \includegraphics[width=0.33\textwidth]{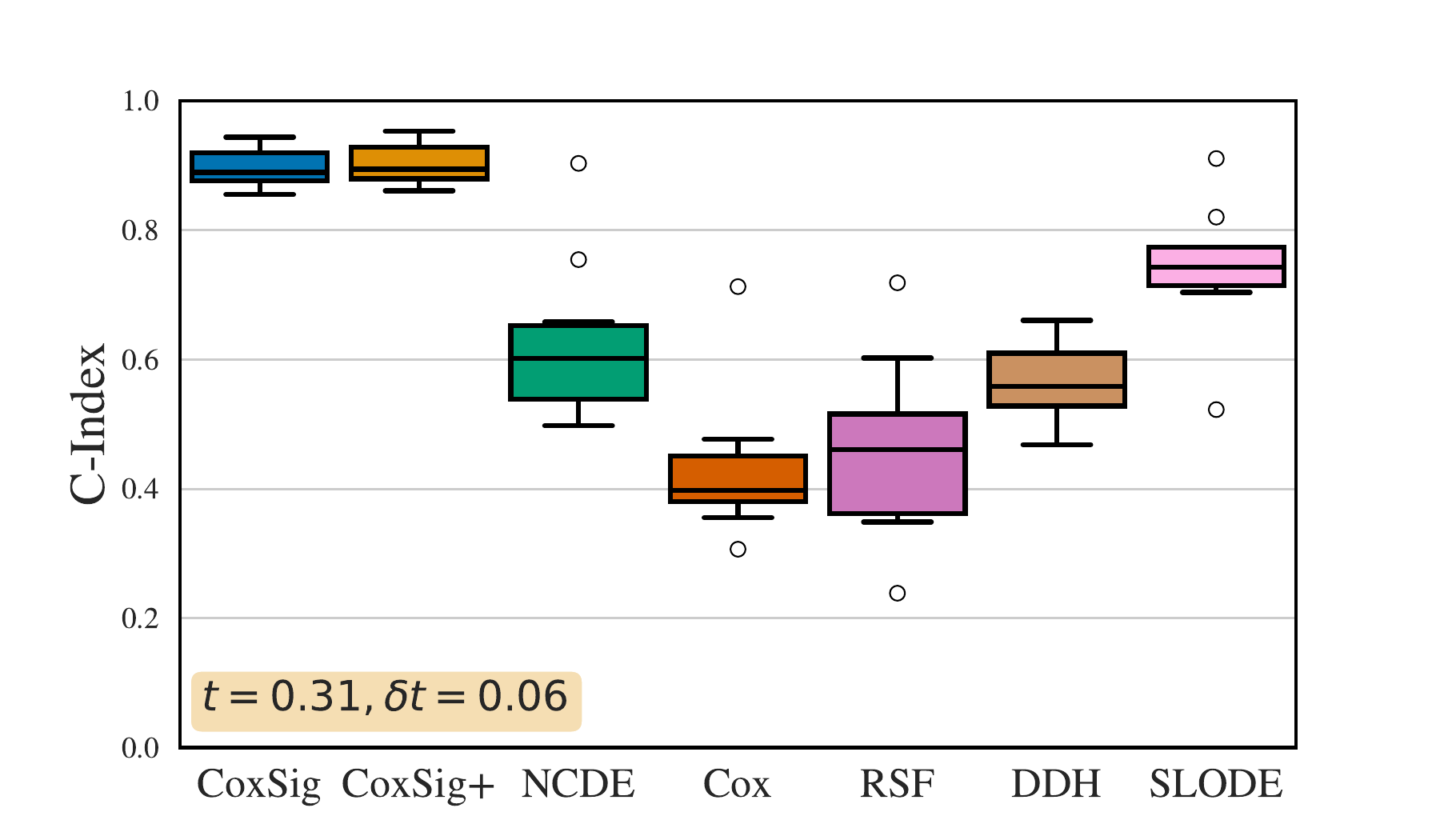}
    \includegraphics[width=0.33\textwidth]{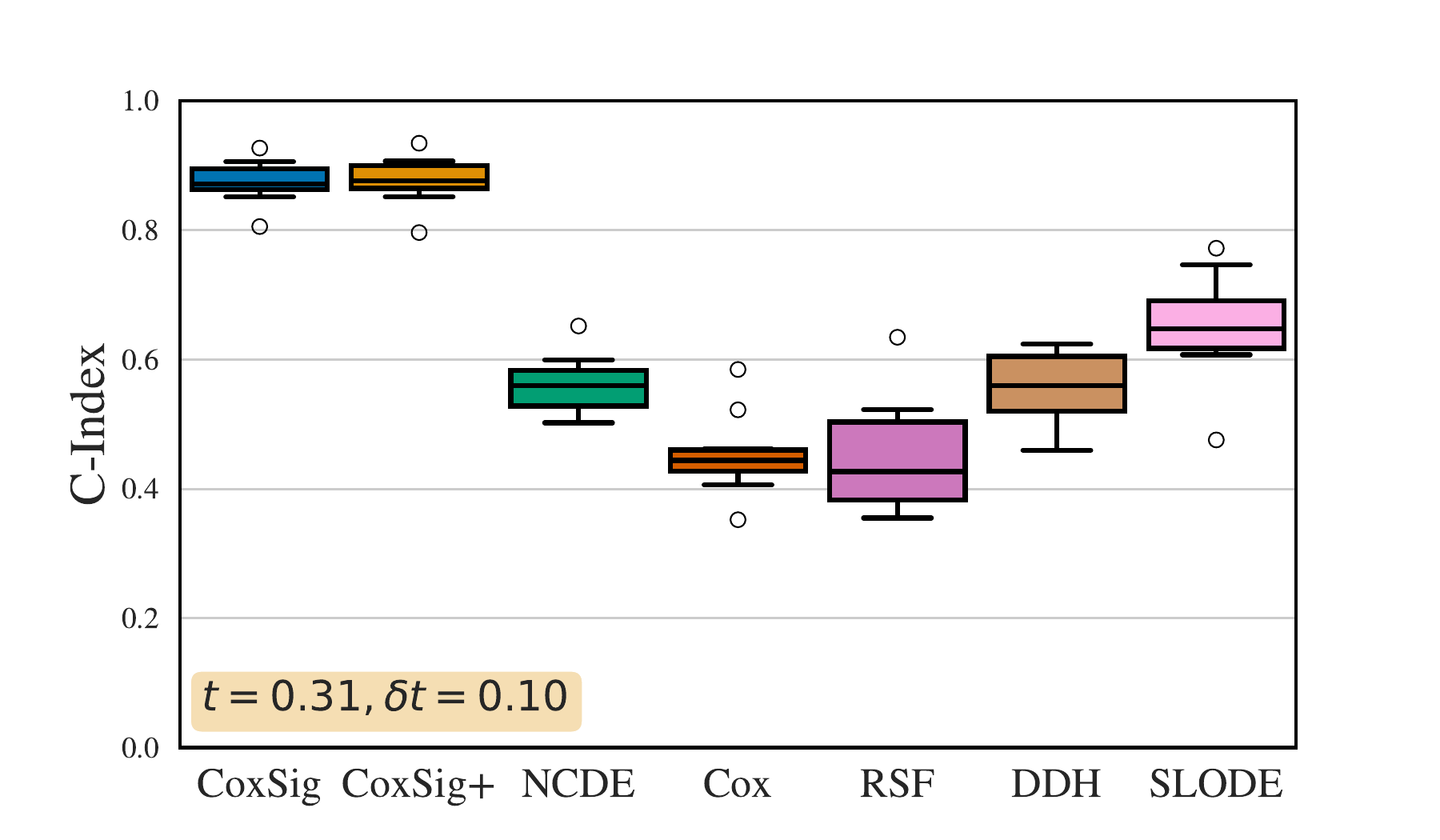}
    \includegraphics[width=0.33\textwidth]{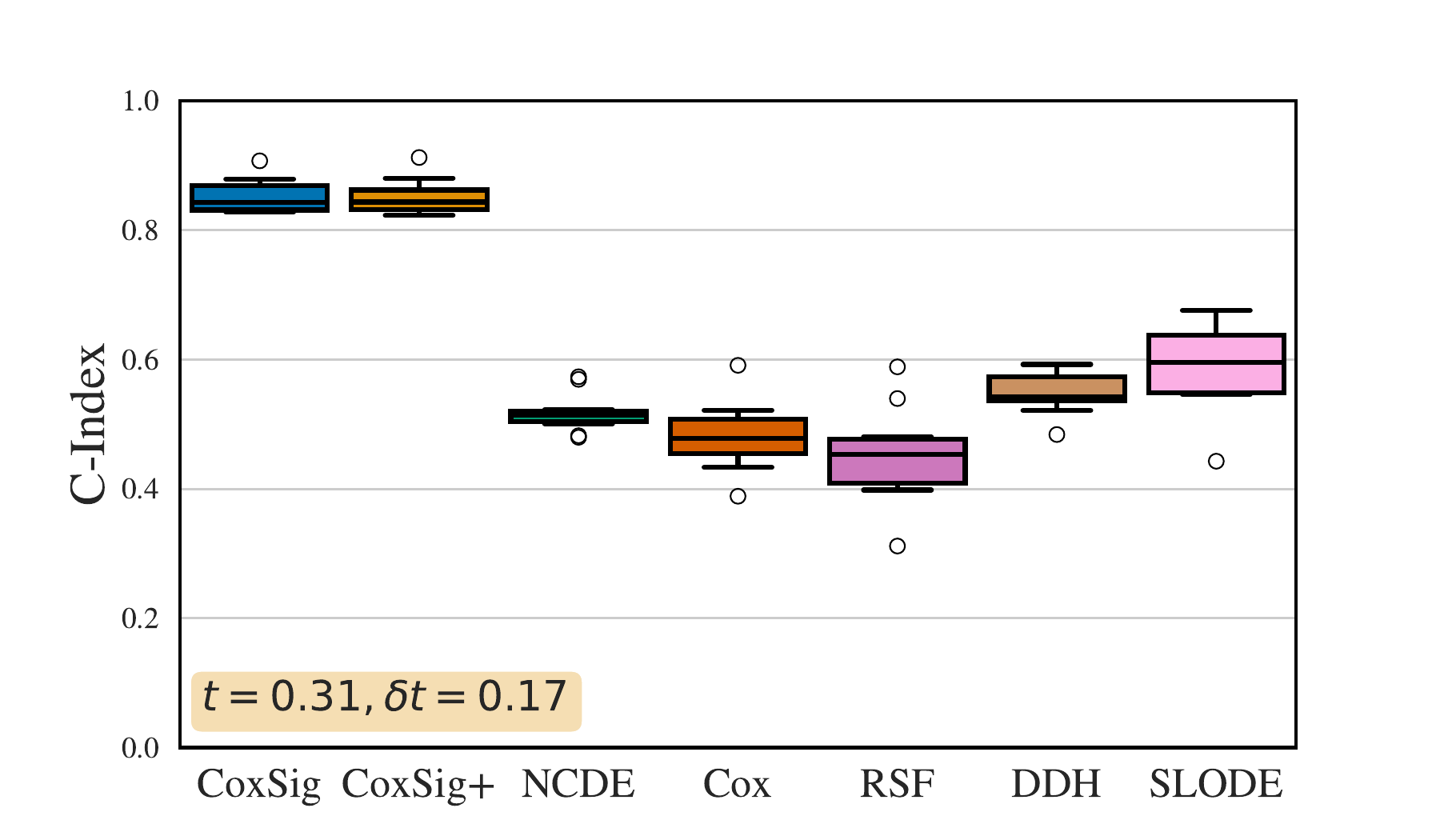}
    \caption{\footnotesize C-Index (\textit{higher} is better) for \textbf{hitting time of a partially observed SDE}  for numerous points $(t,\delta t)$.}
    \label{fig:c_index_OU}
\end{figure*}

\begin{figure*}[h!]
    \centering
    \includegraphics[width=0.33\textwidth]{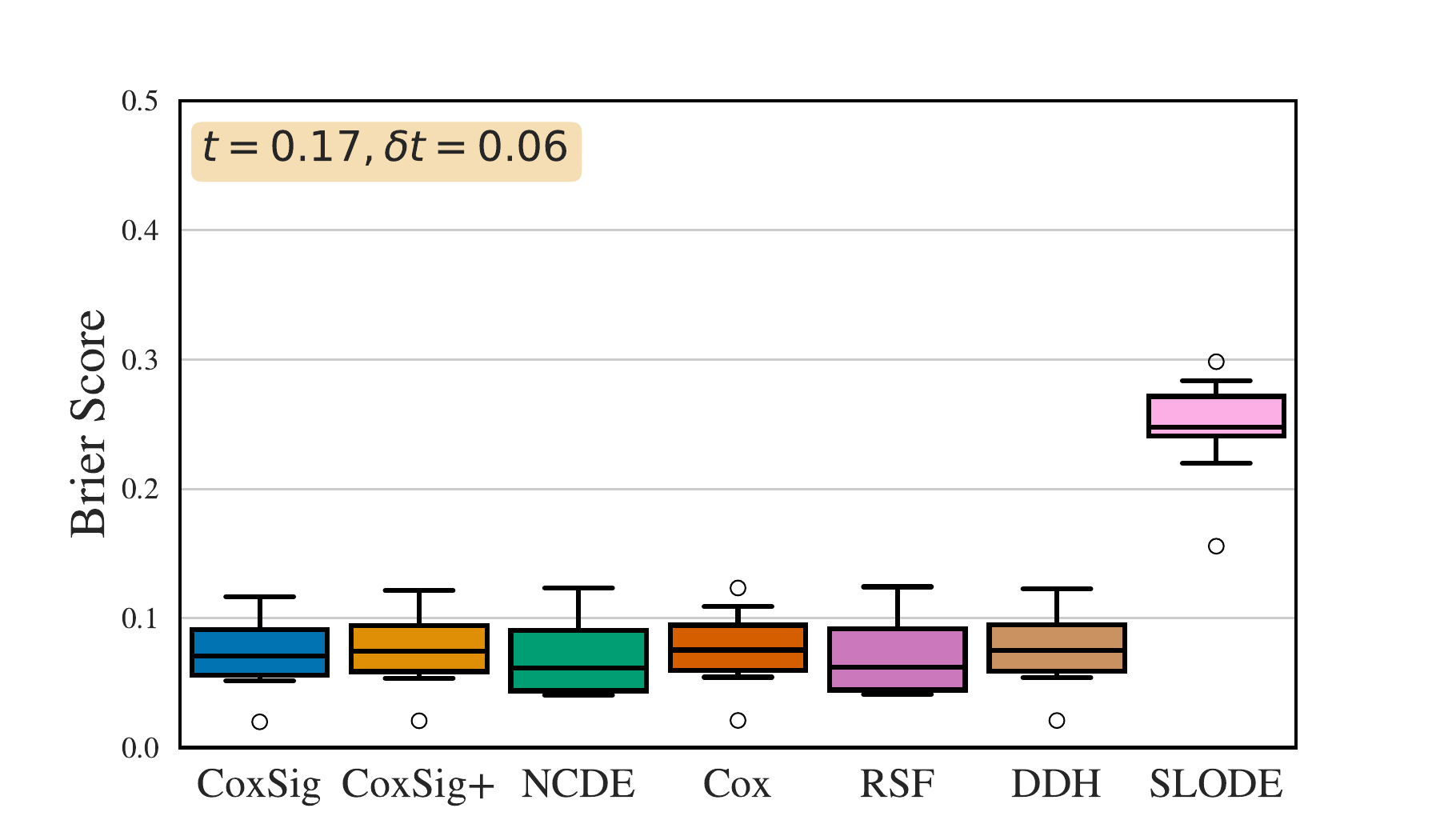}
    \includegraphics[width=0.33\textwidth]{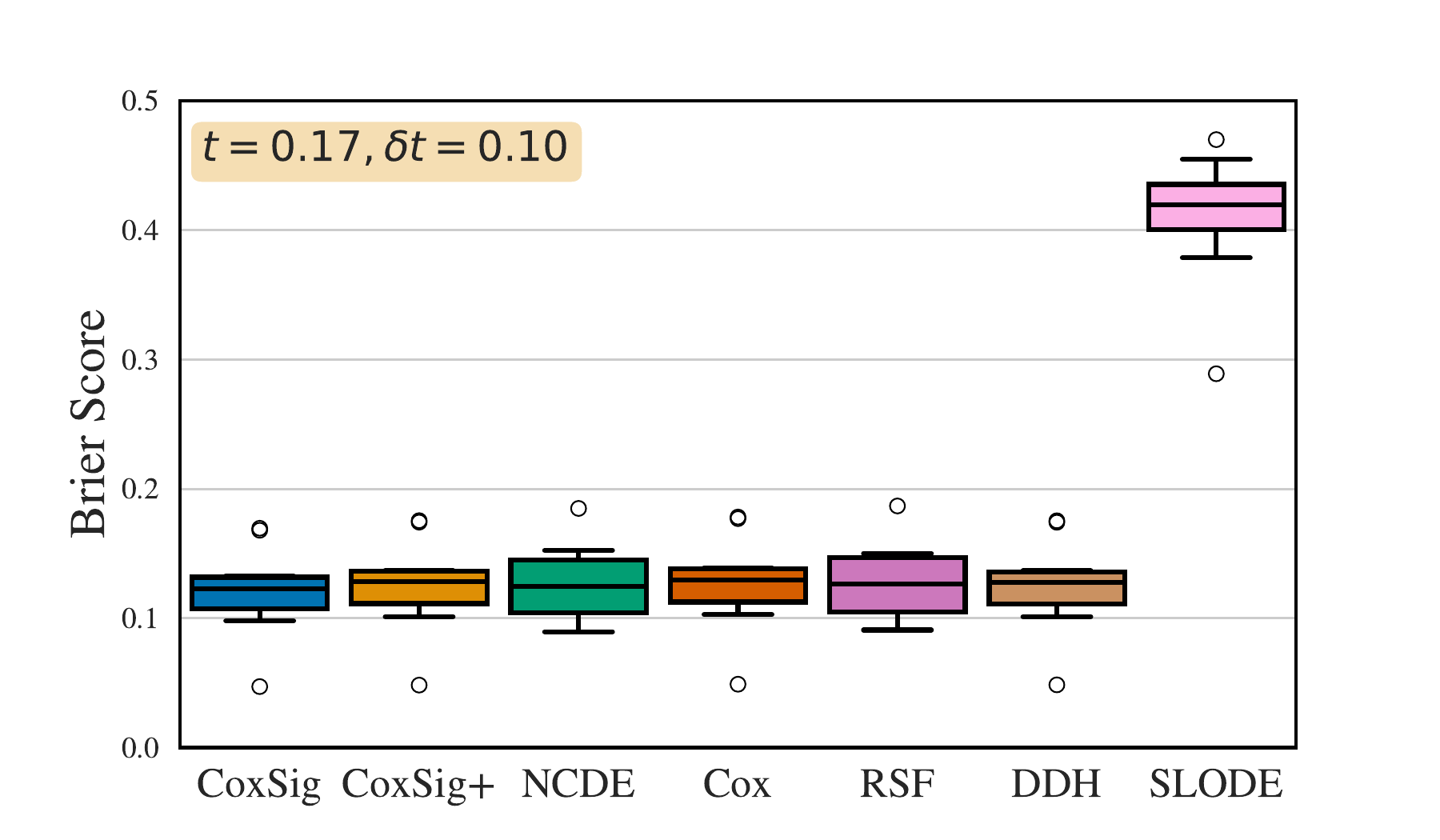}
    \includegraphics[width=0.33\textwidth]{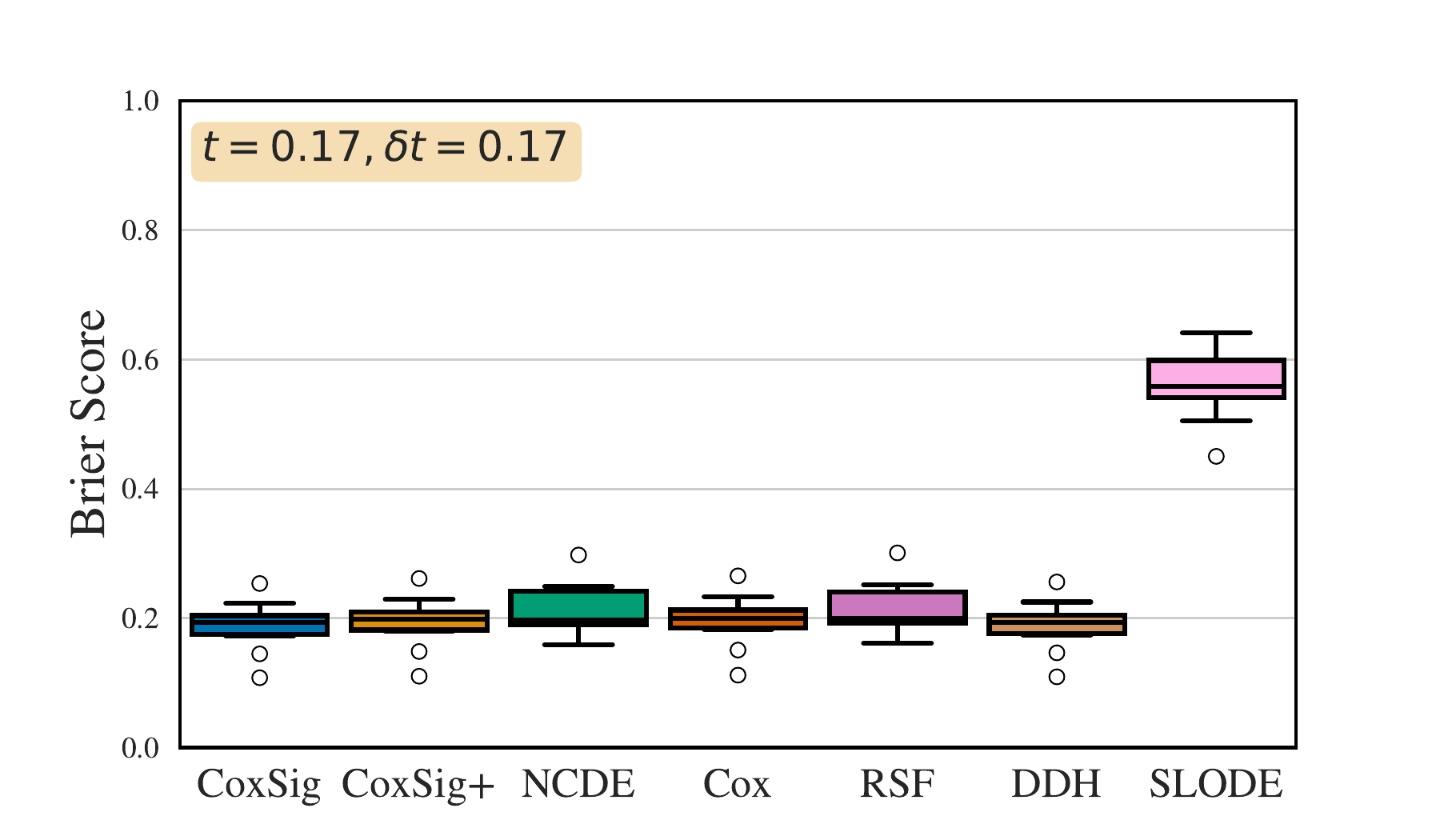}
    \includegraphics[width=0.33\textwidth]{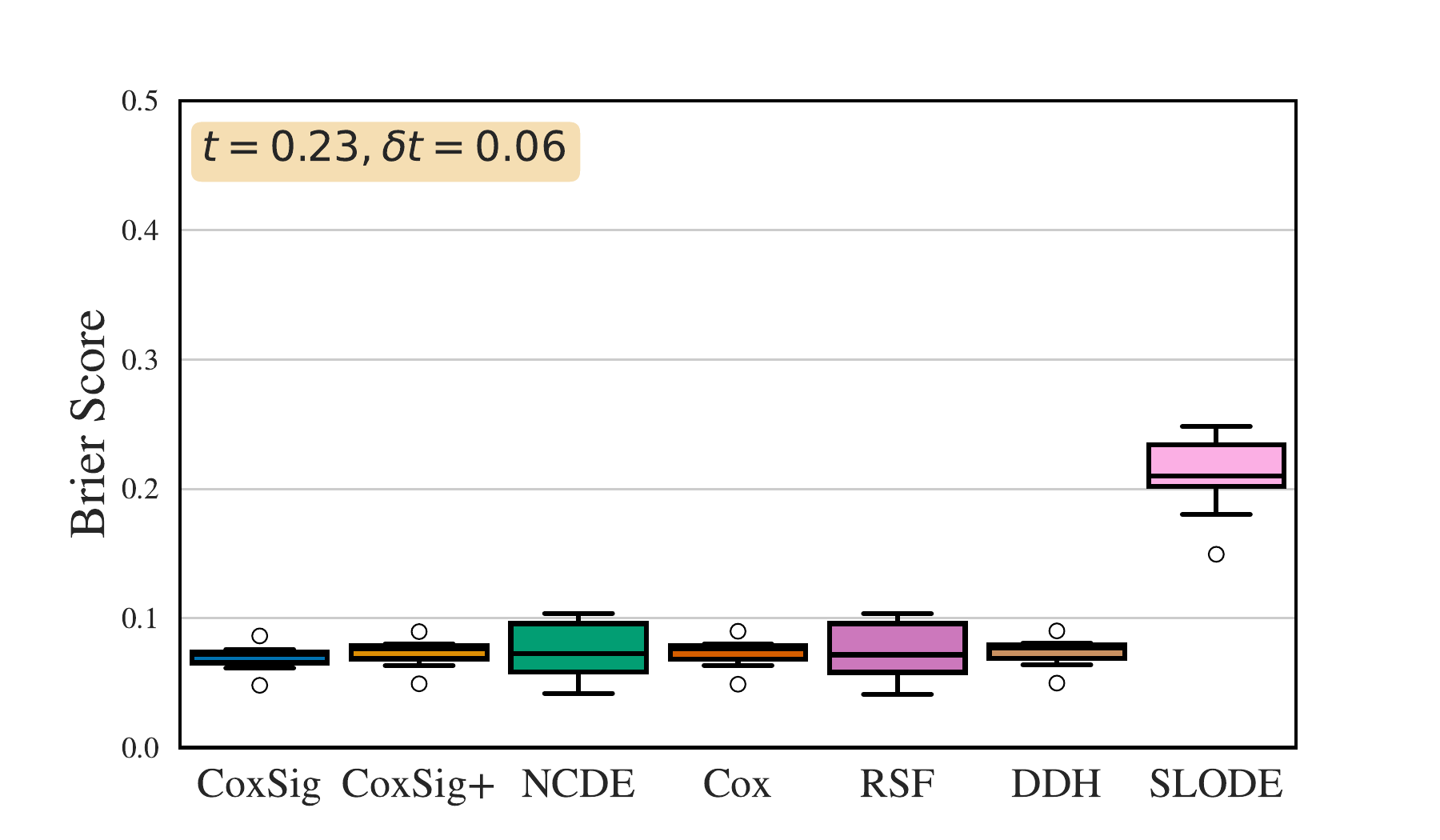}
    \includegraphics[width=0.33\textwidth]{figures/updated_figures/OU/OU_BS_pt=0.23,dt=0.10.pdf}
    \includegraphics[width=0.33\textwidth]{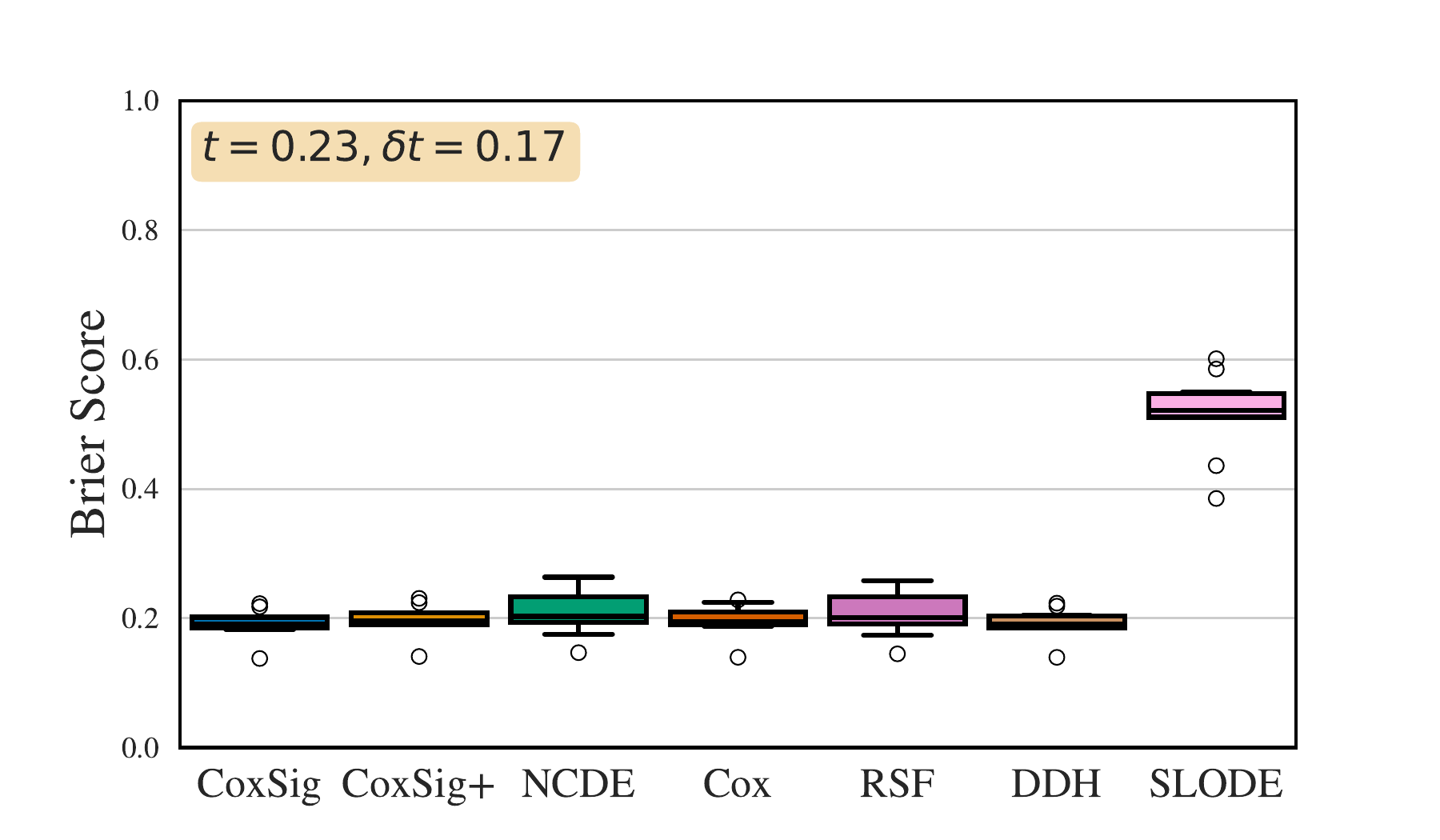}
    \includegraphics[width=0.33\textwidth]{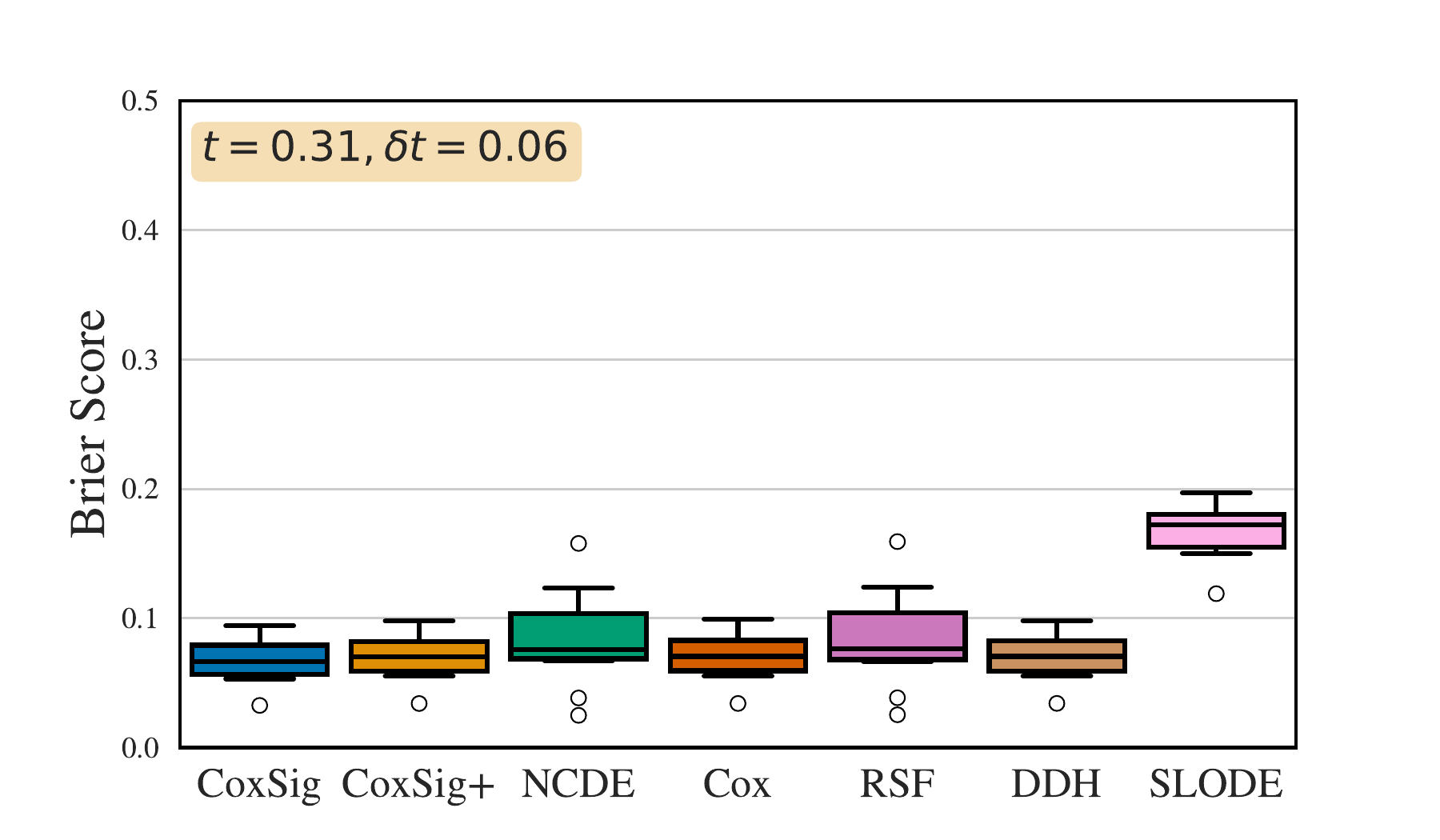}
    \includegraphics[width=0.33\textwidth]{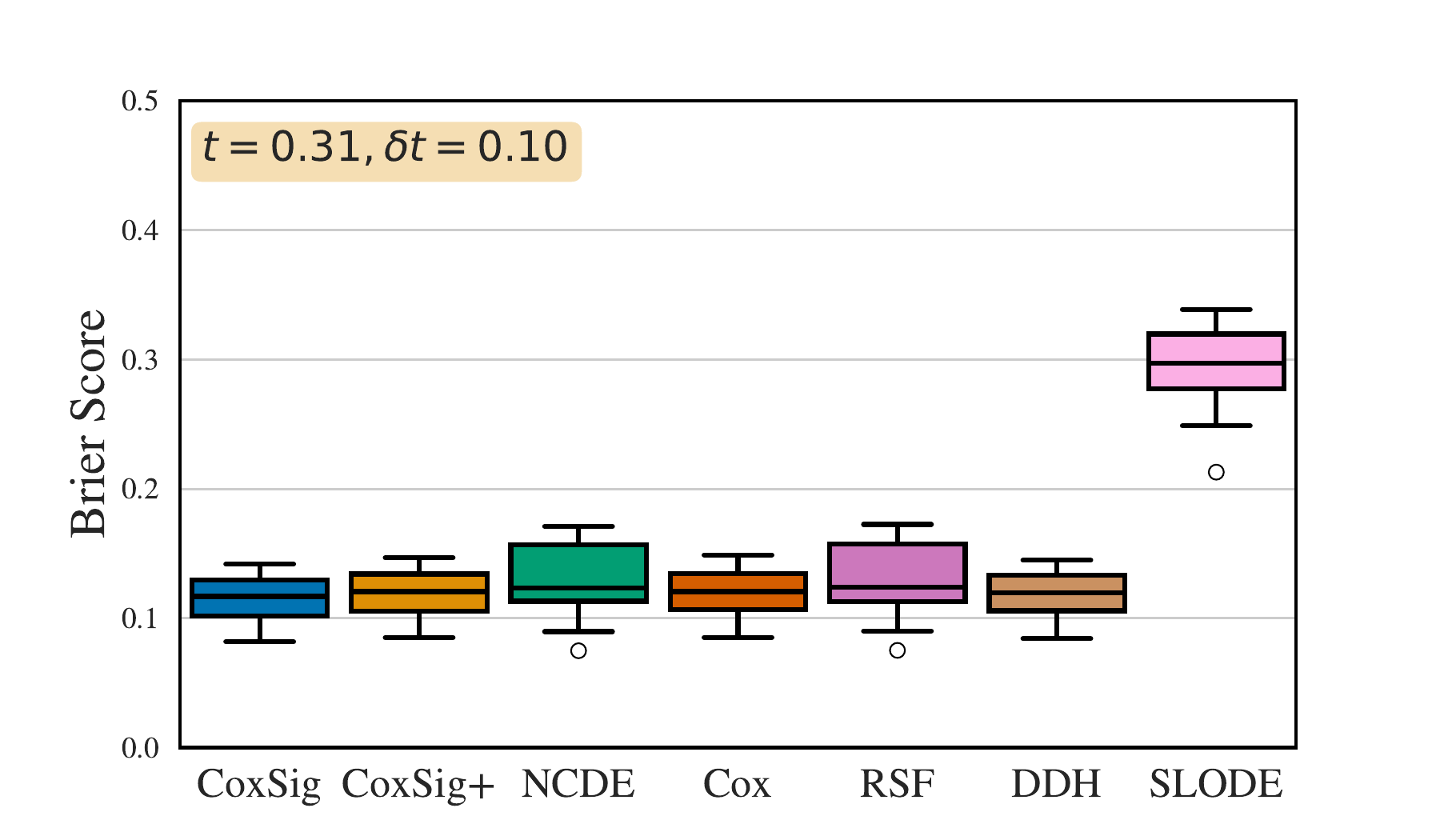}
    \includegraphics[width=0.33\textwidth]{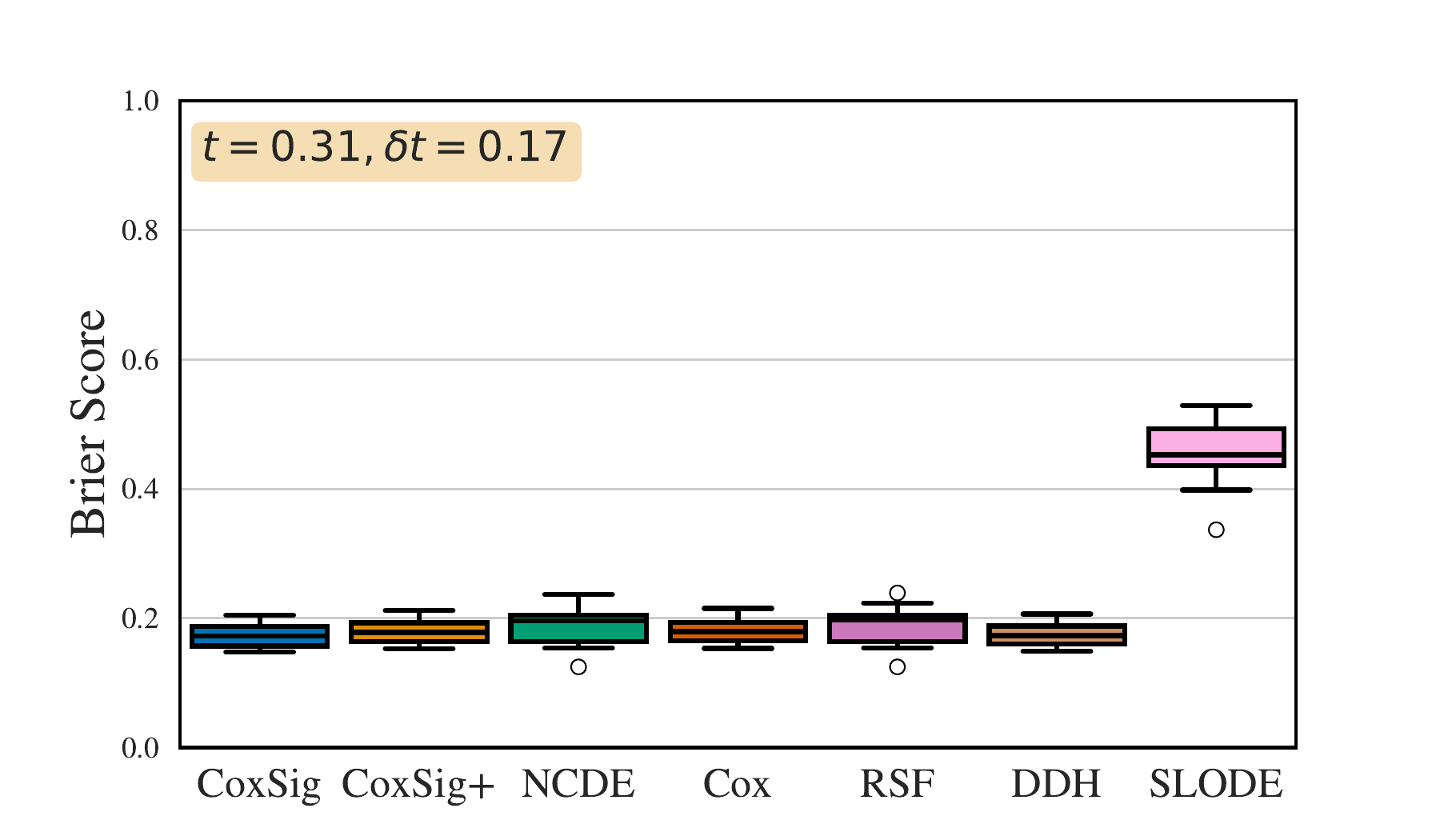}
    \caption{\footnotesize Brier score (\textit{lower} is better) for \textbf{hitting time of a partially observed SDE}  for numerous points $(t,\delta t)$.}
    \label{fig:bs_OU}
\end{figure*}

\begin{figure*}[h!]
    \centering
    \includegraphics[width=0.33\textwidth]{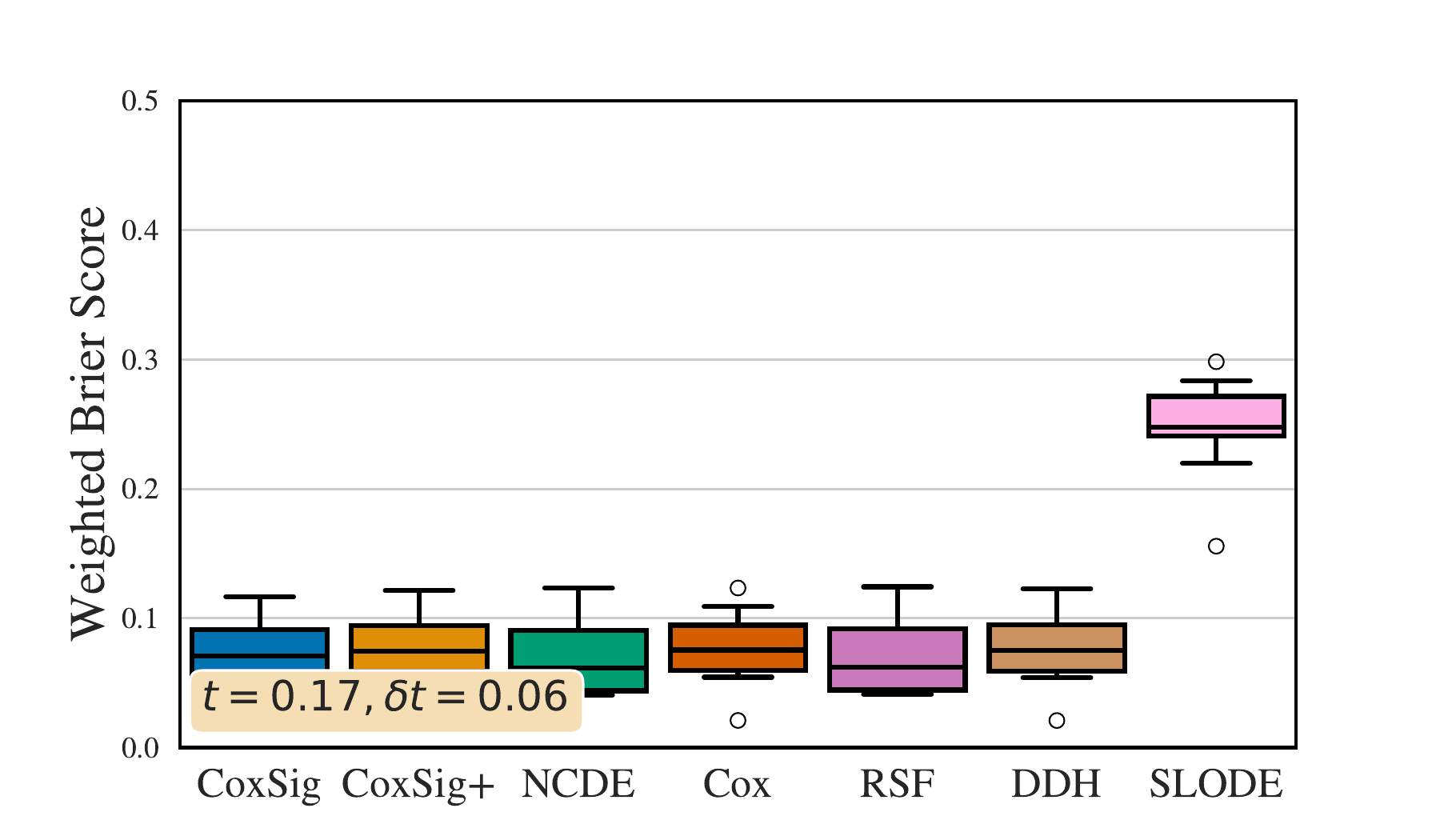}
    \includegraphics[width=0.33\textwidth]{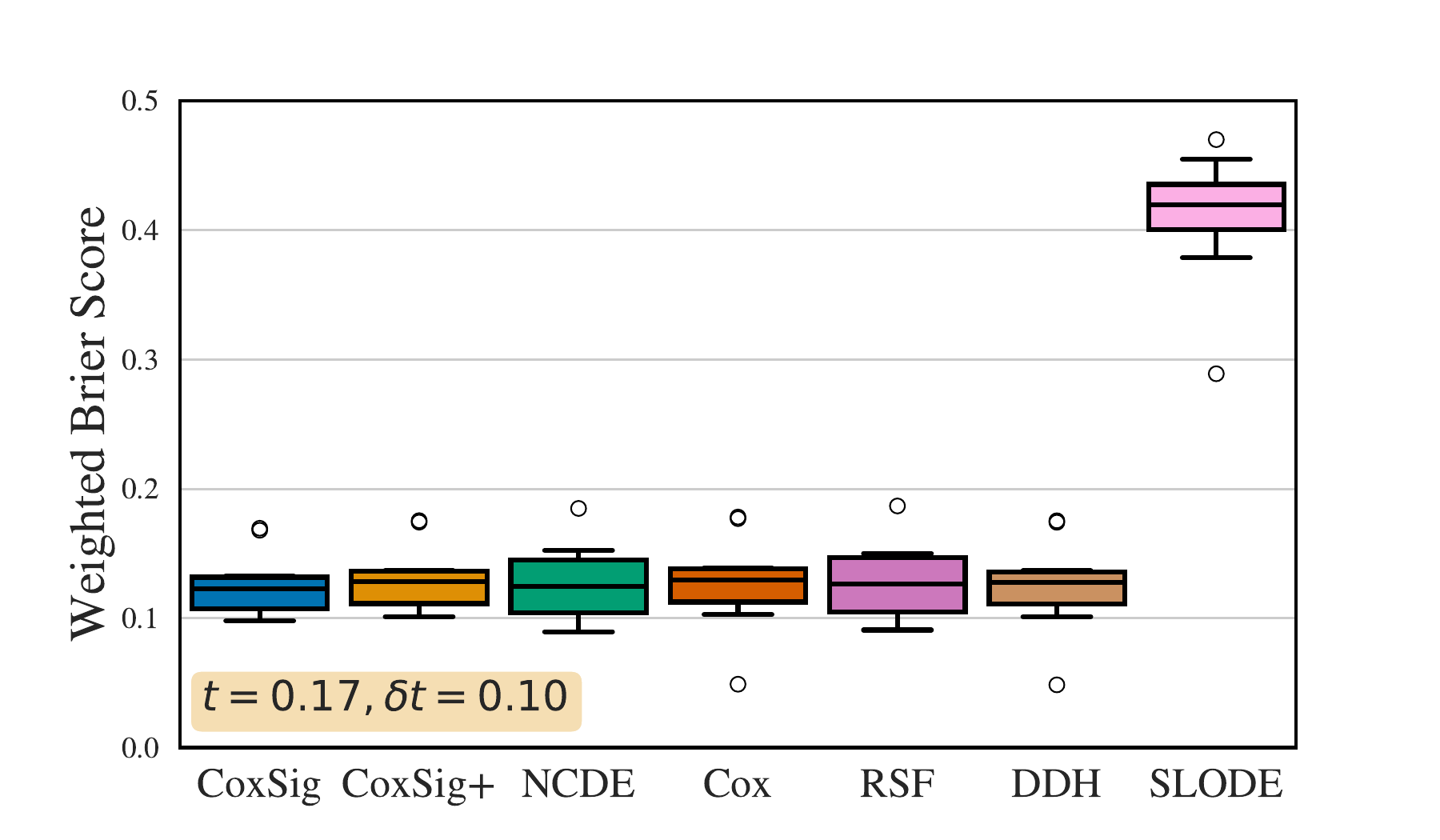}
    \includegraphics[width=0.33\textwidth]{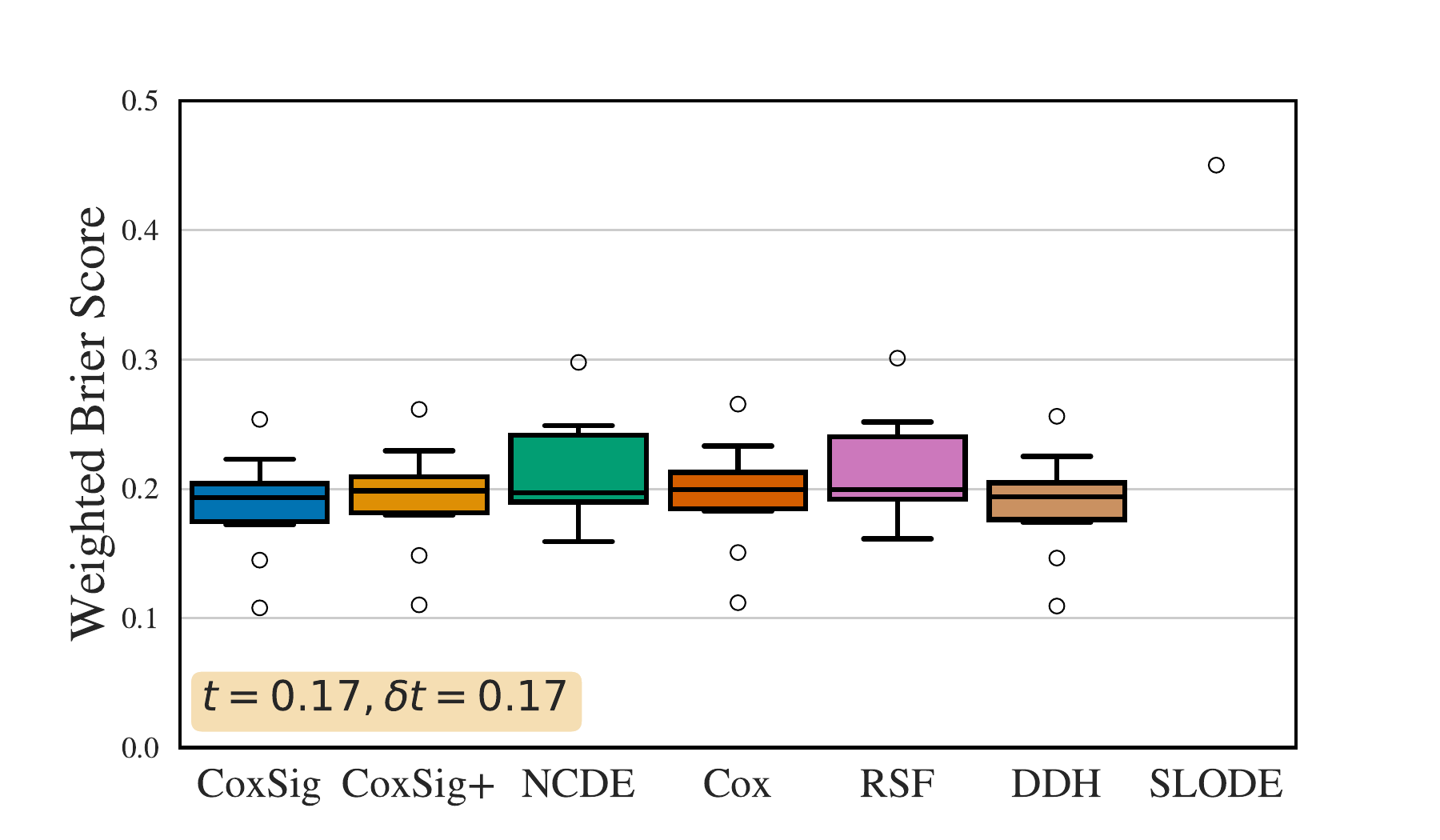}
    \includegraphics[width=0.33\textwidth]{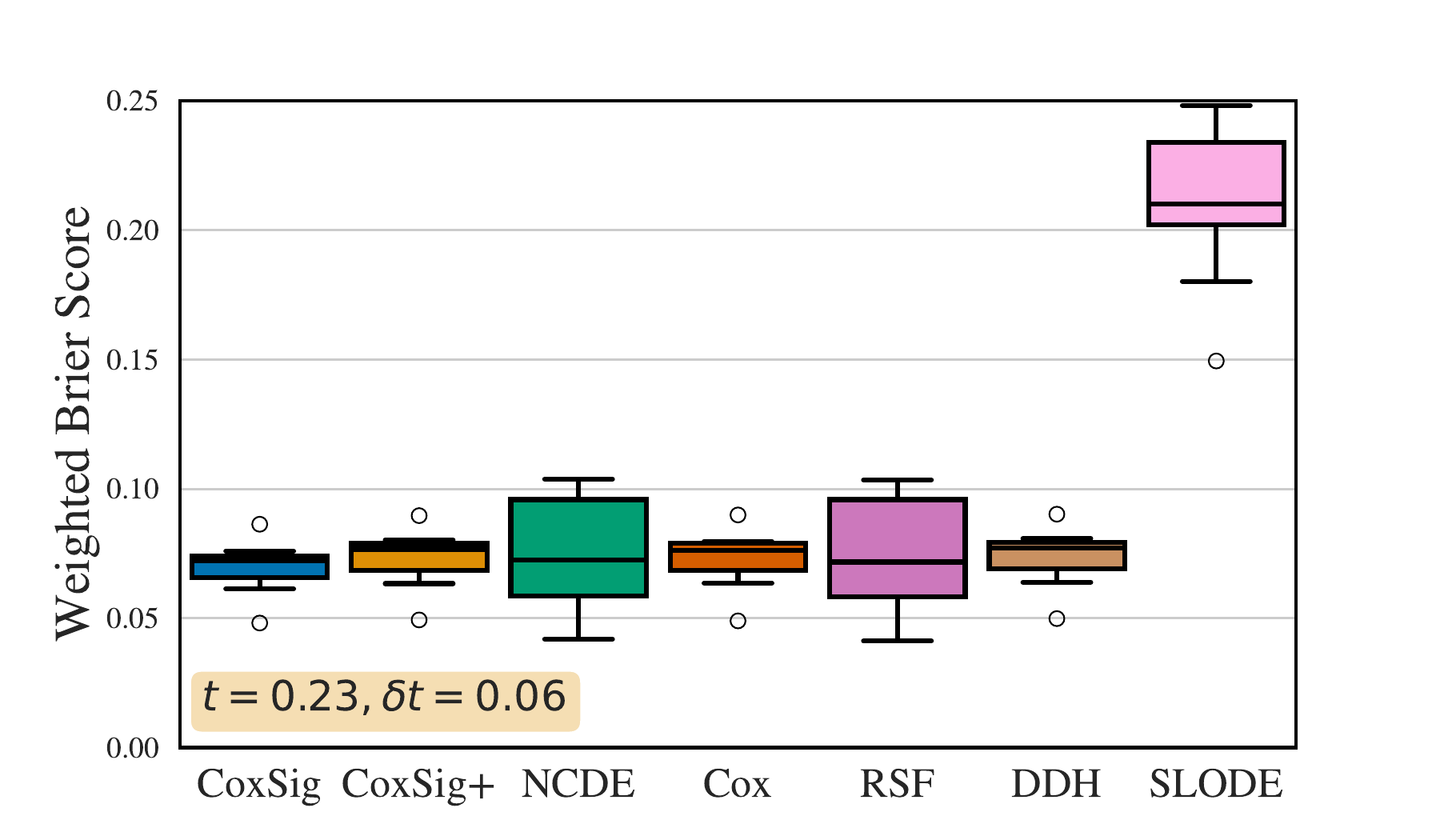}
    \includegraphics[width=0.33\textwidth]{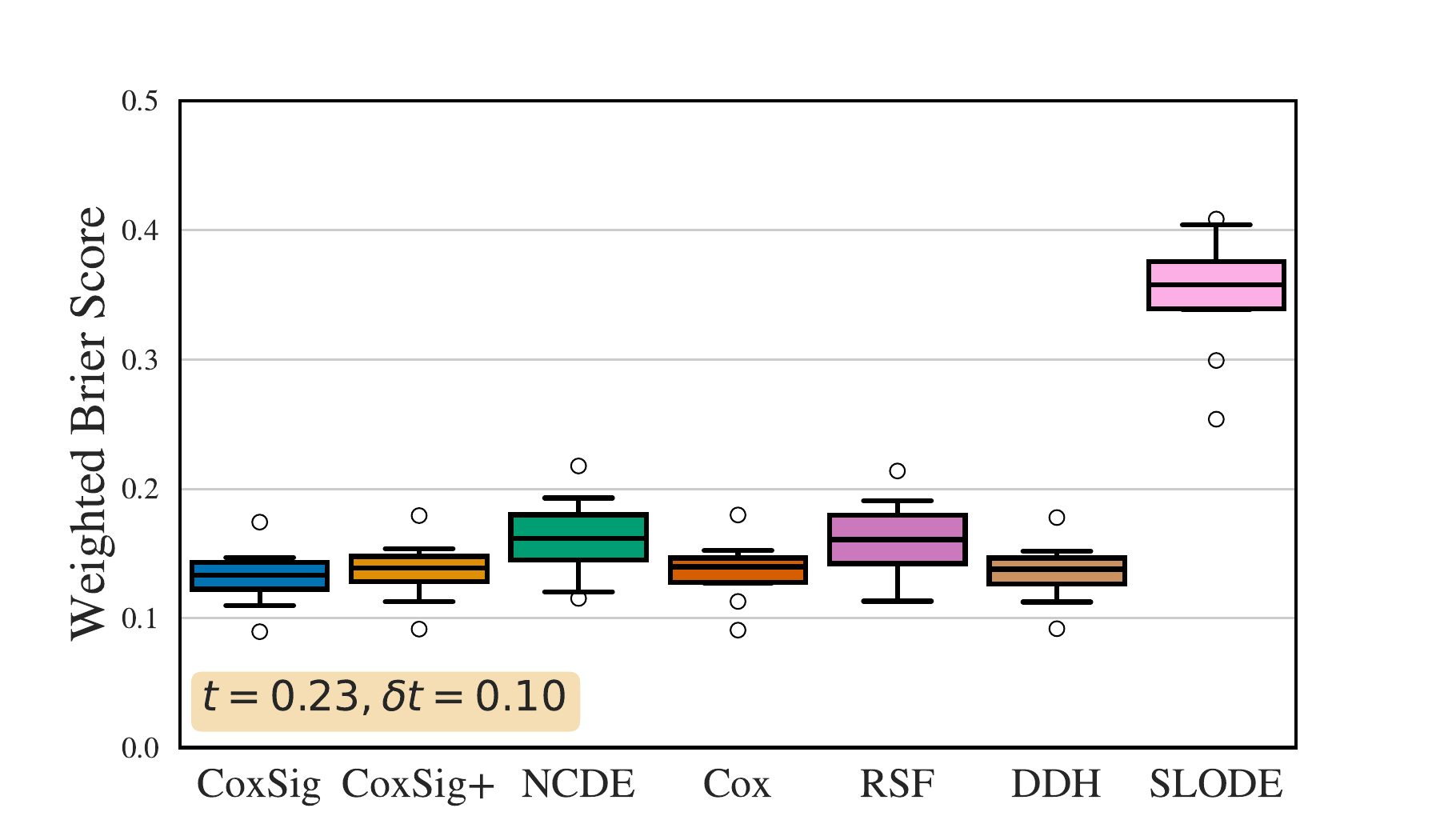}
    \includegraphics[width=0.33\textwidth]{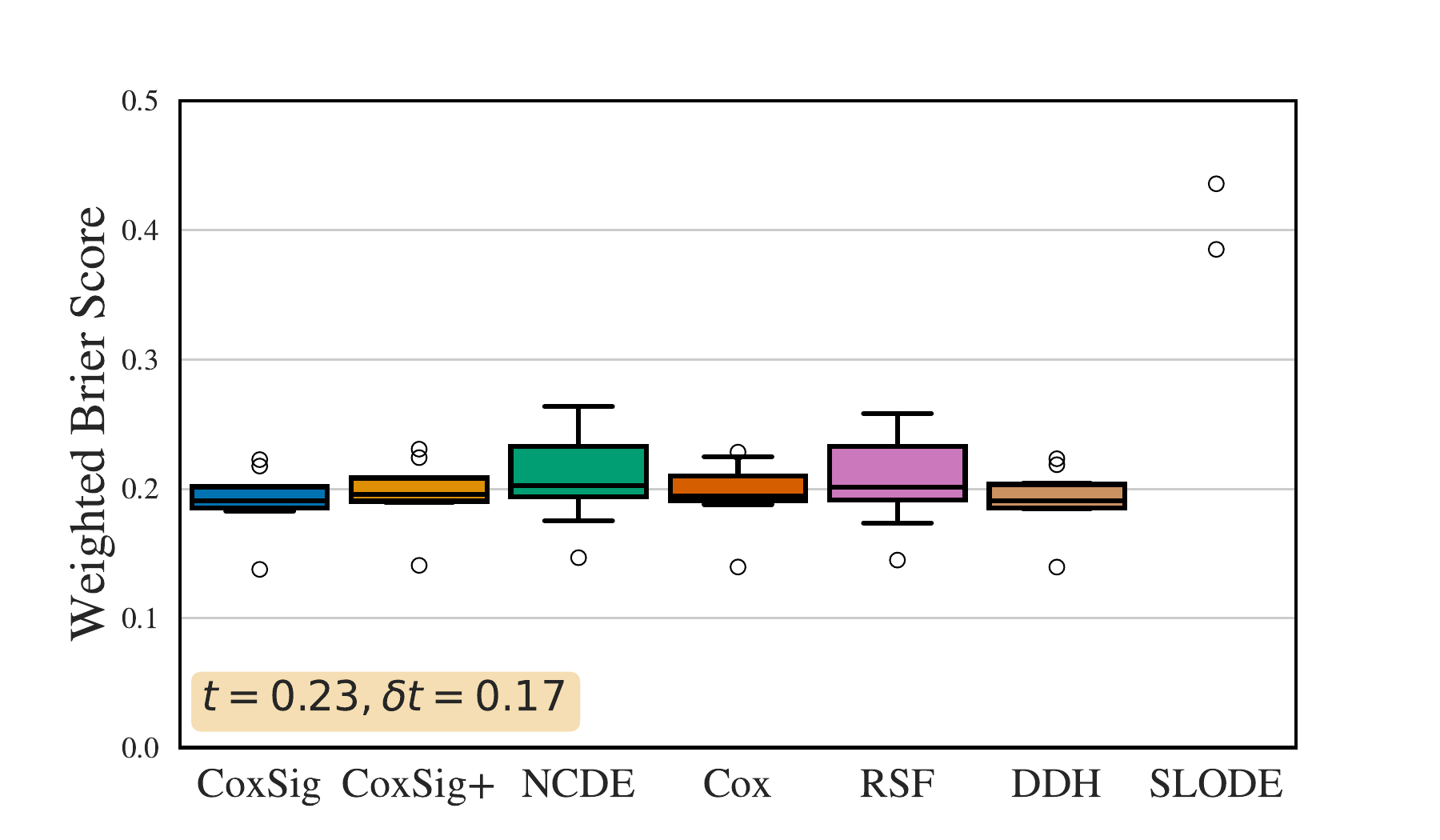}
    \includegraphics[width=0.33\textwidth]{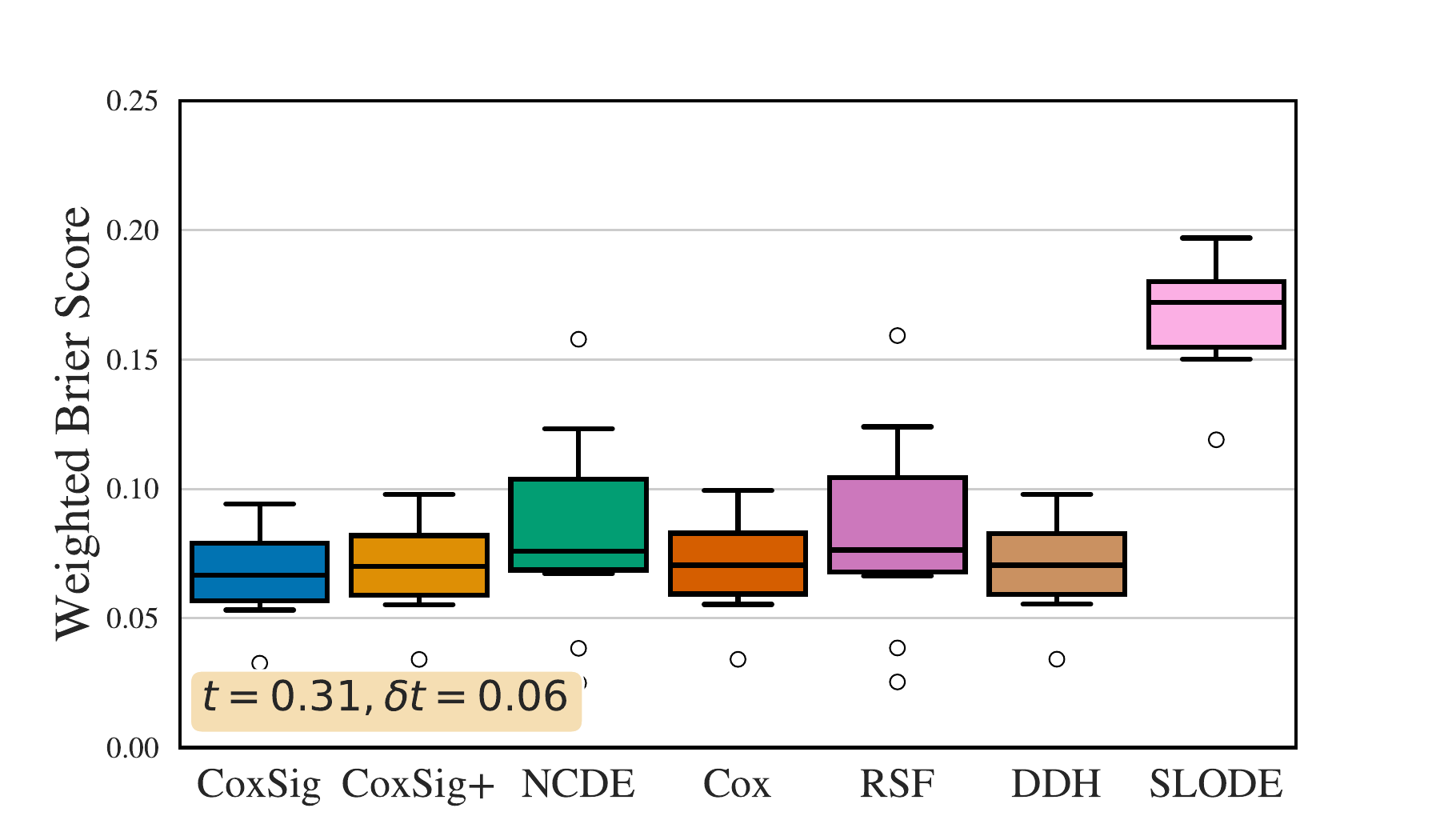}
    \includegraphics[width=0.33\textwidth]{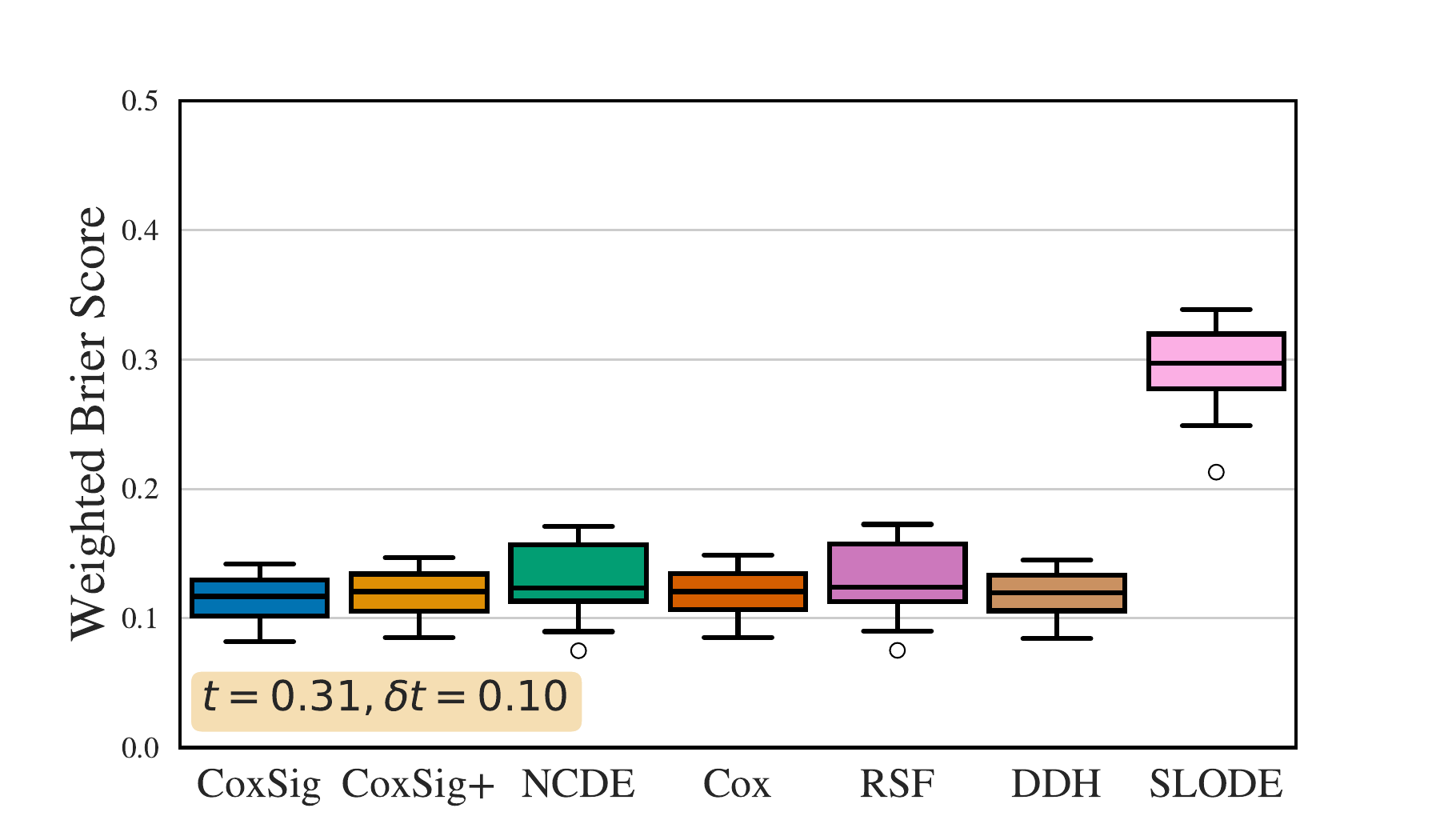}
    \includegraphics[width=0.33\textwidth]{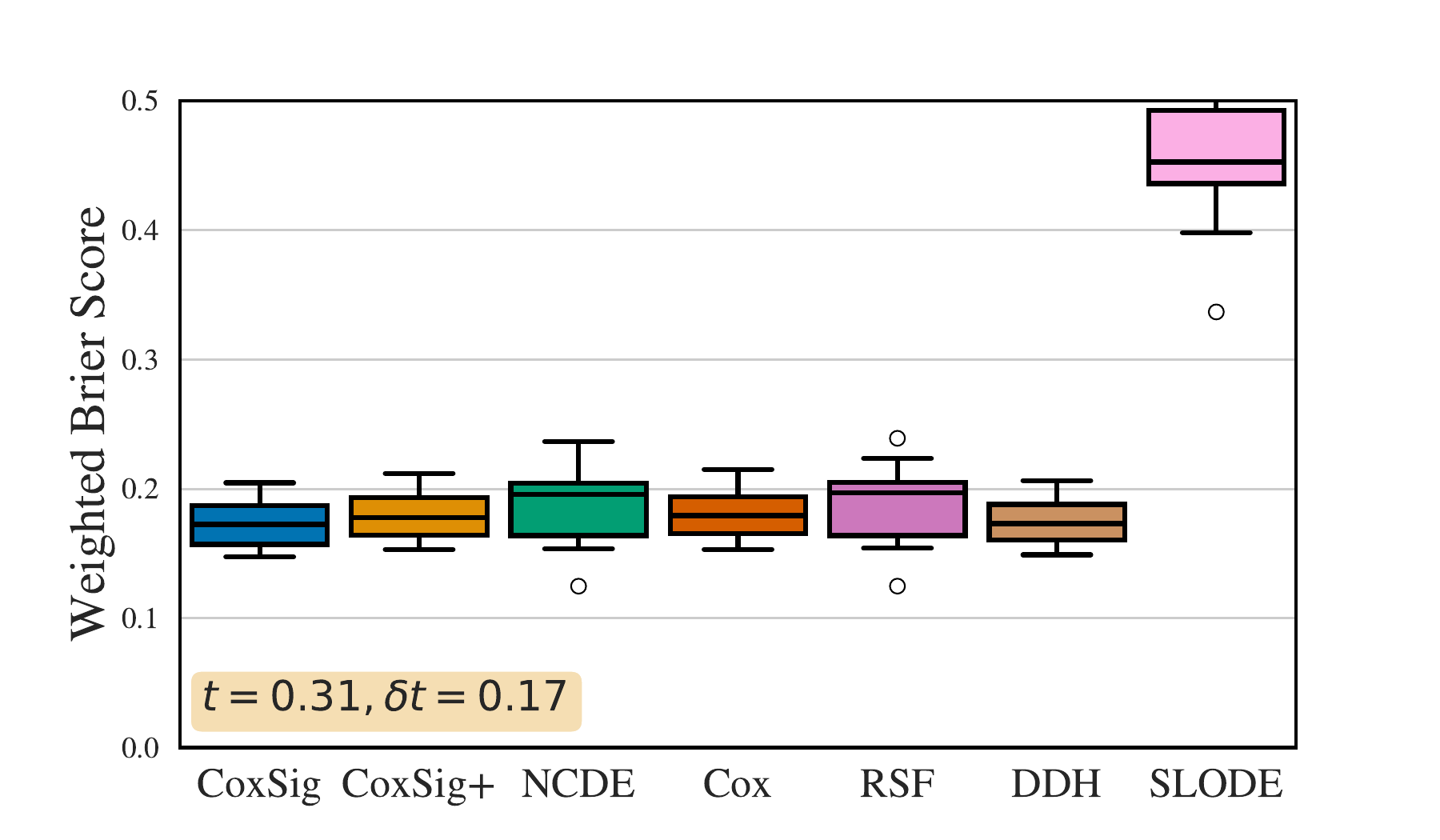}
    \caption{\footnotesize Weighted Brier score (\textit{lower} is better) for \textbf{hitting time of a partially observed SDE}  for numerous points $(t,\delta t)$.}
    \label{fig:wbs_OU}
\end{figure*}

\begin{figure*}[h!]
    \centering
    \includegraphics[width=0.33\textwidth]{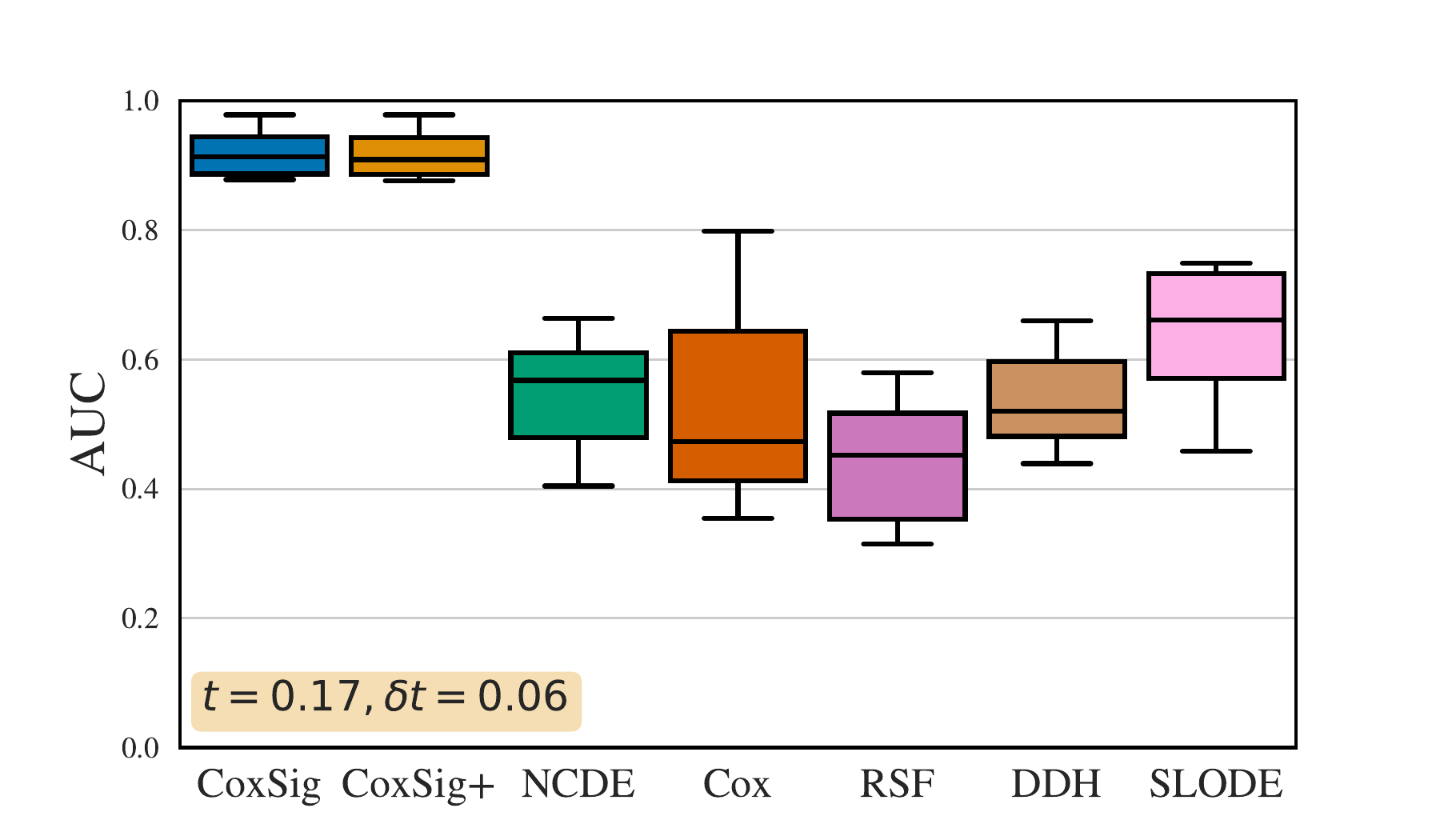}
    \includegraphics[width=0.33\textwidth]{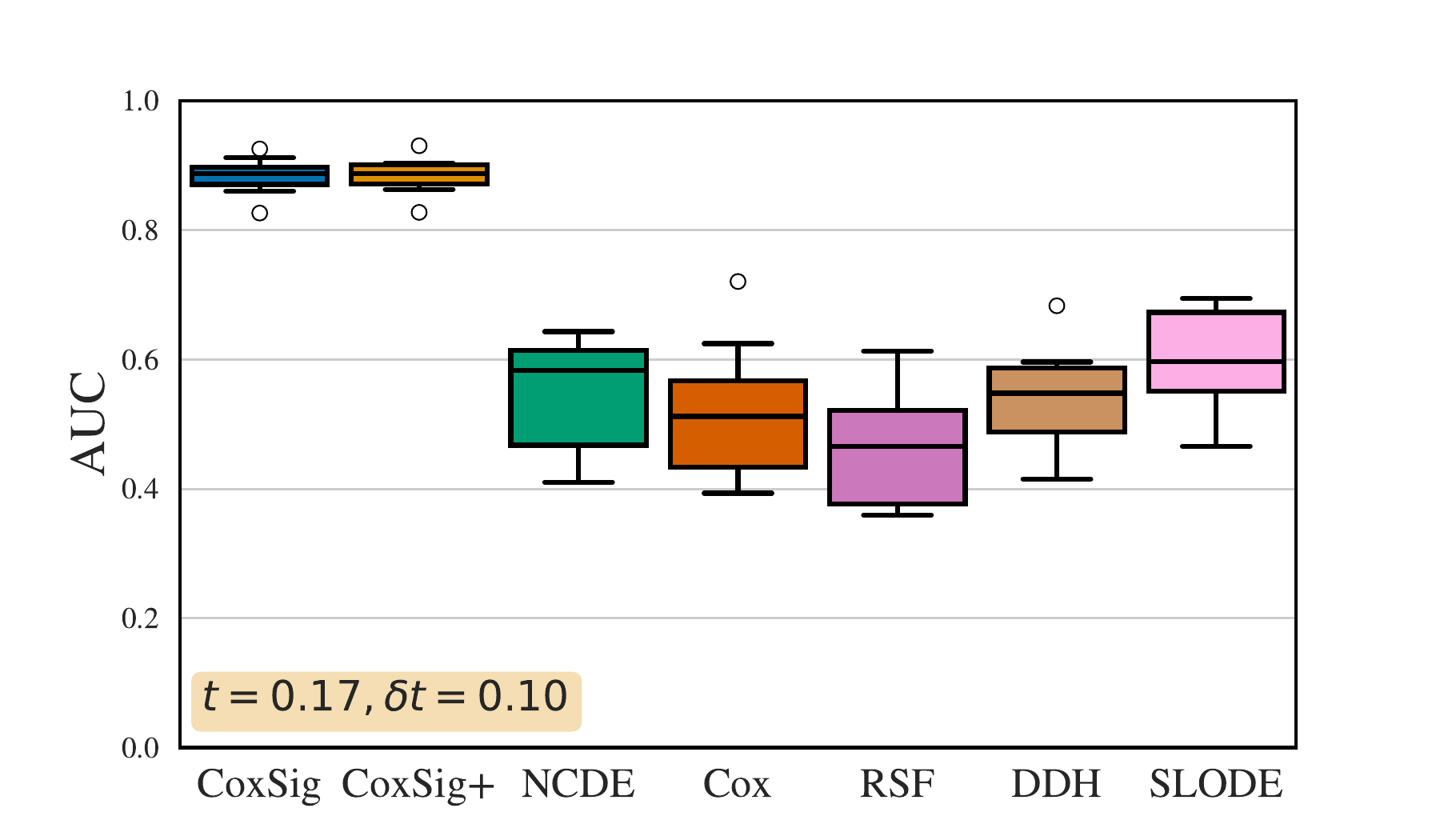}
    \includegraphics[width=0.33\textwidth]{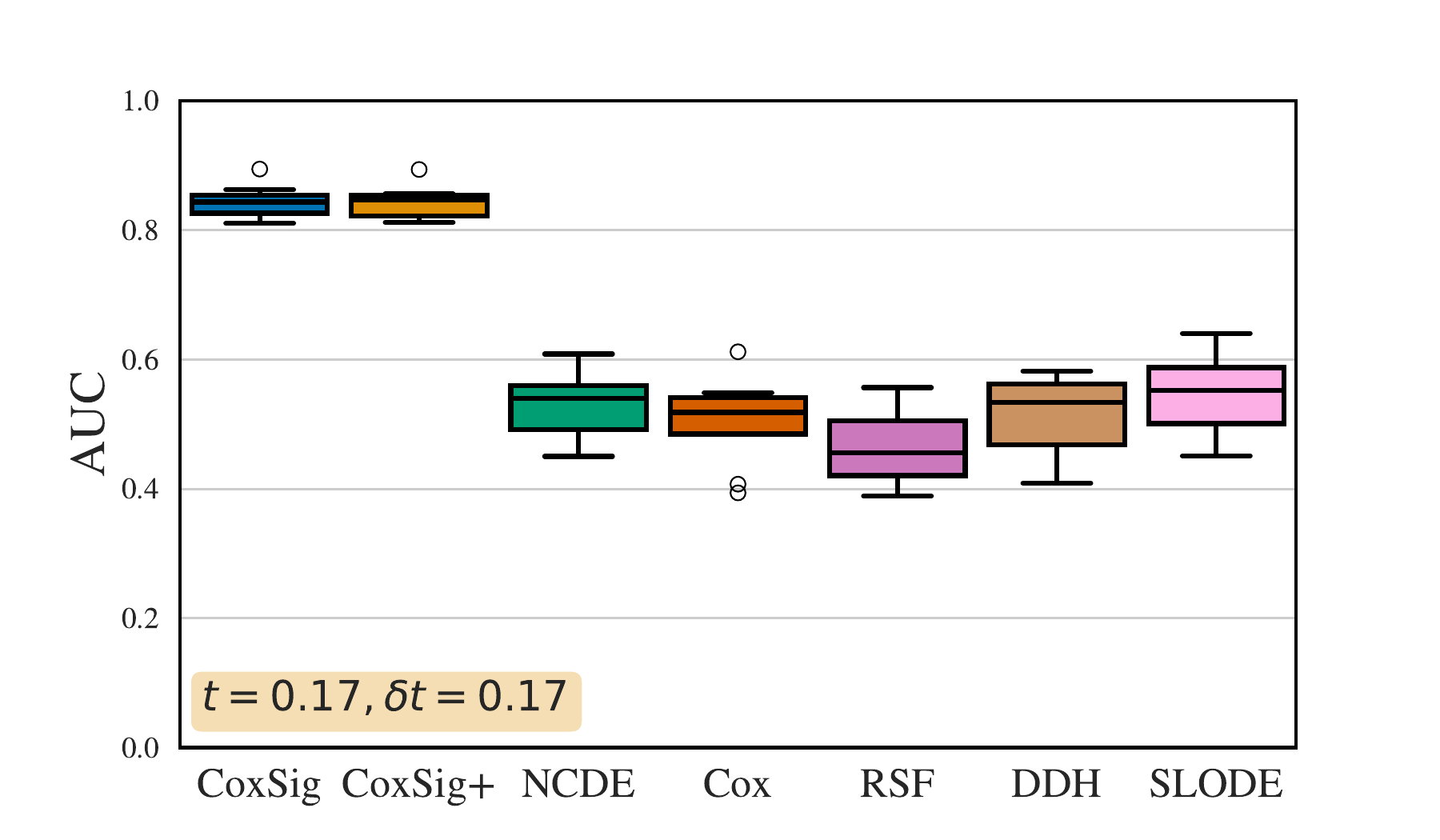}
    \includegraphics[width=0.33\textwidth]{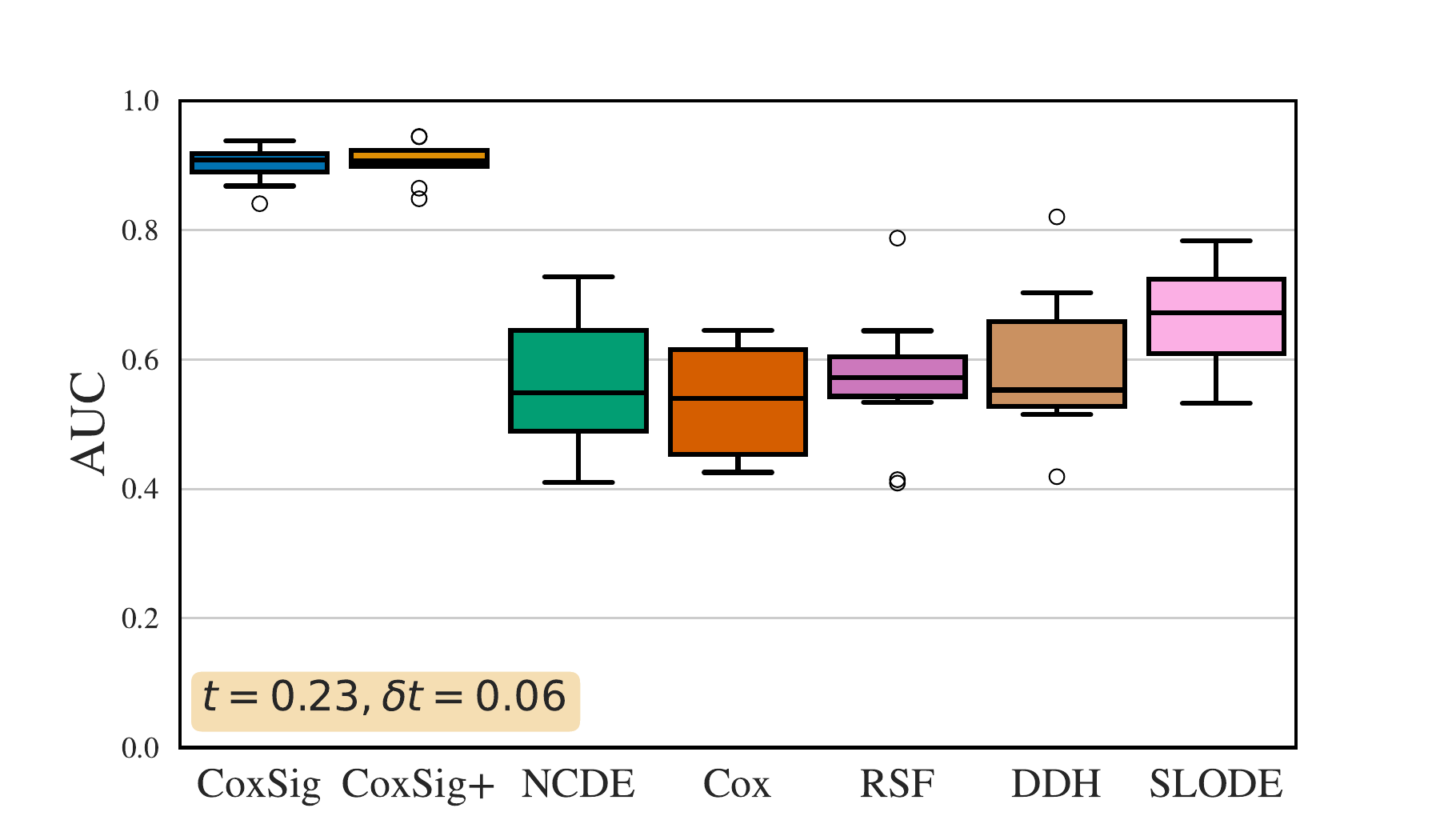}
    \includegraphics[width=0.33\textwidth]{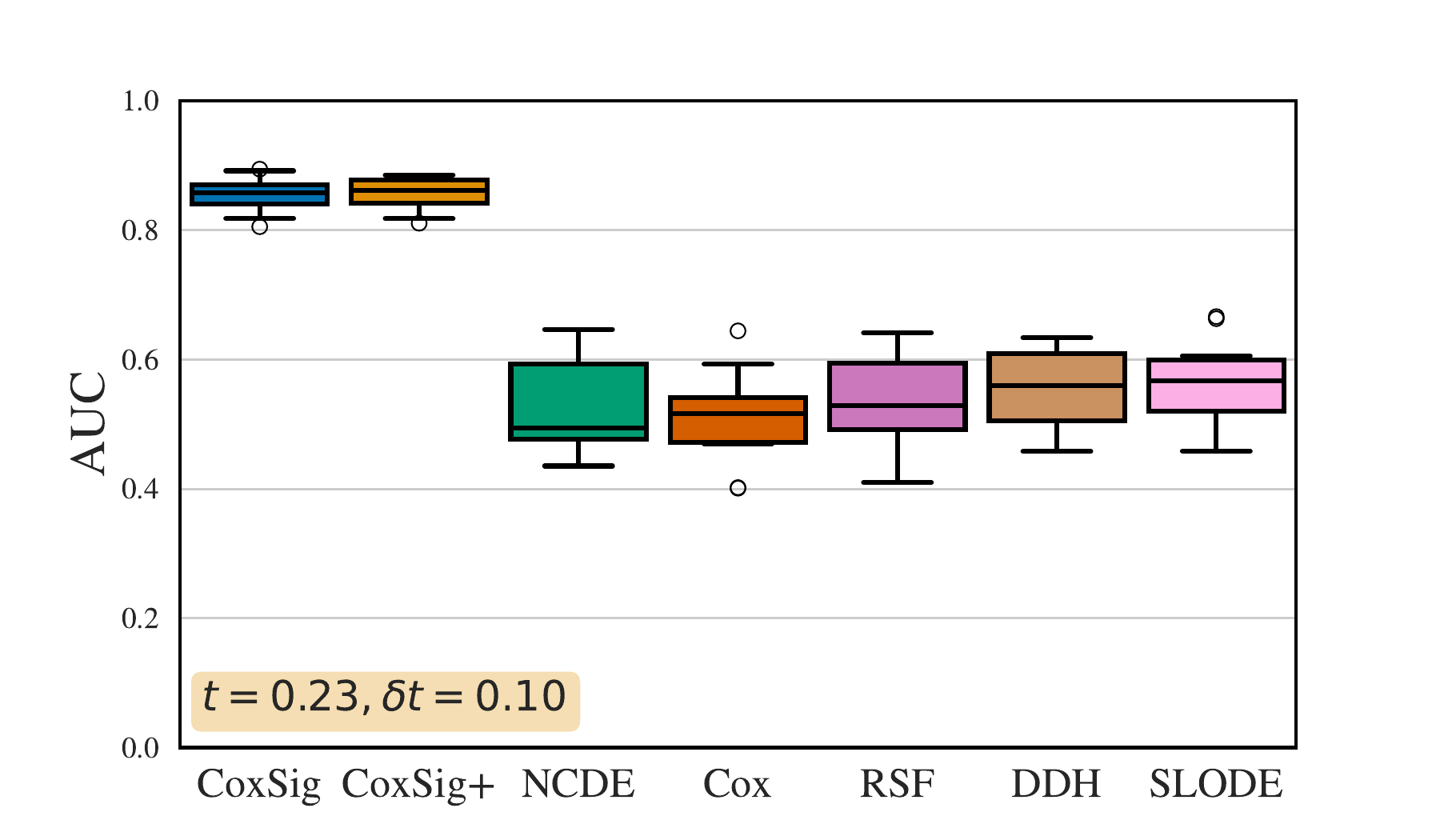}
    \includegraphics[width=0.33\textwidth]{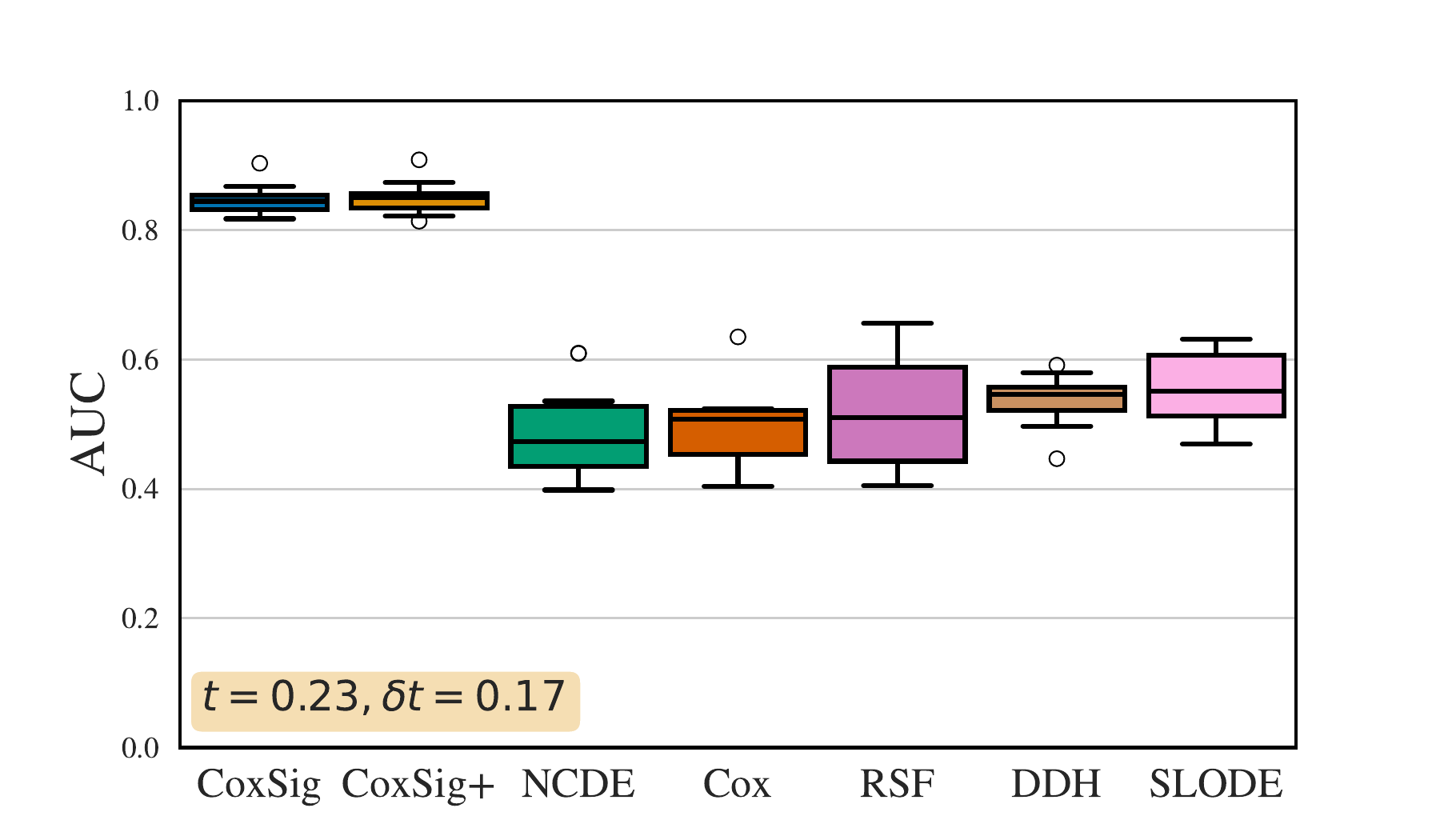}
    \includegraphics[width=0.33\textwidth]{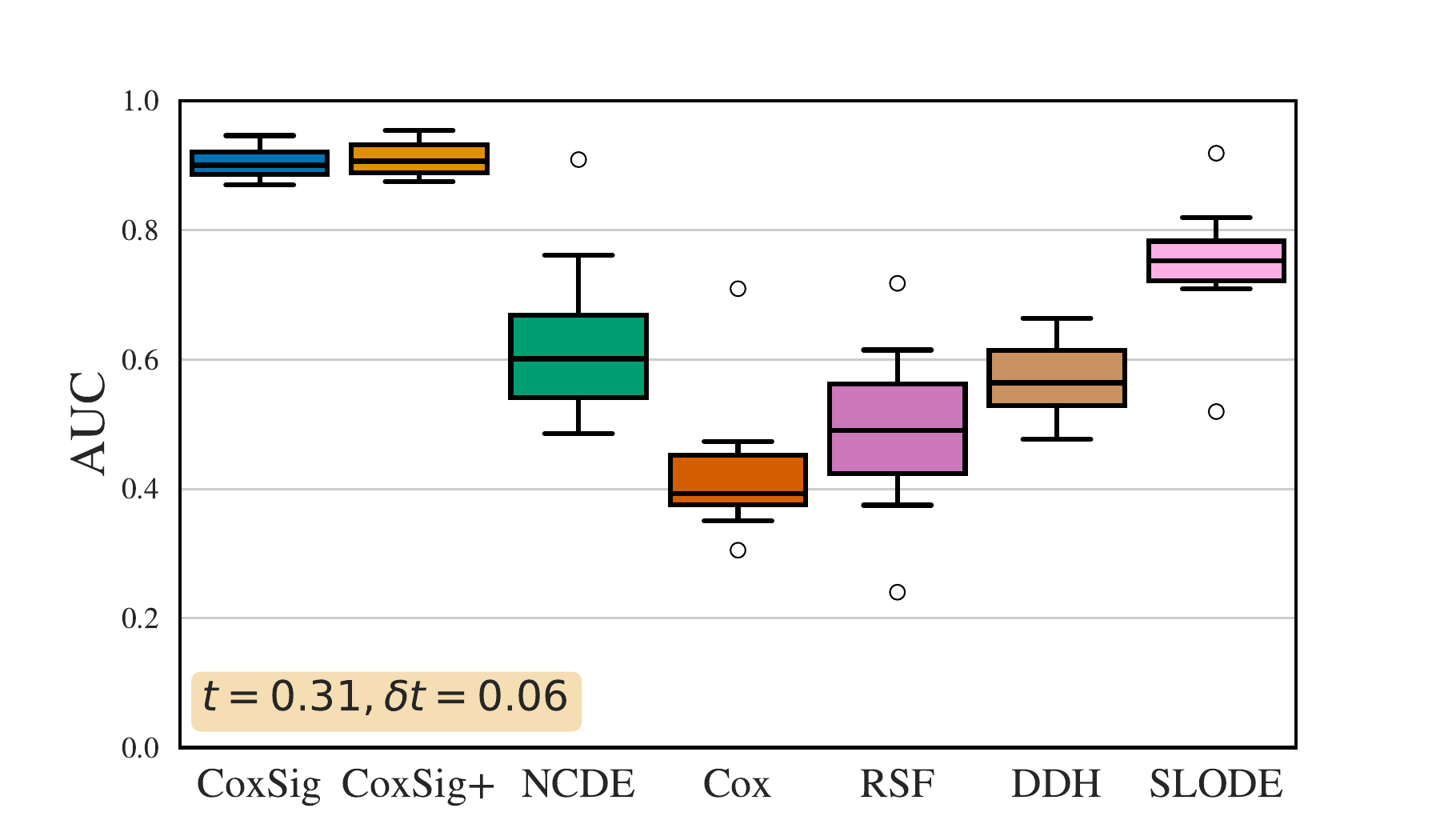}
    \includegraphics[width=0.33\textwidth]{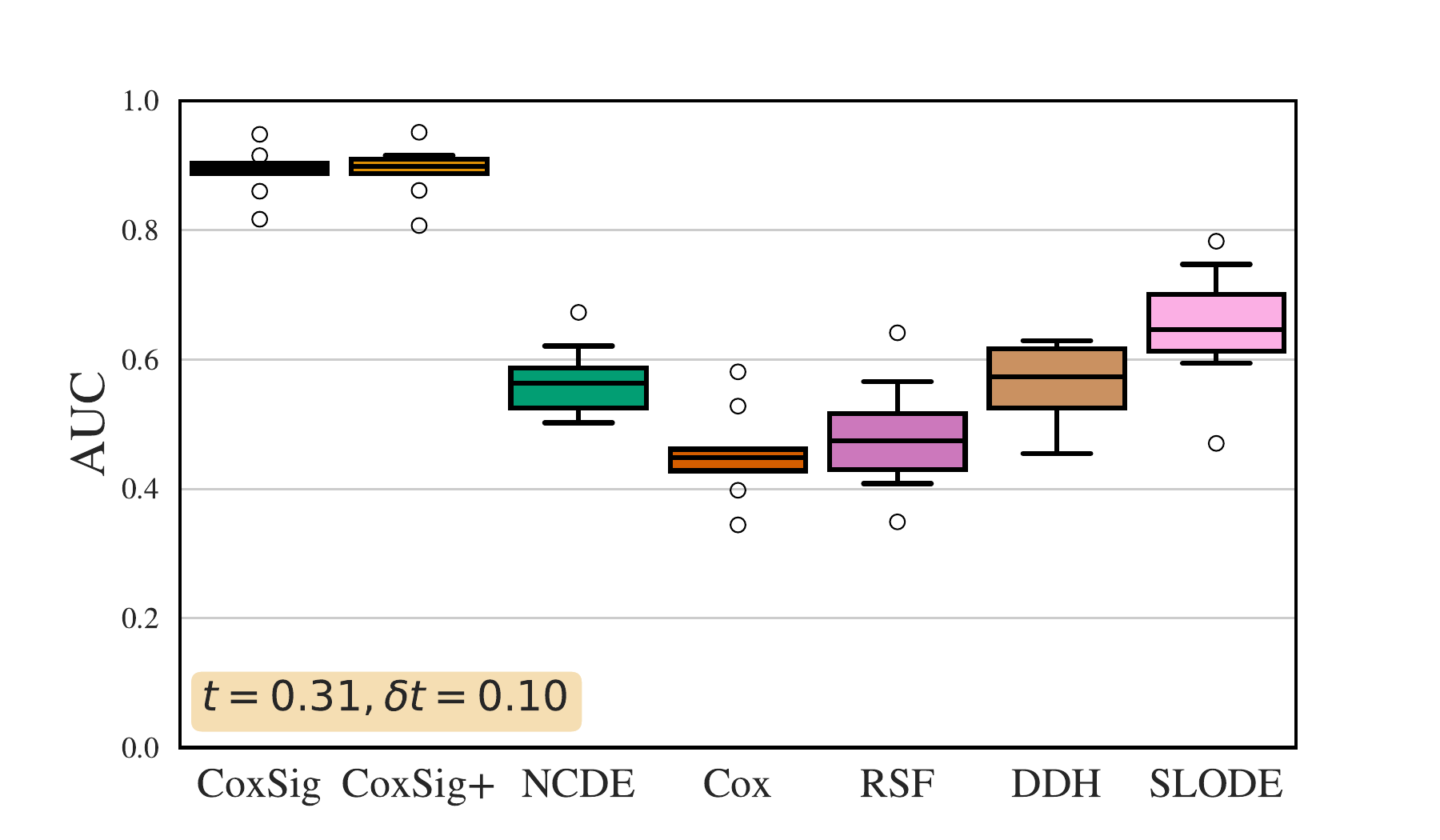}
    \includegraphics[width=0.33\textwidth]{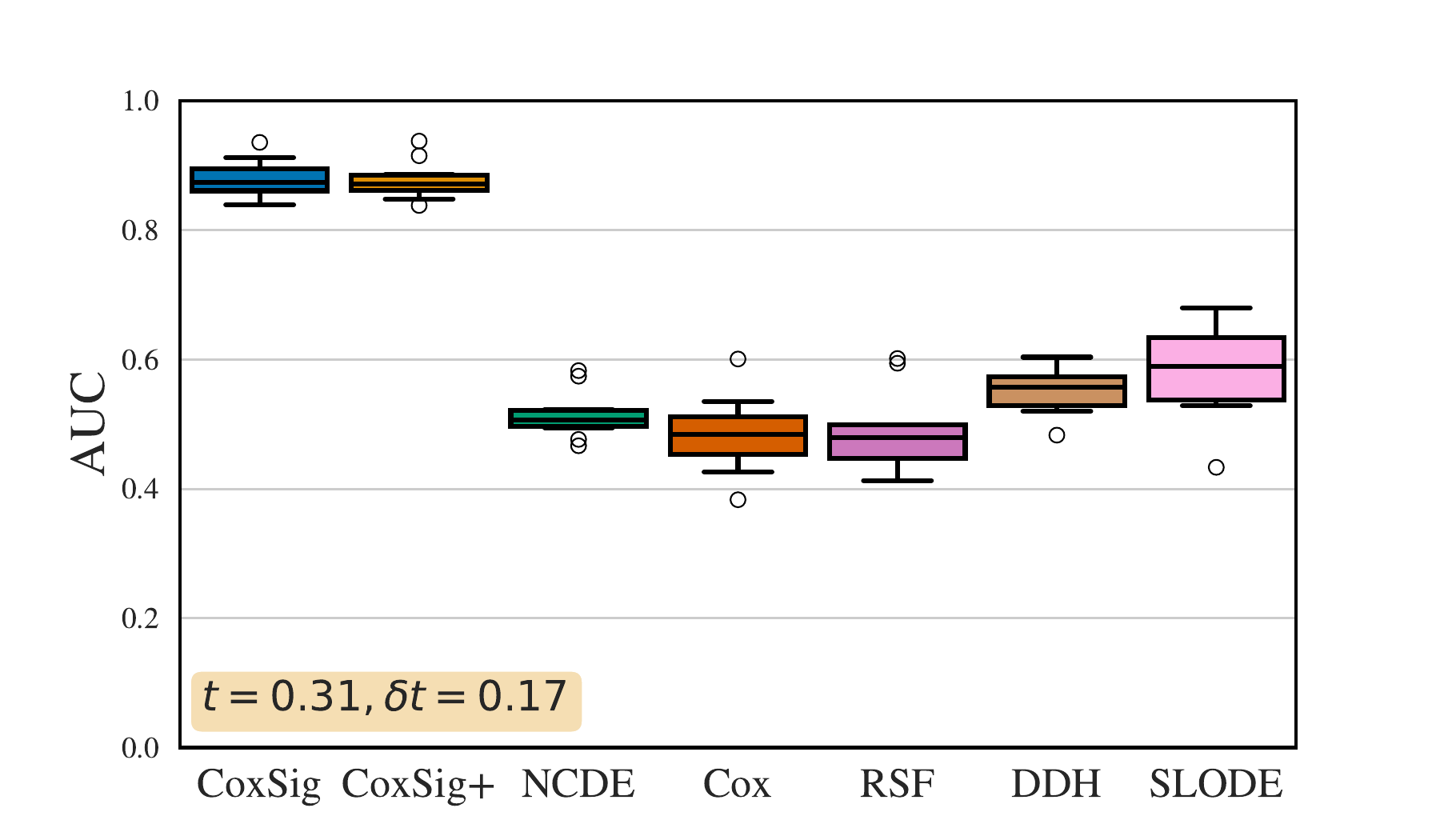}
    \caption{\footnotesize AUC (\textit{higher} is better) for \textbf{hitting time of a partially observed SDE}  for numerous points $(t,\delta t)$.}
    \label{fig:auc_OU}
\end{figure*}

\begin{figure}[h!]
    \centering
    \includegraphics[width=0.49\textwidth]{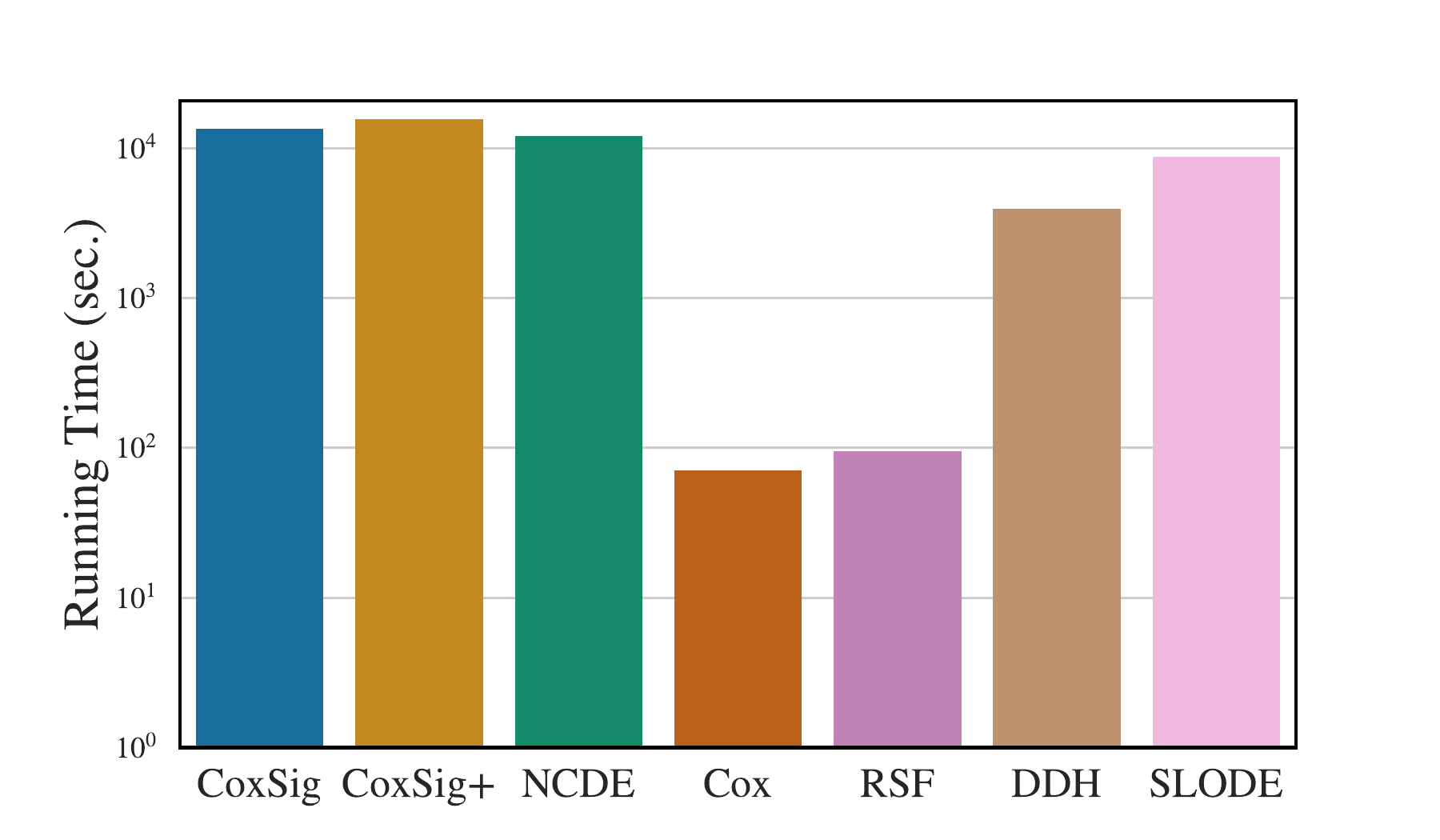}
    \includegraphics[width=0.49\textwidth]{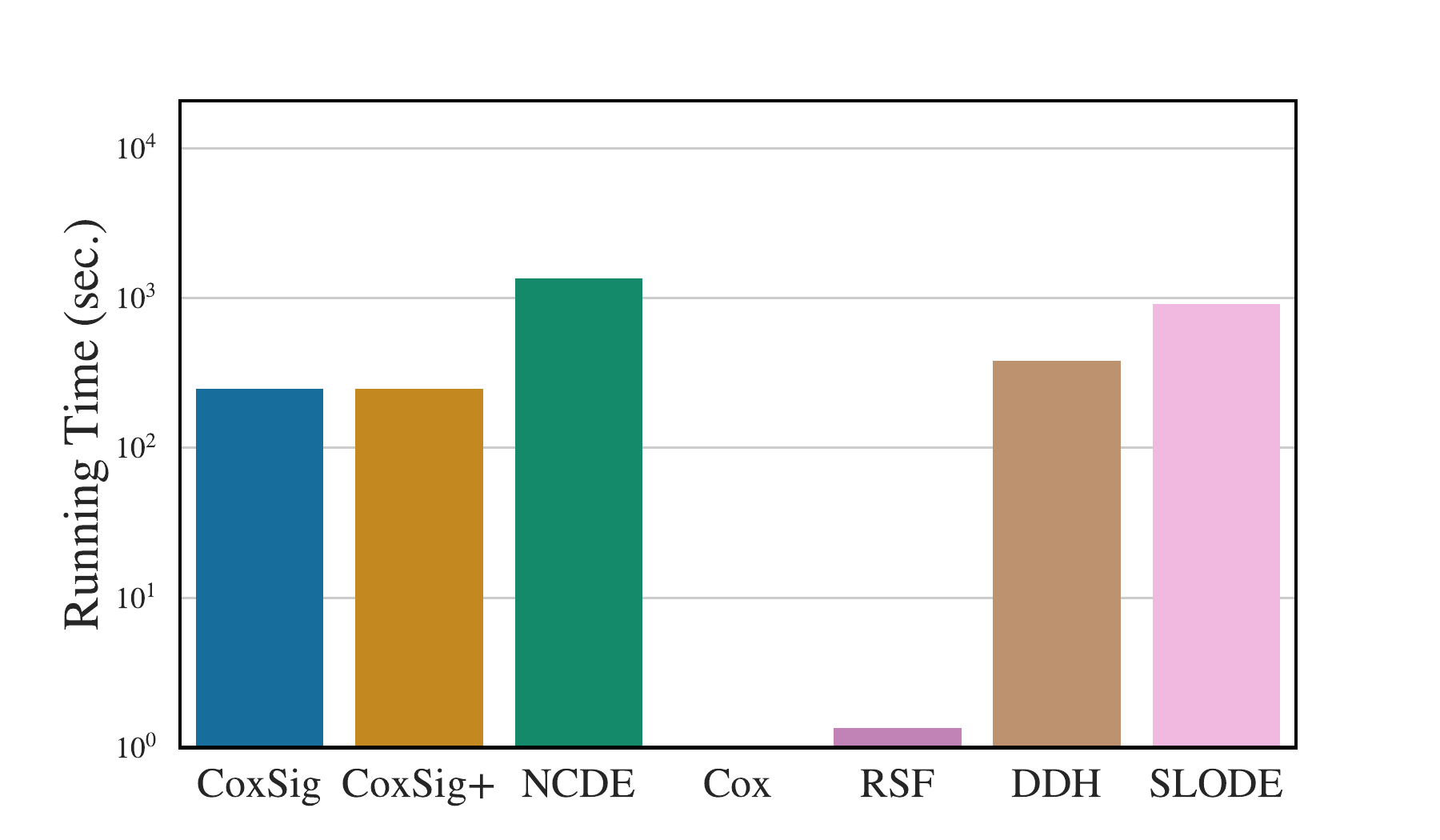}
    \caption{\footnotesize Running times on the partially observed SDE experiment (log-scale) averaged over $10$ runs including cross-validation of the hyperparameters on CoxSig, CoxSig+, Cox and RSF (\textbf{left}) and over $1$ run without cross-validation of the hyperparameters on CoxSig, CoxSig+, Cox and RSF (\textbf{right}).}
    \label{fig:running_times}
\end{figure}

\subsection{Tumor Growth}

\paragraph{Time series.} Similarly to the partially observed SDE experiment described above, we set $d=2$ which includes 1-dimensional sample path $x_t$ of a fractional Brownian motion with Hurst parameter $H=0.6$. The paths are sampled at $1000$ times over the time interval $[0,10]$. All simulations are done using the \texttt{stochastic} package. The time series $\mathbf{X}^i$ are identical, up to observation time, to the ones used for simulations.

\paragraph{Event definition.} Following \citet{simeoni2004predictive}, we consider the differential equations 
\begin{align*}
    \label{eq:experiment_tumor_growth}
    & \frac{du^{(1)}_t}{dt} = \frac{\lambda_0 u^{(1)}_t}{\big[1 + (\frac{\lambda_0}{\lambda_1} w_t)^\Psi\big]^{1/\Psi}} - \kappa_2 x_t u^{(1)}_t\\
    & \frac{du^{(2)}_t}{dt} = \kappa_2 x_t u^{(1)}_t - \kappa_1 u^{(2)}_t\\
    & \frac{du^{(3)}_t}{dt} = \kappa_1 (u^{(2)}_t - u^{(3)}_t)\\
    & \frac{du^{(4)}_t}{dt} = \kappa_1 (u^{(3)}_t - u^{(4)}_t)\\
    & w_t = u^{(1)}_t + u^{(2)}_t + u^{(3)}_t + u^{(4)}_t,
\end{align*}
where $w_t$ is trajectory of each individual with initial status of $(u^{(1)}_0, u^{(2)}_0, u^{(3)}_0, u^{(4)}_0) = (0.8, 0, 0, 0)$ and $(\lambda_0, \lambda_1, \kappa_1, \kappa_2, \Psi) \in \mathbb{R}^5$ are fixed parameters. In our experiment, the parameters are chosen to be $\lambda_0=0.9$, $\lambda_1=0.7$, $\kappa_1=10$, $\kappa_2=0.15$ and $\Psi=20$.  We then define the time-of-event as the time when trajectory cross the threshold $w_\star \in \mathbb{R}$ during the observation period $[t_0 \, \, t_N]$, which is
\begin{equation*}
    T^\star = \min \{ t_0 \leq t \leq t_N  \,|\, w_t \geq w_\star\}.
\end{equation*}
In our experiments, we use the threshold value $w_\star=1.7$. The target differential equations are simulated using an Euler discretization. We train on $n=500$ individuals.

\paragraph{Censorship.} Similarly to the partially observed SDE experiment, we consider terminal censorship: individuals that do not experience the event within the observation period are censored. The censoring level is 8.4$\%$.

\paragraph{Supplementary Figures.} Figure \ref{fig:tumor_growth_appendix_path_hist} provides an example of the full sample path of an individual and the distribution of the event times of the whole population. We add additional results on the test set in Figures \ref{fig:c_index_tumor_growth}, \ref{fig:bs_tumor_growth}, \ref{fig:wbs_tumor_growth} and \ref{fig:auc_tumor_growth}.

\begin{figure*}
    \centering
    \includegraphics[width=0.49\textwidth]{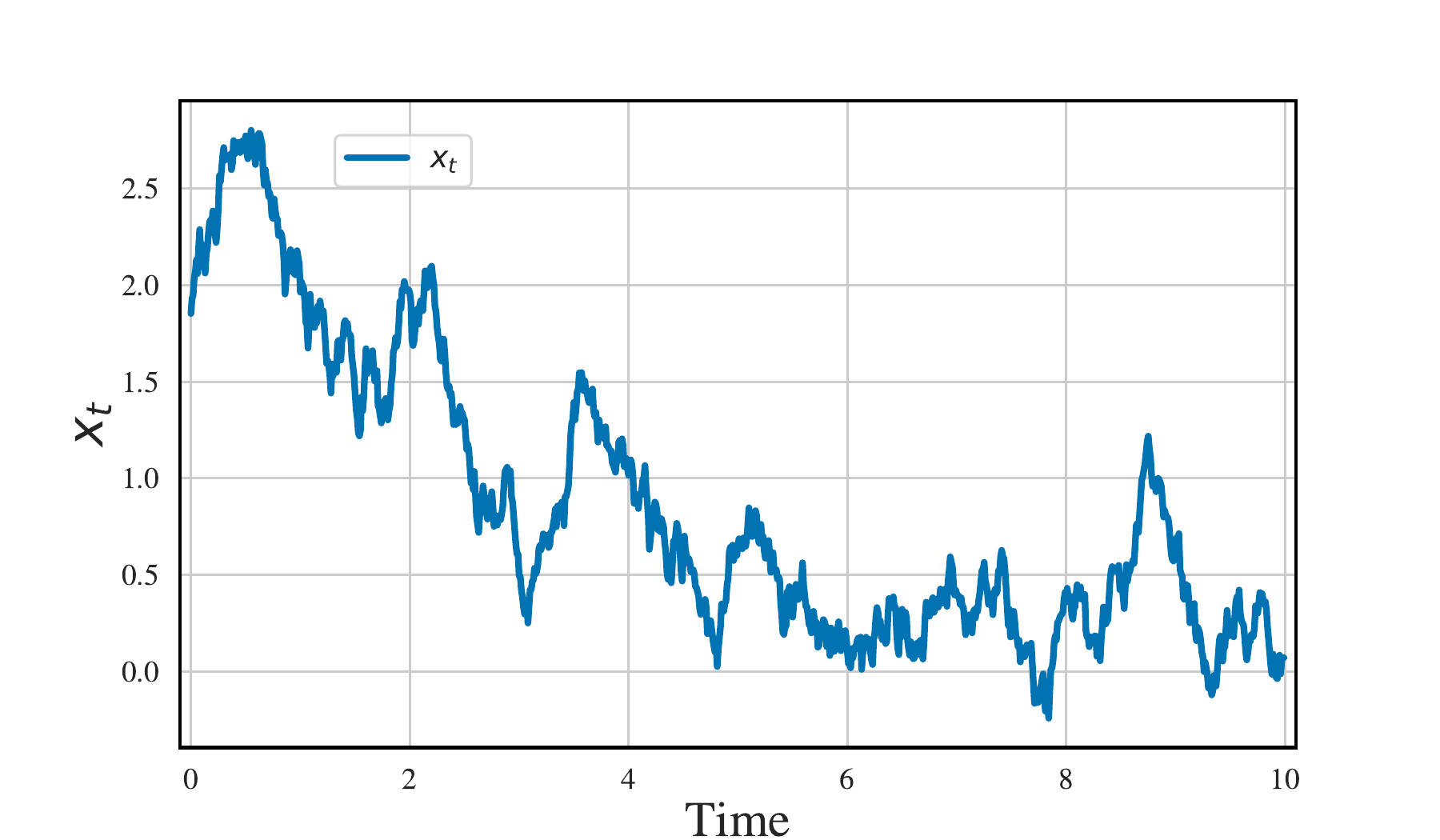}
    \includegraphics[width=0.49\textwidth]{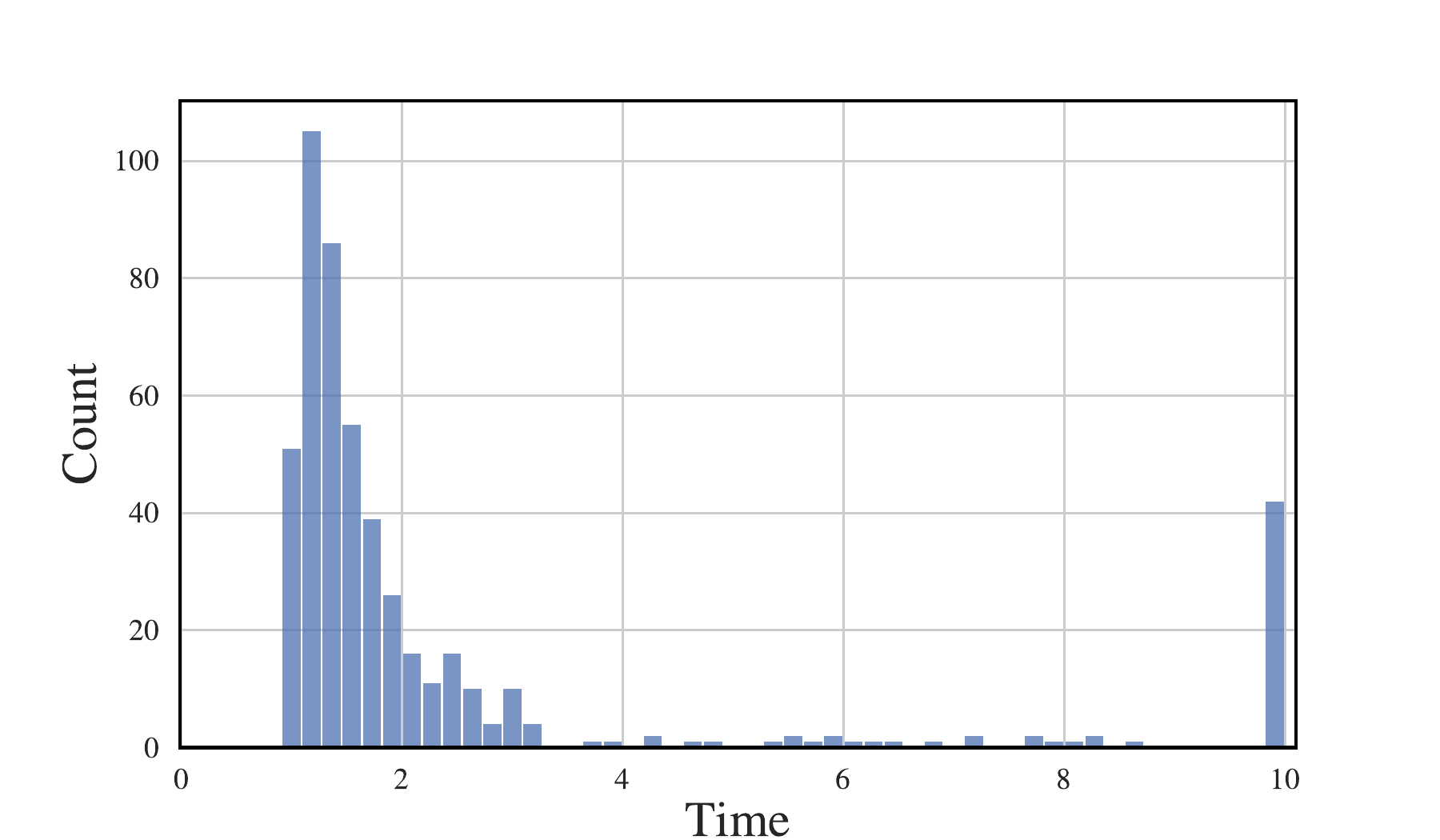}
    
    \caption{\footnotesize Full sample path of an individual (\textbf{left}) and distribution of the event times (\textbf{left}) for the tumor growth experiment. The surge in events at the terminal time indicates terminal censorship.}
    \label{fig:tumor_growth_appendix_path_hist}
\end{figure*}

\begin{figure*}[h!]
    \centering
    \includegraphics[width=0.33\textwidth]{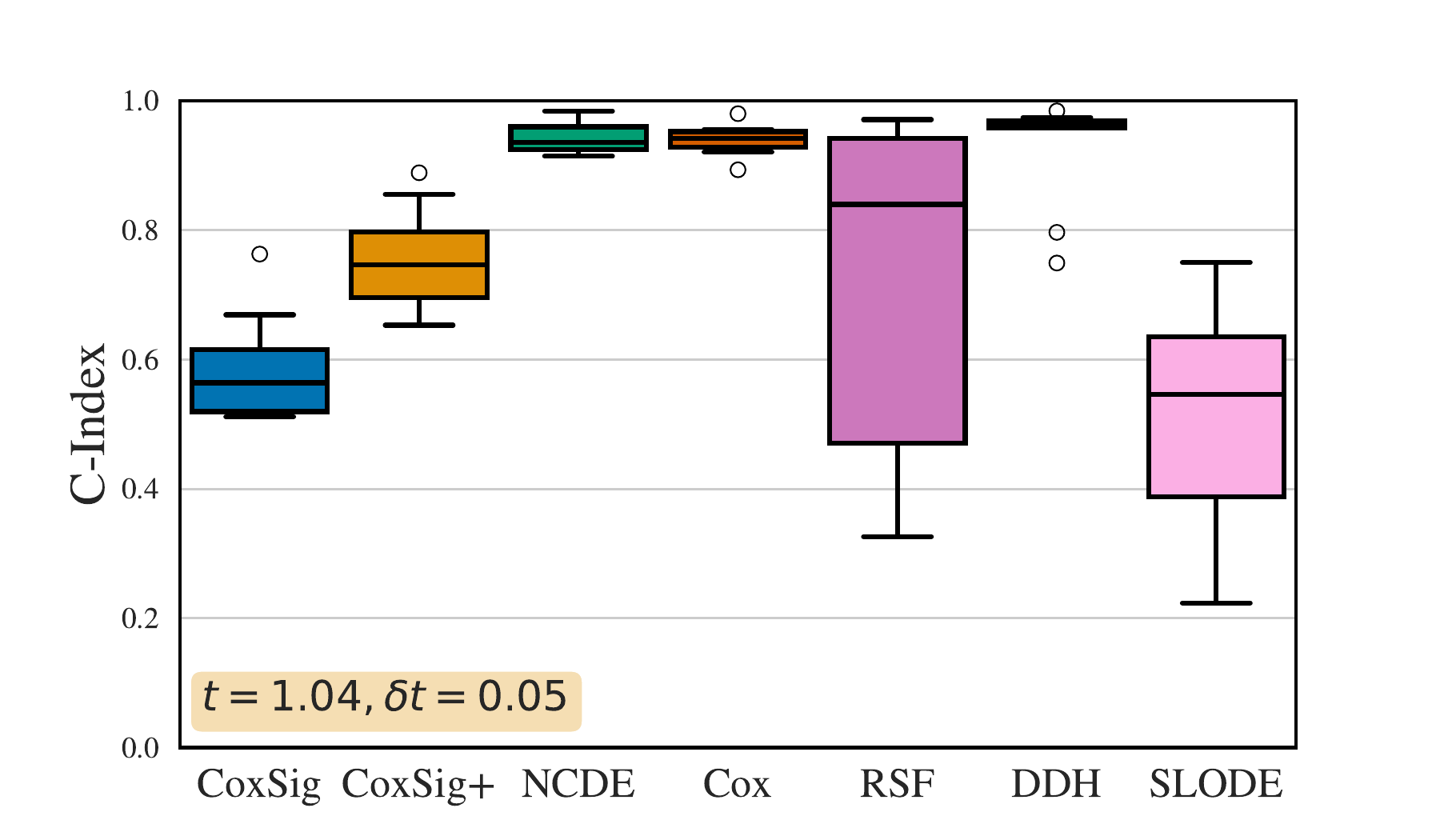}
    \includegraphics[width=0.33\textwidth]{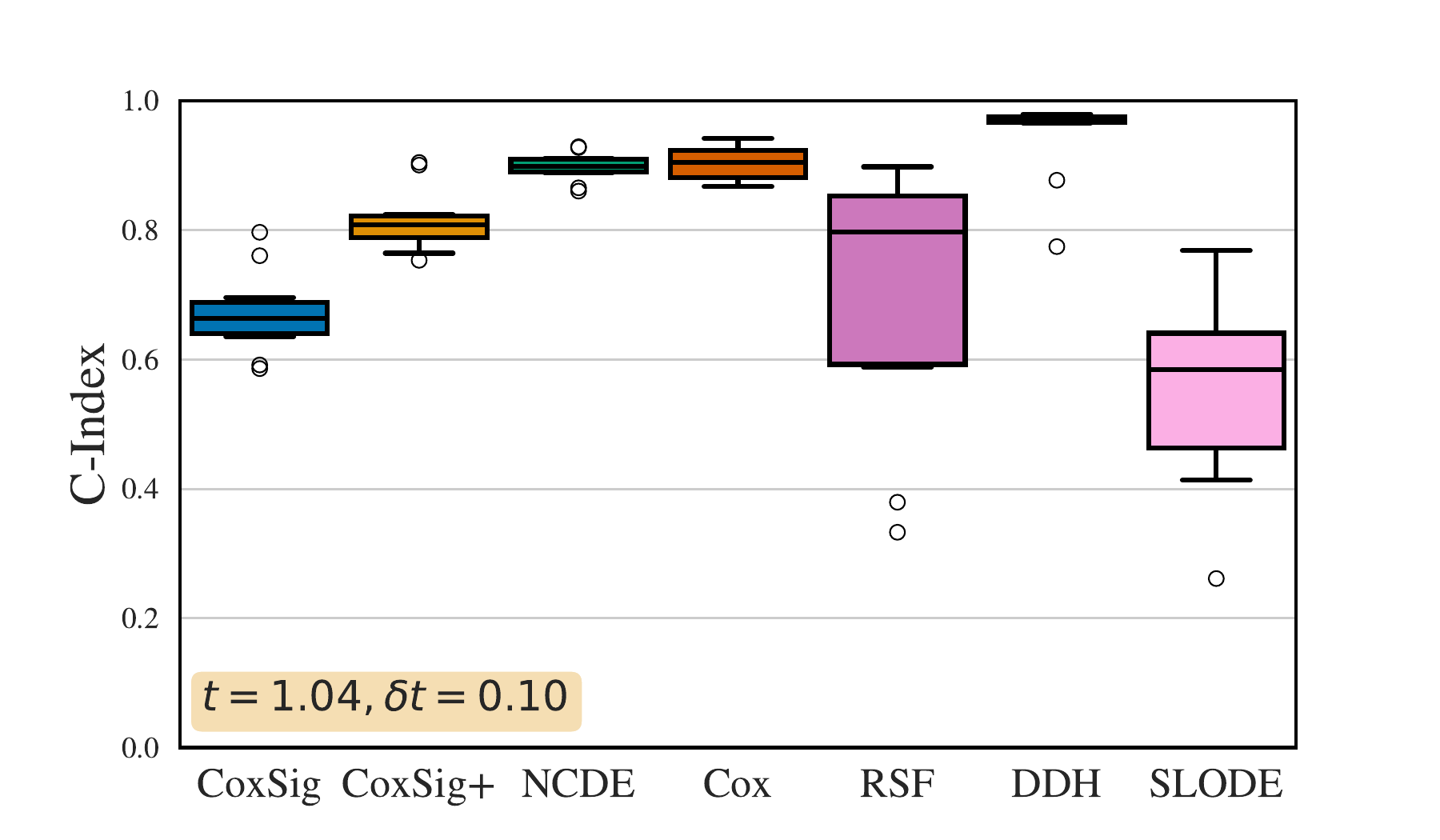}
    \includegraphics[width=0.33\textwidth]{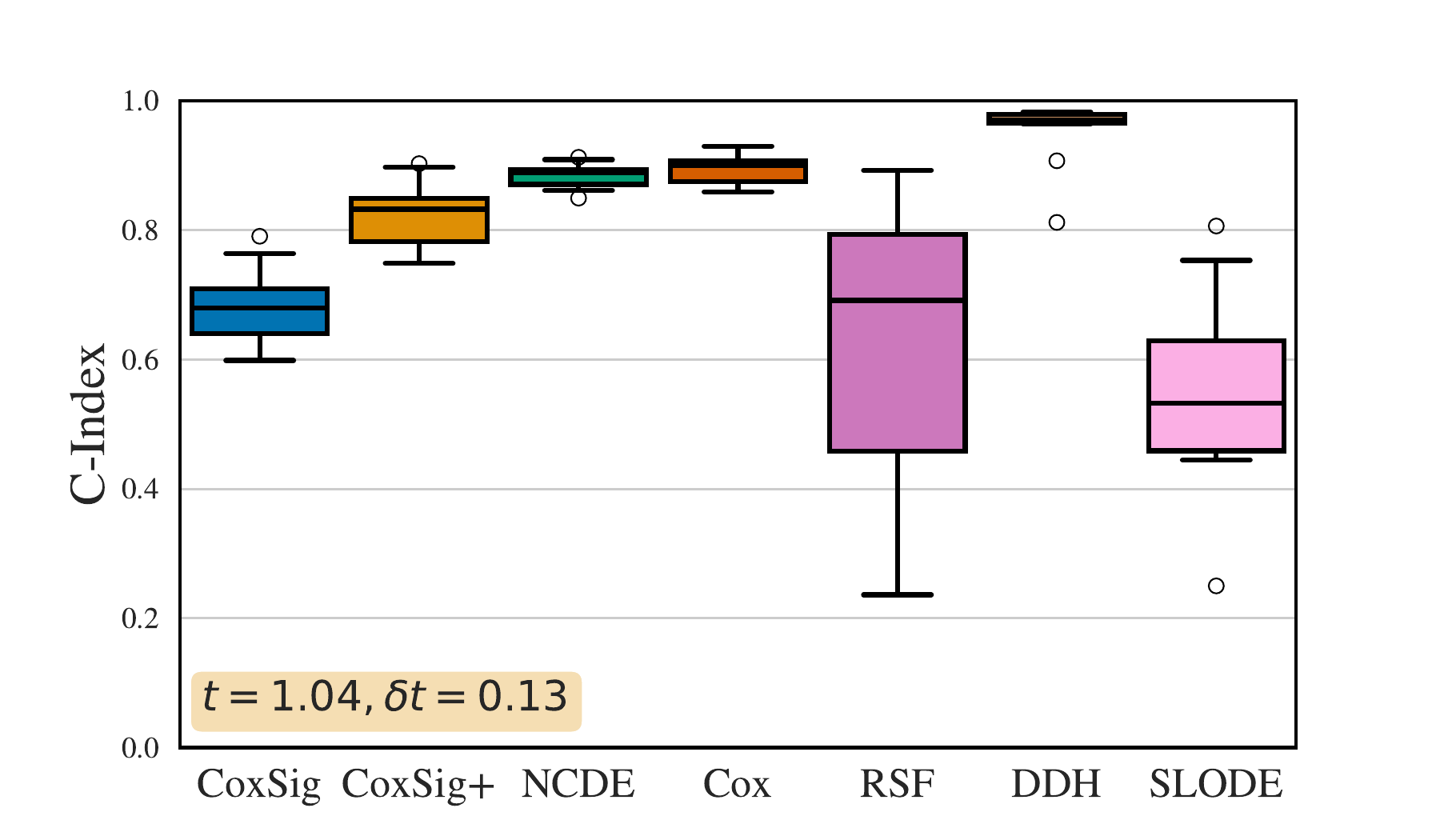}
    \includegraphics[width=0.33\textwidth]{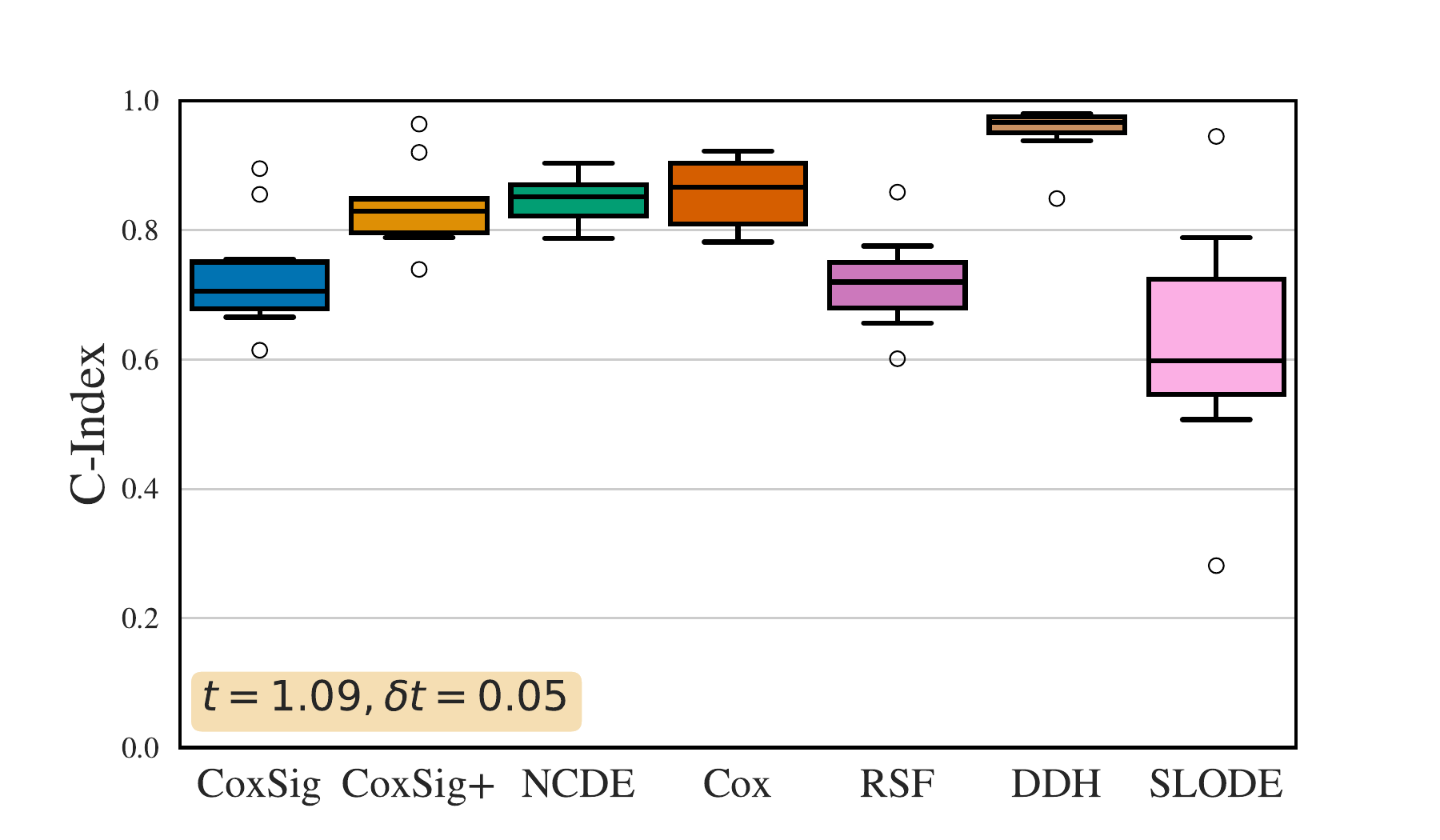}
    \includegraphics[width=0.33\textwidth]{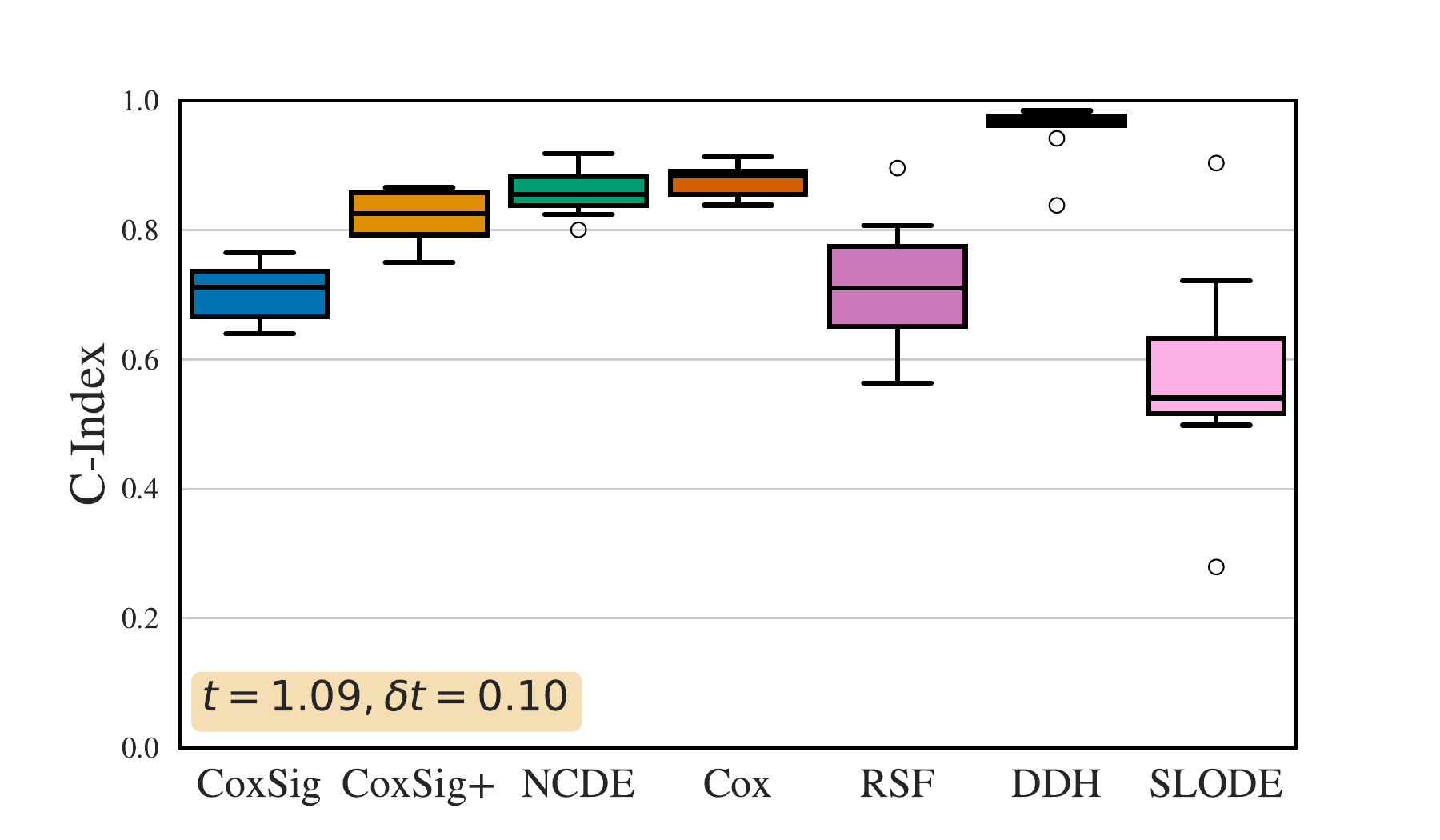}
    \includegraphics[width=0.33\textwidth]{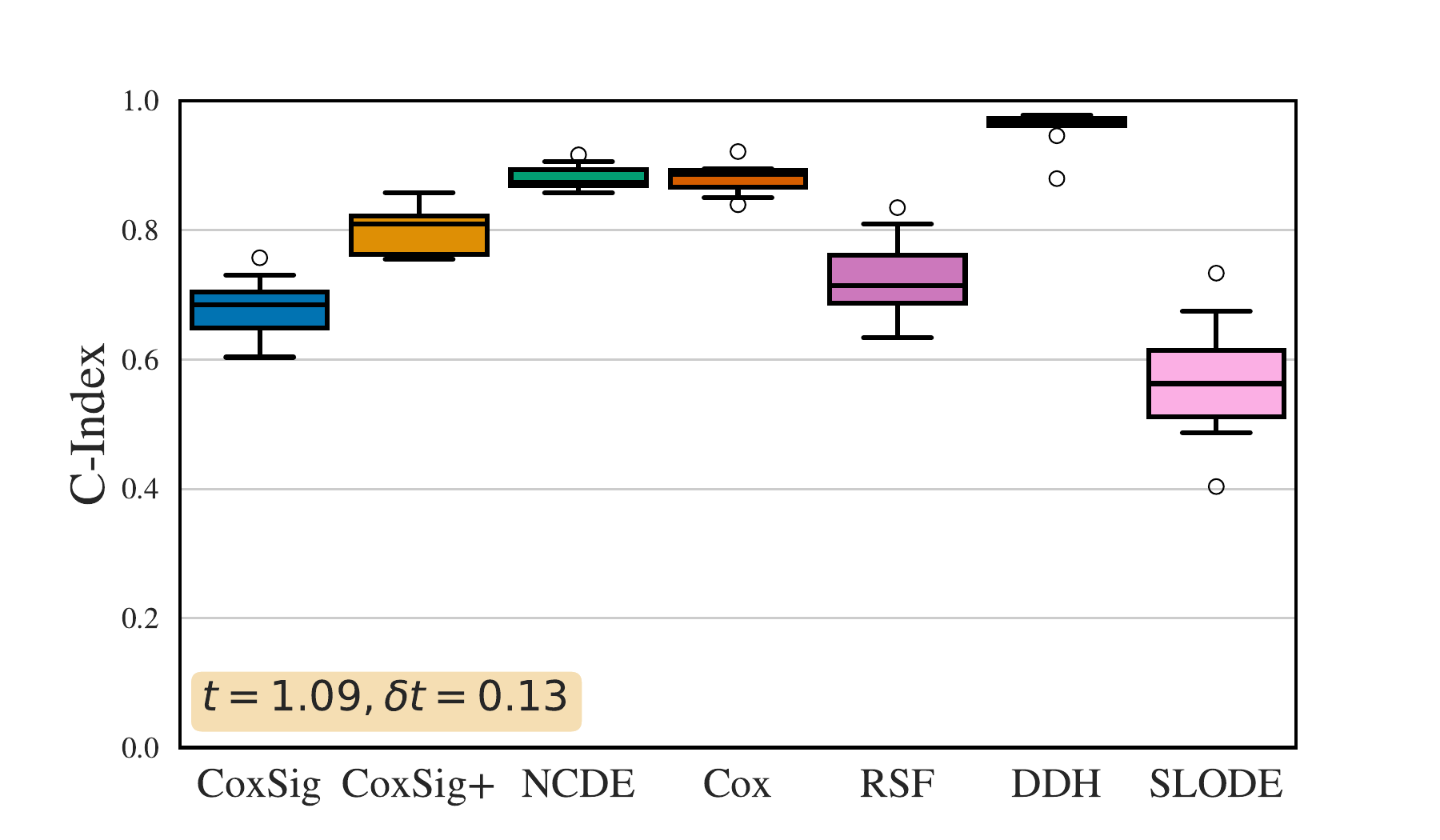}
    \includegraphics[width=0.33\textwidth]{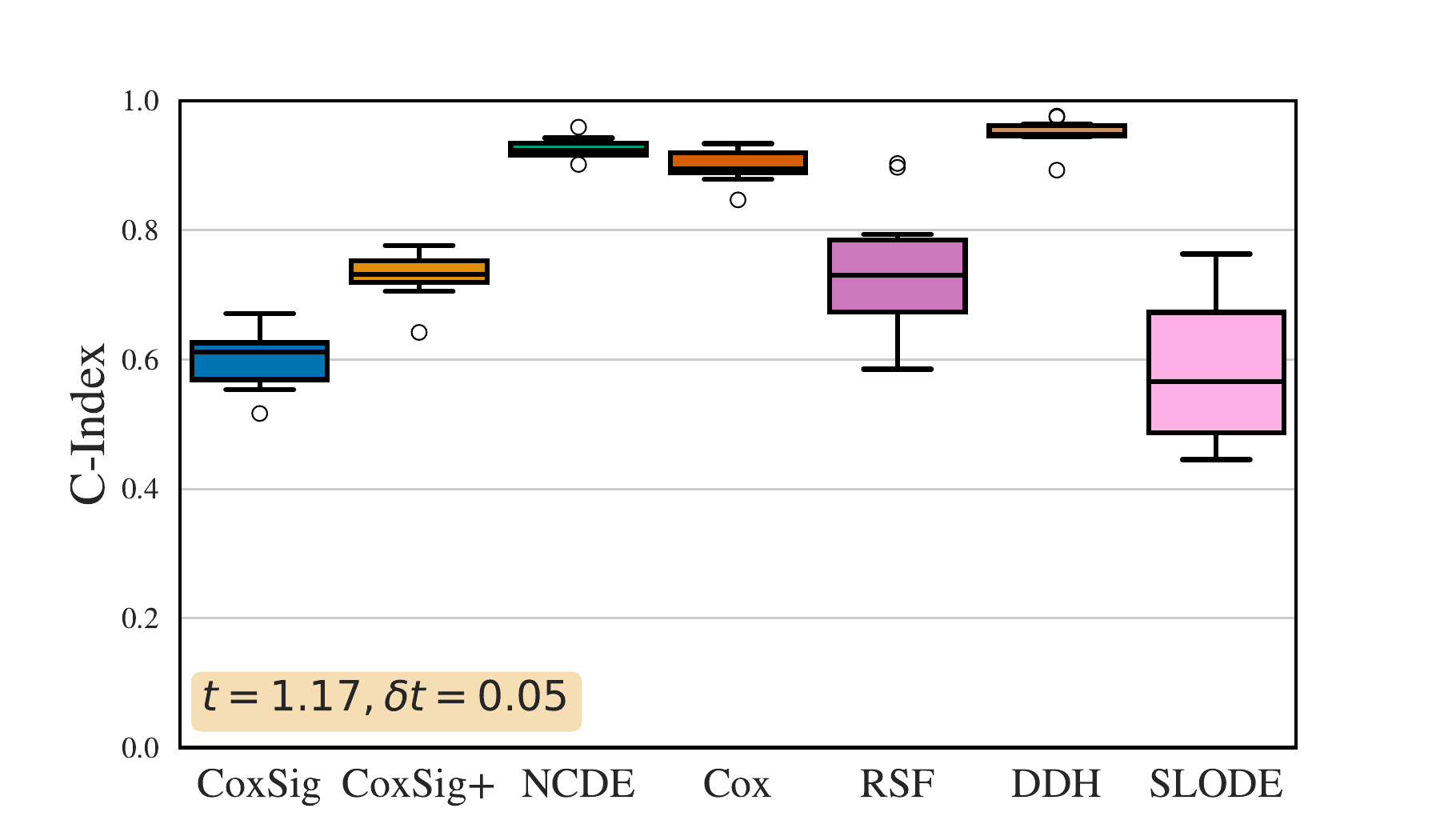}
    \includegraphics[width=0.33\textwidth]{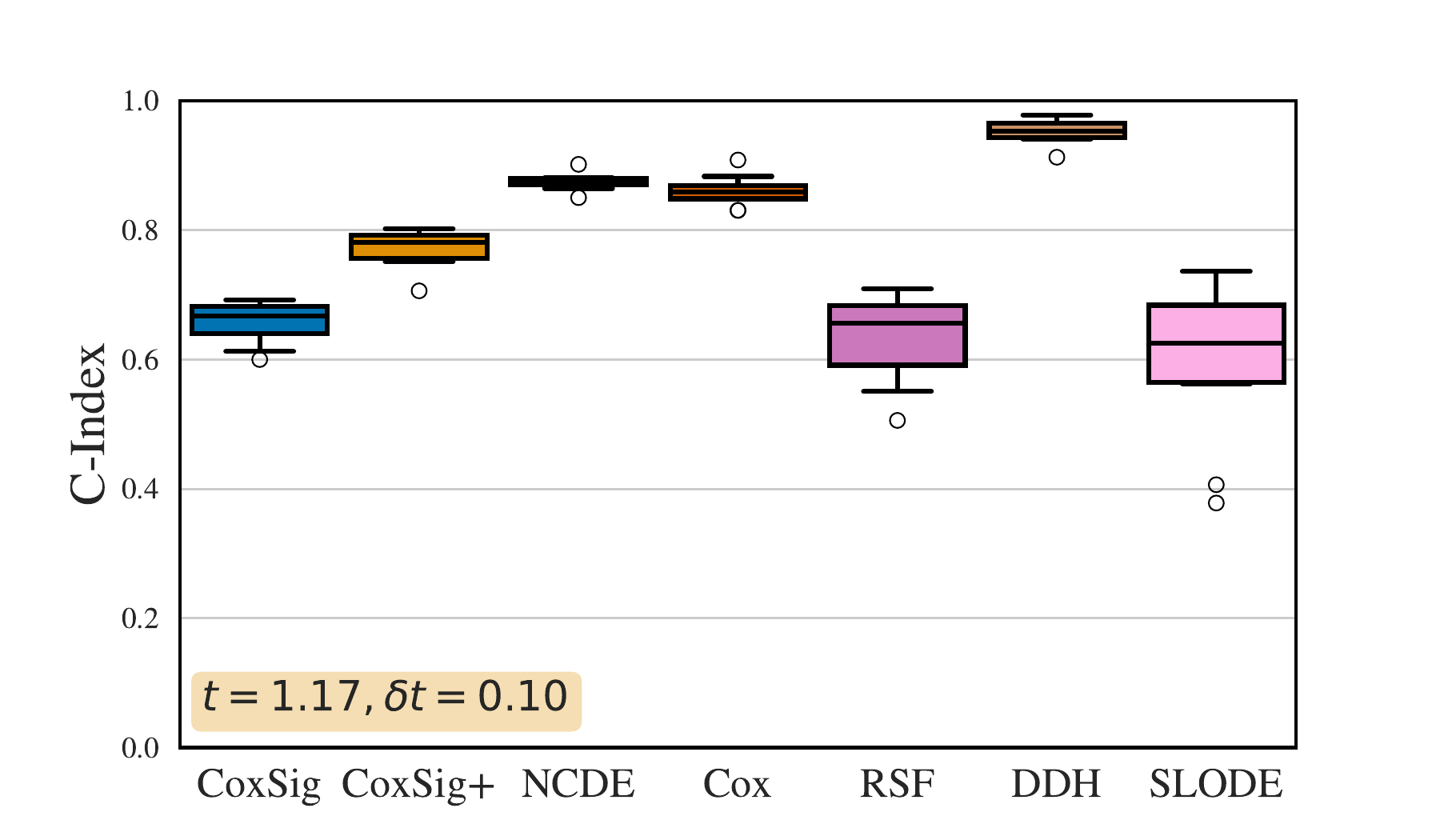}
    \includegraphics[width=0.33\textwidth]{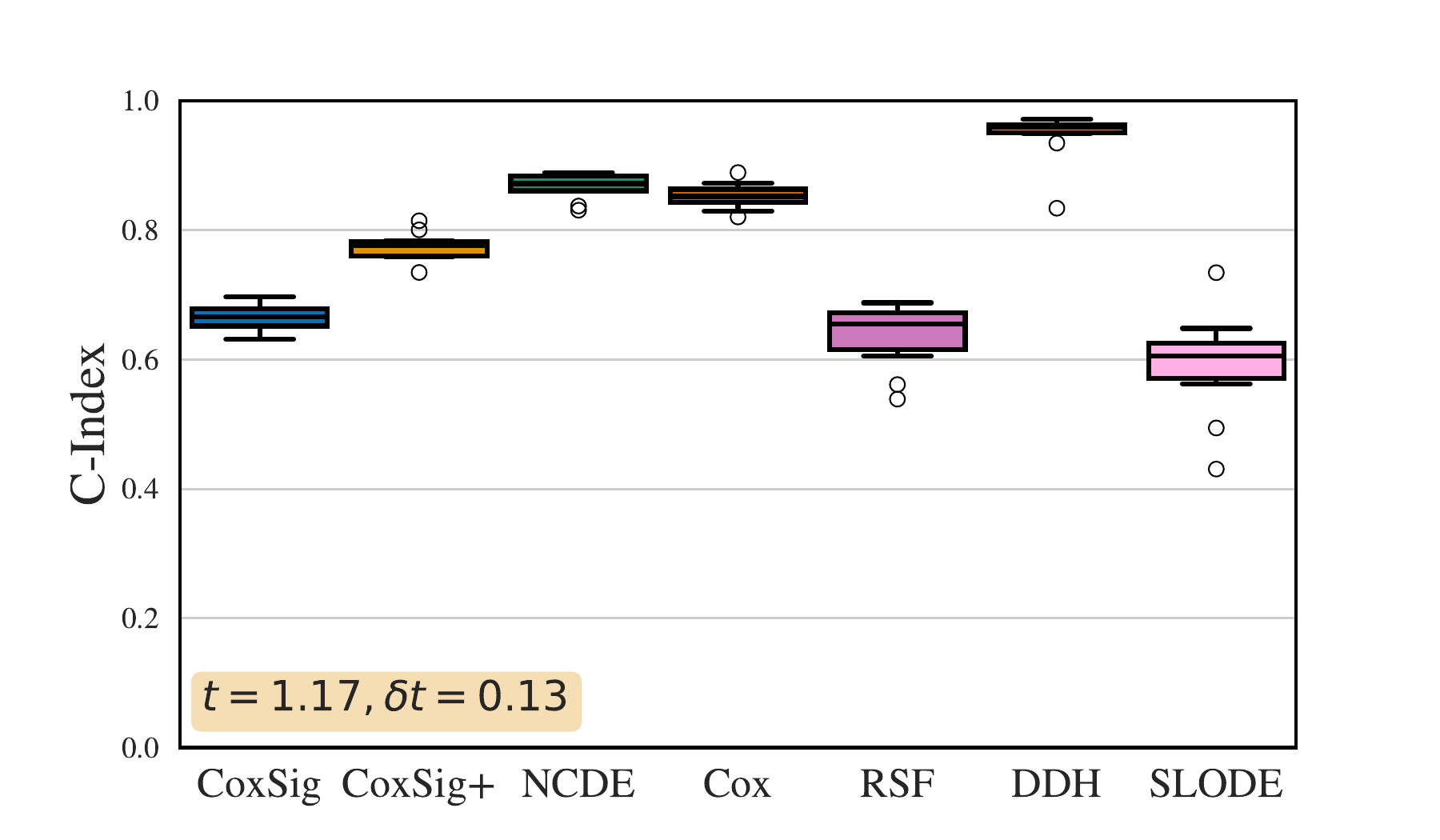}
    \caption{\footnotesize C-Index (\textit{higher} is better) for \textbf{Tumor Growth} for numerous points $(t,\delta t)$.}
    \label{fig:c_index_tumor_growth}
\end{figure*}

\begin{figure*}
    \centering
    \includegraphics[width=0.33\textwidth]{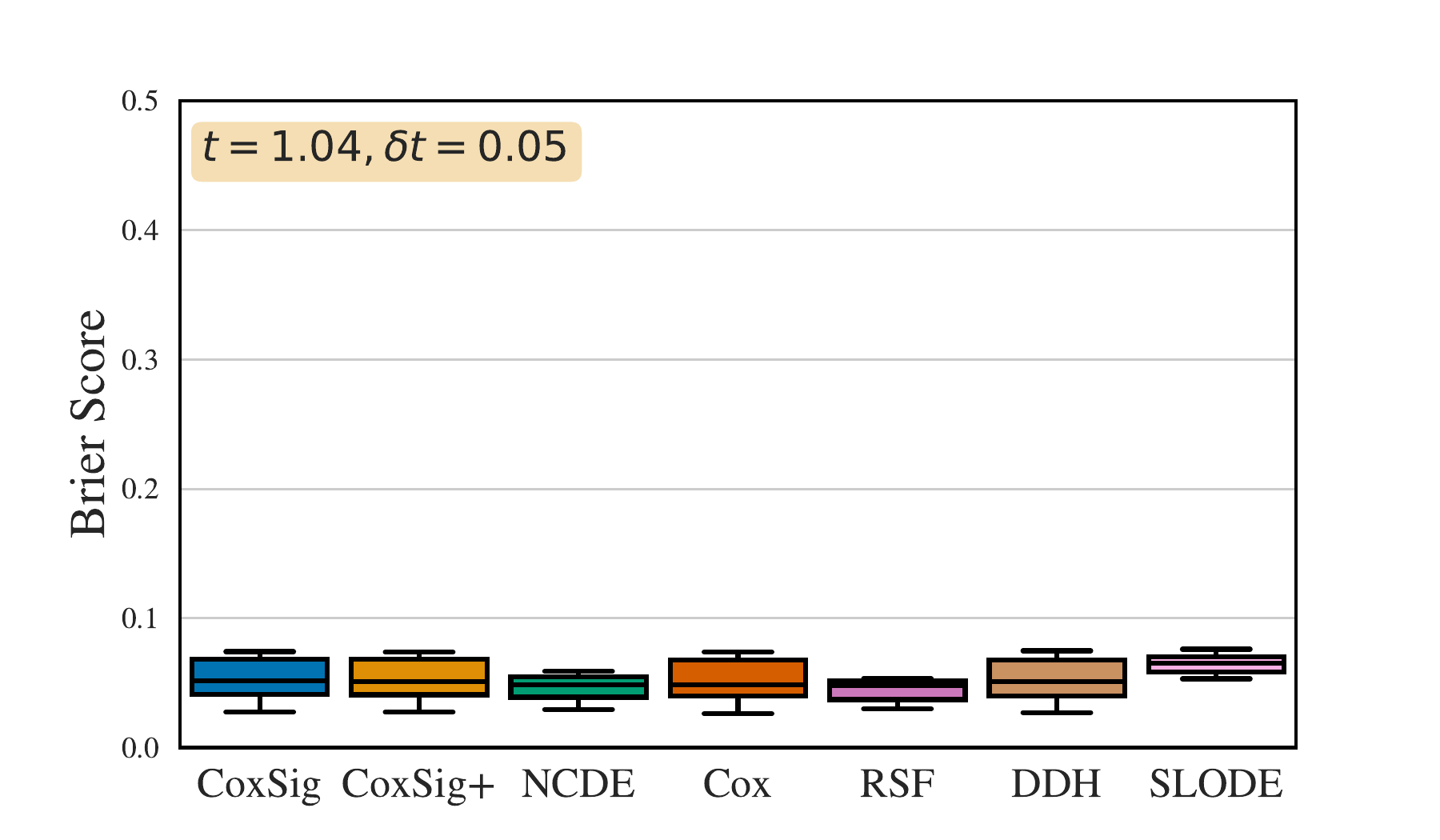}
    \includegraphics[width=0.33\textwidth]{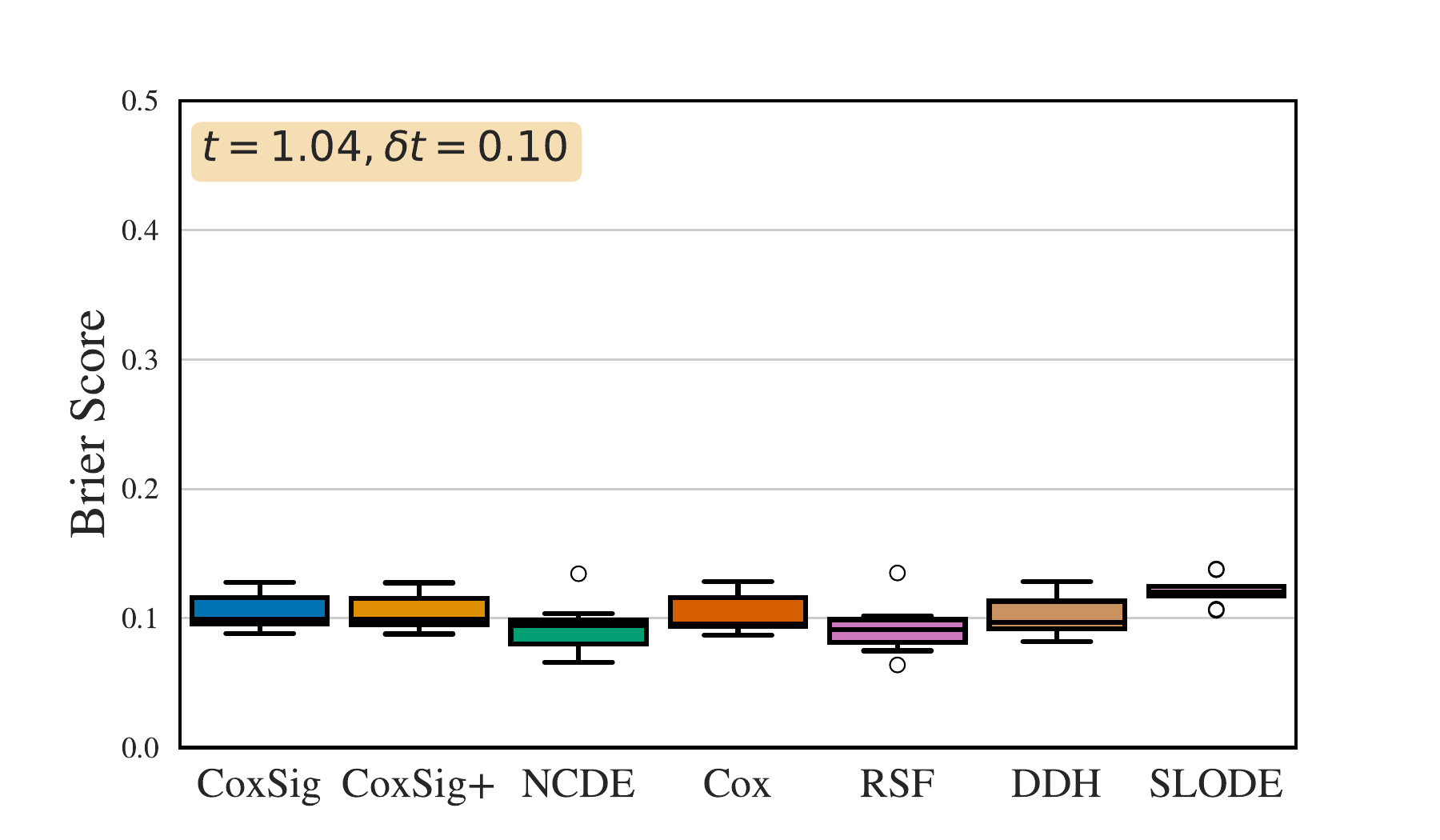}
    \includegraphics[width=0.33\textwidth]{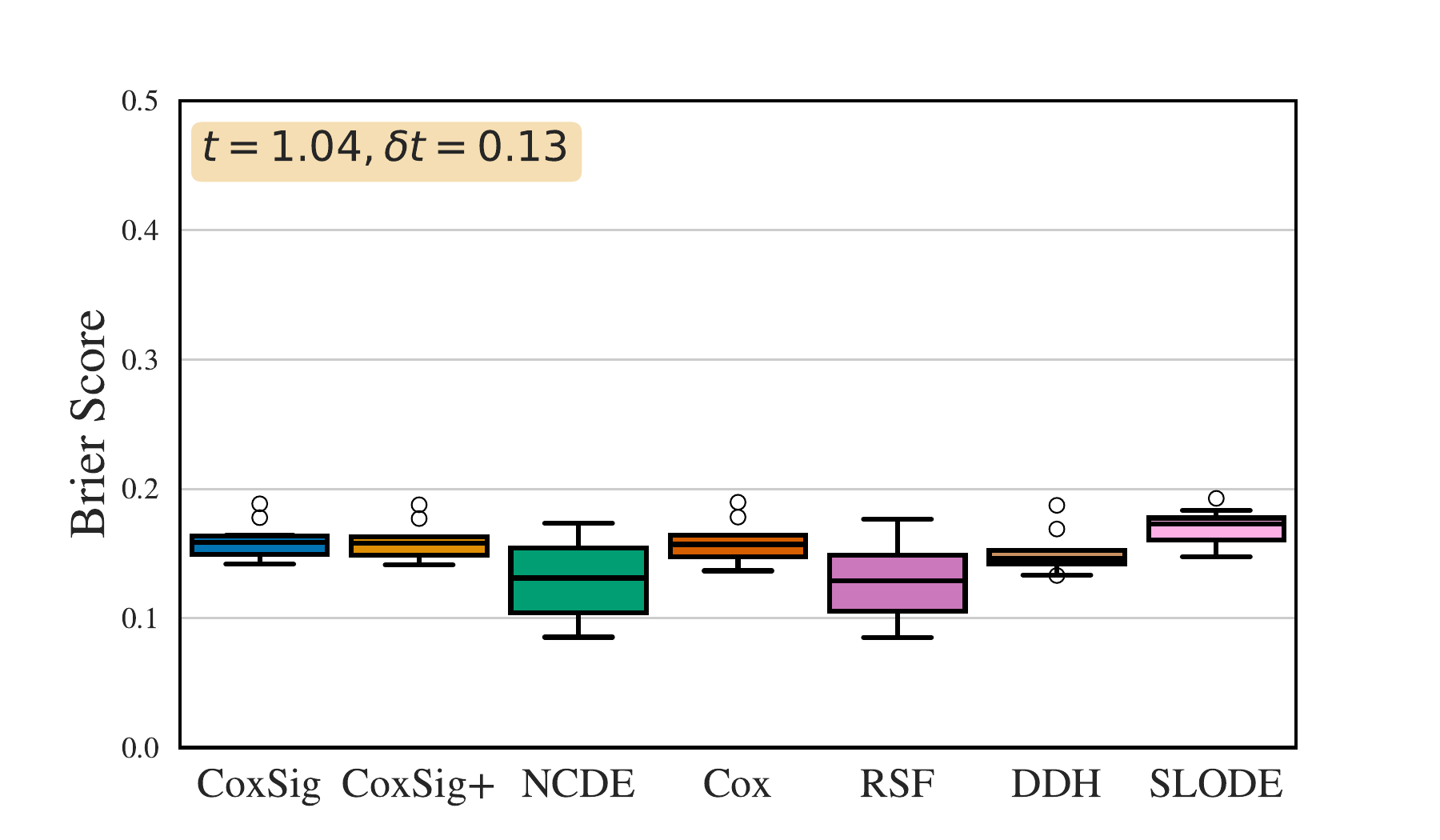}
    \includegraphics[width=0.33\textwidth]{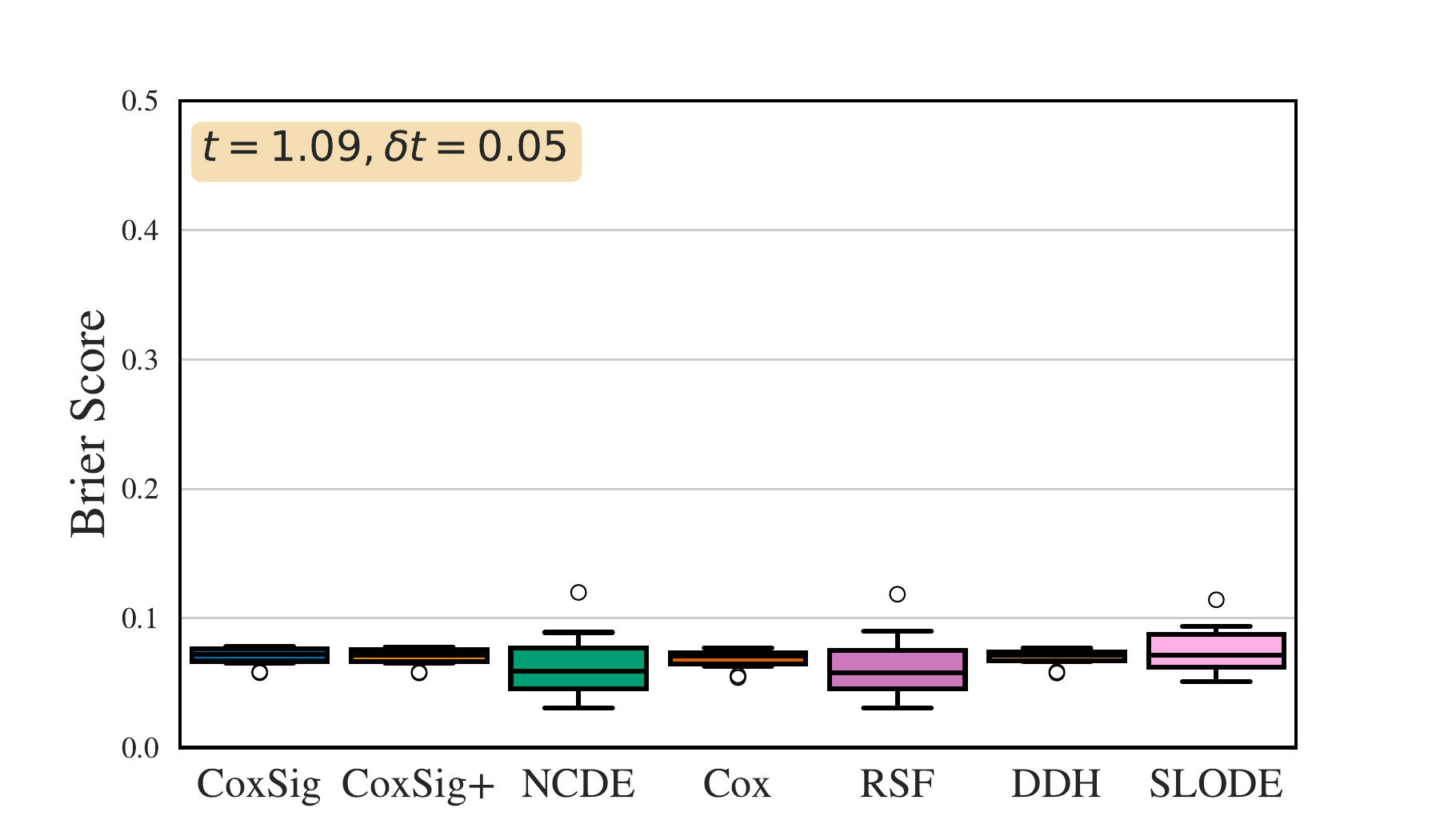}
    \includegraphics[width=0.33\textwidth]{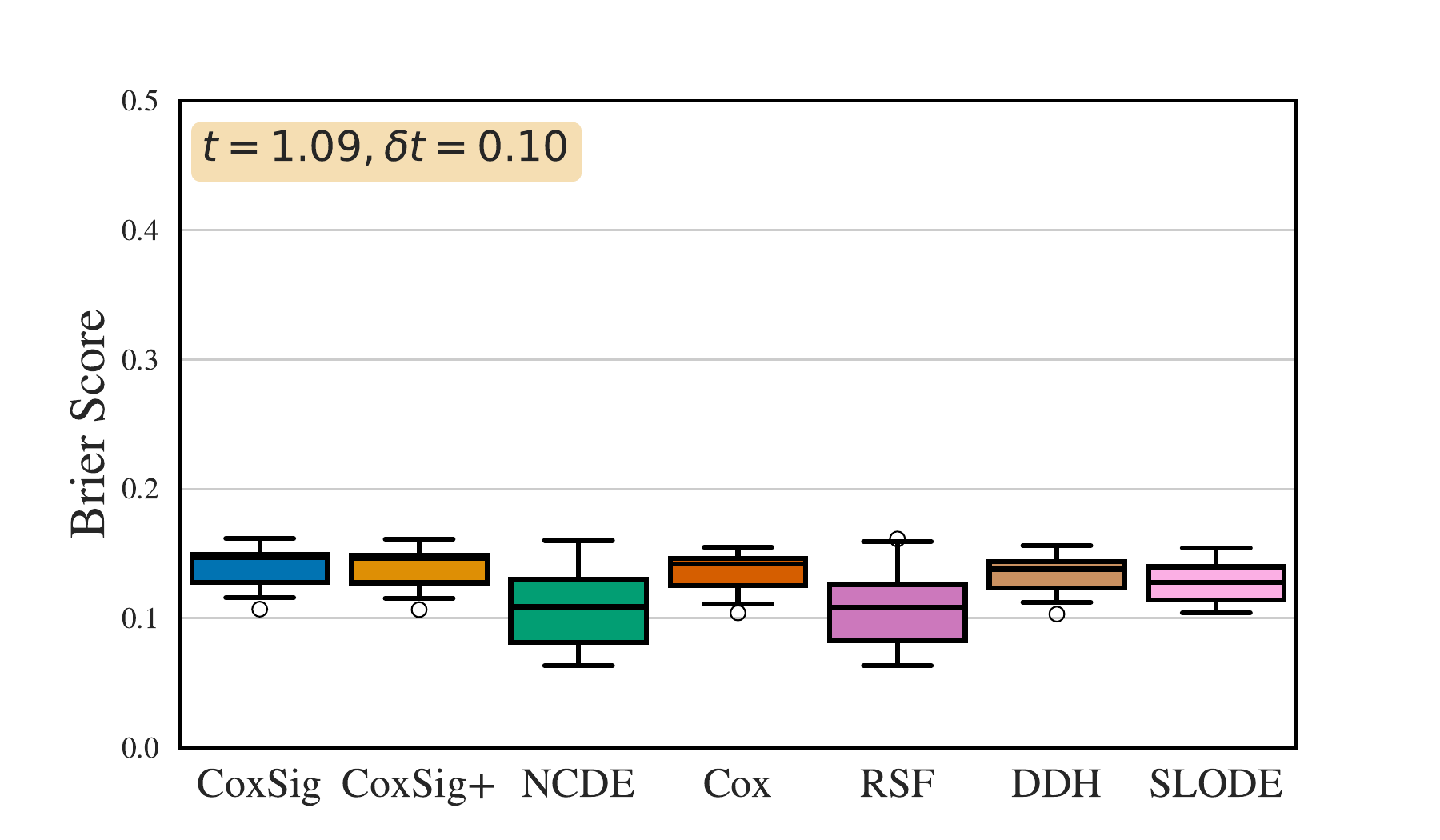}
    \includegraphics[width=0.33\textwidth]{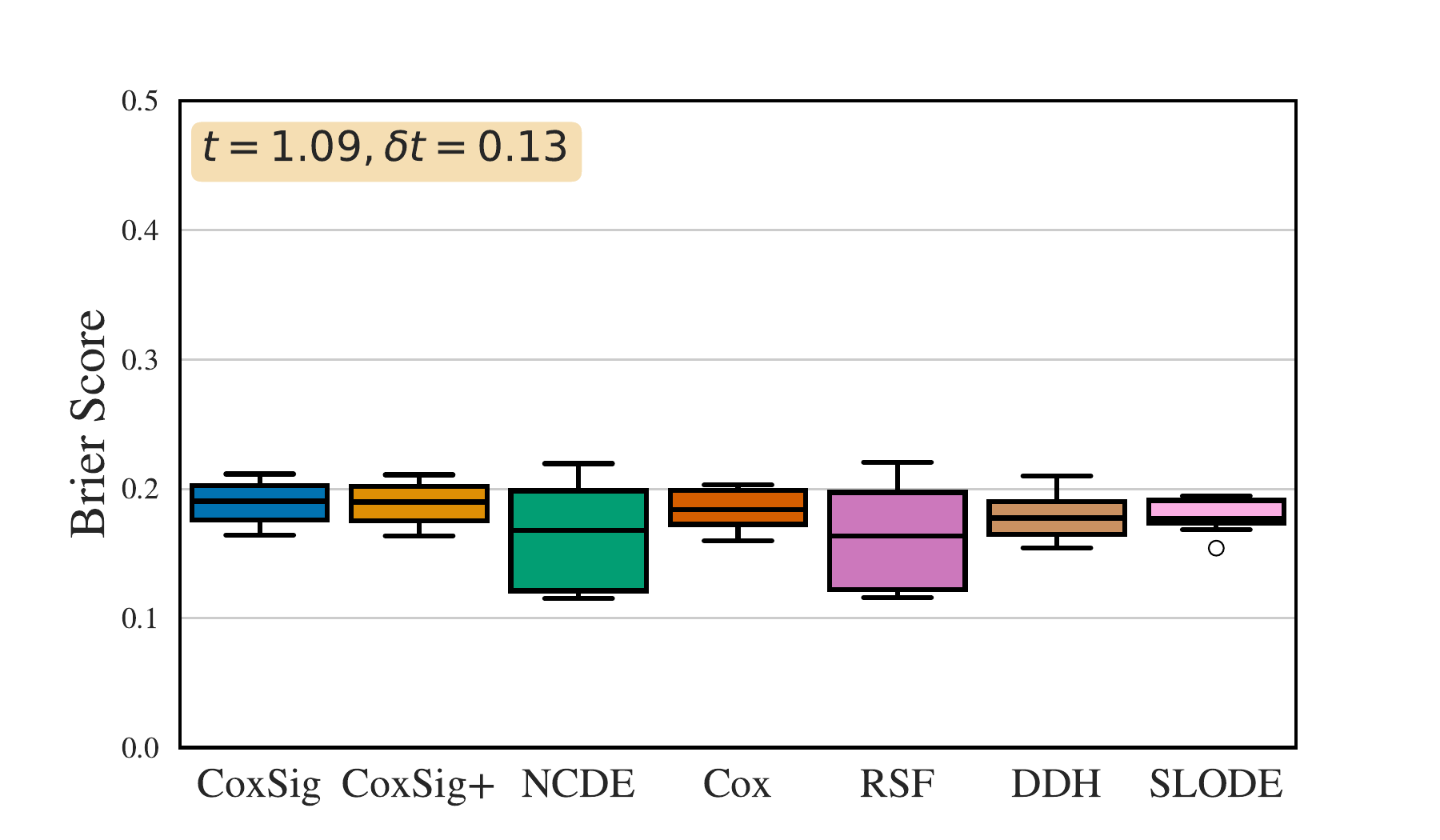}
    \includegraphics[width=0.33\textwidth]{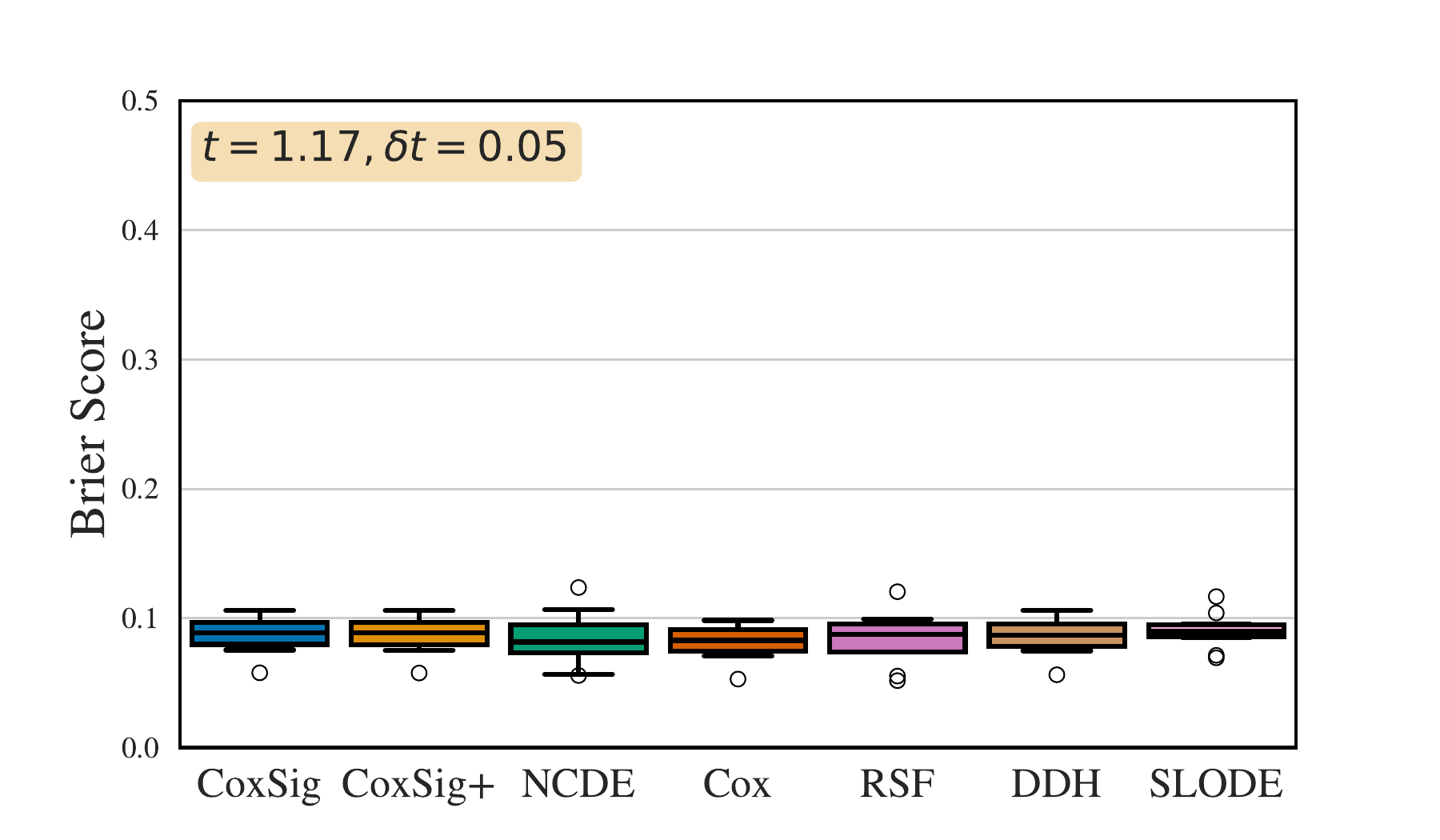}
    \includegraphics[width=0.33\textwidth]{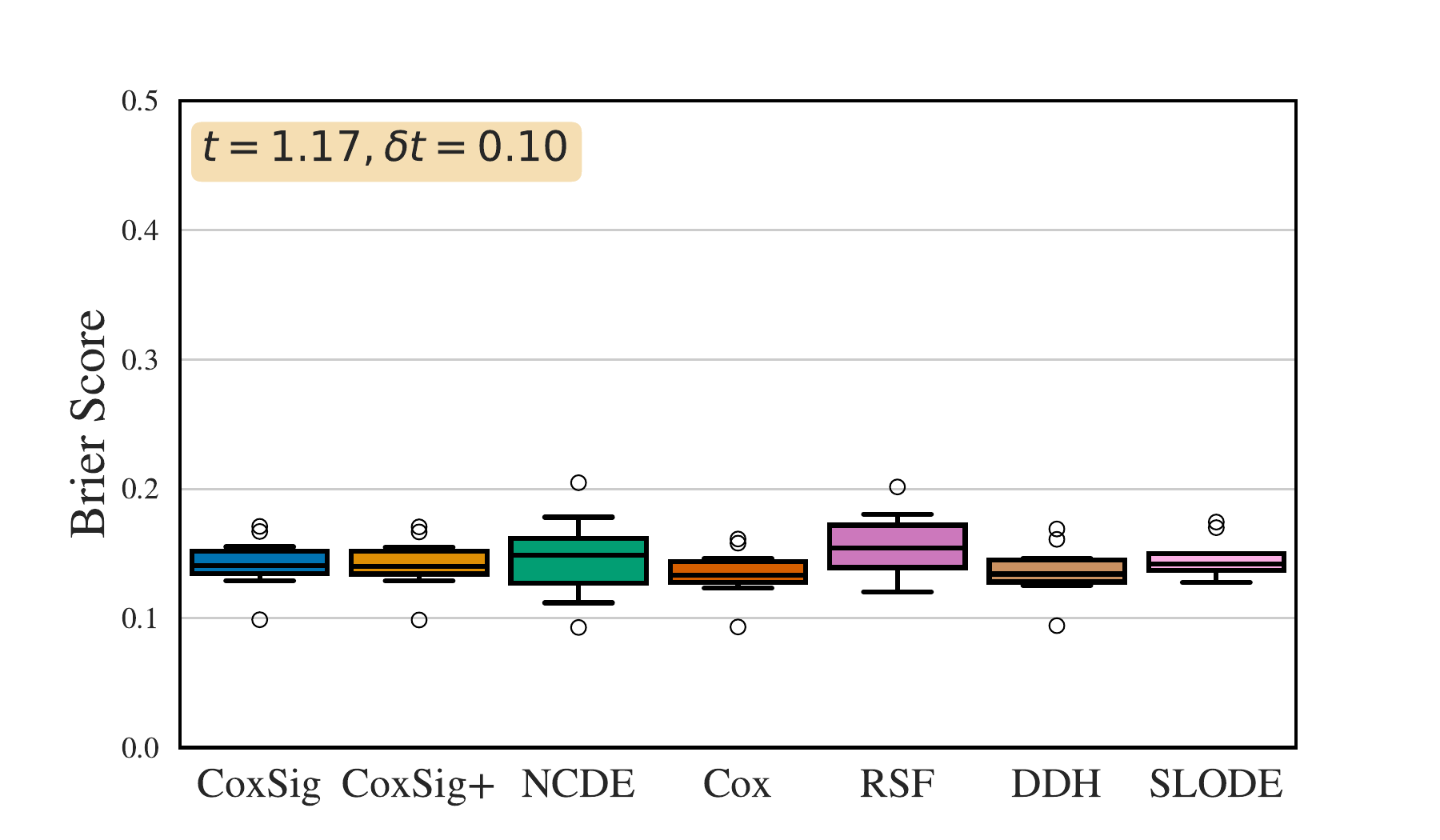}
    \includegraphics[width=0.33\textwidth]{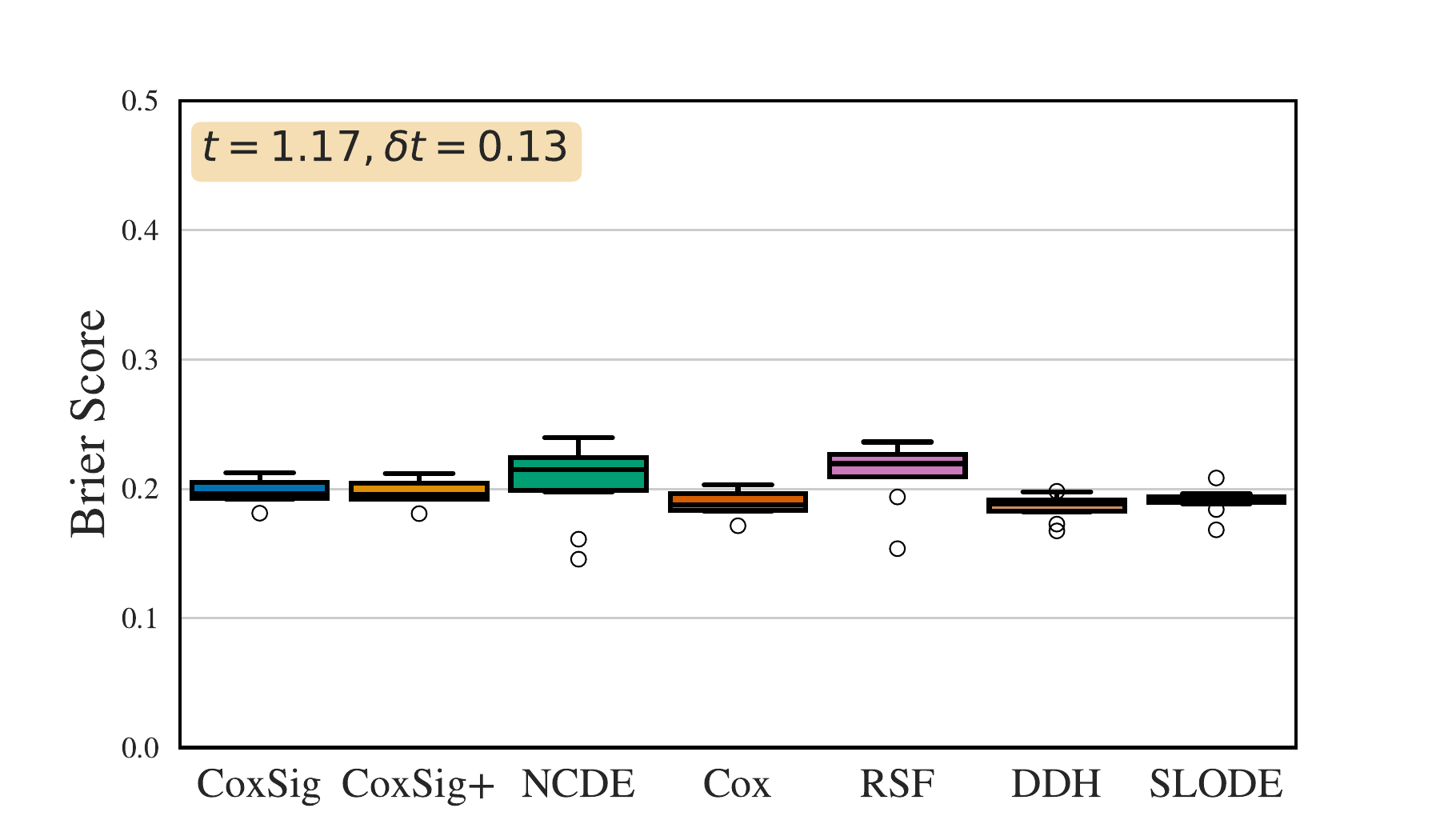}
    \caption{\footnotesize Brier score (\textit{lower} is better) for \textbf{Tumor Growth} for numerous points $(t,\delta t)$.}
    \label{fig:bs_tumor_growth}
\end{figure*}

\begin{figure*}
    \centering
    \includegraphics[width=0.33\textwidth]{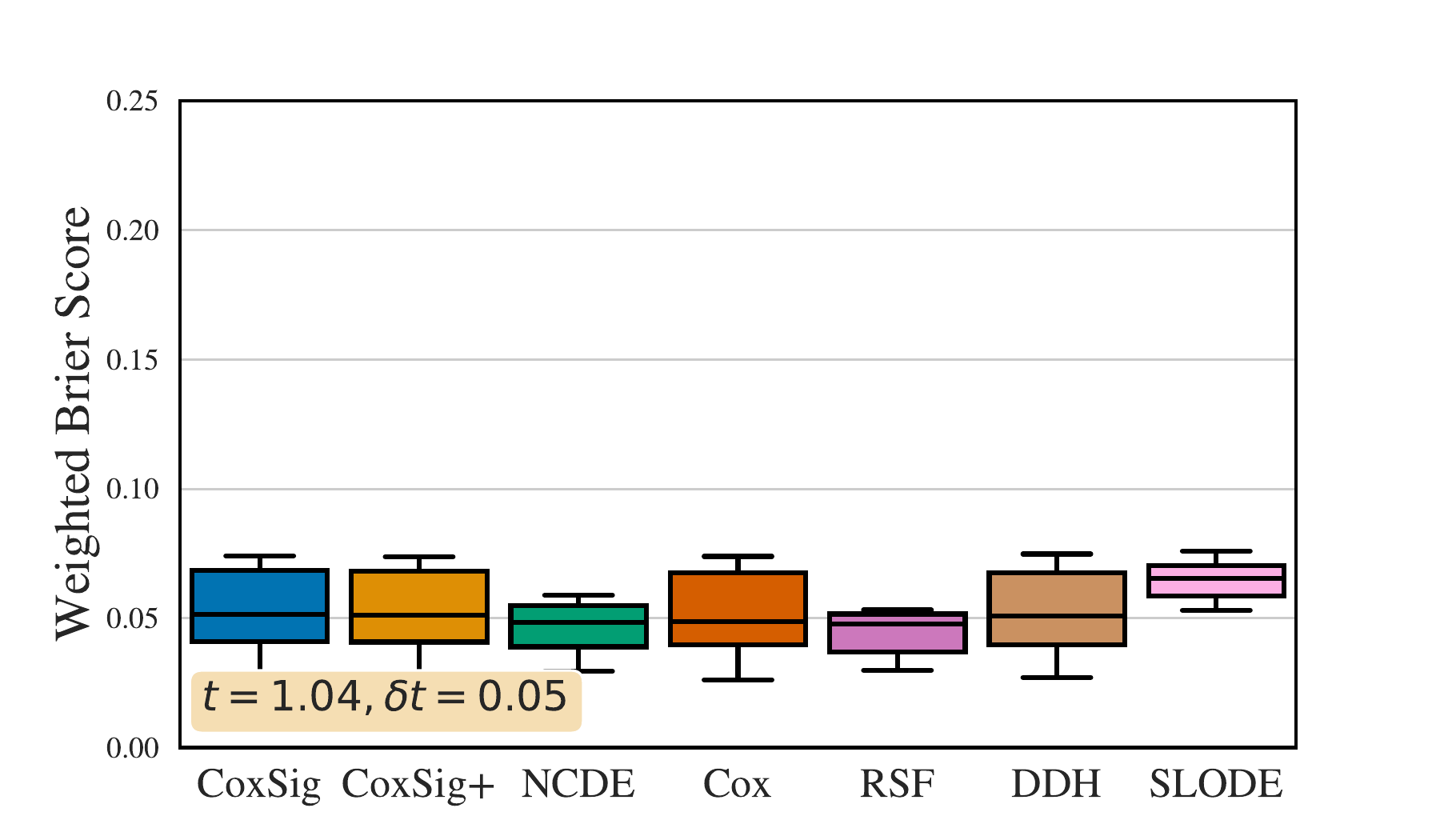}
    \includegraphics[width=0.33\textwidth]{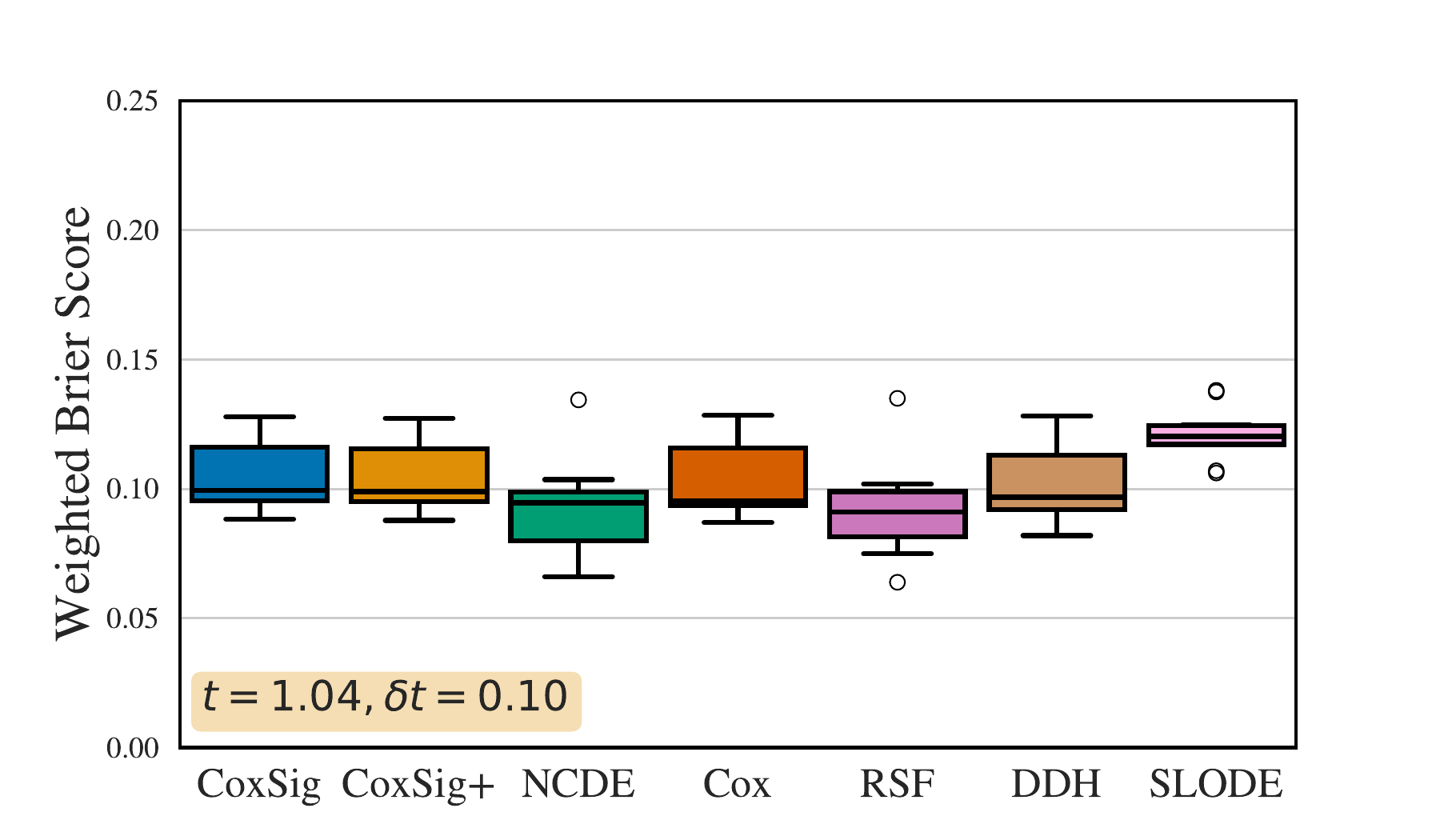}
    \includegraphics[width=0.33\textwidth]{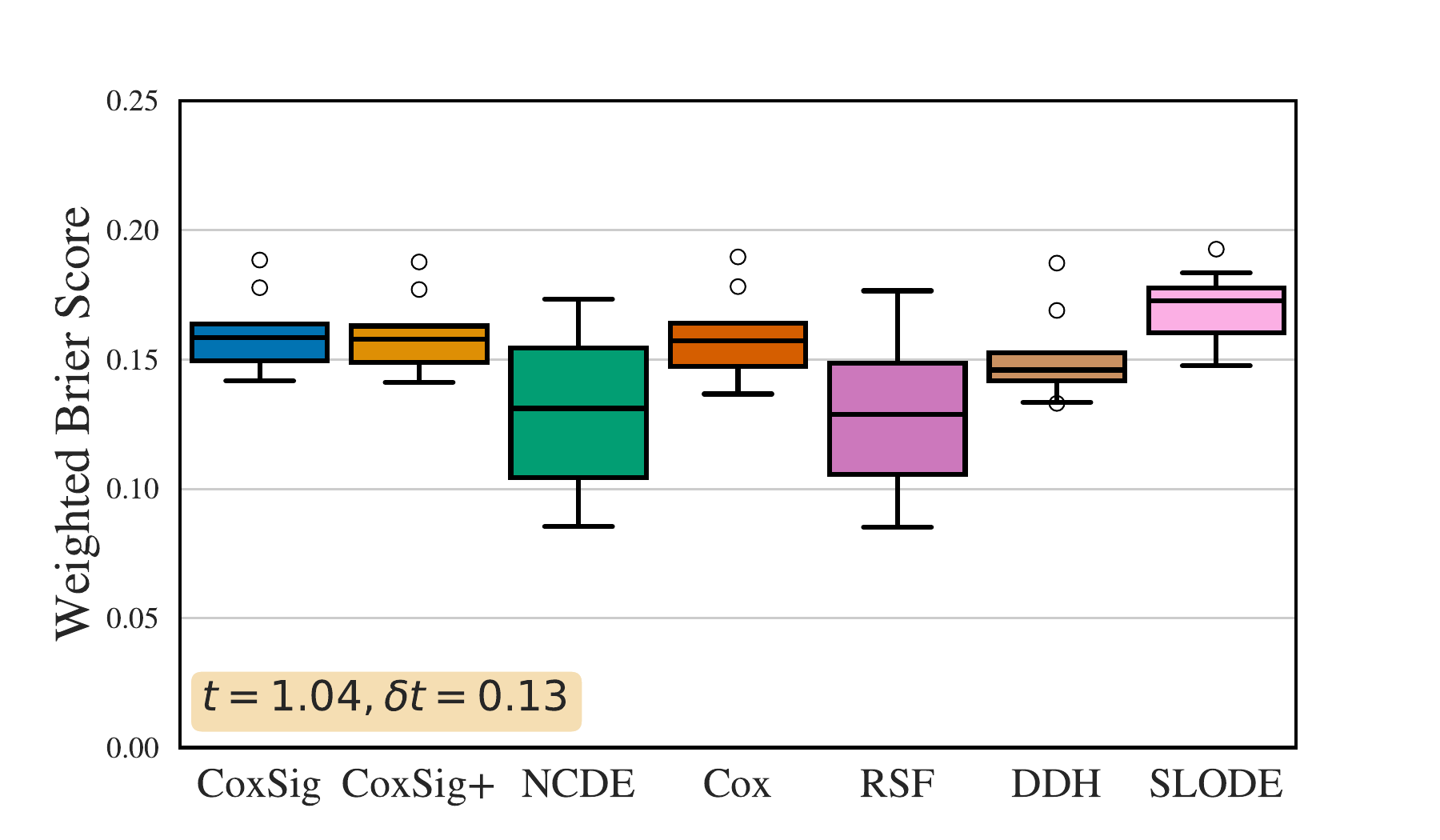}
    \includegraphics[width=0.33\textwidth]{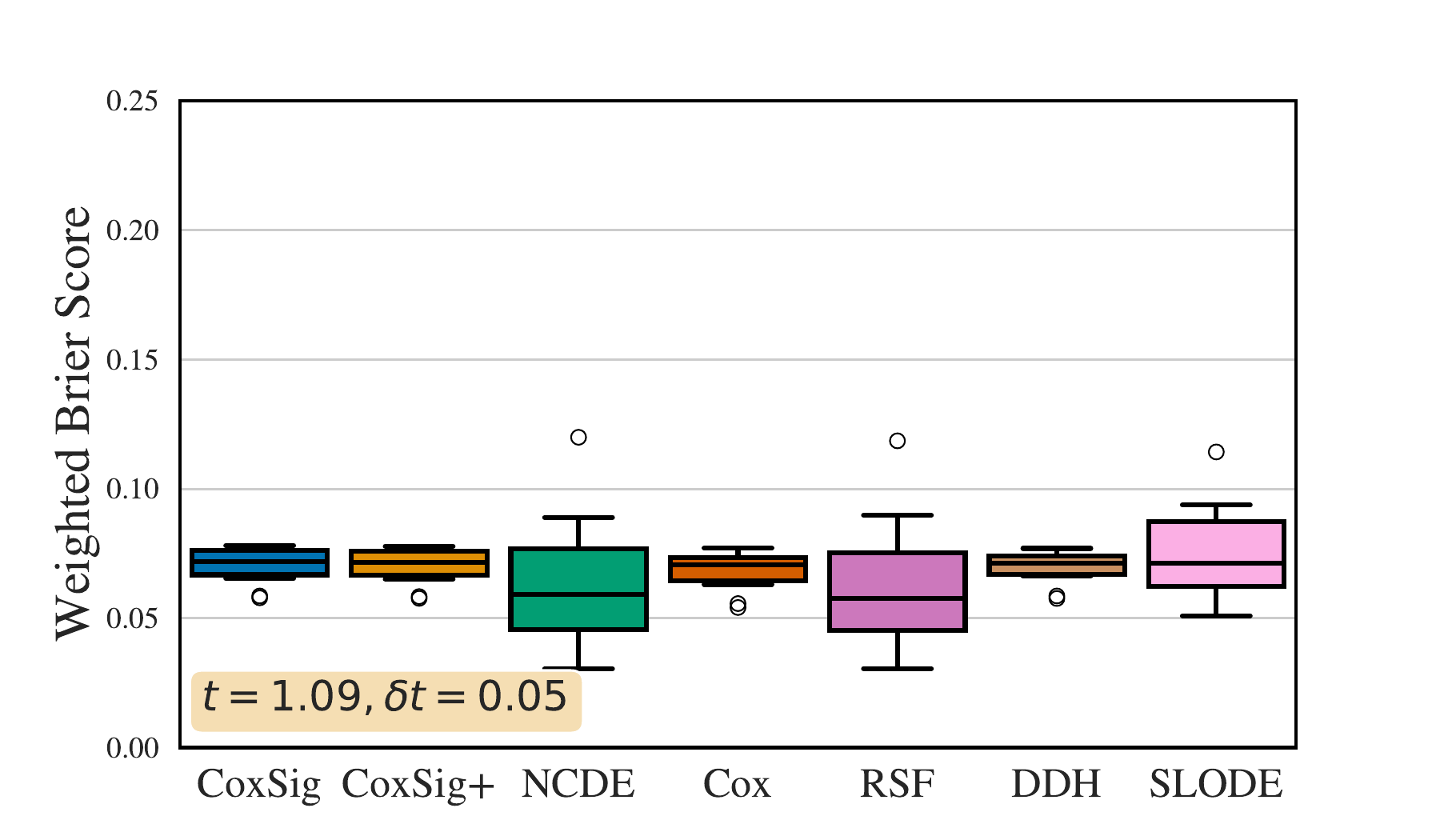}
    \includegraphics[width=0.33\textwidth]{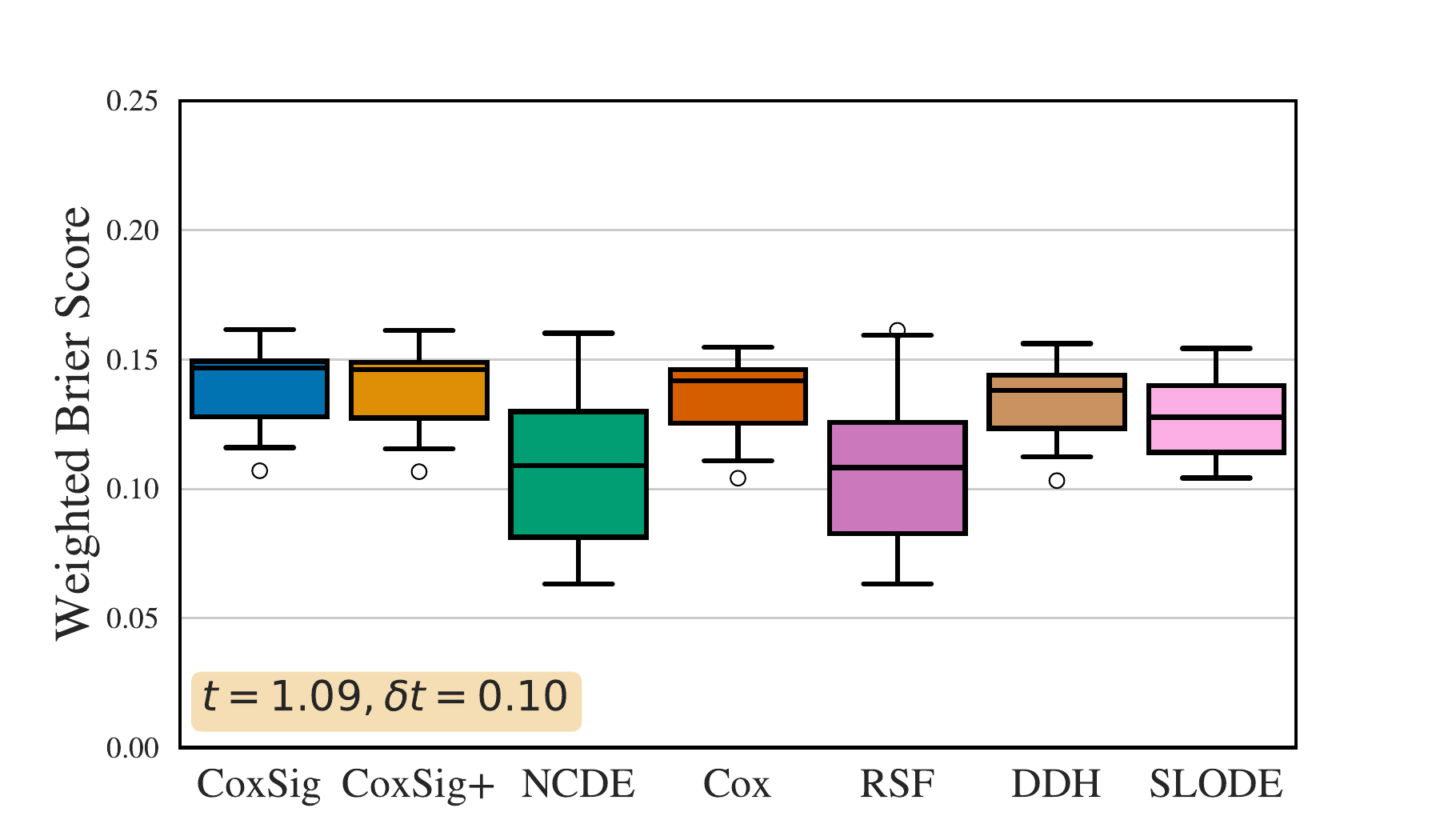}
    \includegraphics[width=0.33\textwidth]{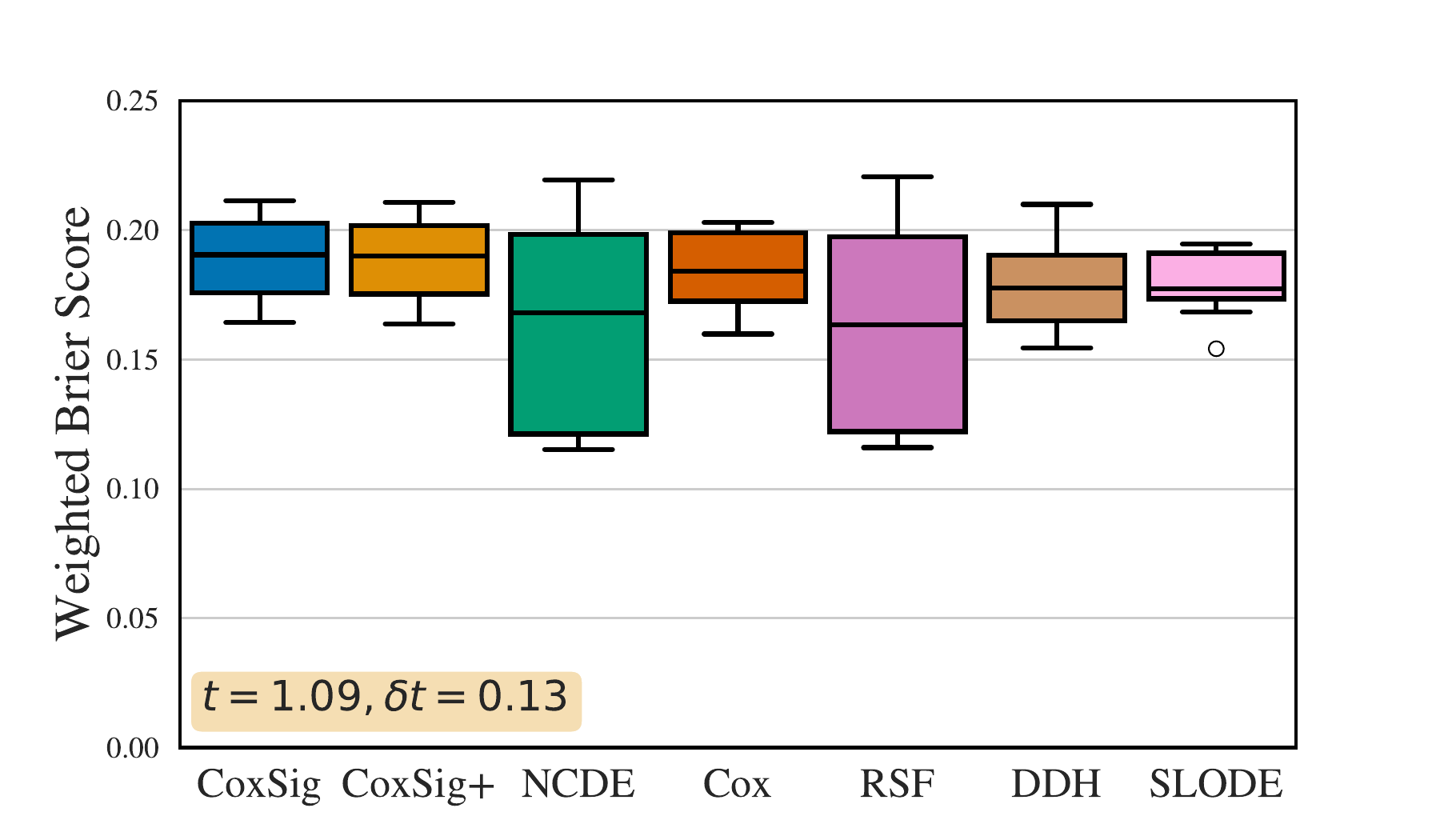}
    \includegraphics[width=0.33\textwidth]{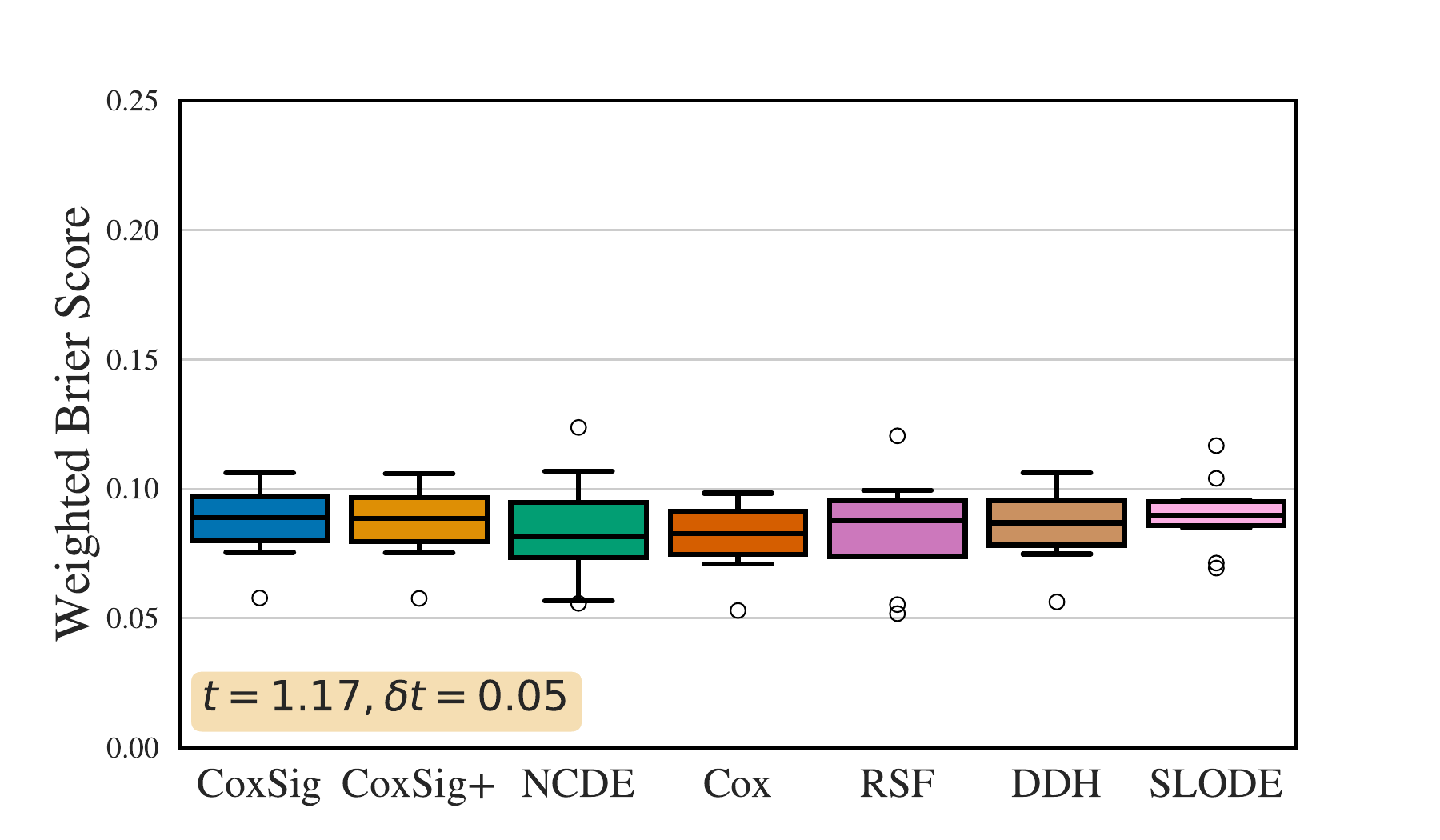}
    \includegraphics[width=0.33\textwidth]{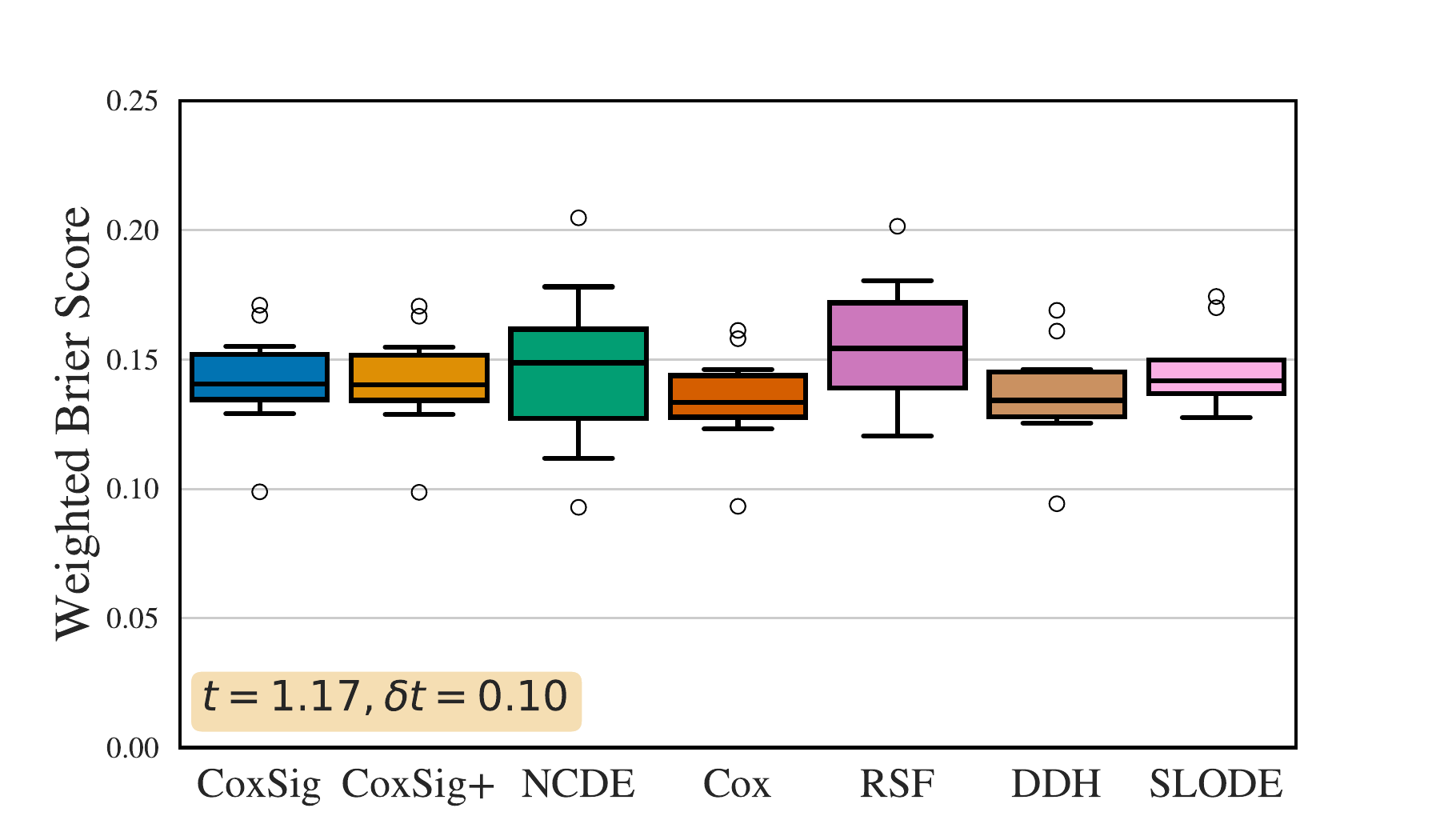}
    \includegraphics[width=0.33\textwidth]{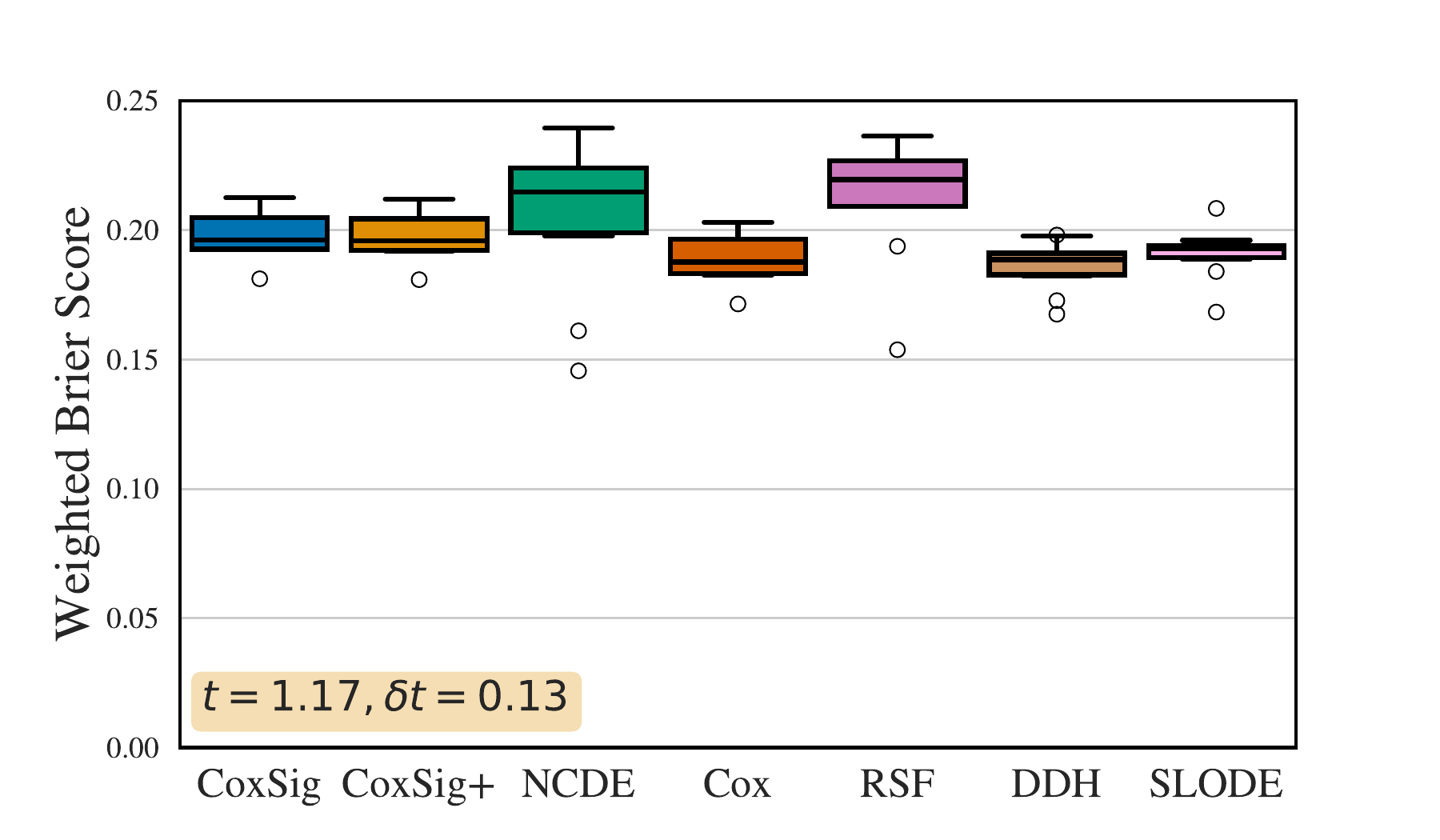}
    \caption{\footnotesize Weighted Brier score (\textit{lower} is better) for \textbf{Tumor Growth} for numerous points $(t,\delta t)$.}
    \label{fig:wbs_tumor_growth}
\end{figure*}

\begin{figure*}
    \centering
    \includegraphics[width=0.33\textwidth]{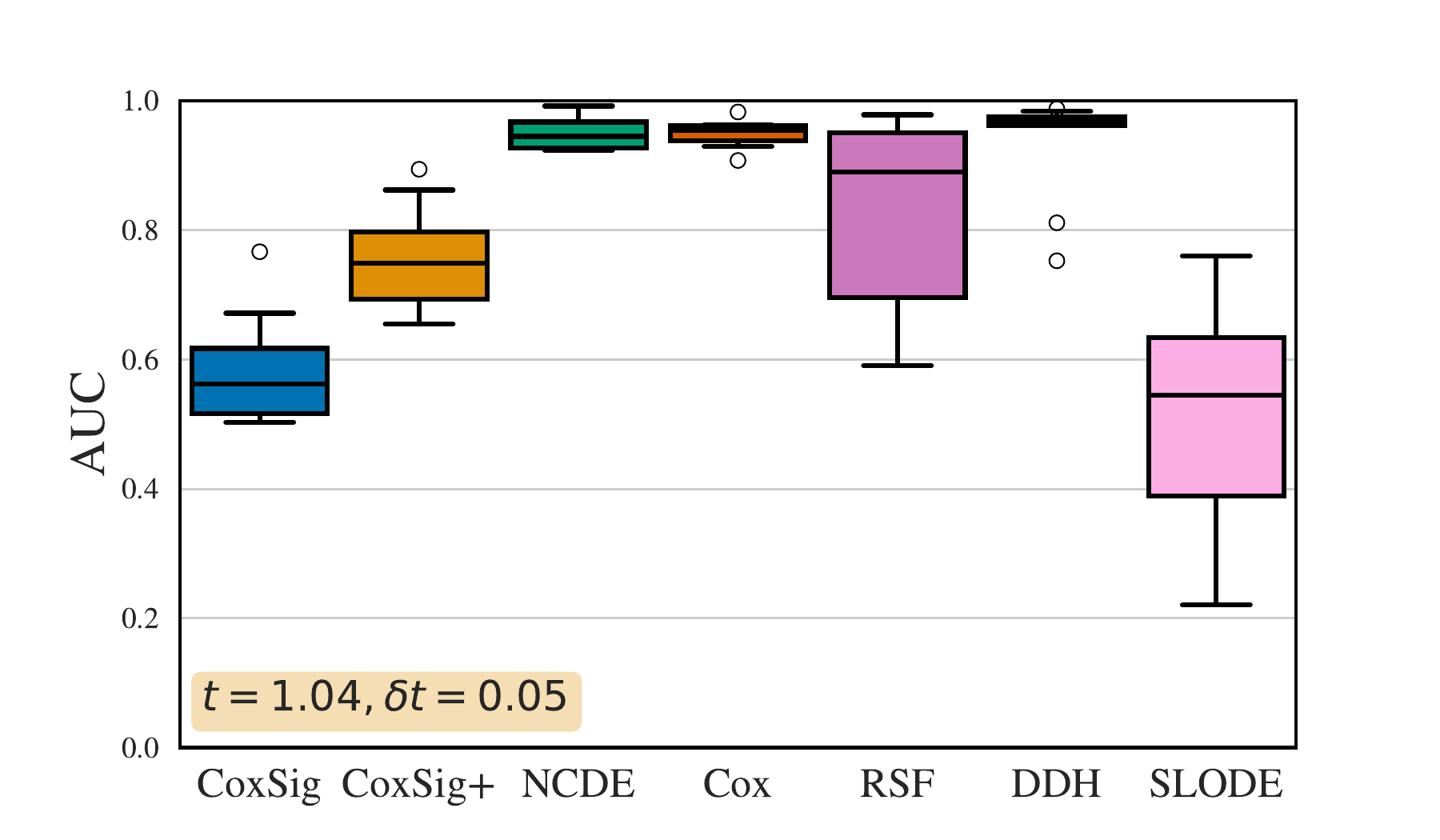}
    \includegraphics[width=0.33\textwidth]{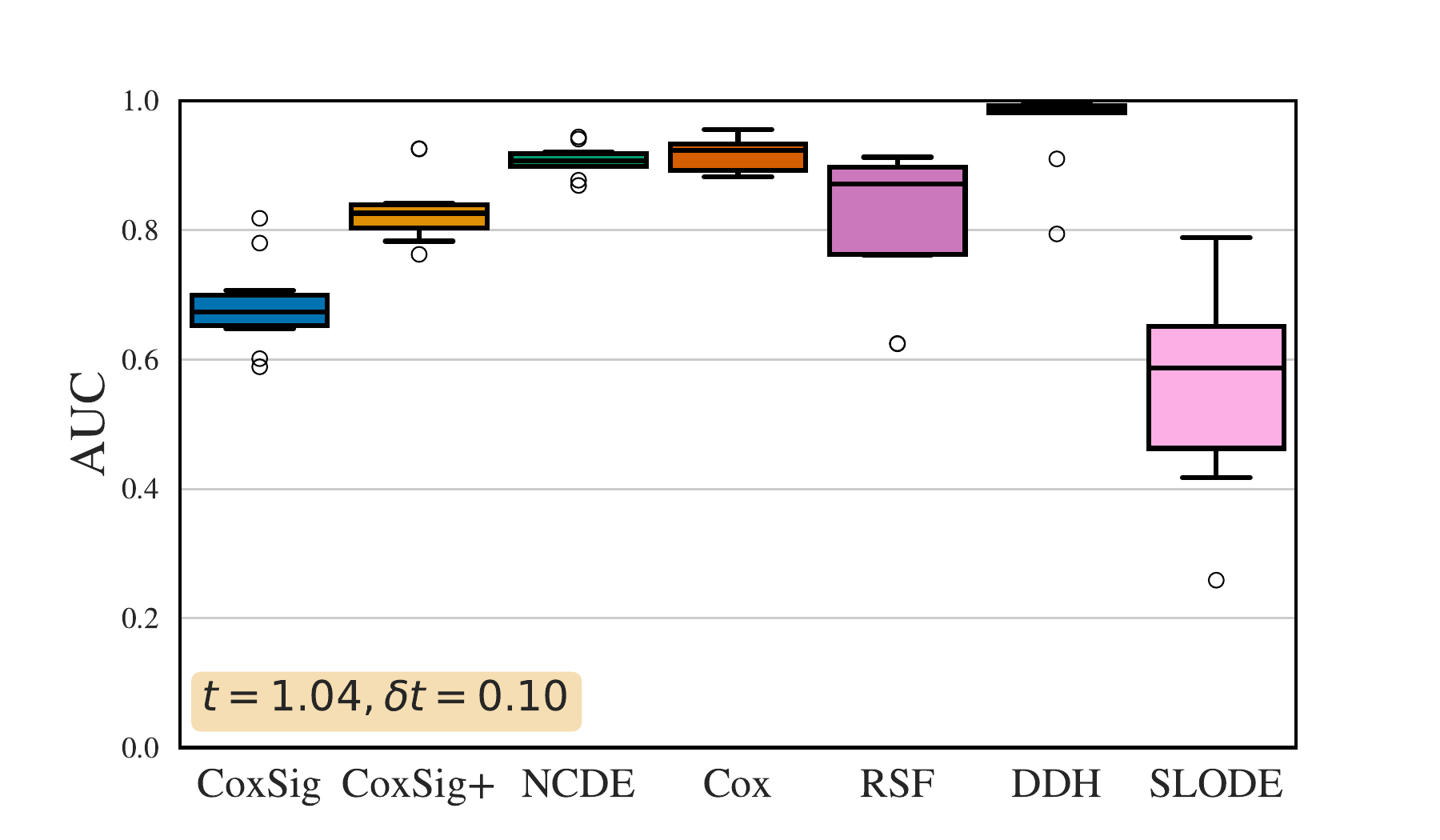}
    \includegraphics[width=0.33\textwidth]{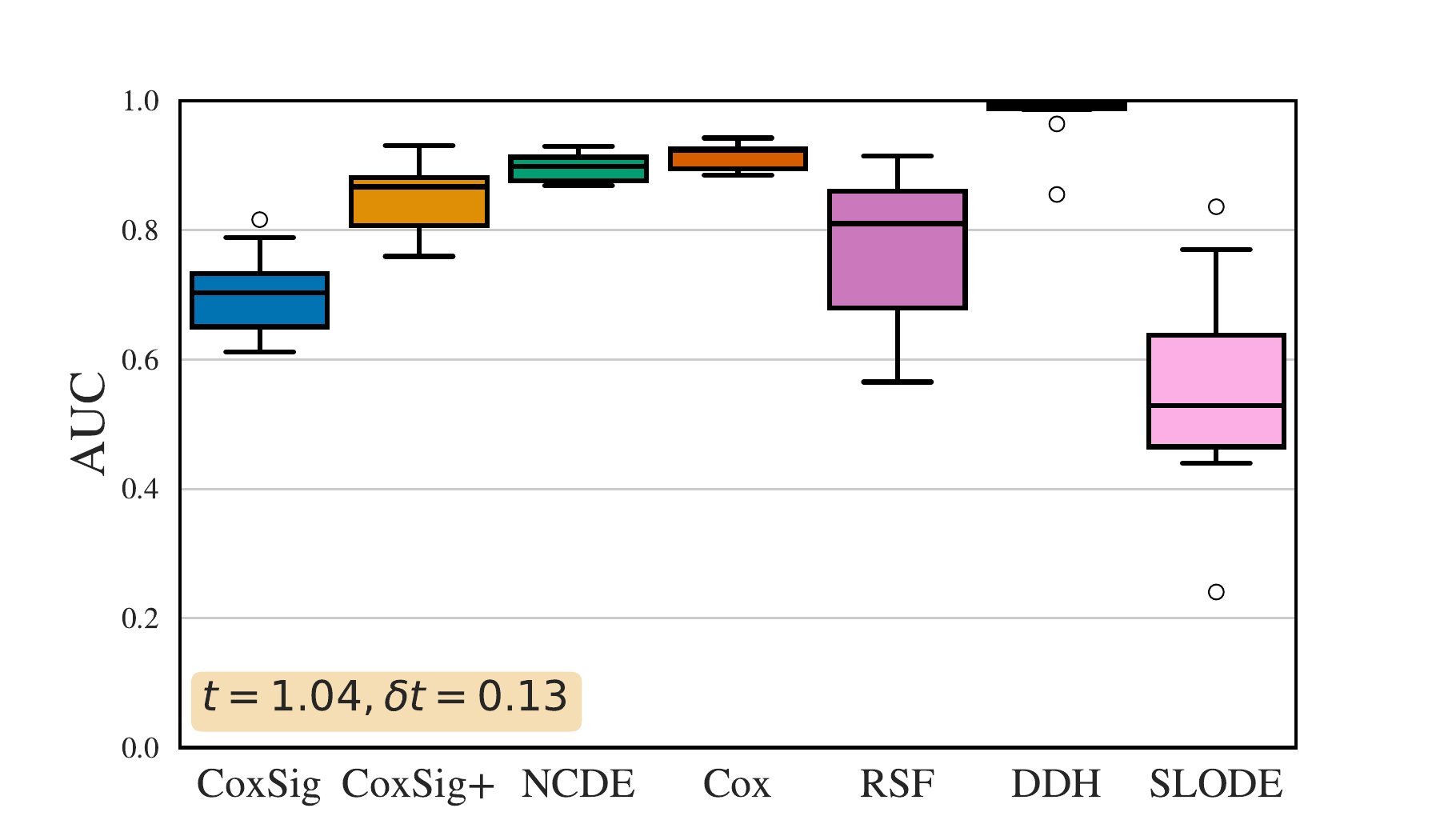}
    \includegraphics[width=0.33\textwidth]{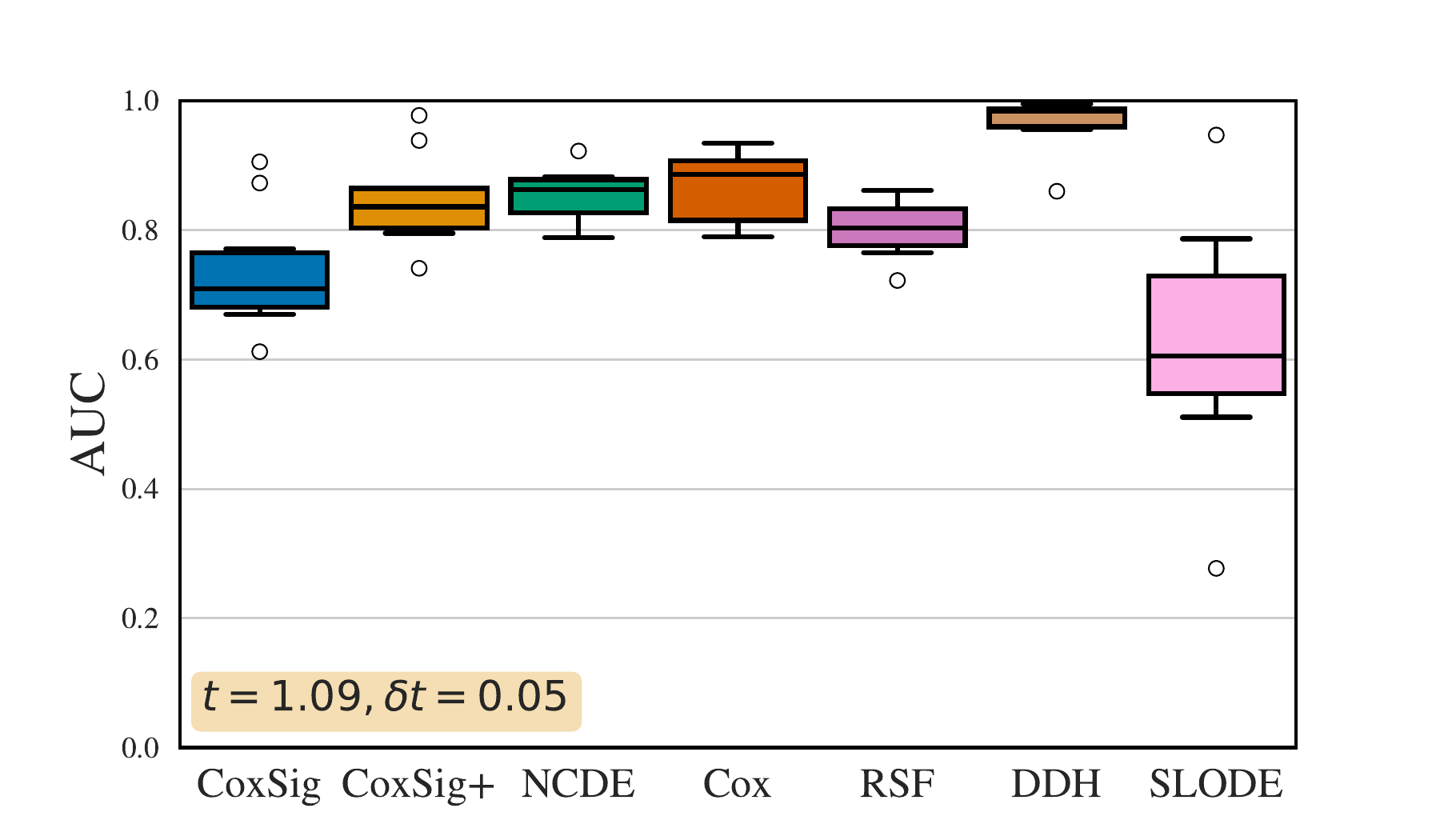}
    \includegraphics[width=0.33\textwidth]{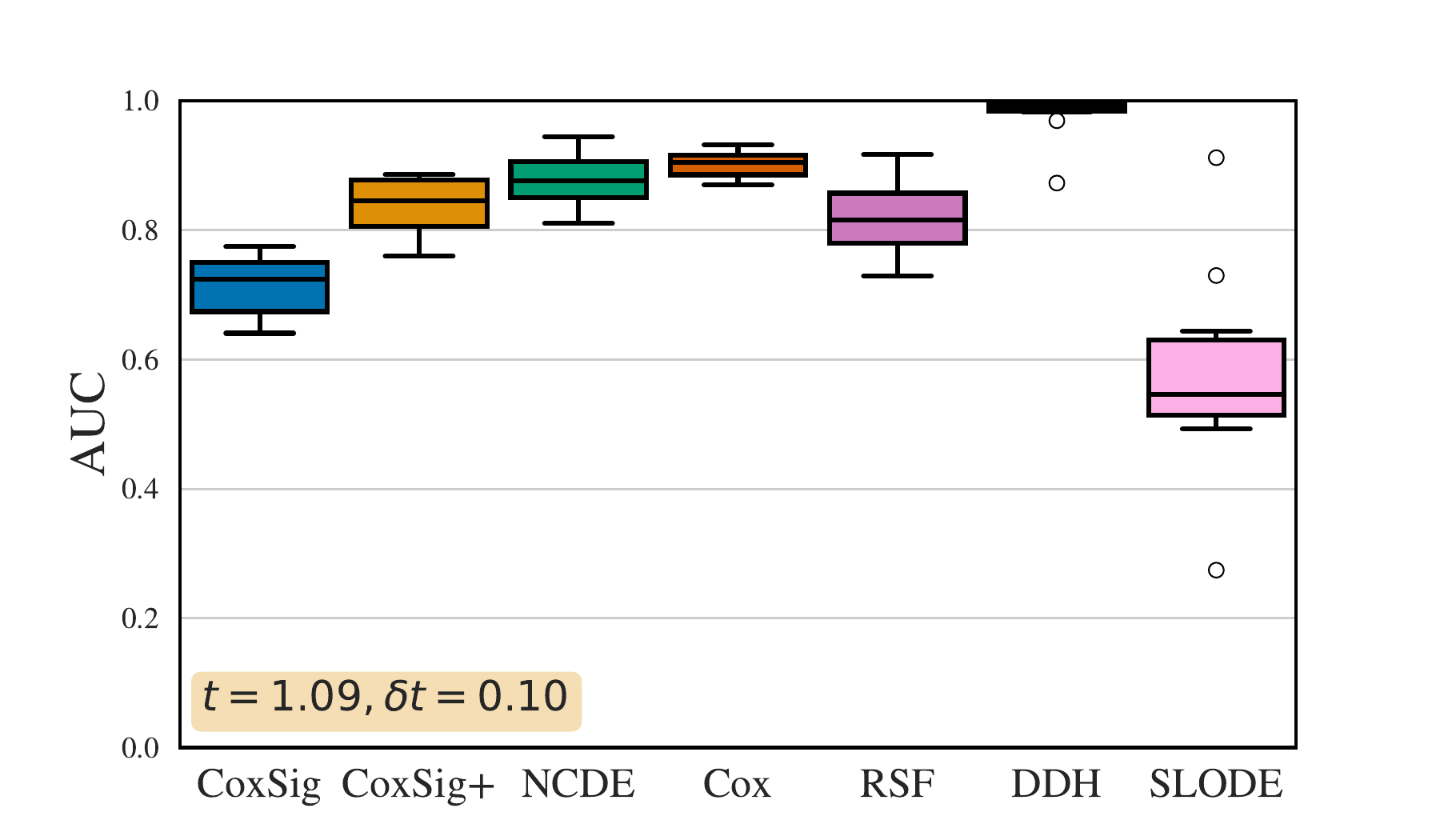}
    \includegraphics[width=0.33\textwidth]{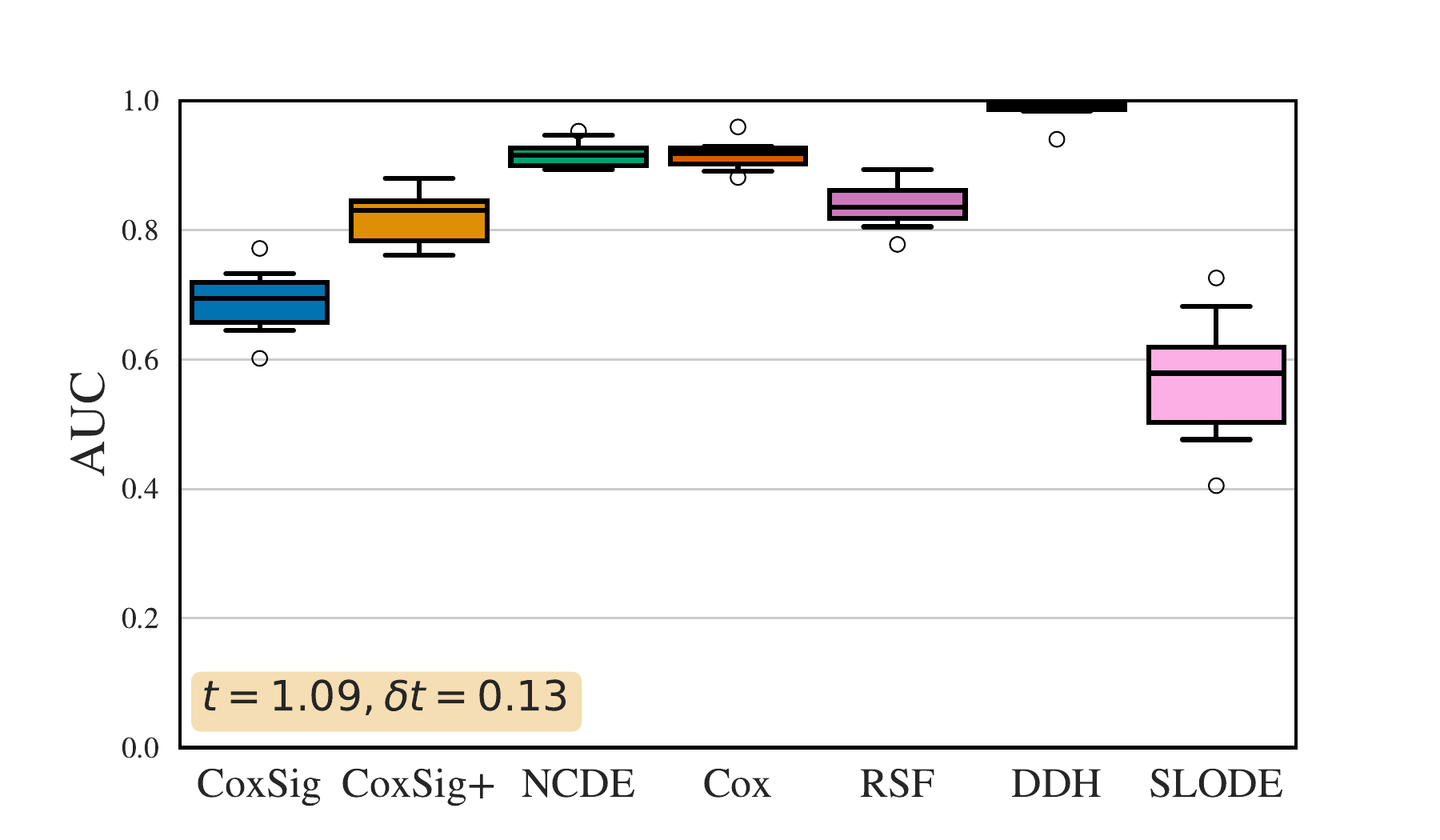}
    \includegraphics[width=0.33\textwidth]{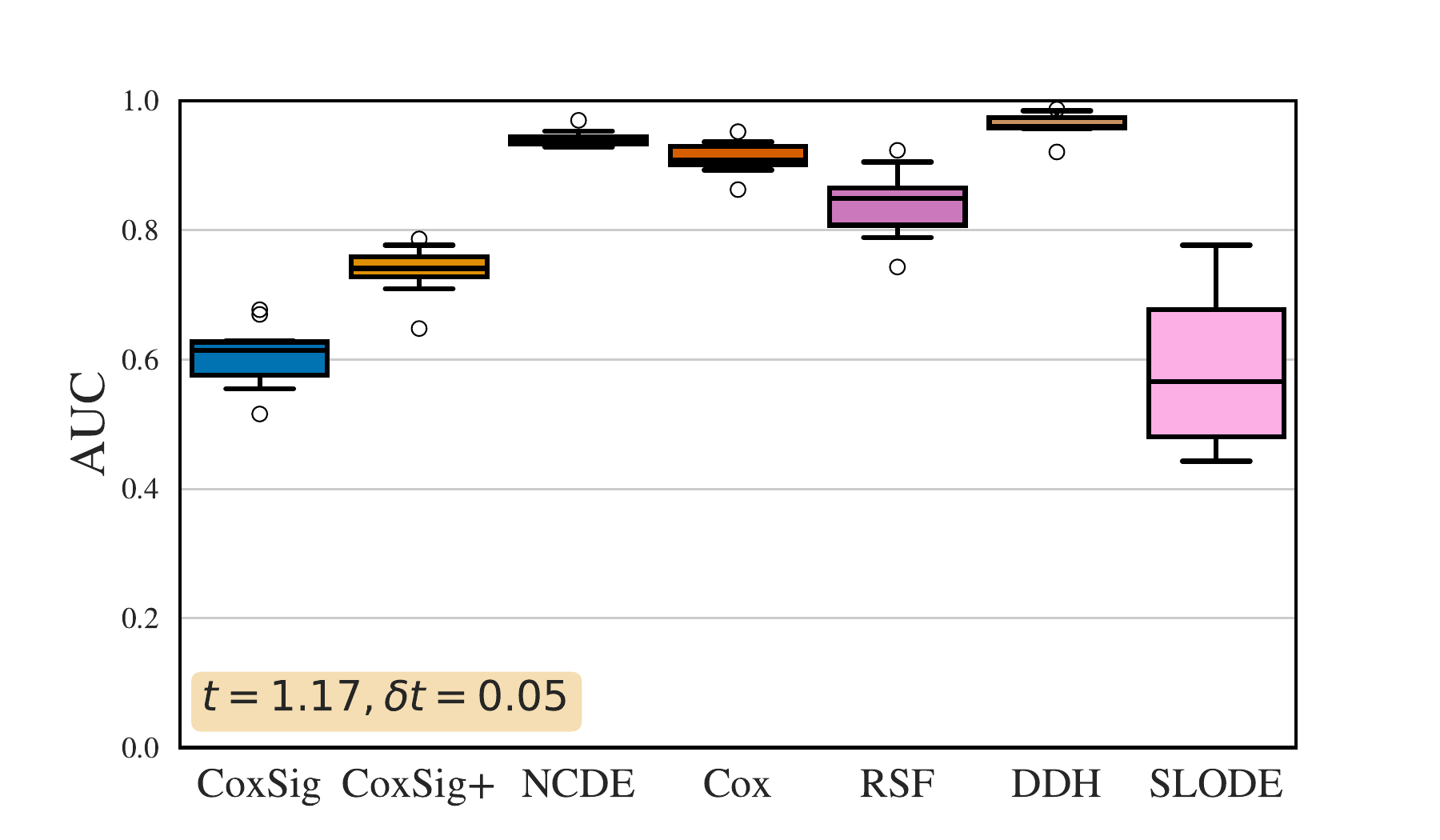}
    \includegraphics[width=0.33\textwidth]{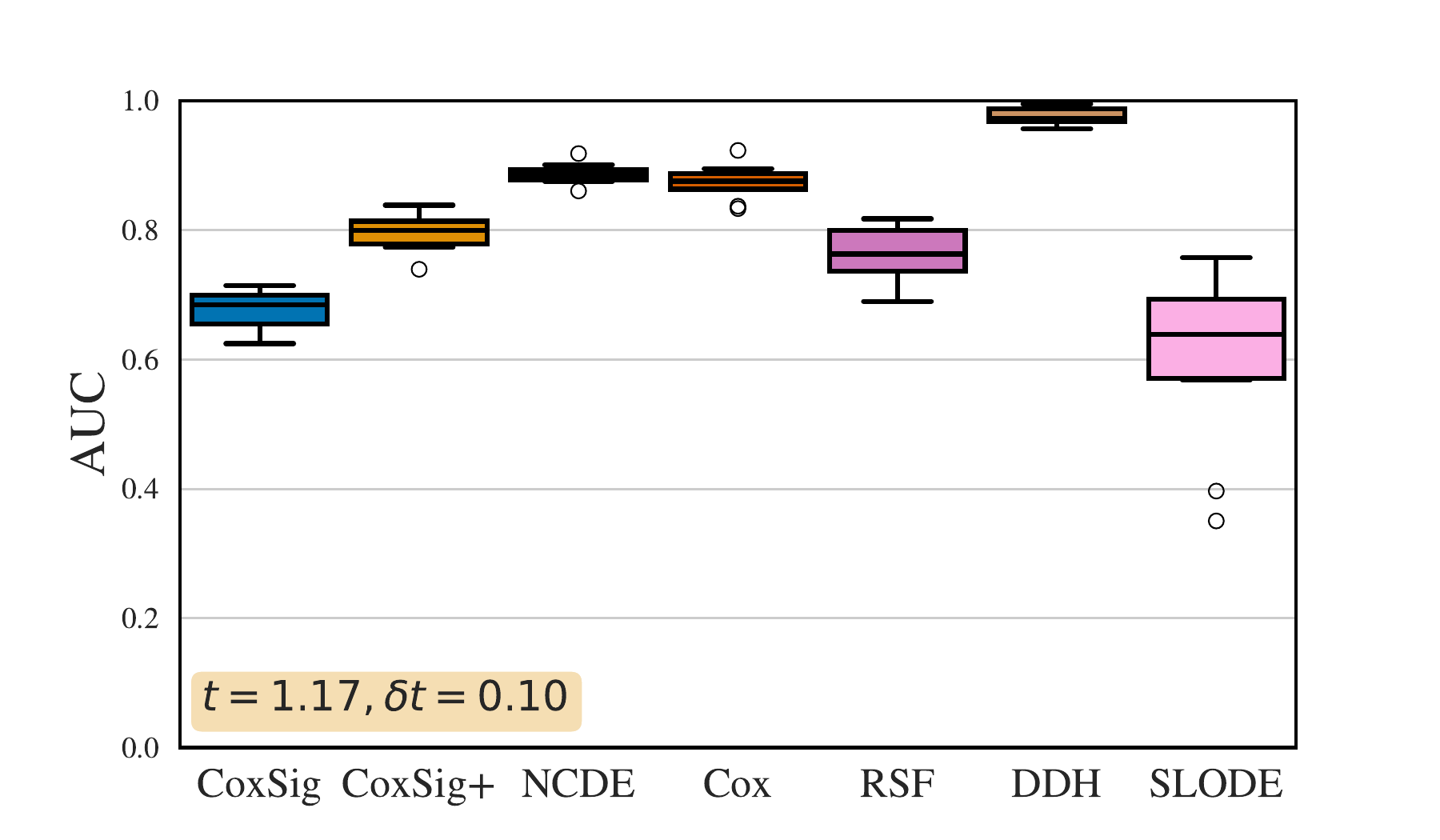}
    \includegraphics[width=0.33\textwidth]{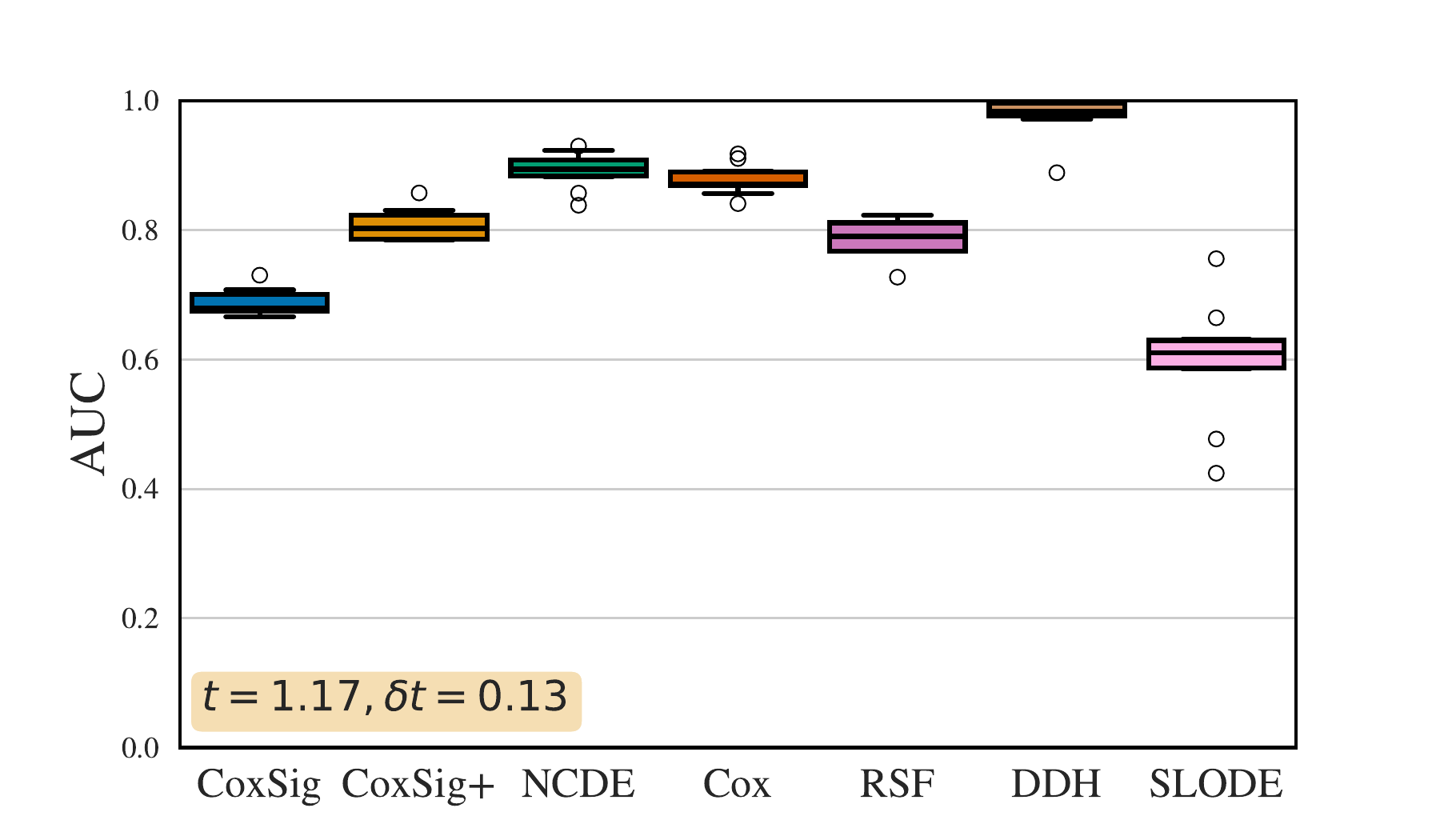}
    \caption{\footnotesize AUC (\textit{higher} is better) for \textbf{Tumor Growth} for numerous points $(t,\delta t)$.}
    \label{fig:auc_tumor_growth}
\end{figure*}
\subsection{Predictive Maintenance}

\paragraph{Time series.} This dataset describes the degradation of 200 aircraft gas turbine engines, where 22 measurements of sensors and 3 operational settings are recorded each operational cycle until its failure. After removing low-variance features, 16 longitudinal features are selected for training models. The average time length of these features is about 25 cycles. Note that we apply standardization for selected features before training.

\paragraph{Event definition.} The times of event are given as-is in the dataset. We refer to \citet{saxena2008damage} for a precise description of the data generation.

\paragraph{Censorship.} Censorship is given as-in in the dataset. The censoring level of this dataset is 50$\%$, which is a high censorship rate in survival analysis. We refer again to \citet{saxena2008damage} for more details.

\paragraph{Supplementary Figures.} Figure \ref{fig:NASA_appendix_path_hist} provides an example of several randomly picked sample paths of an individual and the distribution of the event times of the whole population. We add additional results in Figures \ref{fig:c_index_nasa}, \ref{fig:bs_nasa}, \ref{fig:wbs_nasa} and \ref{fig:auc_nasa}.

\begin{figure*}
    \centering
    \includegraphics[width=0.49\textwidth]{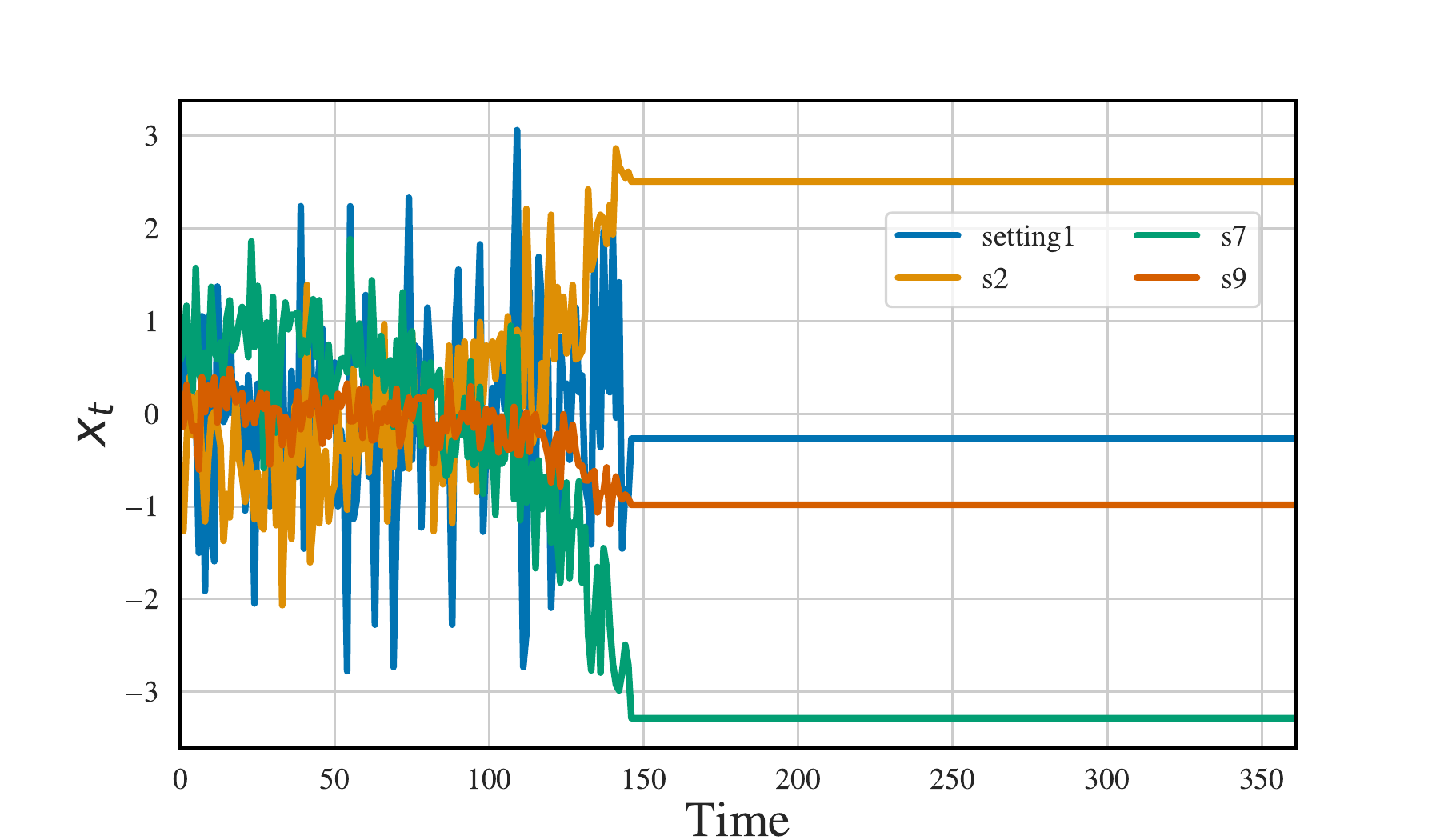}
    \includegraphics[width=0.49\textwidth]{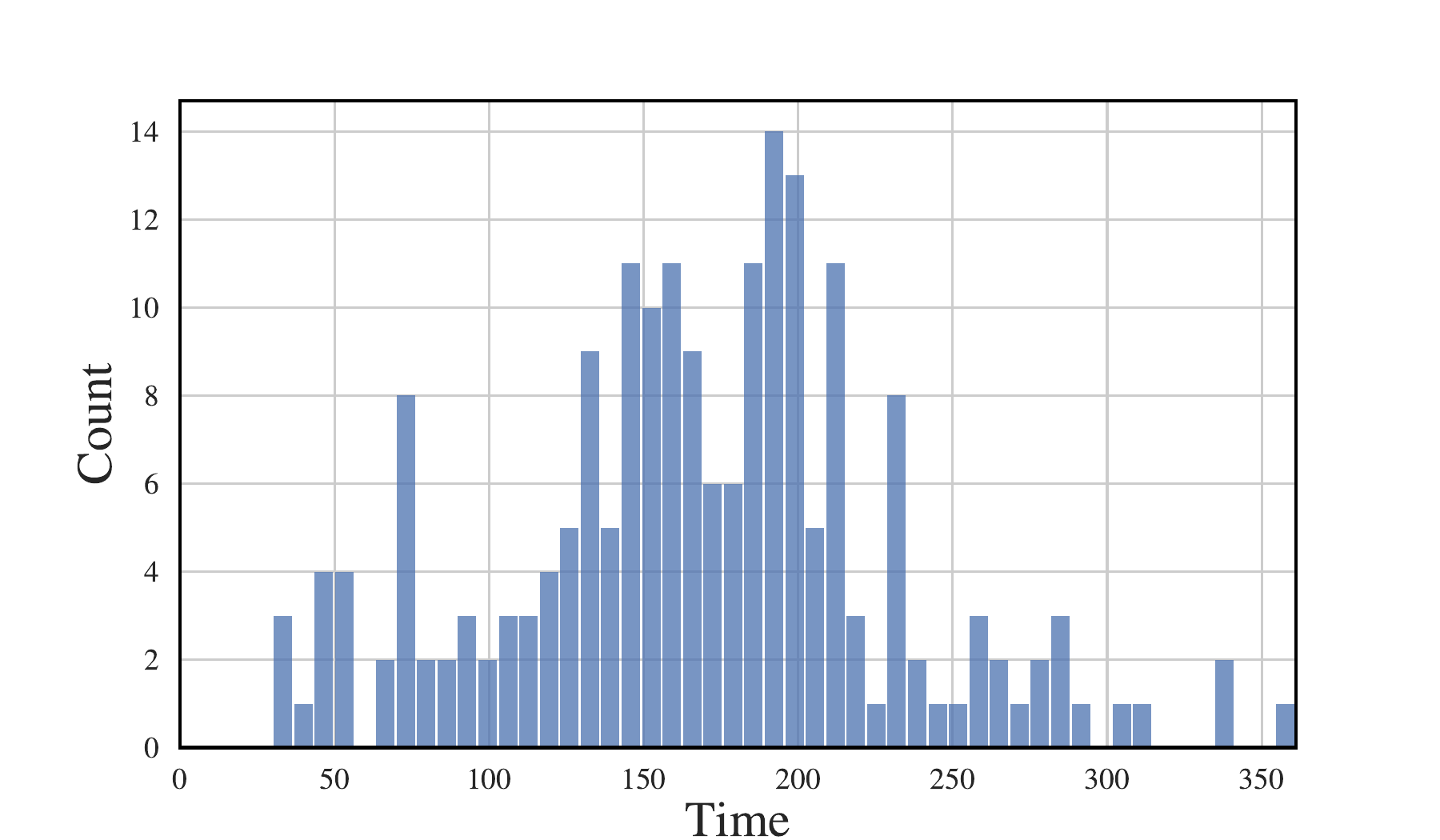}
    
    \caption{\footnotesize Partial sample path of an individual (\textbf{left}) and distribution of the event times (\textbf{left}) for the predictive maintenance experiment. On the left, the time series is filled with the last observed value from the time of the event on.}
    \label{fig:NASA_appendix_path_hist}
\end{figure*}

\begin{figure*}[h!]
    \centering
    \includegraphics[width=0.33\textwidth]{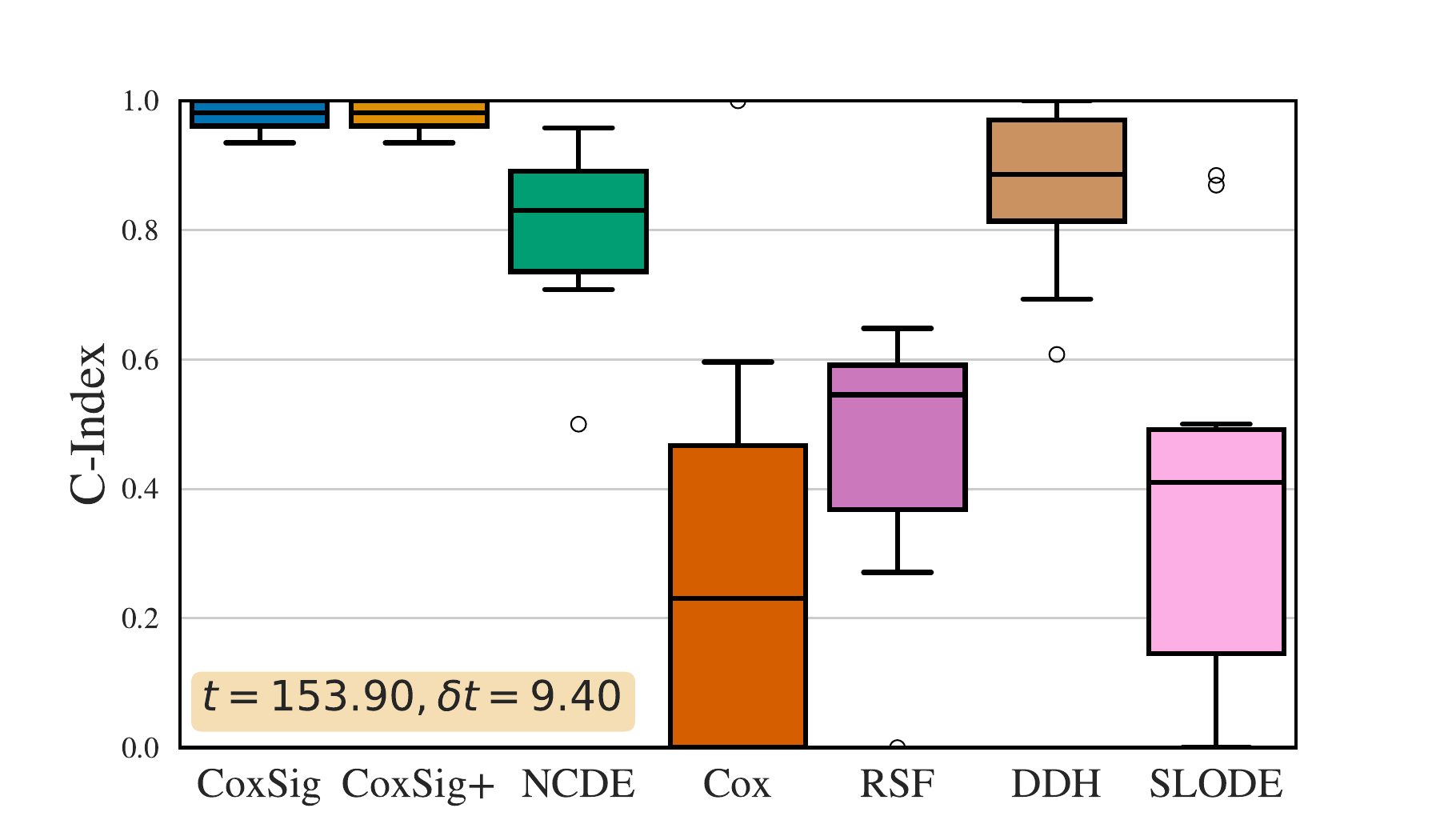}
    \includegraphics[width=0.33\textwidth]{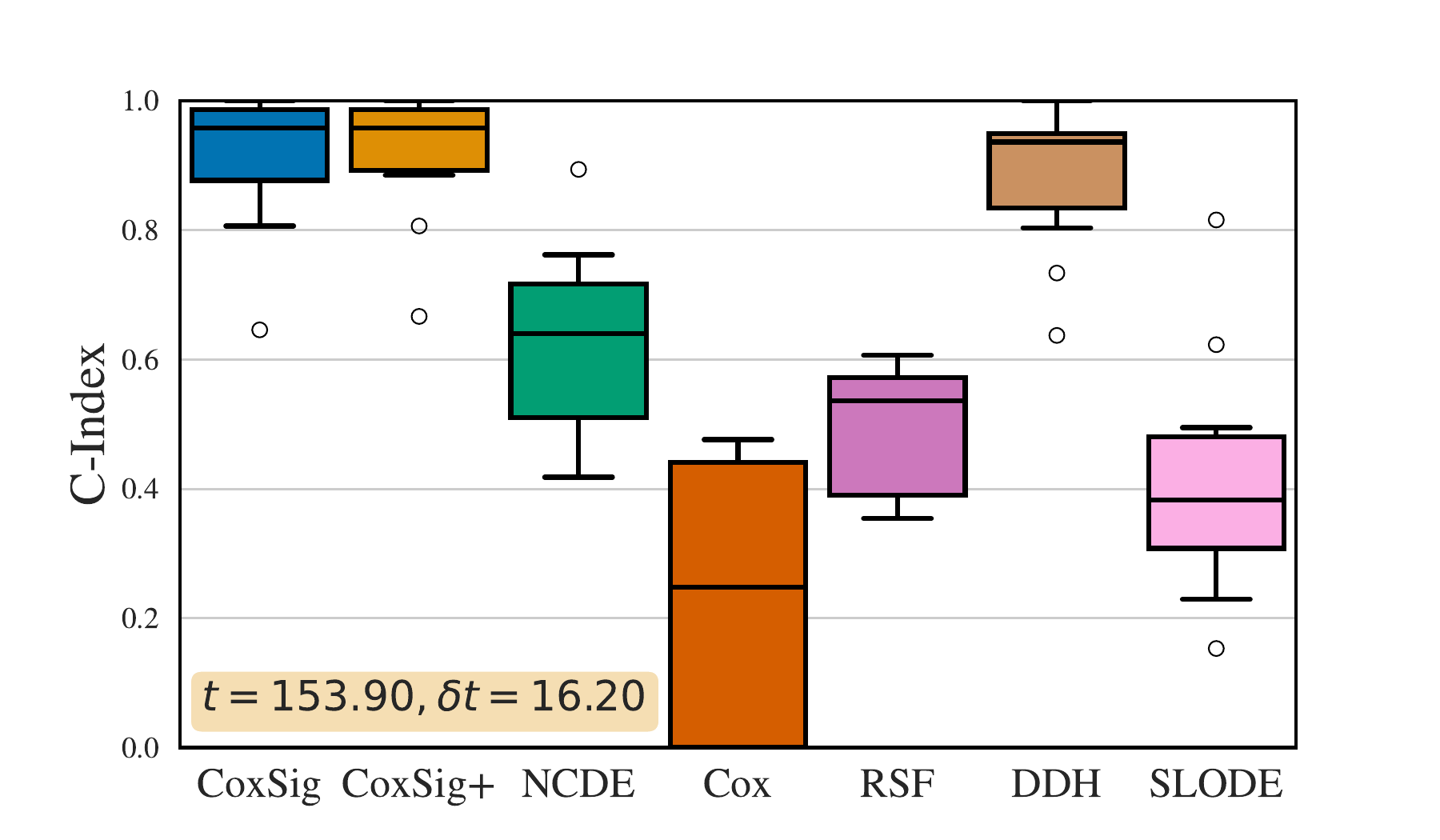}
    \includegraphics[width=0.33\textwidth]{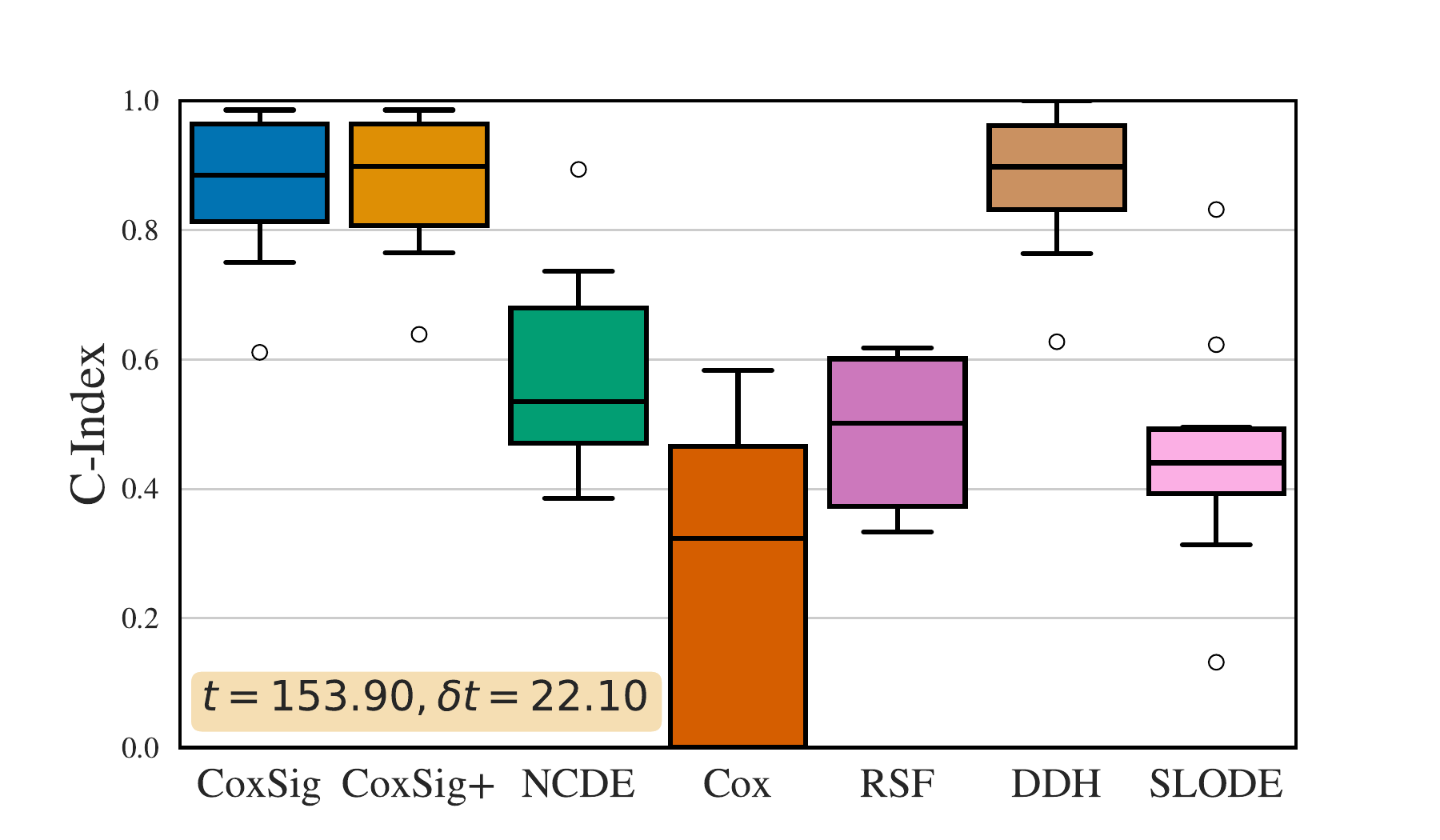}
    \includegraphics[width=0.33\textwidth]{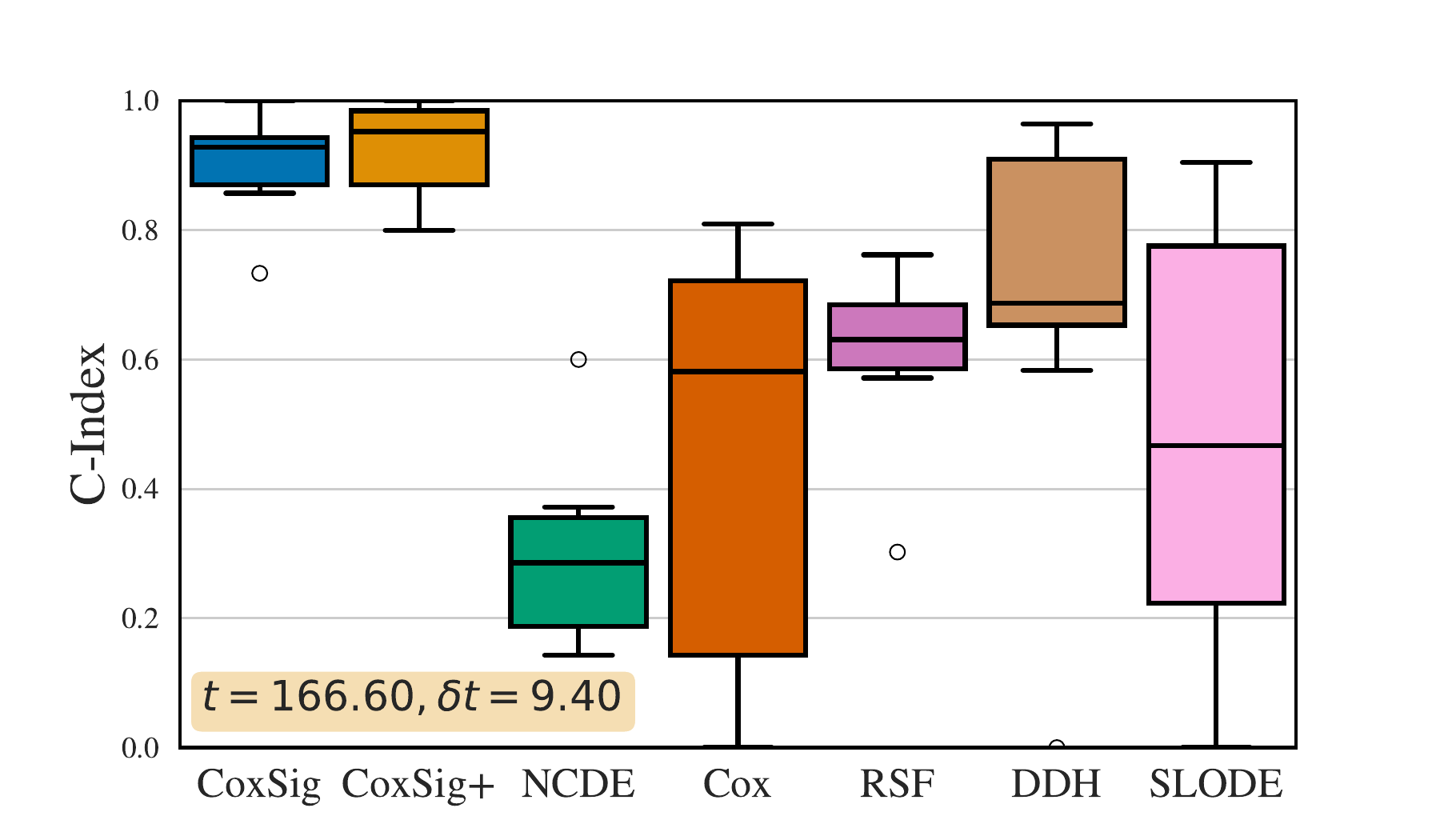}
    \includegraphics[width=0.33\textwidth]{figures/updated_figures/nasa/NASA_C-Index_pt=166.60,dt=16.20.pdf}
    \includegraphics[width=0.33\textwidth]{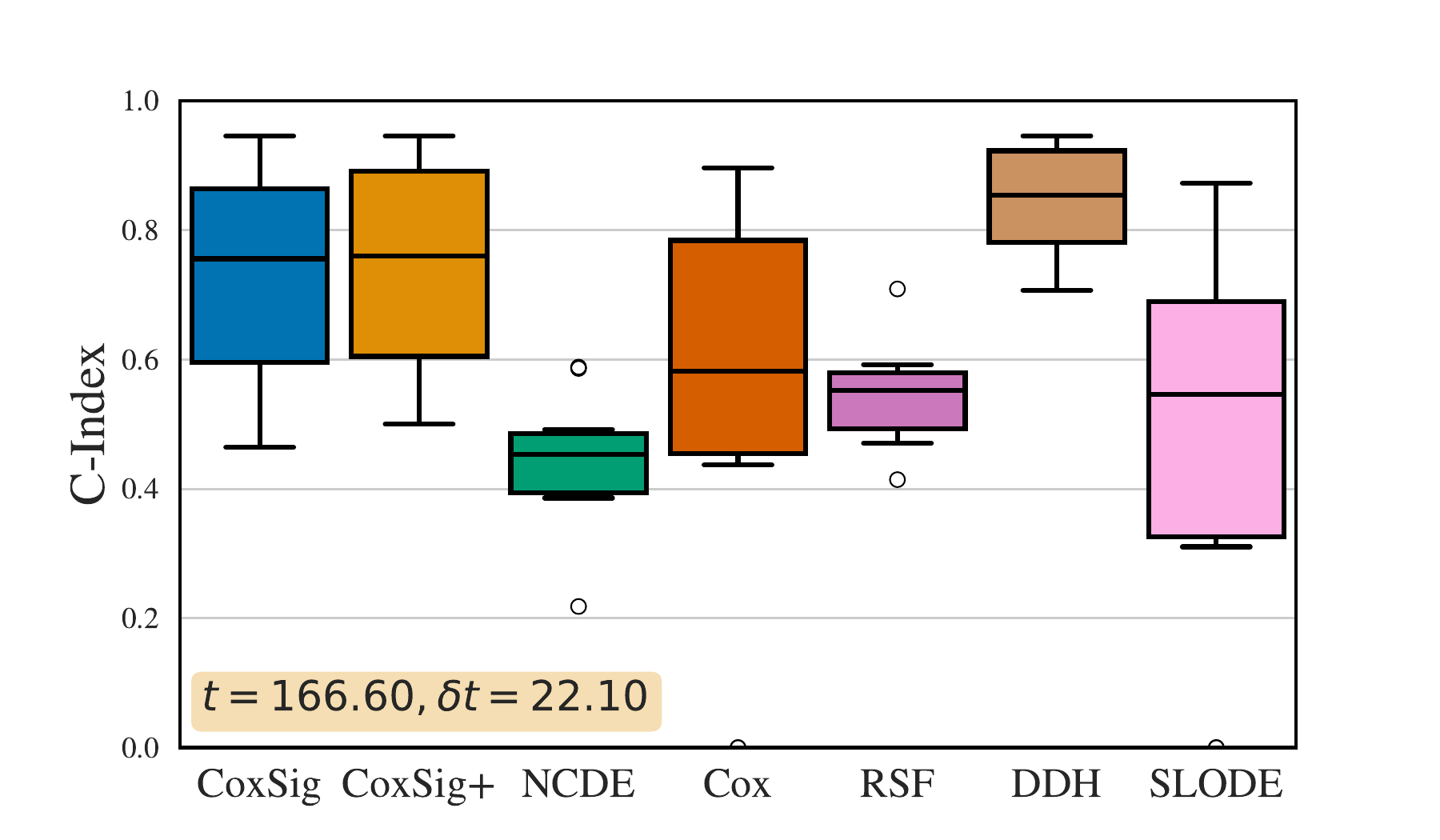}
    \includegraphics[width=0.33\textwidth]{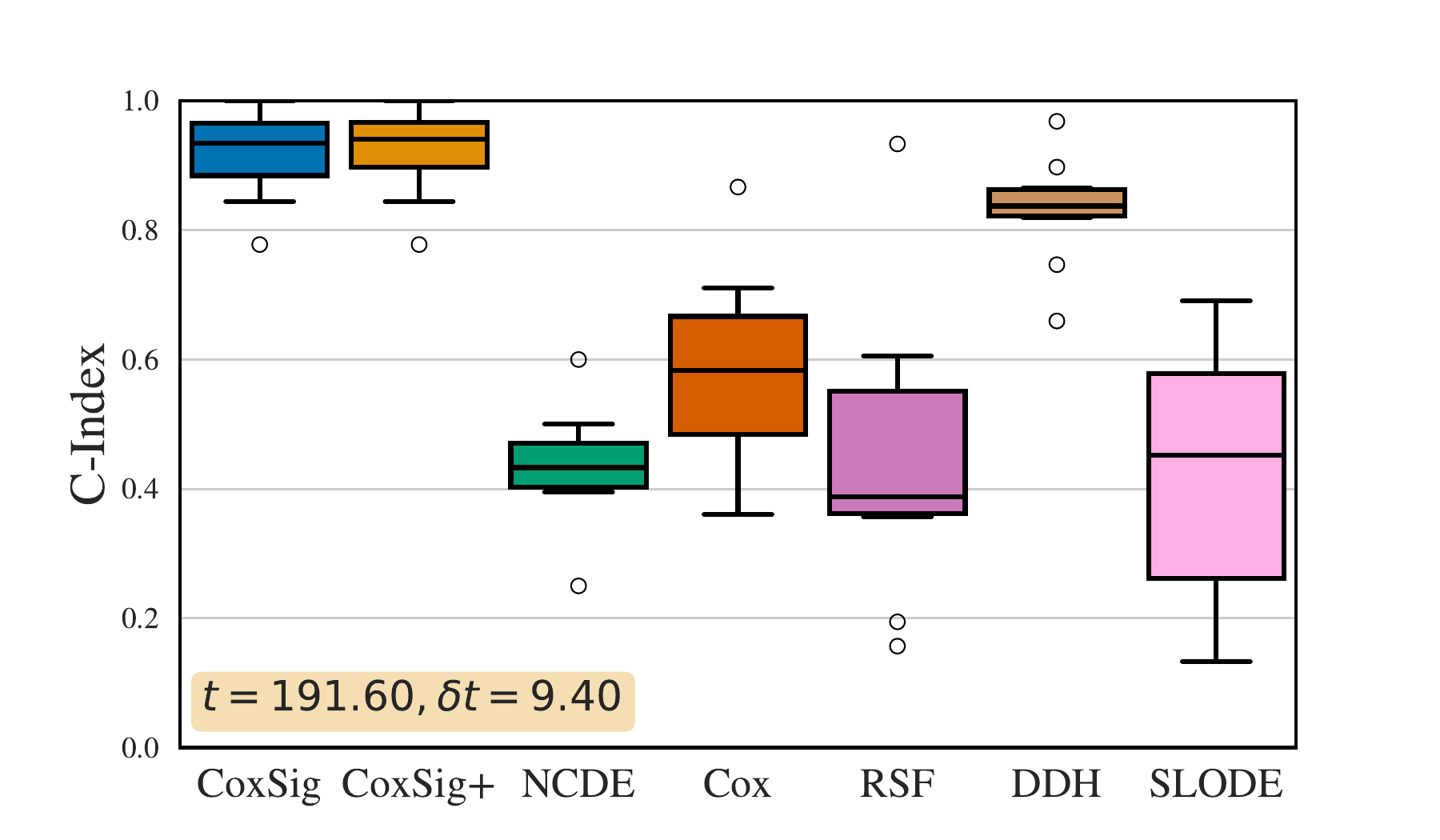}
    \includegraphics[width=0.33\textwidth]{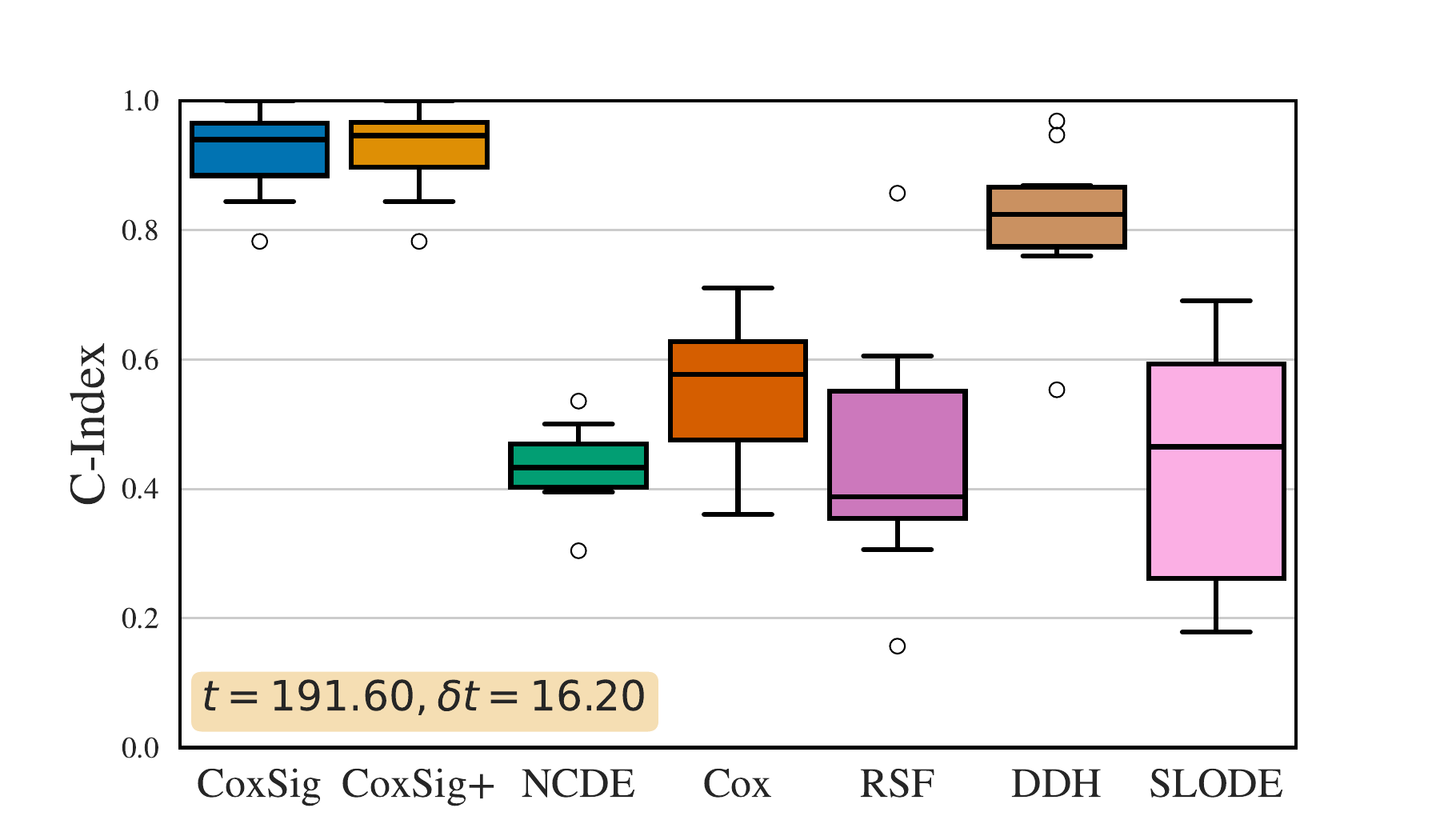}
    \includegraphics[width=0.33\textwidth]{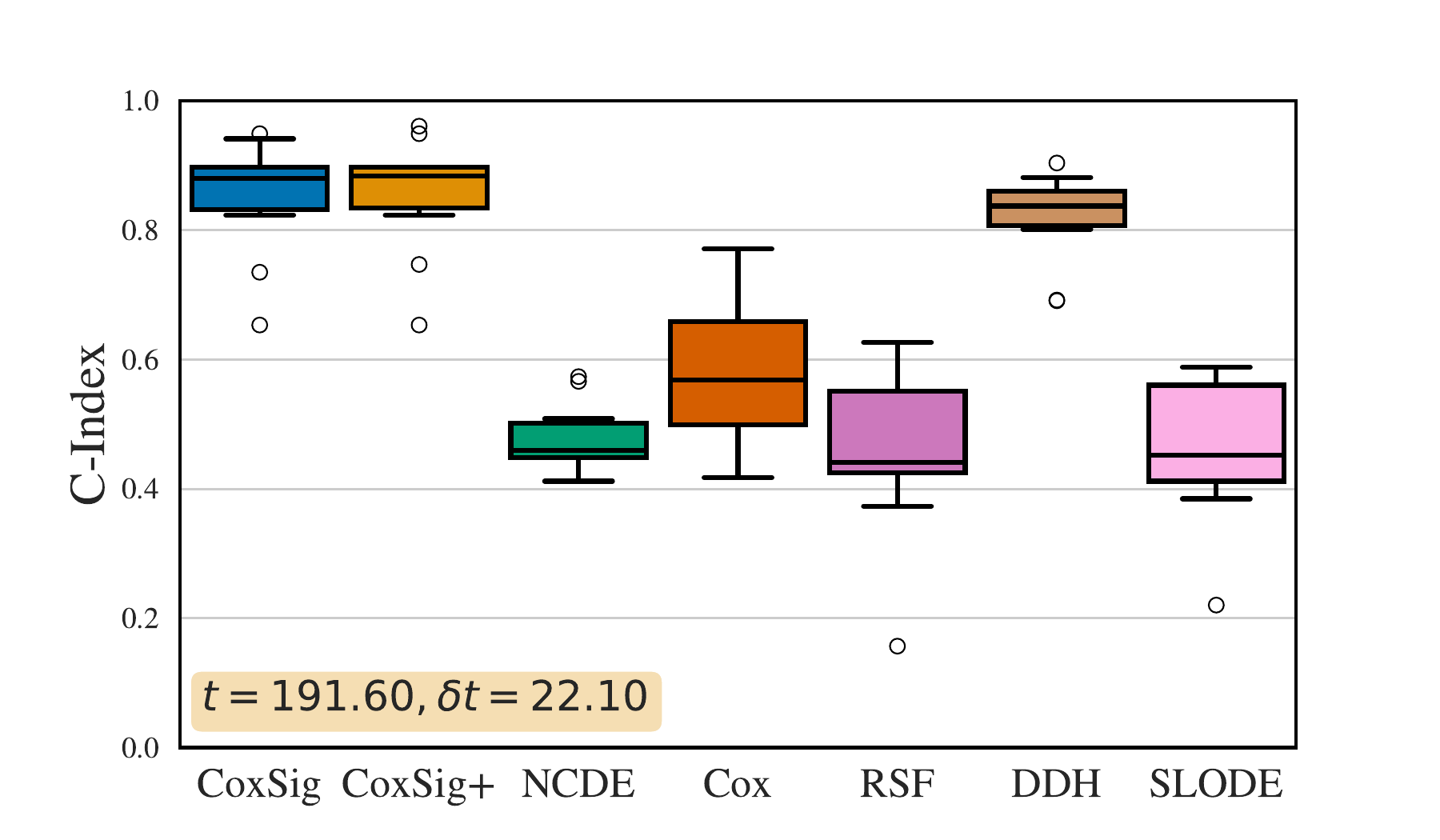}
    \caption{\footnotesize C-Index (\textit{higher} is better) for \textbf{predictive maintenance} for numerous points $(t,\delta t)$.}
    \label{fig:c_index_nasa}
\end{figure*}

\begin{figure*}
    \centering
    \includegraphics[width=0.33\textwidth]{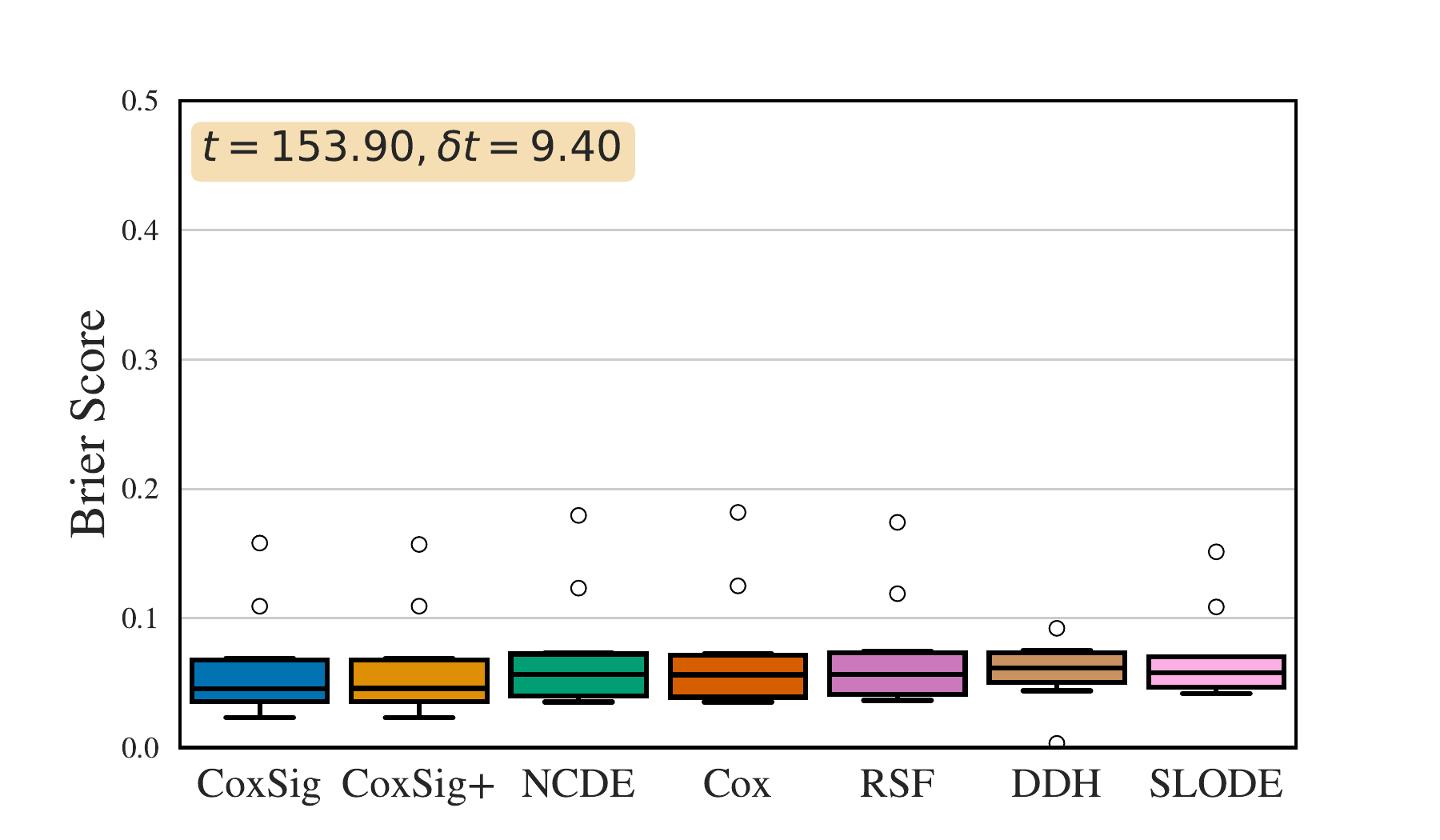}
    \includegraphics[width=0.33\textwidth]{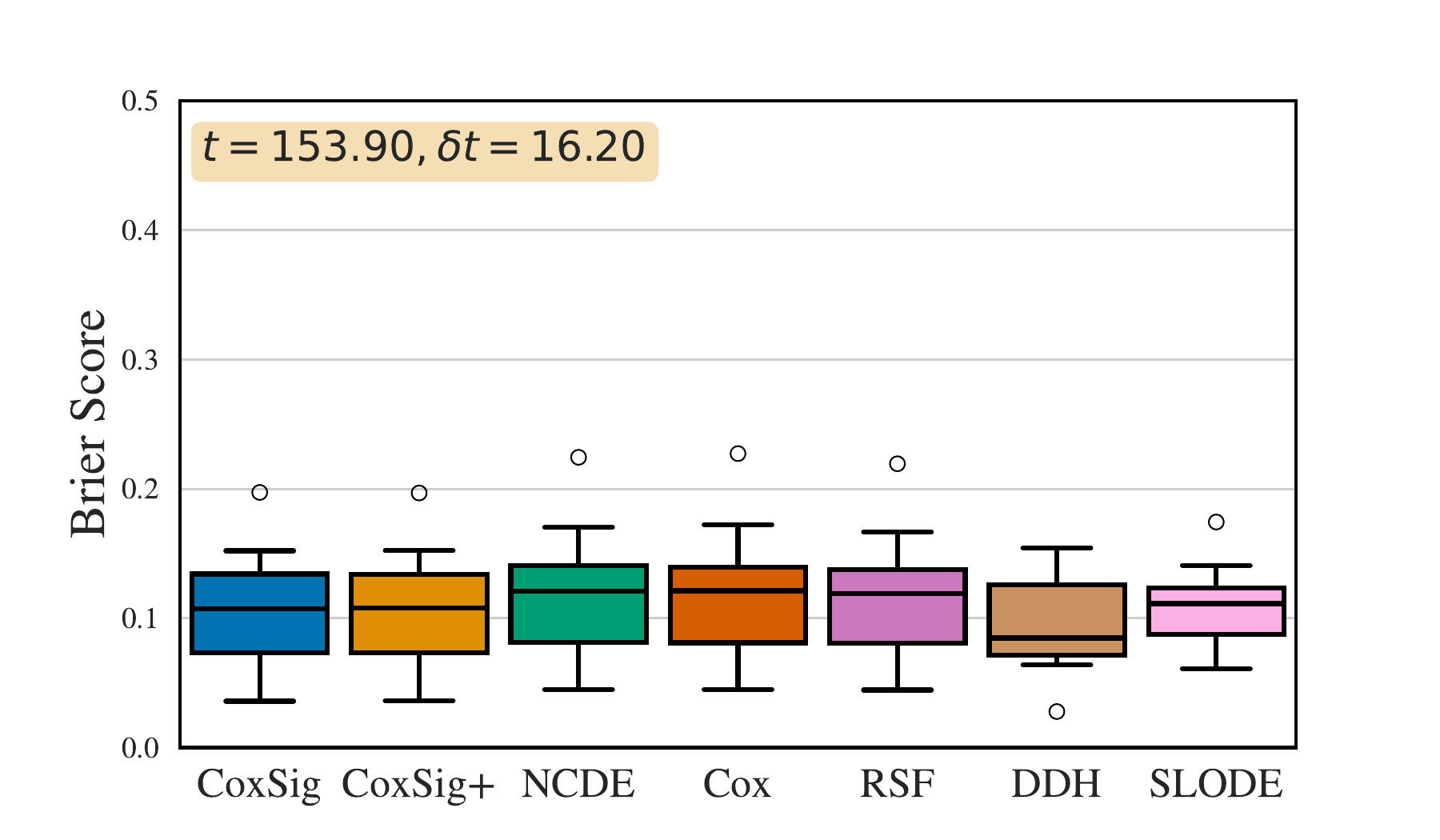}
    \includegraphics[width=0.33\textwidth]{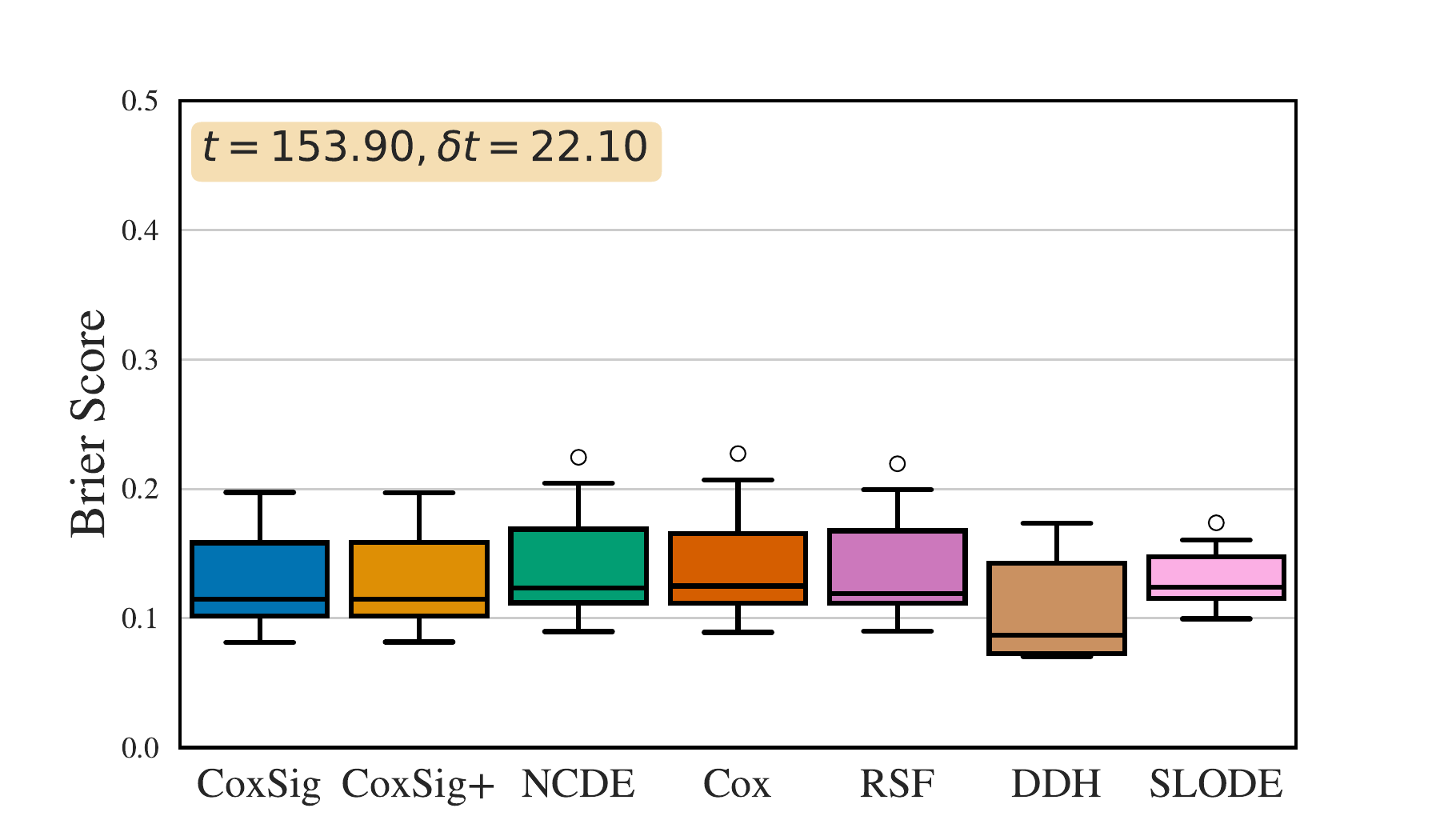}
    \includegraphics[width=0.33\textwidth]{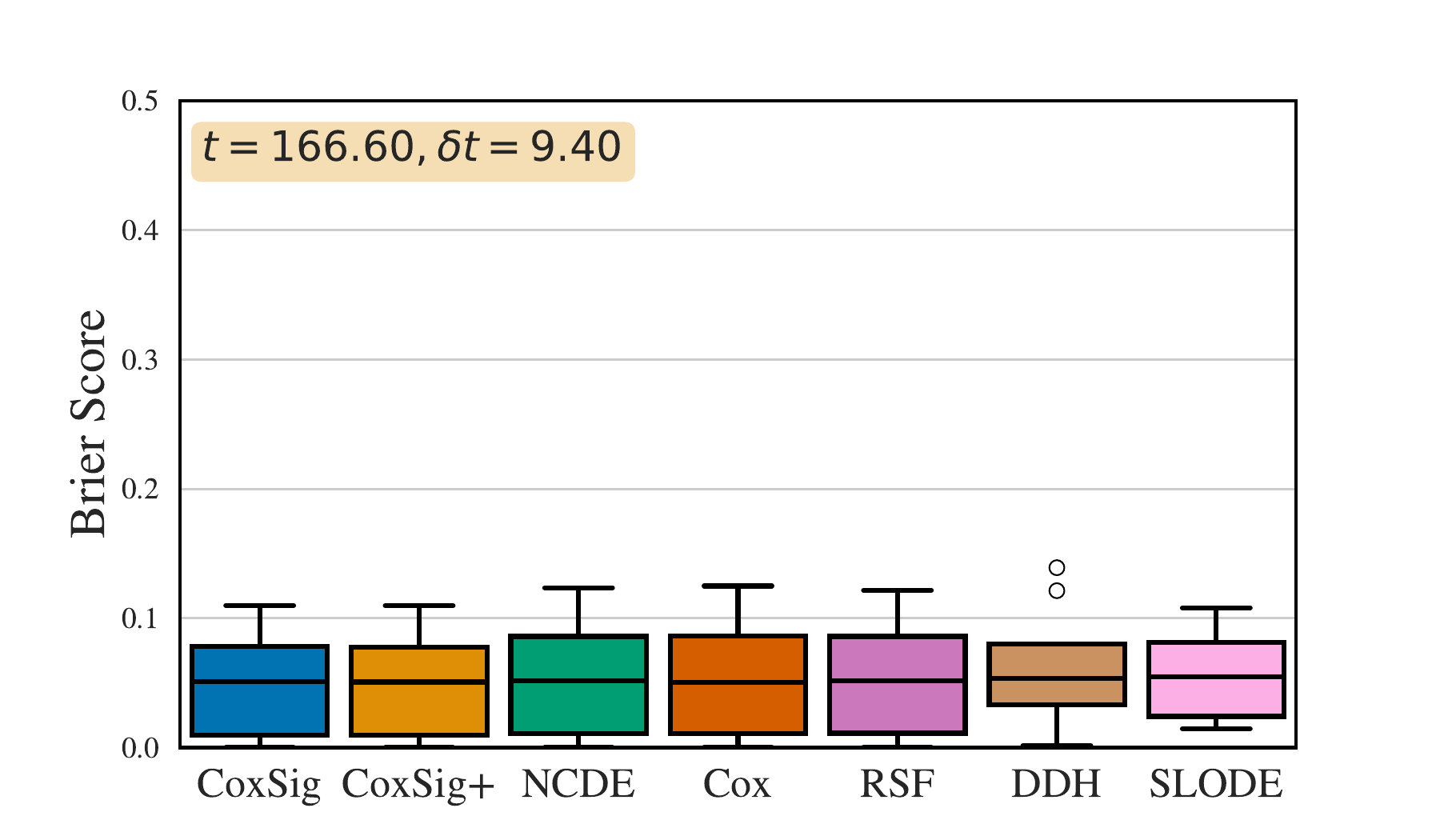}
    \includegraphics[width=0.33\textwidth]{figures/updated_figures/nasa/NASA_BS_pt=166.60,dt=16.20.pdf}
    \includegraphics[width=0.33\textwidth]{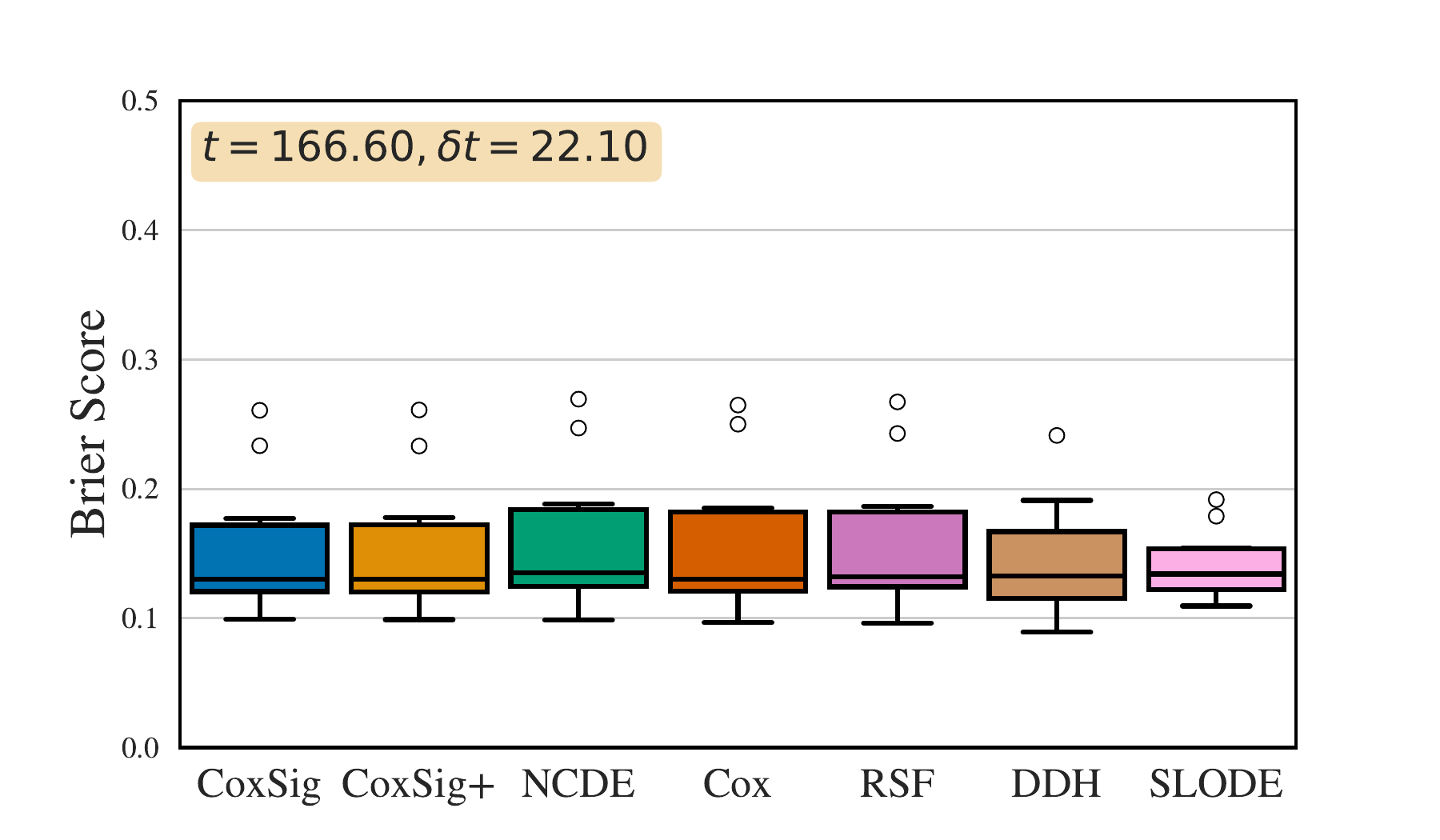}
    \includegraphics[width=0.33\textwidth]{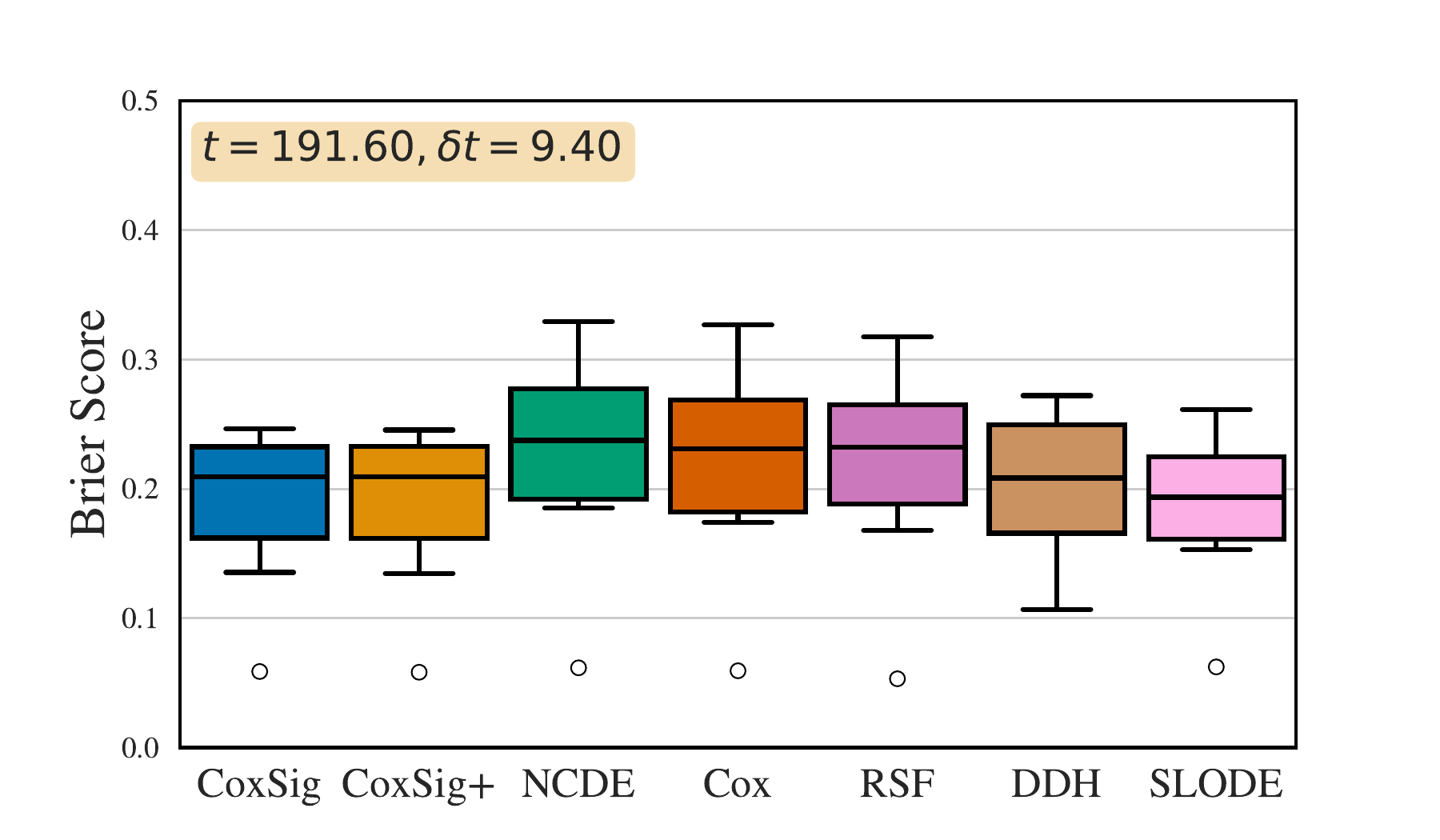}
    \includegraphics[width=0.33\textwidth]{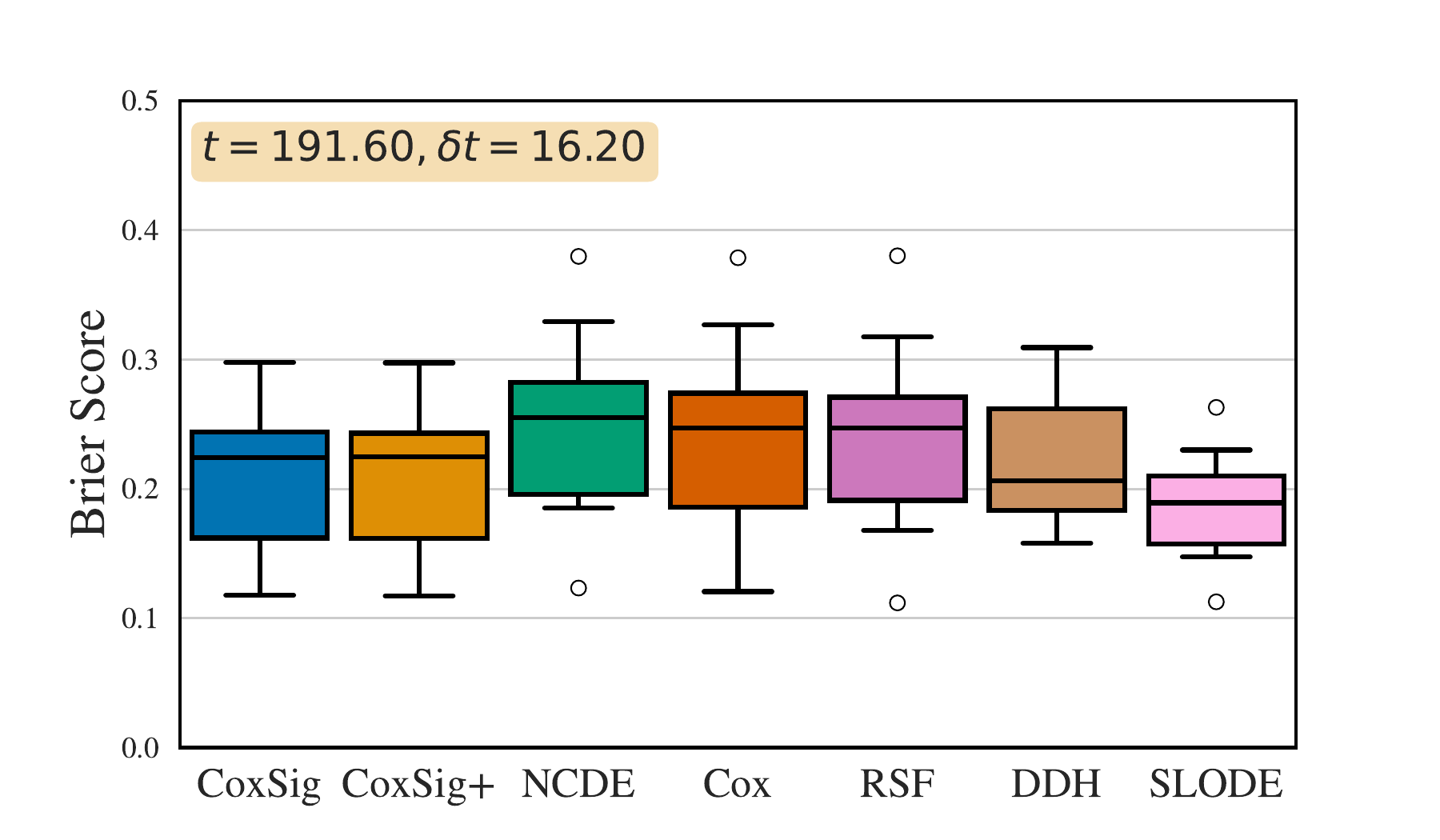}
    \includegraphics[width=0.33\textwidth]{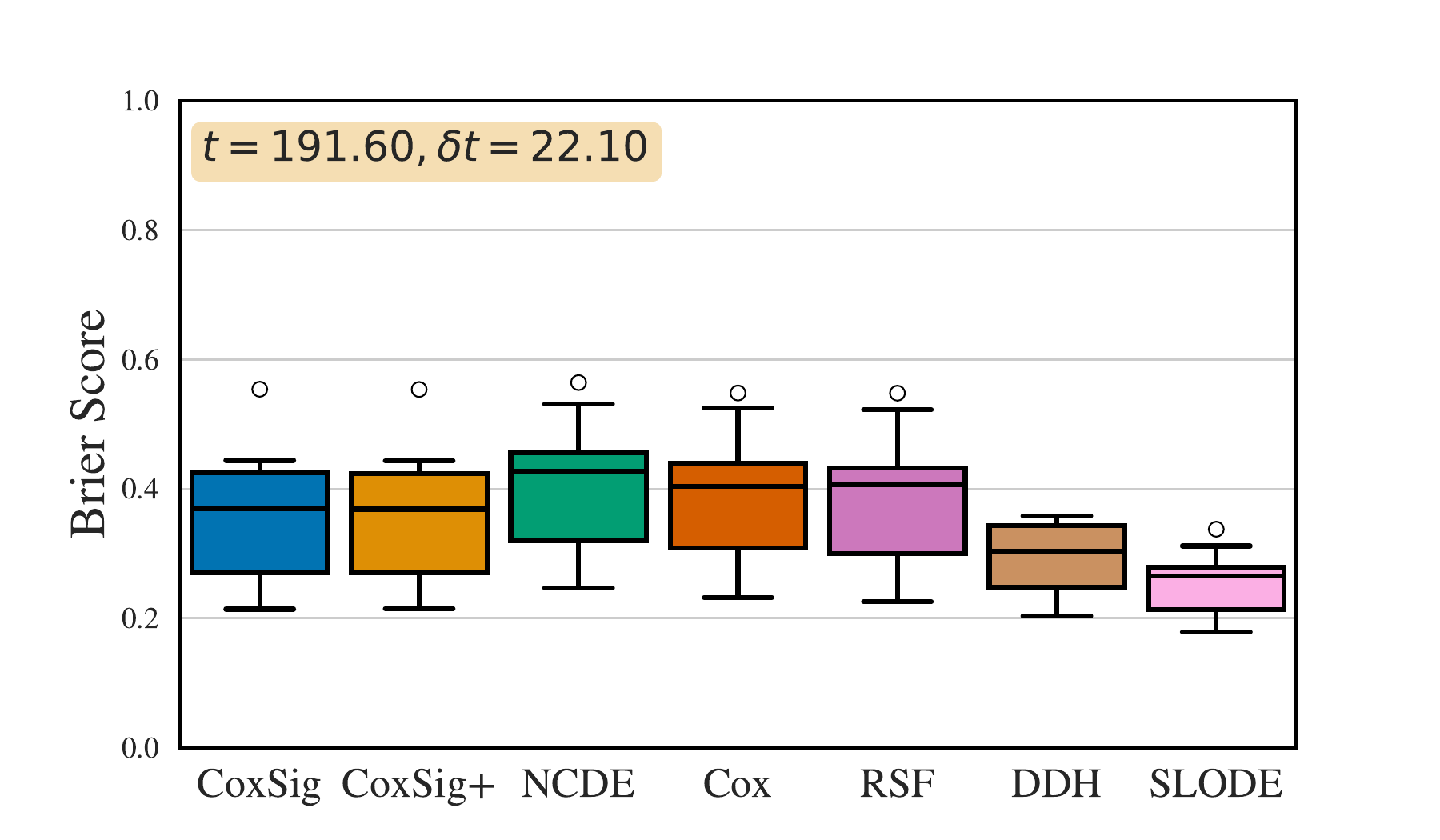}
    \caption{\footnotesize Brier Score (\textit{lower} is better) for \textbf{predictive maintenance} for numerous points $(t,\delta t)$.}
    \label{fig:bs_nasa}
\end{figure*}

\begin{figure*}
    \centering
    \includegraphics[width=0.33\textwidth]{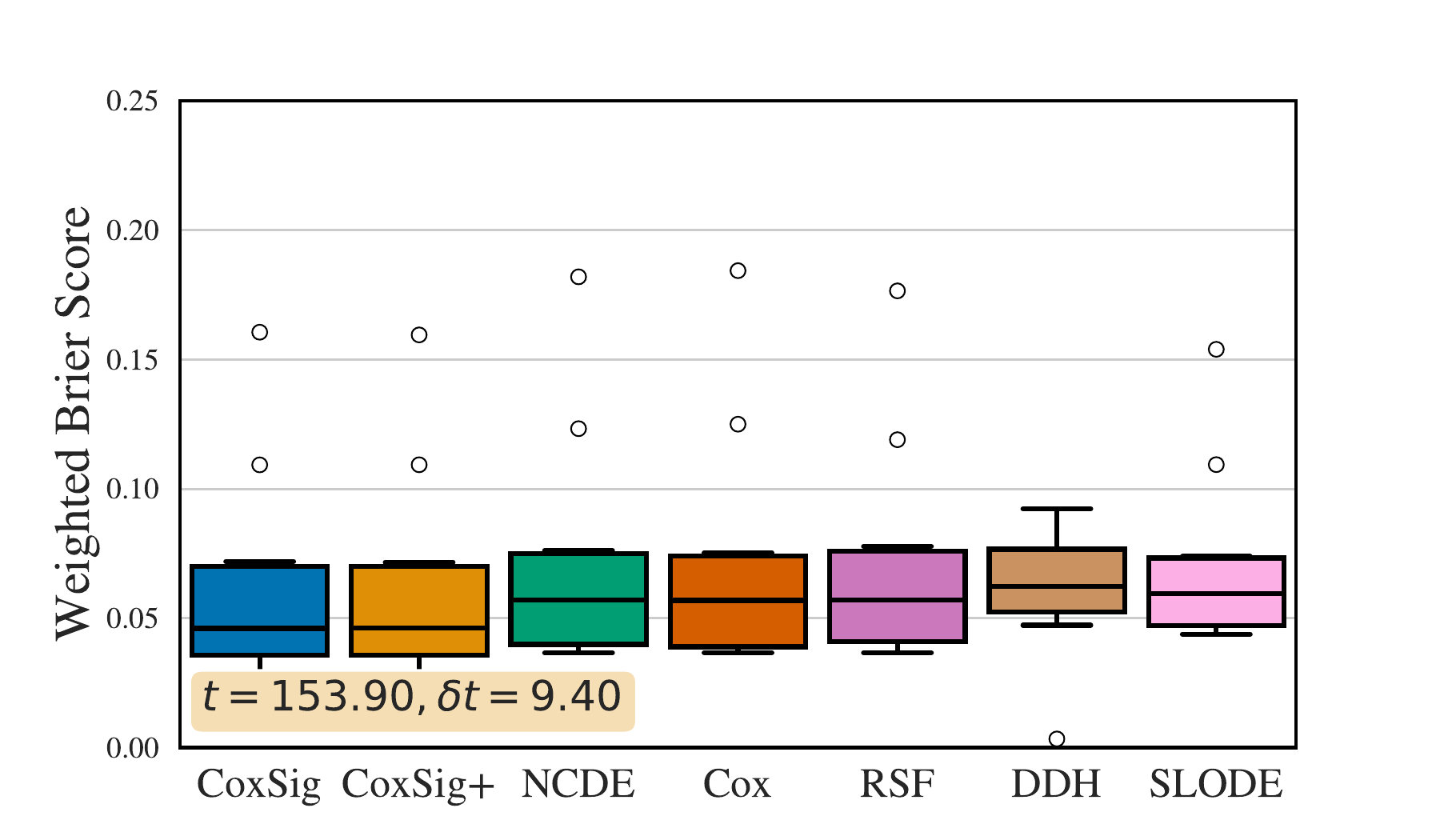}
    \includegraphics[width=0.33\textwidth]{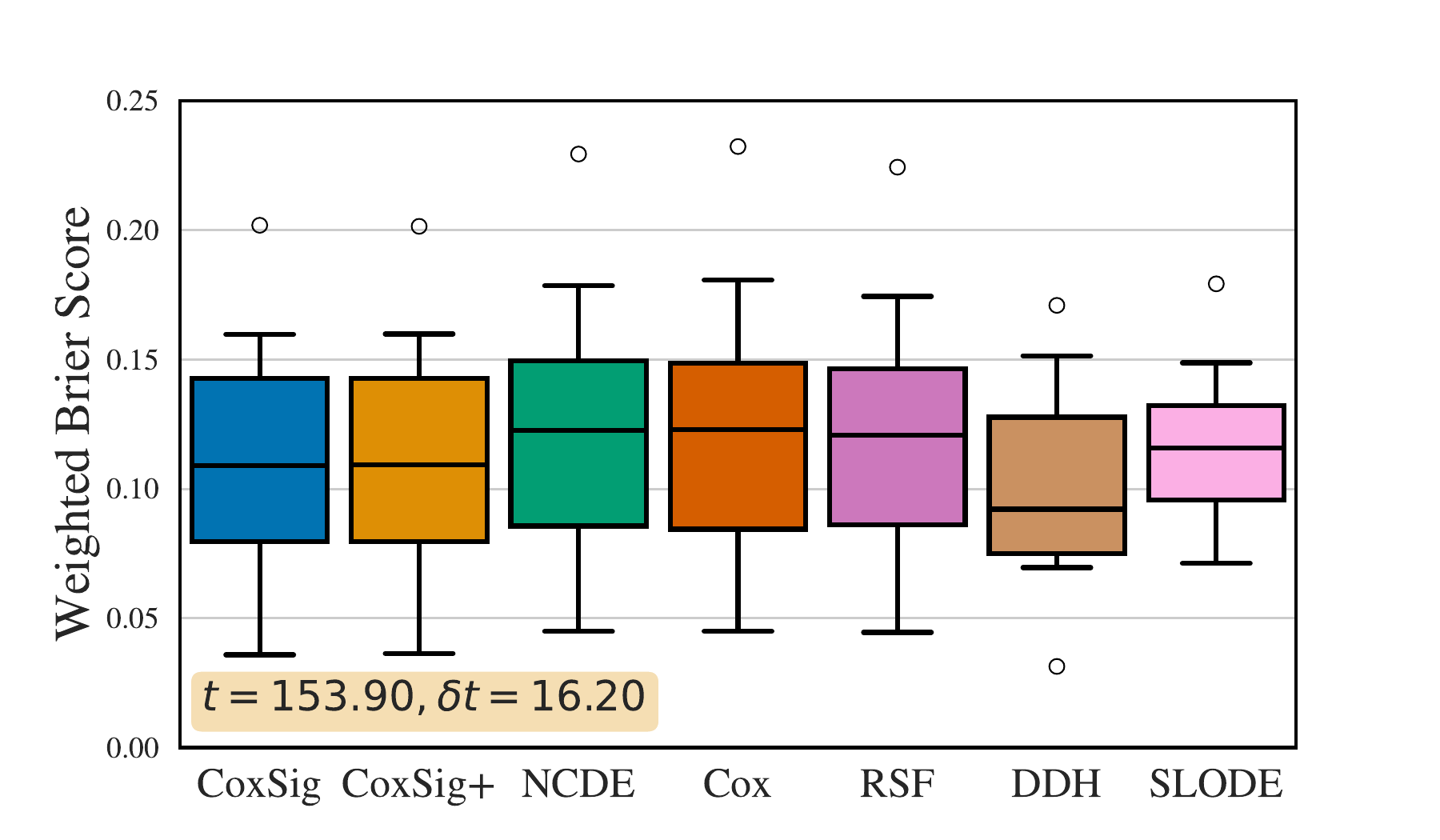}
    \includegraphics[width=0.33\textwidth]{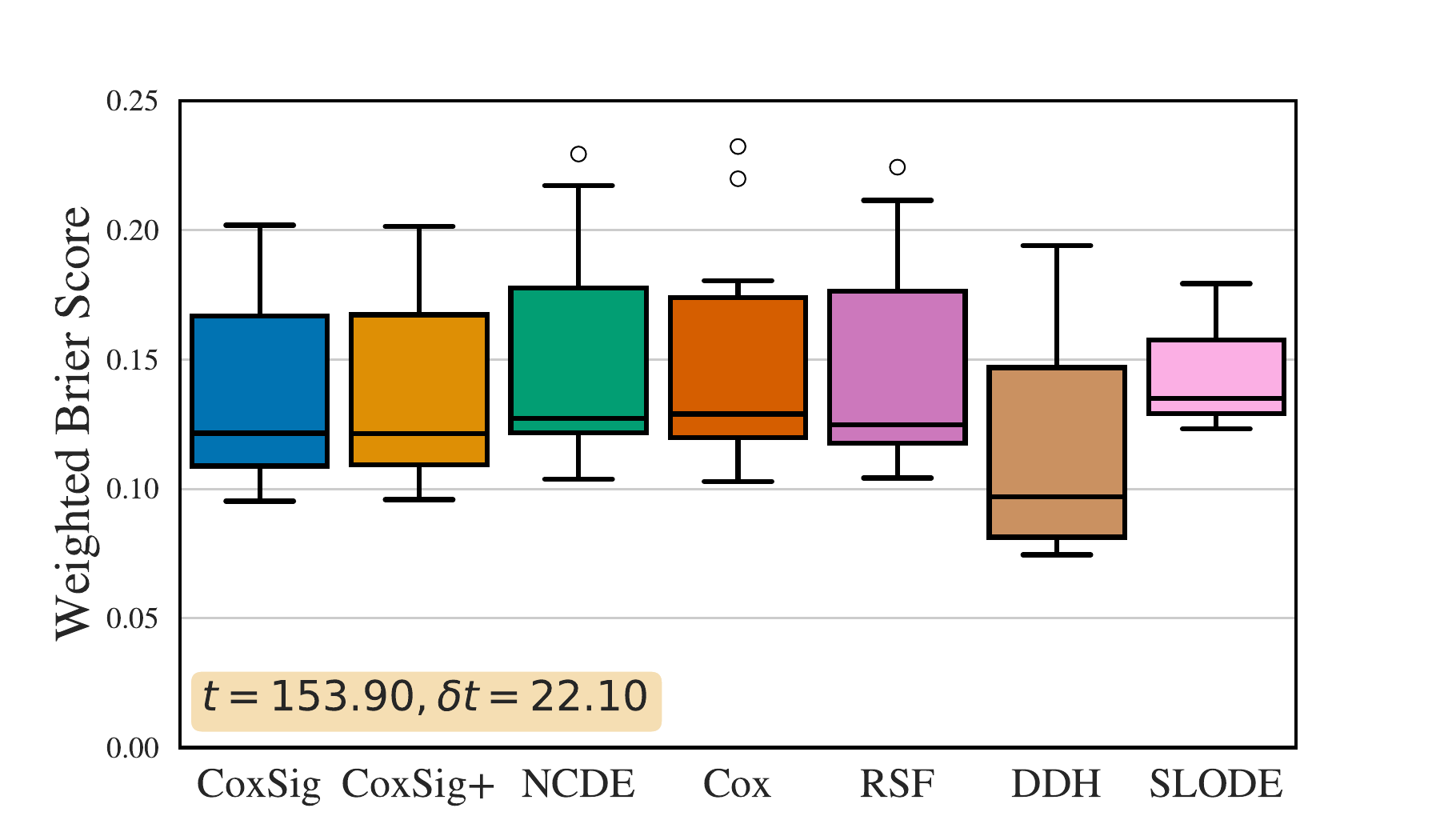}
    \includegraphics[width=0.33\textwidth]{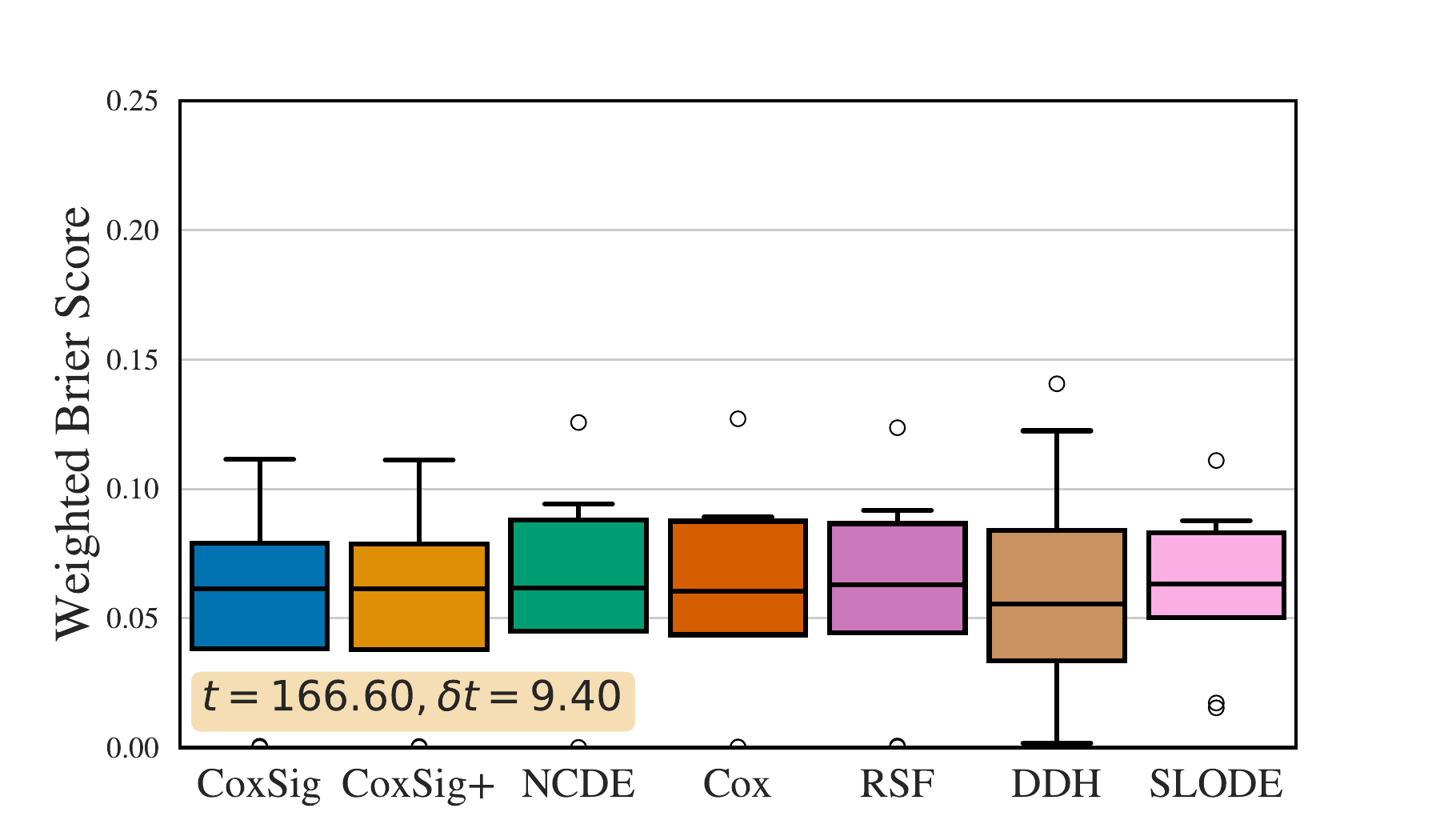}
    \includegraphics[width=0.33\textwidth]{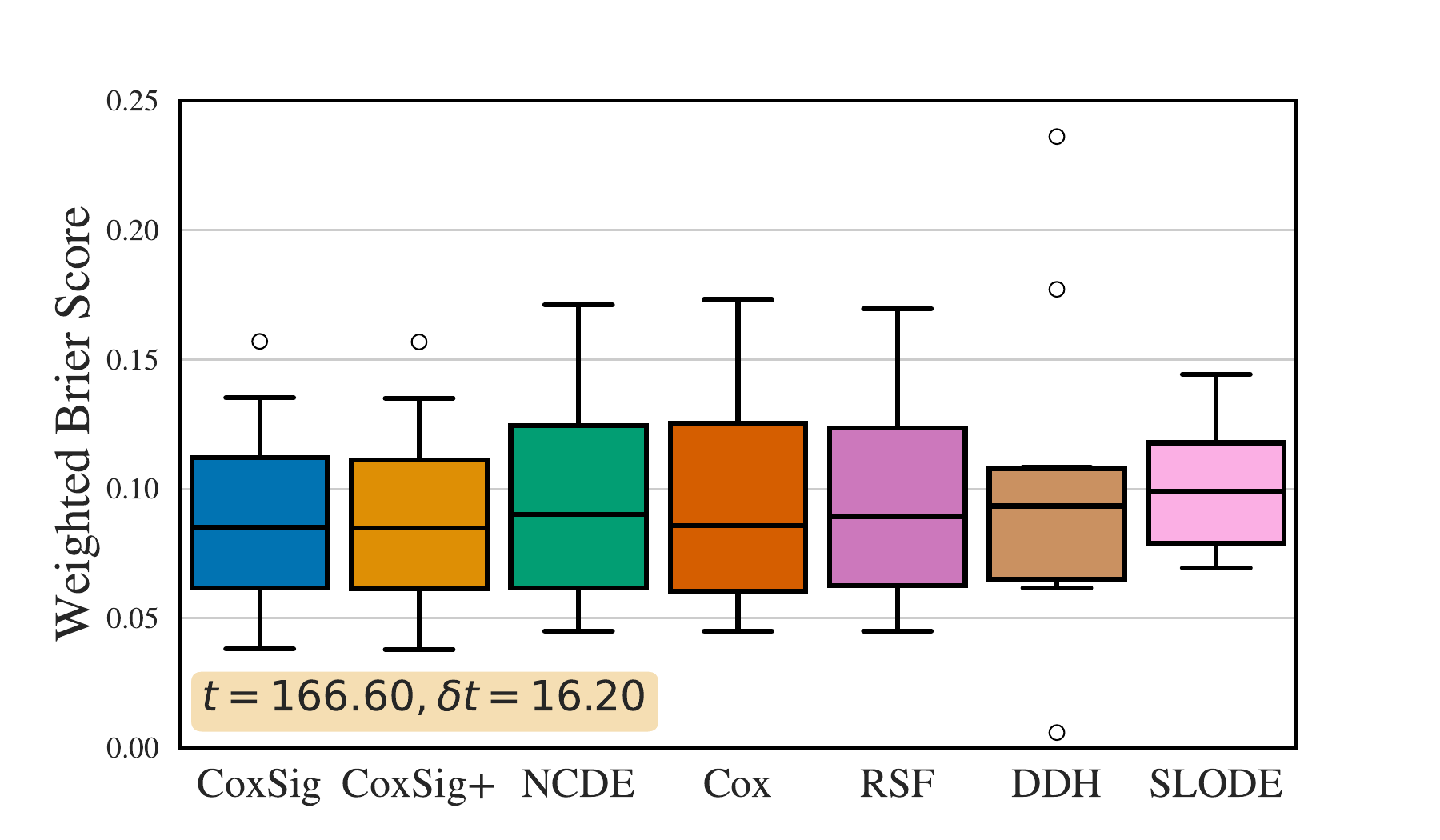}
    \includegraphics[width=0.33\textwidth]{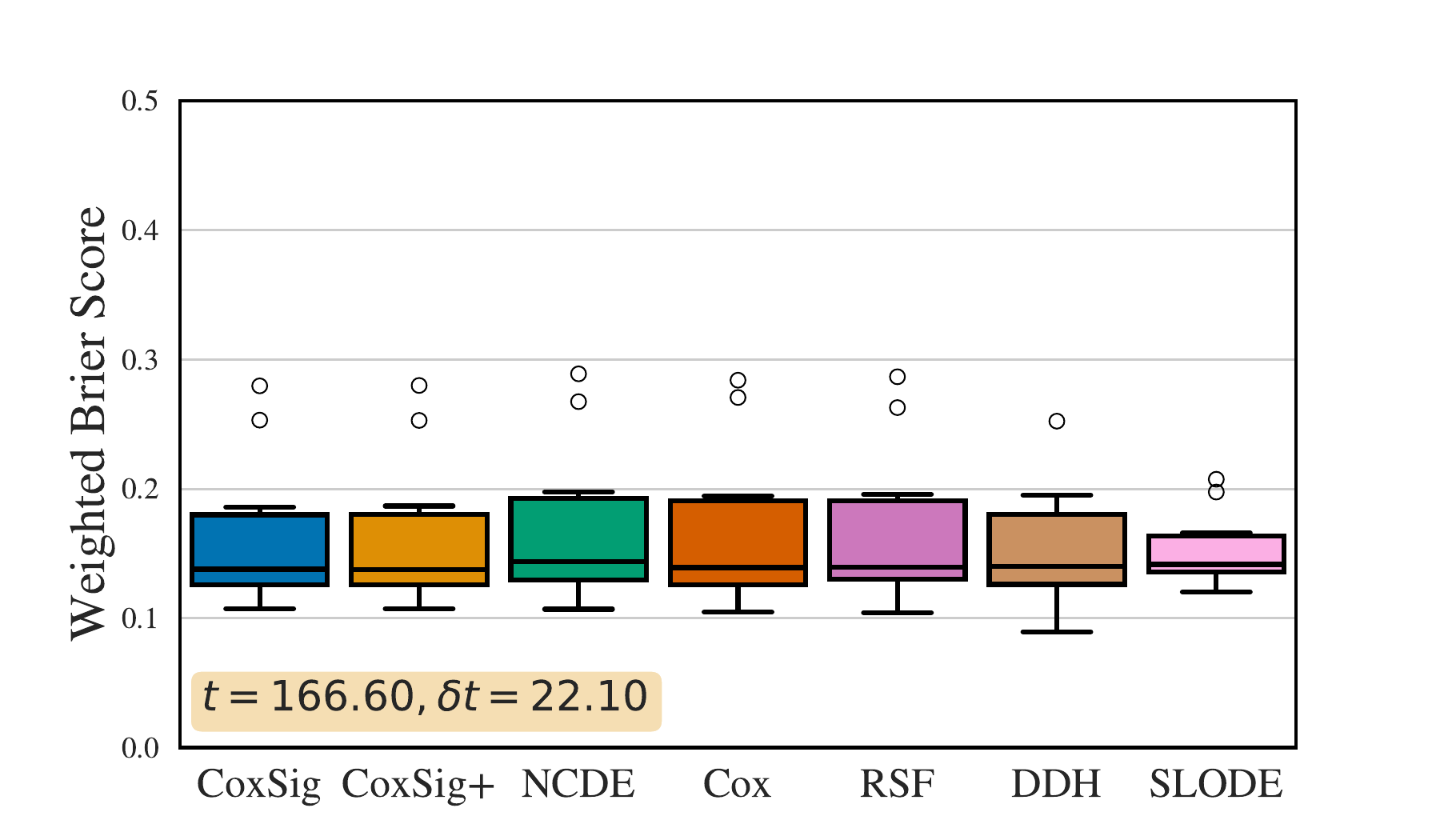}
    \includegraphics[width=0.33\textwidth]{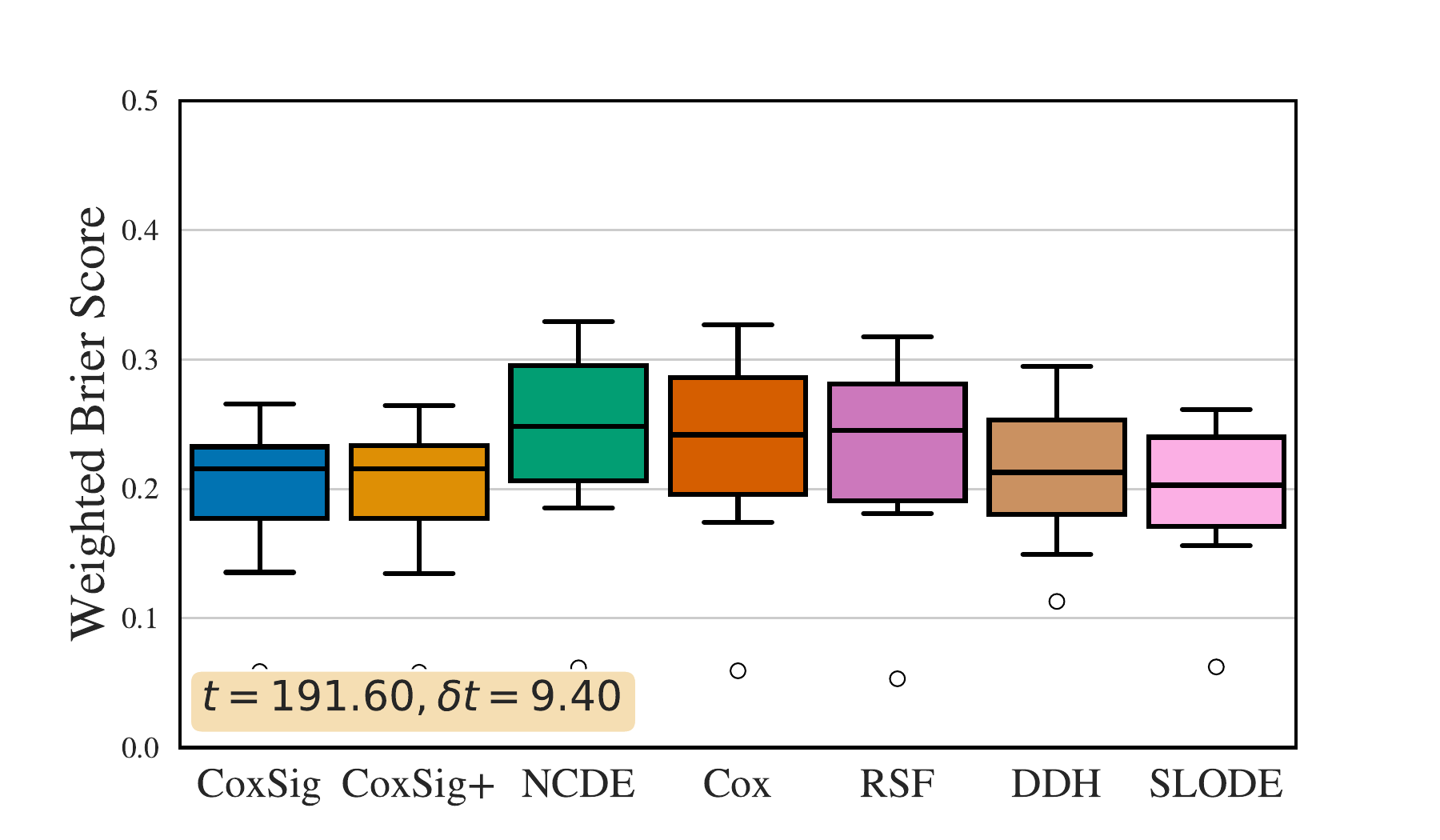}
    \includegraphics[width=0.33\textwidth]{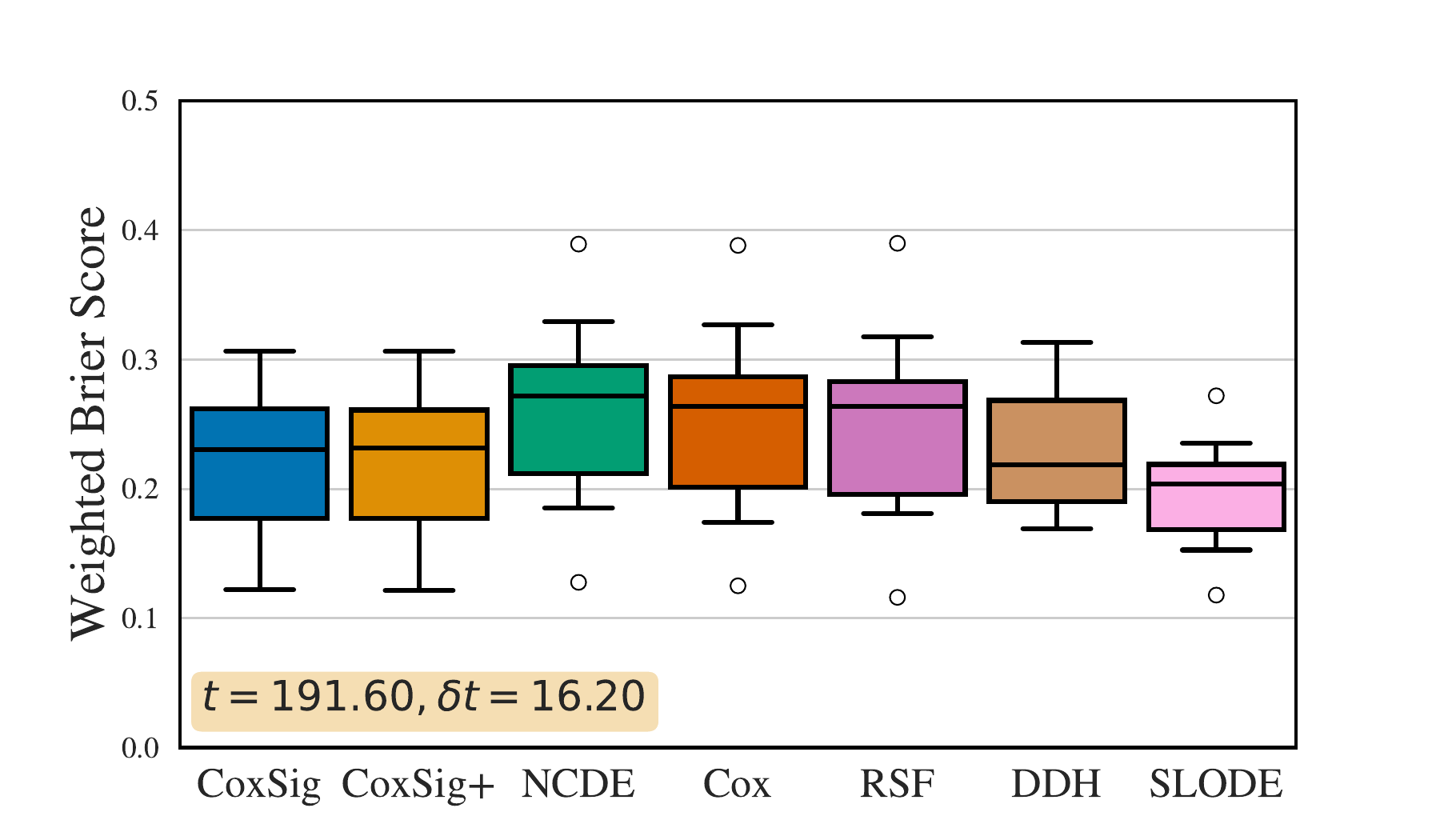}
    \includegraphics[width=0.33\textwidth]{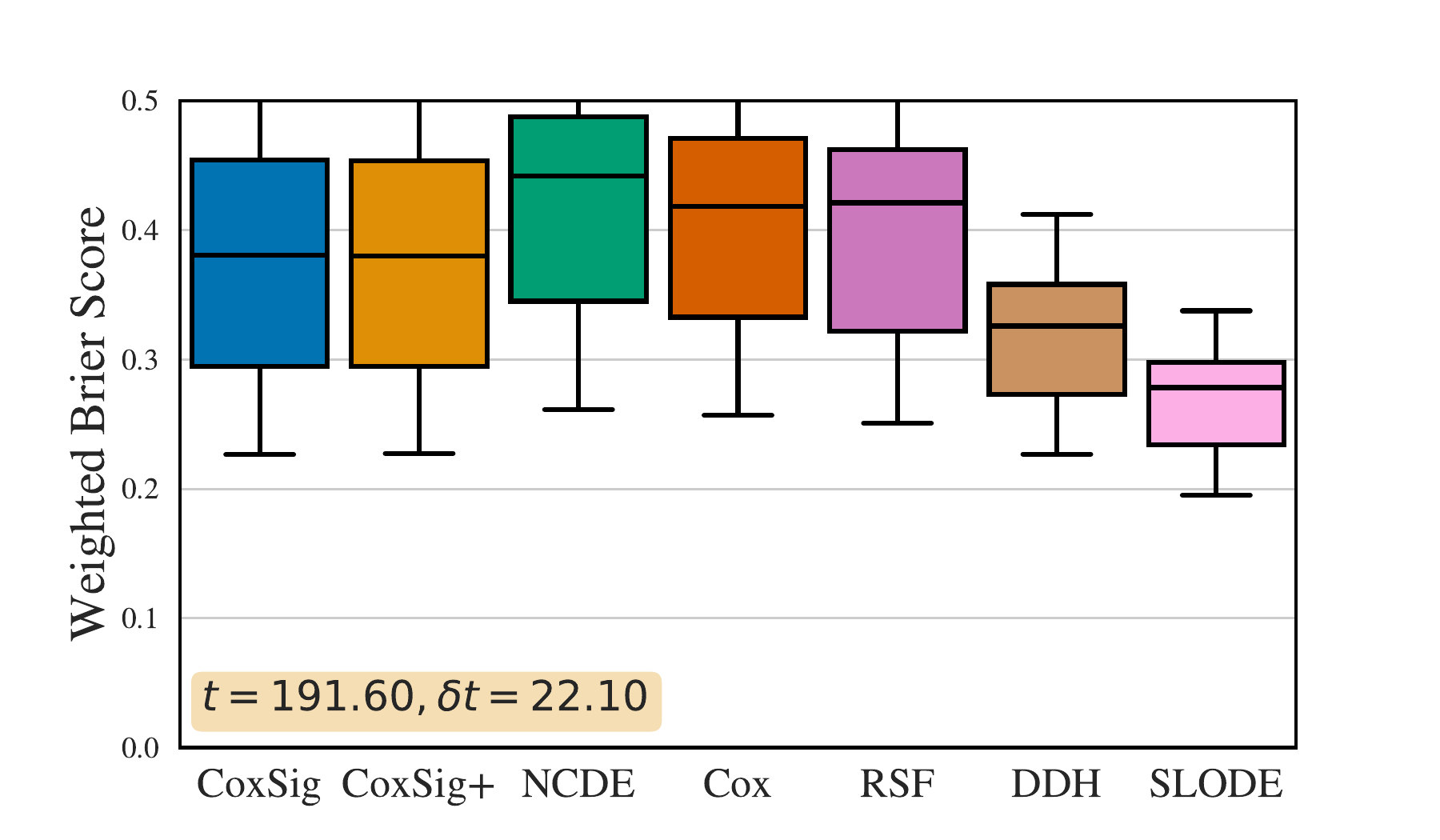}
    \caption{\footnotesize Weighted Brier Score (\textit{lower} is better) for \textbf{predictive maintenance} for numerous points $(t,\delta t)$.}
    \label{fig:wbs_nasa}
\end{figure*}

\begin{figure*}
    \centering
    \includegraphics[width=0.33\textwidth]{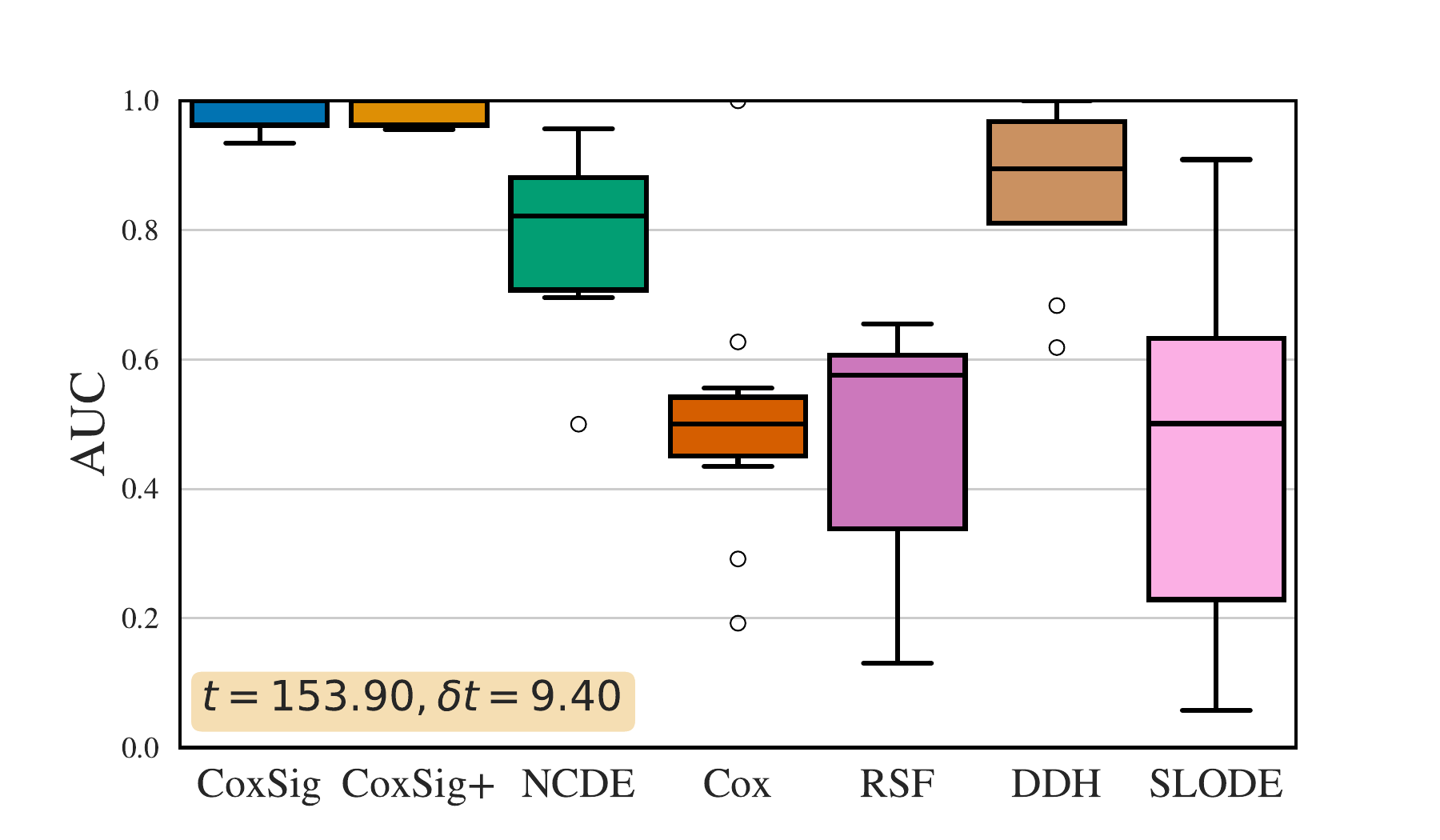}
    \includegraphics[width=0.33\textwidth]{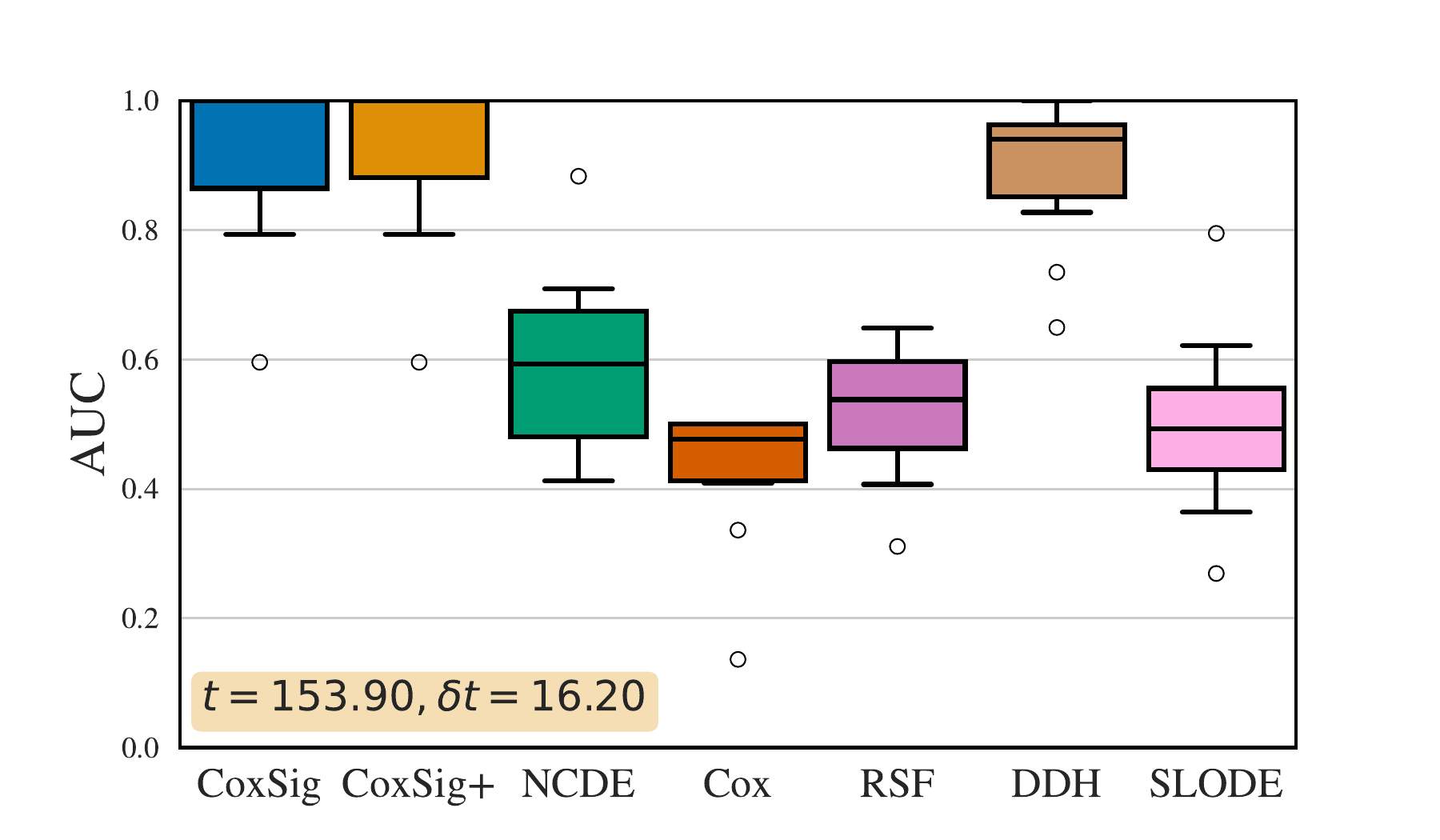}
    \includegraphics[width=0.33\textwidth]{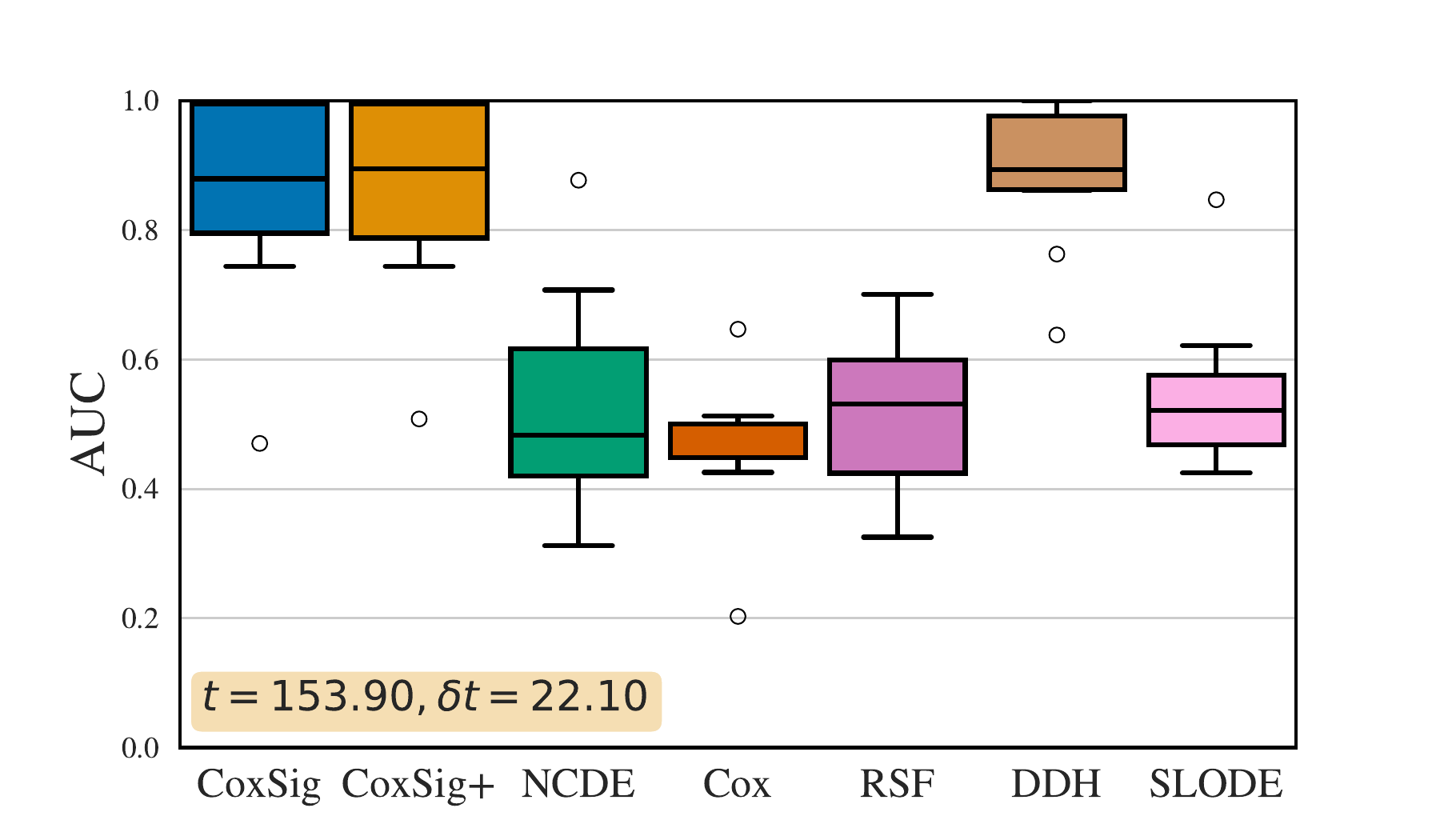}
    \includegraphics[width=0.33\textwidth]{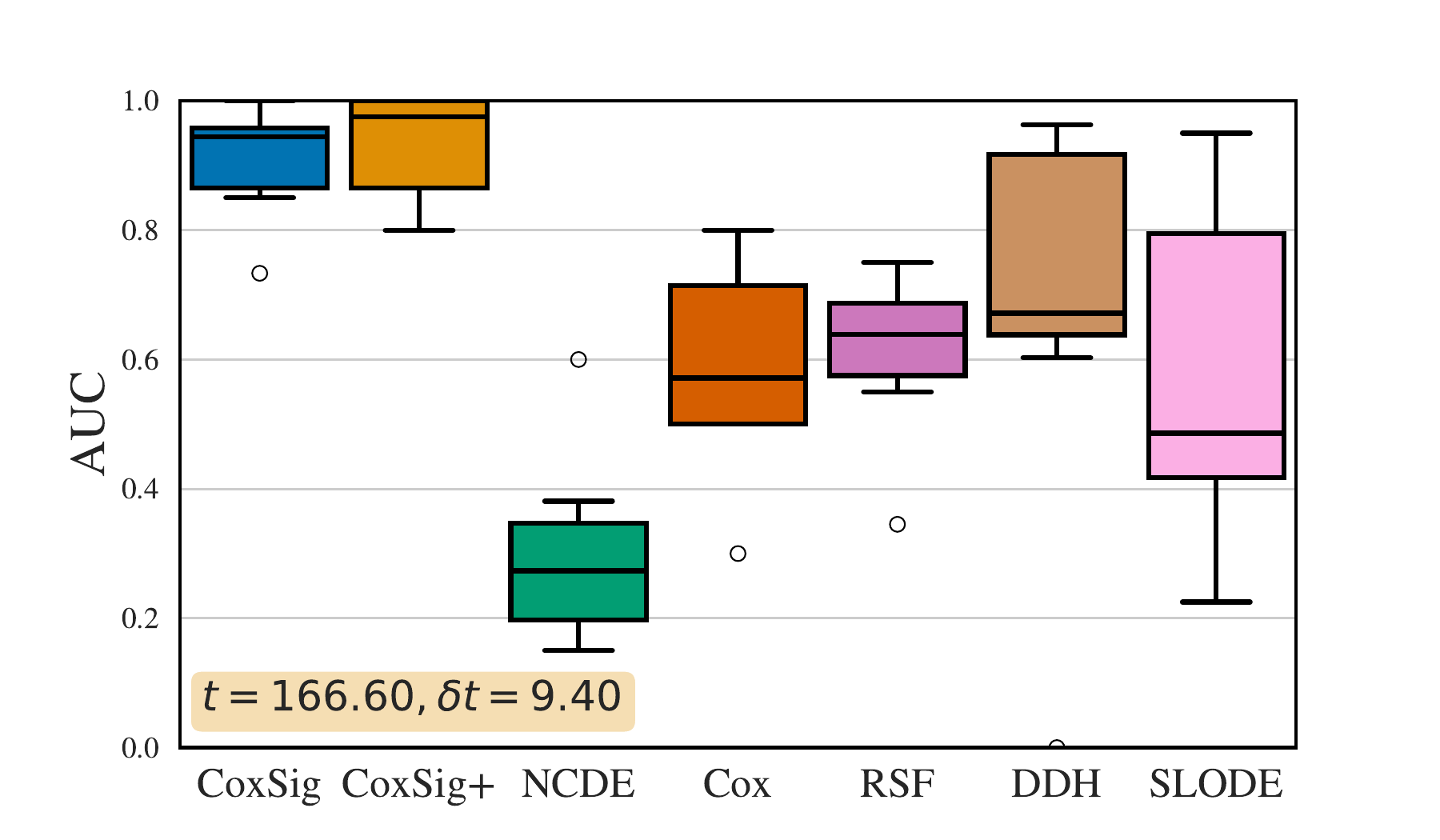}
    \includegraphics[width=0.33\textwidth]{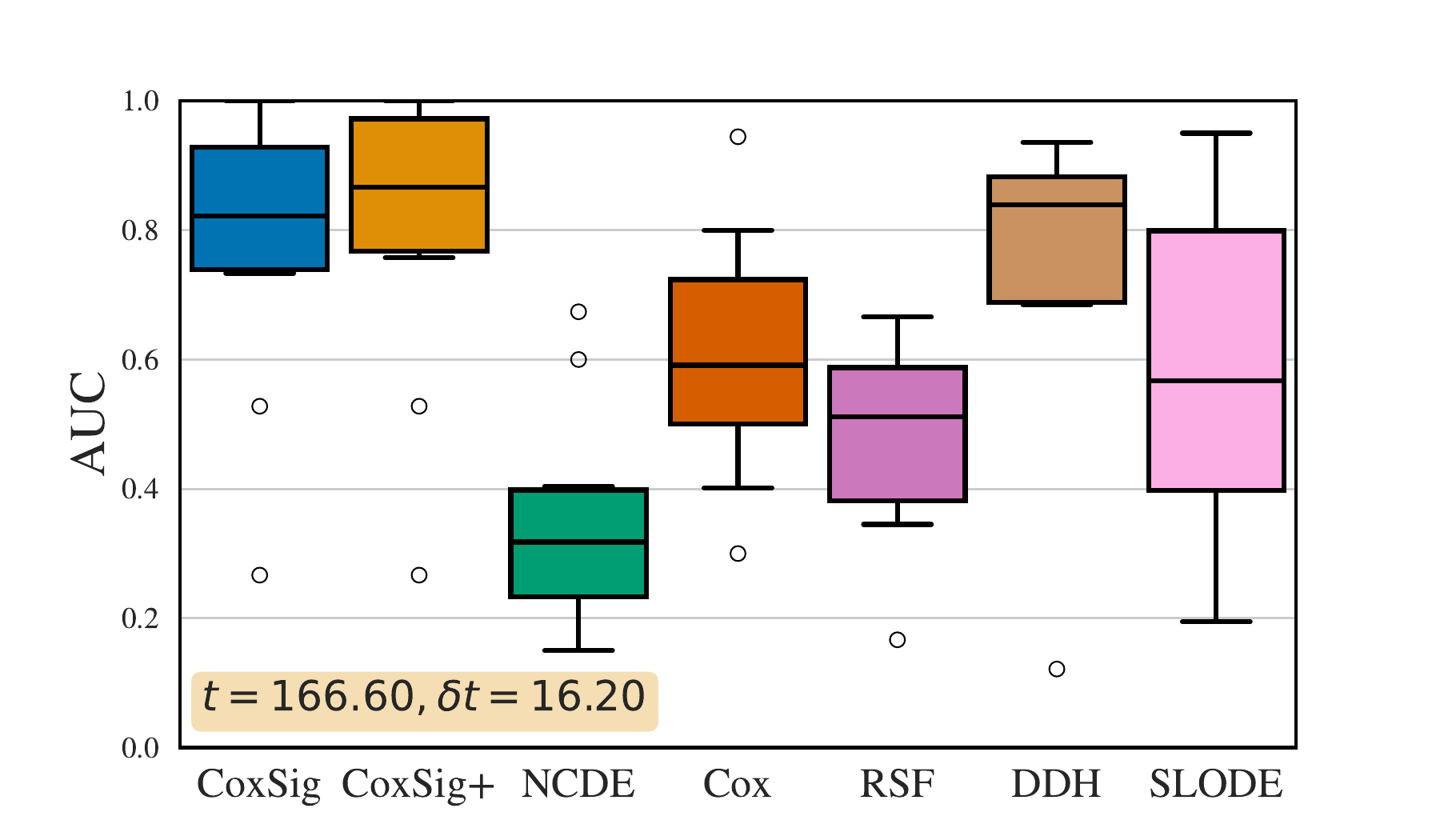}
    \includegraphics[width=0.33\textwidth]{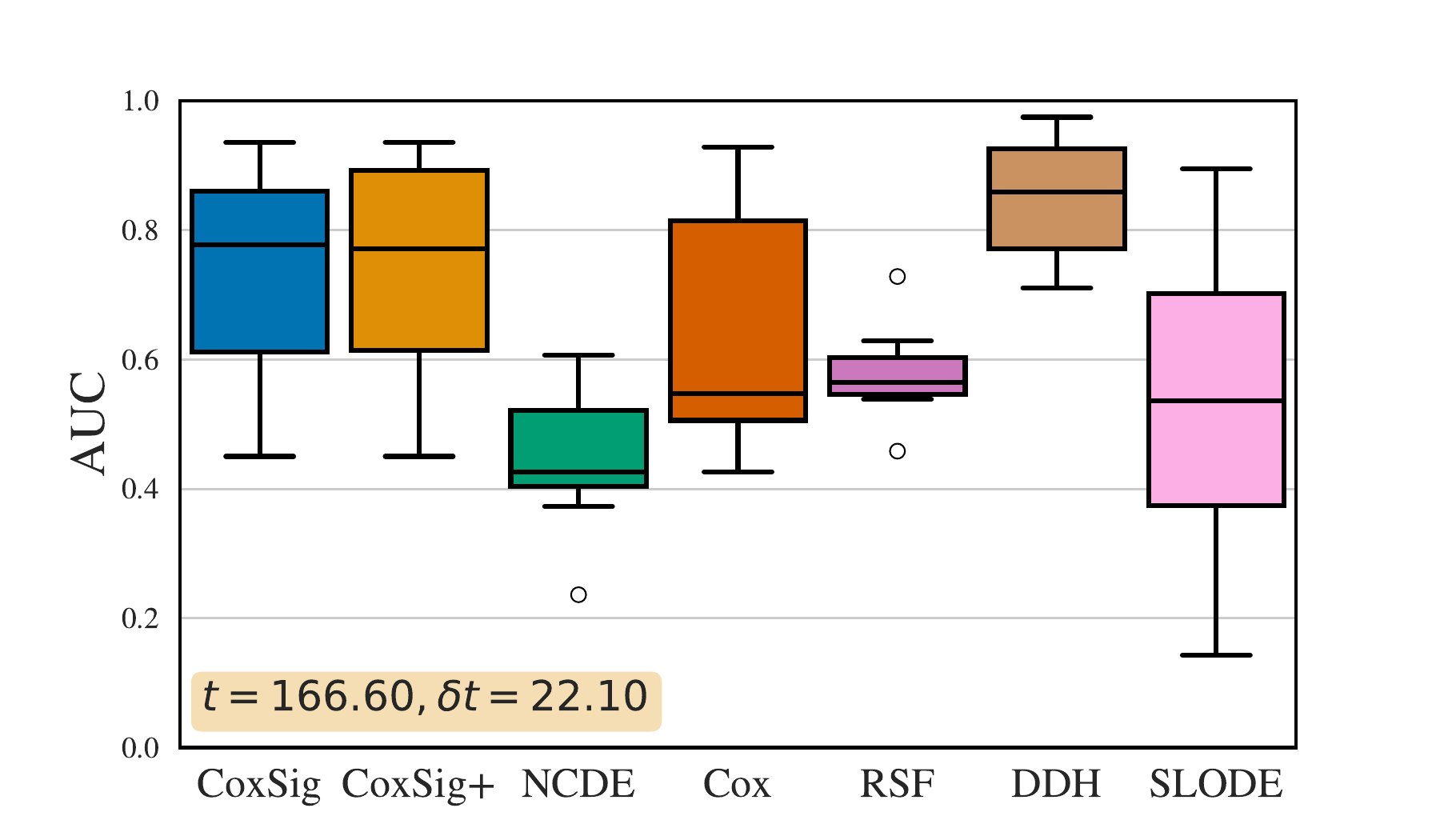}
    \includegraphics[width=0.33\textwidth]{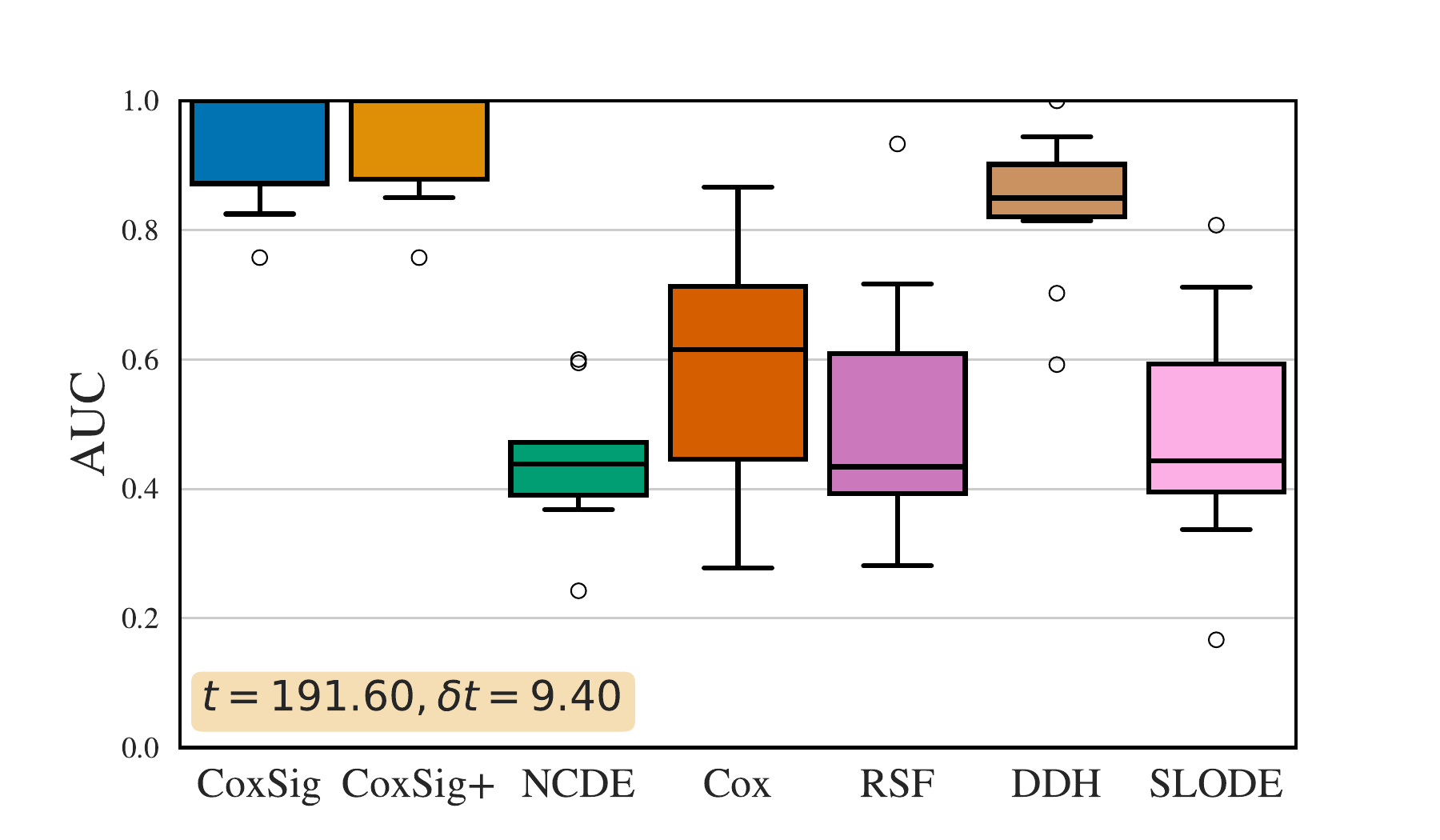}
    \includegraphics[width=0.33\textwidth]{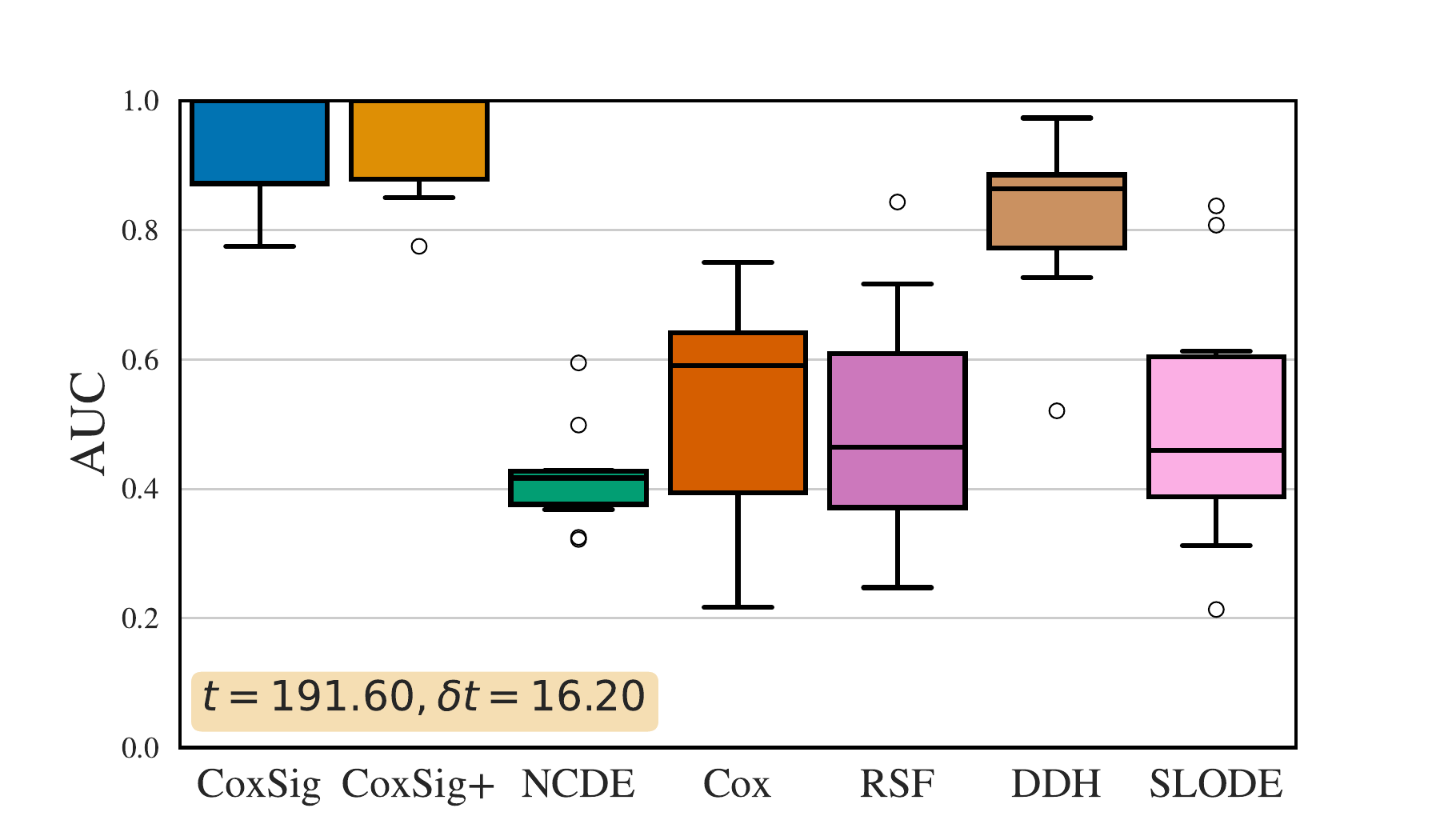}
    \includegraphics[width=0.33\textwidth]{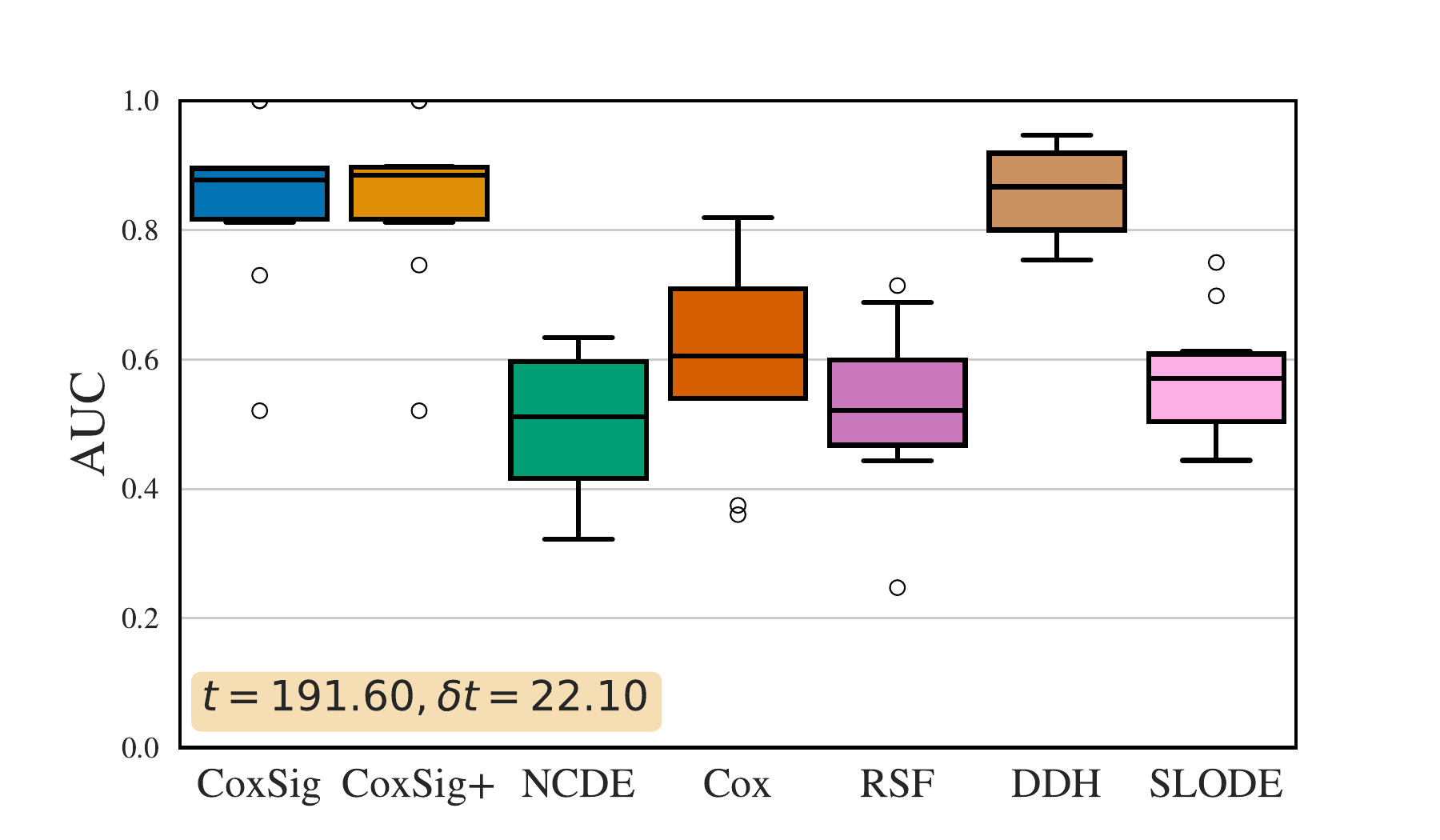}
    \caption{\footnotesize AUC (\textit{higher} is better) for \textbf{predictive maintenance} for numerous points $(t,\delta t)$.}
    \label{fig:auc_nasa}
\end{figure*}

\subsection{Churn Prediction}

For this dataset, the amount of details that we can release is limited both because of the sensitive nature of the data and of the anonymity requirements of the reviewing process. 

\paragraph{Time series.} All longitudinal features have been computed on a temporal window of one week, the raw data corresponding to all product orders placed on the platform from 06-12-2021 to 12-11-2023. For clients who have no order during the week, we fill zero value for all longitudinal measurements this week. After removing features with more than 90 $\%$ of missingness, 14 longitudinal features of 1043 clients are selected for the training step. Note that we apply standardization for selected features before training.

\paragraph{Event definition.} We consider that a customer has churned if she has no passed any order in the last 4 weeks. If the customer starts ordering again after a churn, we register her as a new customer. 

\paragraph{Censorship.} Censorship is terminal based on the data collection period (give dates here). Hence any customer that has not churned by 12-11-2023 is censored. In this dataset, 38.4$\%$ of the clients are terminally censored.

\paragraph{Supplementary Figures.} Figure \ref{fig:califrais_appendix_path} provides an example of four sample paths of four randomly chosen individuals. We add additional results in Figures \ref{fig:c_index_churn}, \ref{fig:bs_churn}, \ref{fig:wbs_churn} and \ref{fig:auc_churn}.

\begin{figure*}
    \centering
    \includegraphics[width=0.49\textwidth]{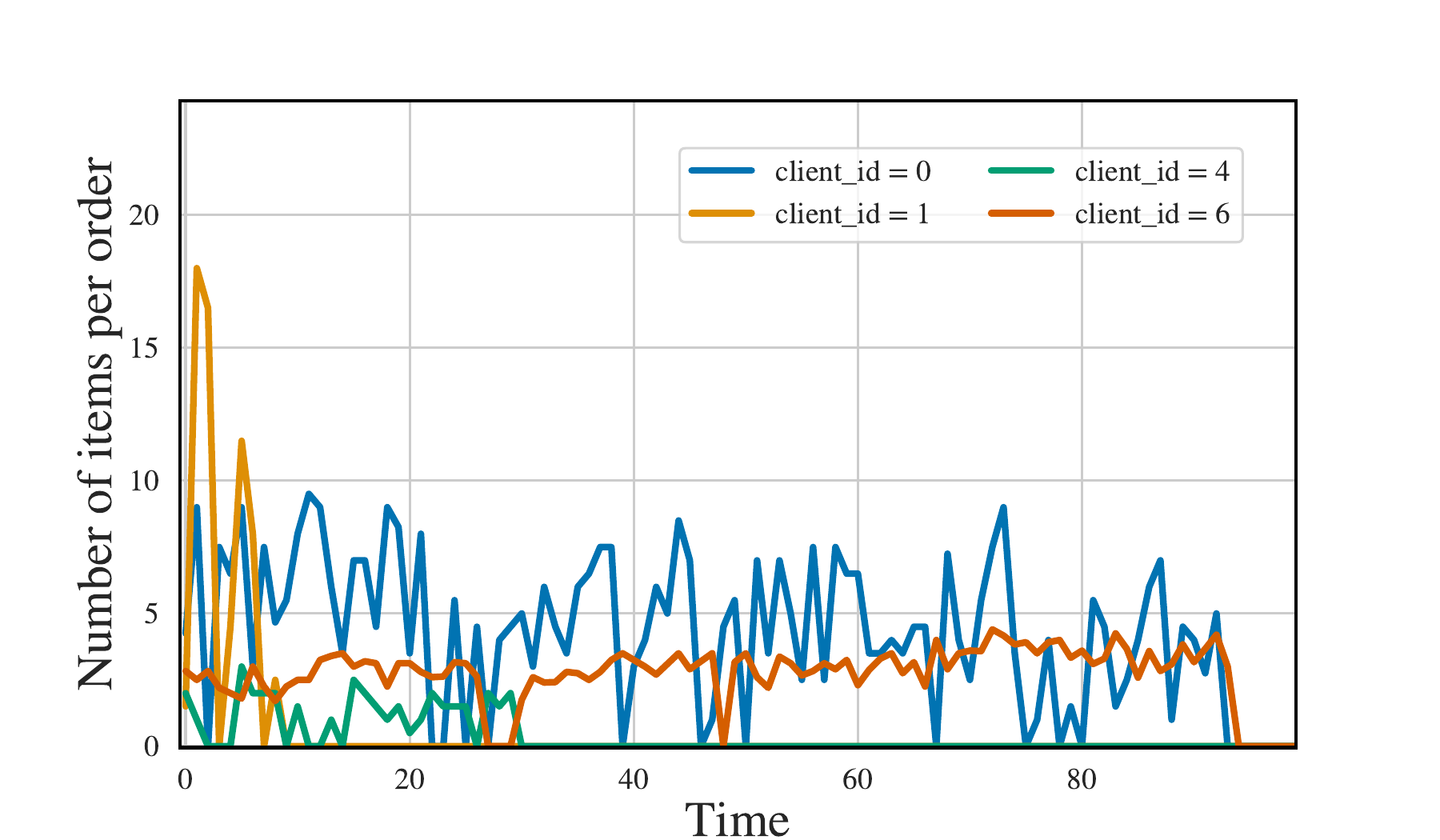}
    \includegraphics[width=0.49\textwidth]{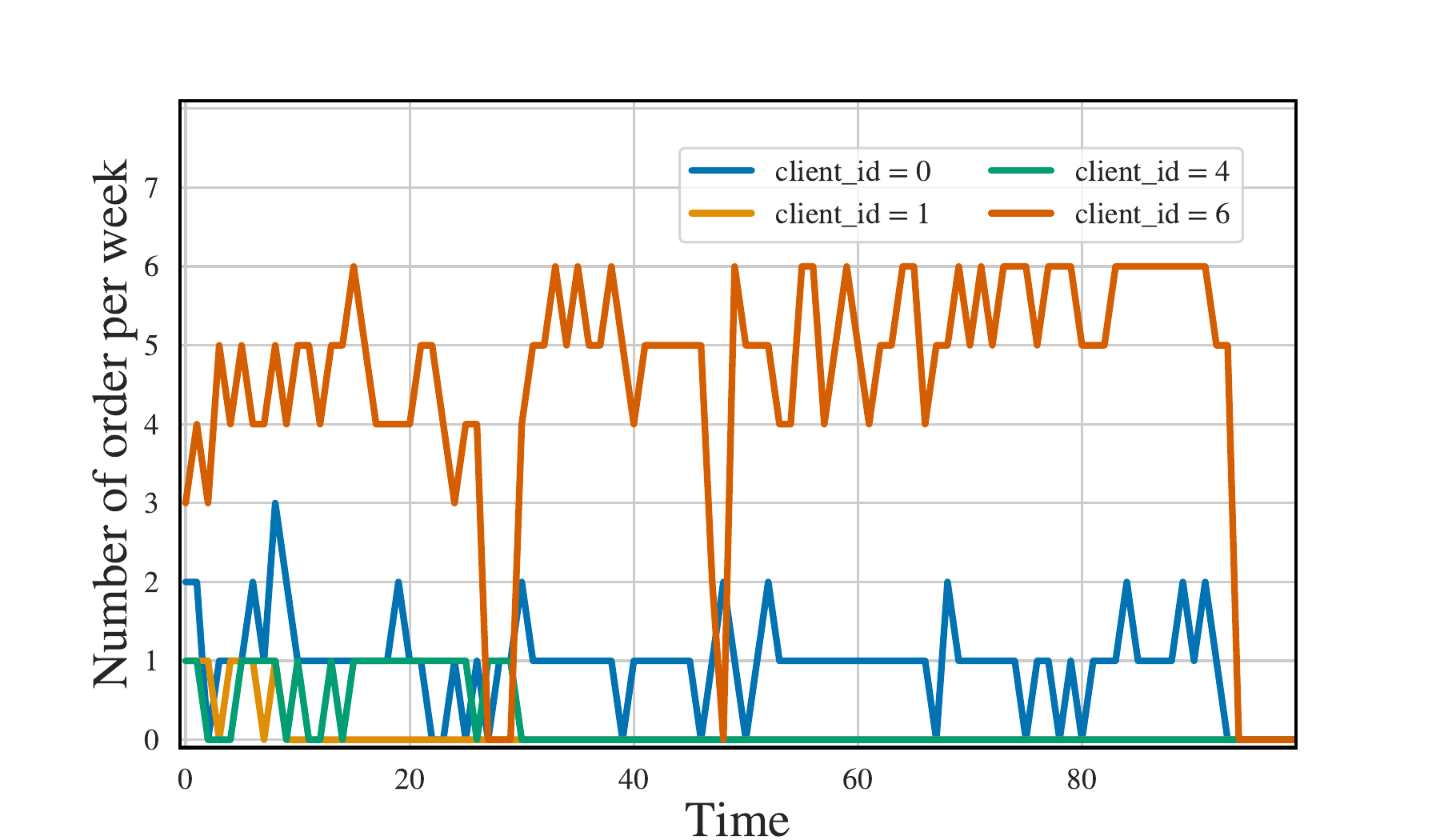}
    \includegraphics[width=0.49\textwidth]{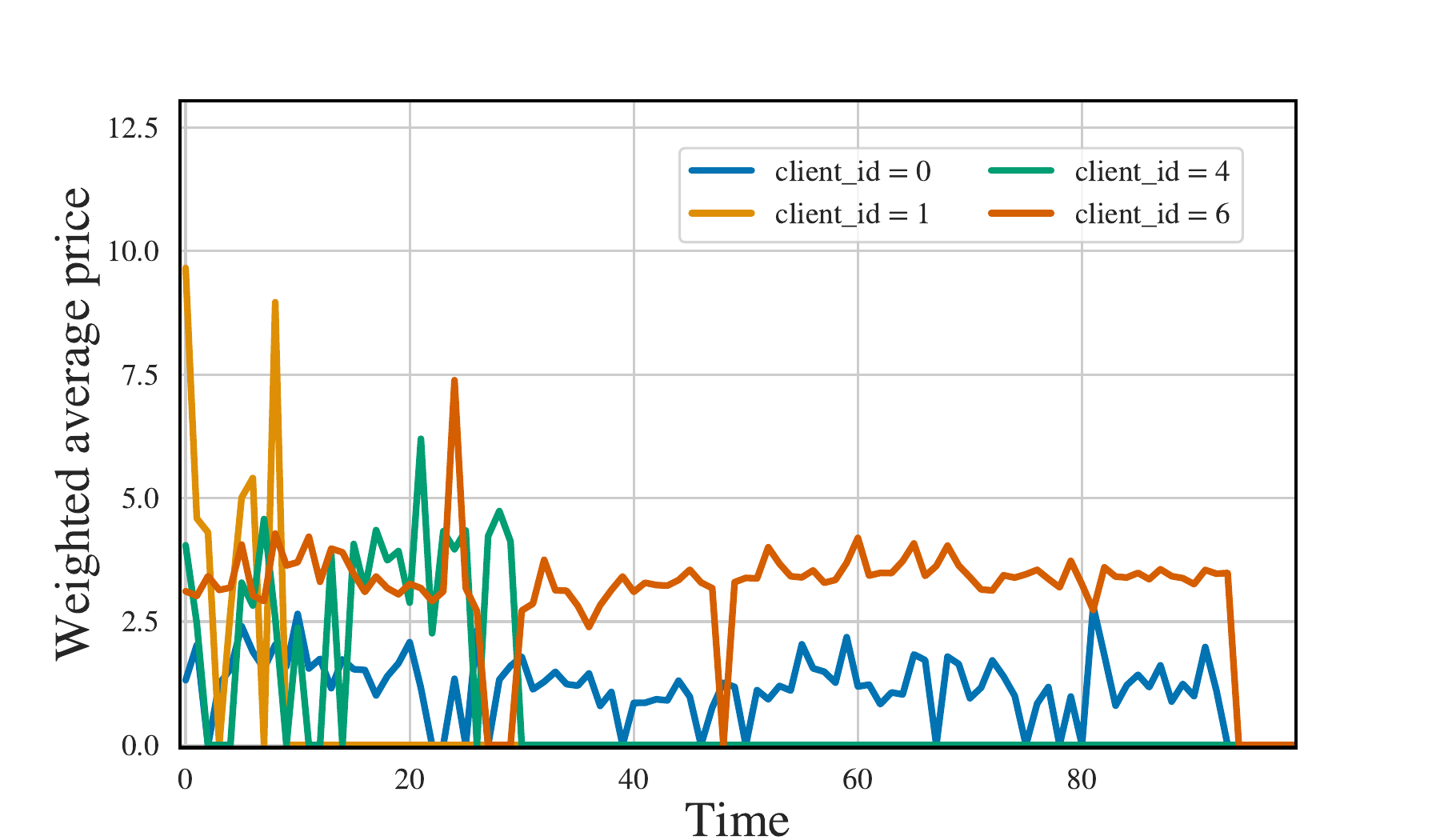}
    \includegraphics[width=0.49\textwidth]{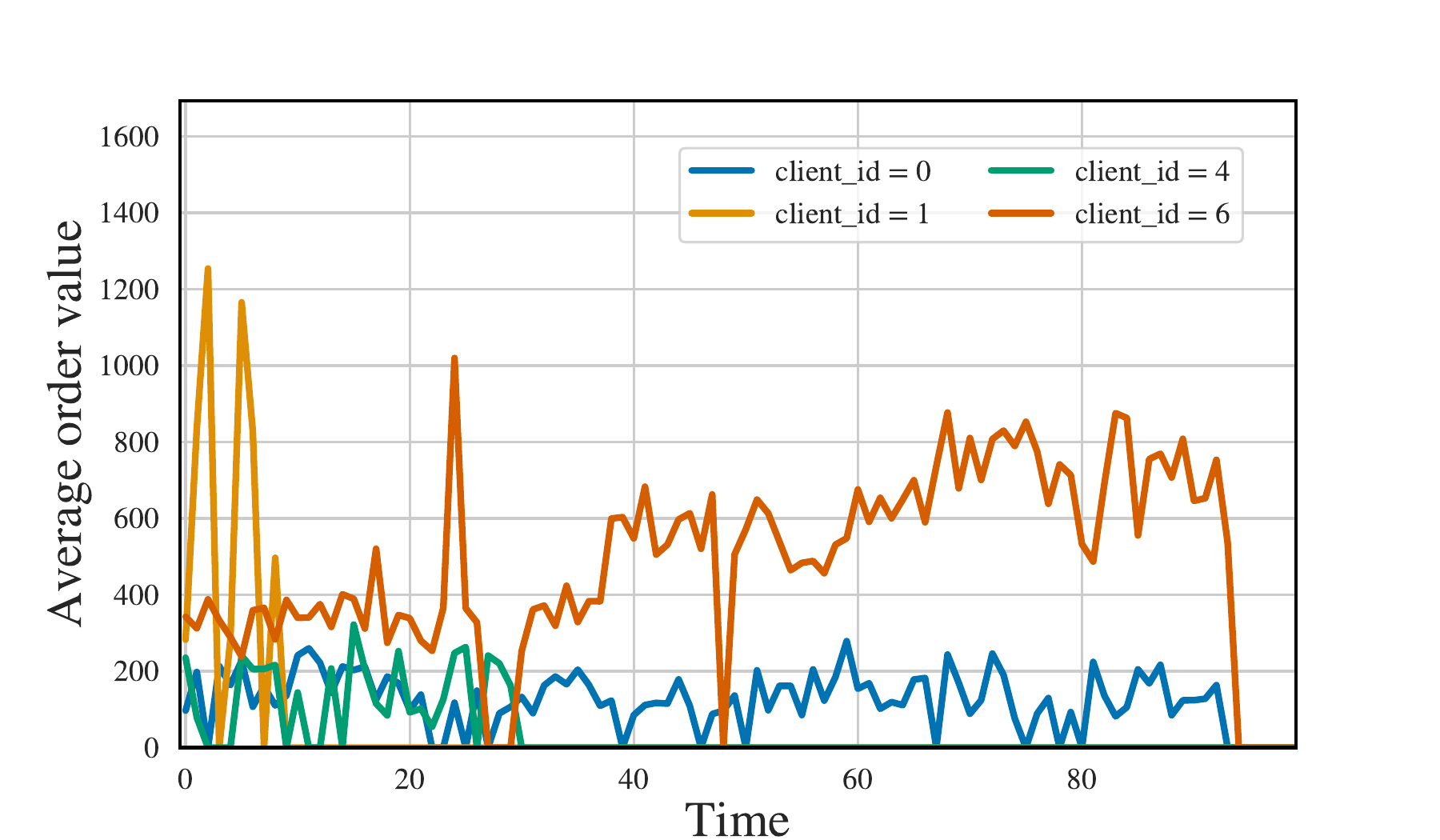}    
    \label{fig:Califrais_appendix_path_hist}
    \caption{\footnotesize Values of 4 different time-dependent features for 4 randomly chosen individuals from the \textbf{churn prediction} dataset. Individual time-to-event and distribution of the event times cannot be displayed to protect consumer and business privacy. A precise description of the different time-dependent features will be provided upon publication.}
    \label{fig:califrais_appendix_path}
\end{figure*}

\begin{figure*}[h!]
    \centering
    \includegraphics[width=0.33\textwidth]{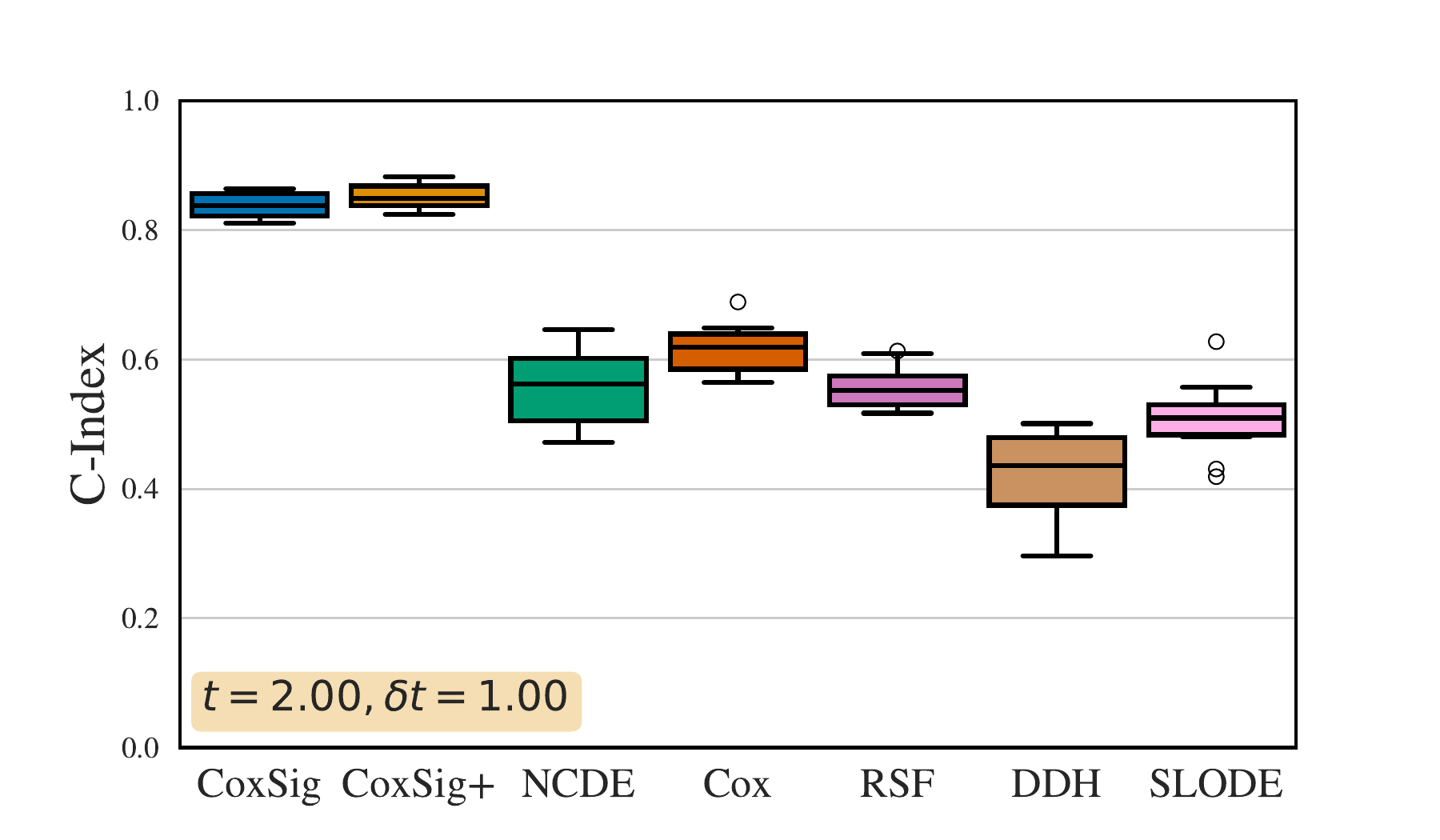}
    \includegraphics[width=0.33\textwidth]{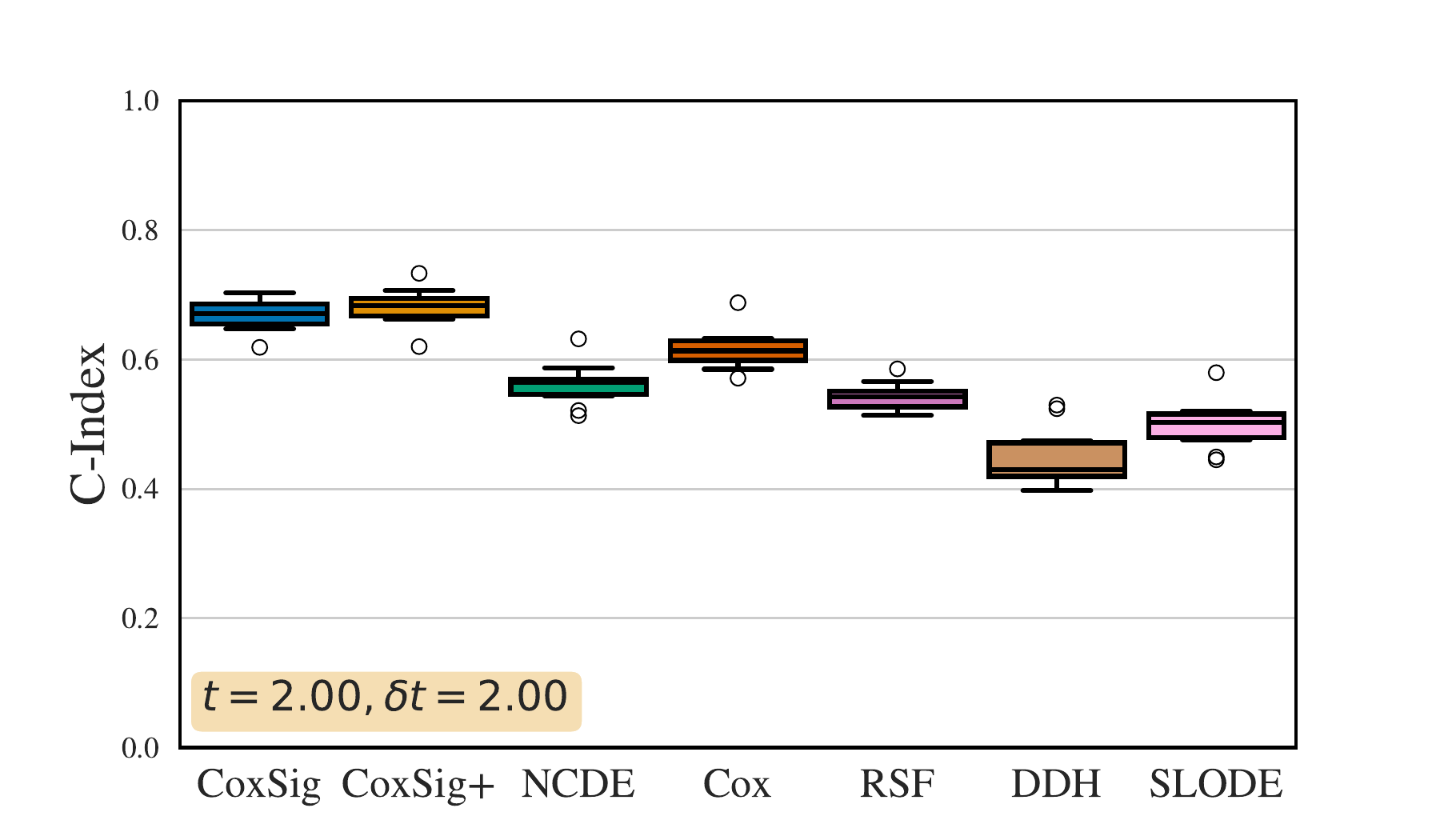}
    \includegraphics[width=0.33\textwidth]{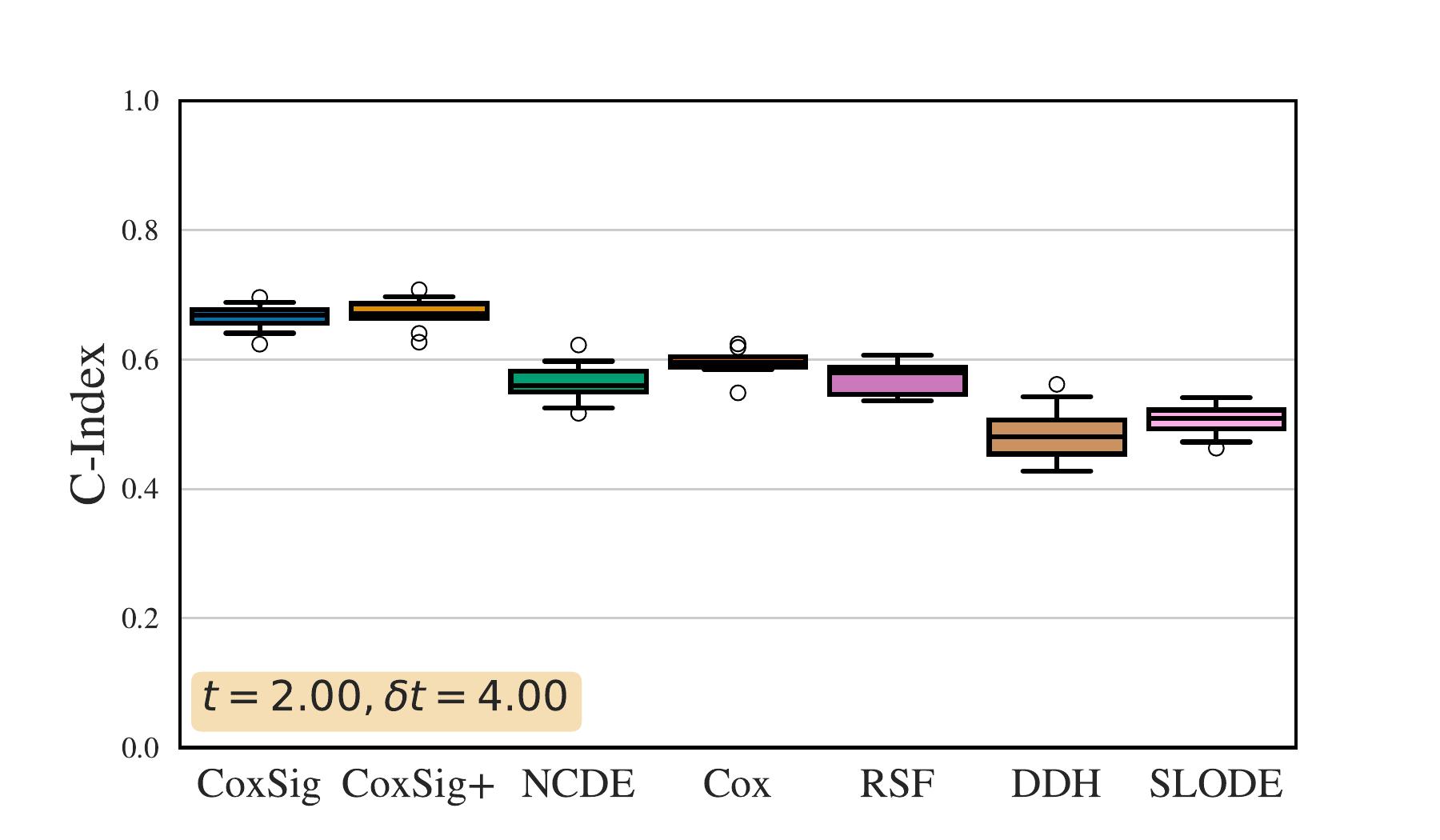}
    \includegraphics[width=0.33\textwidth]{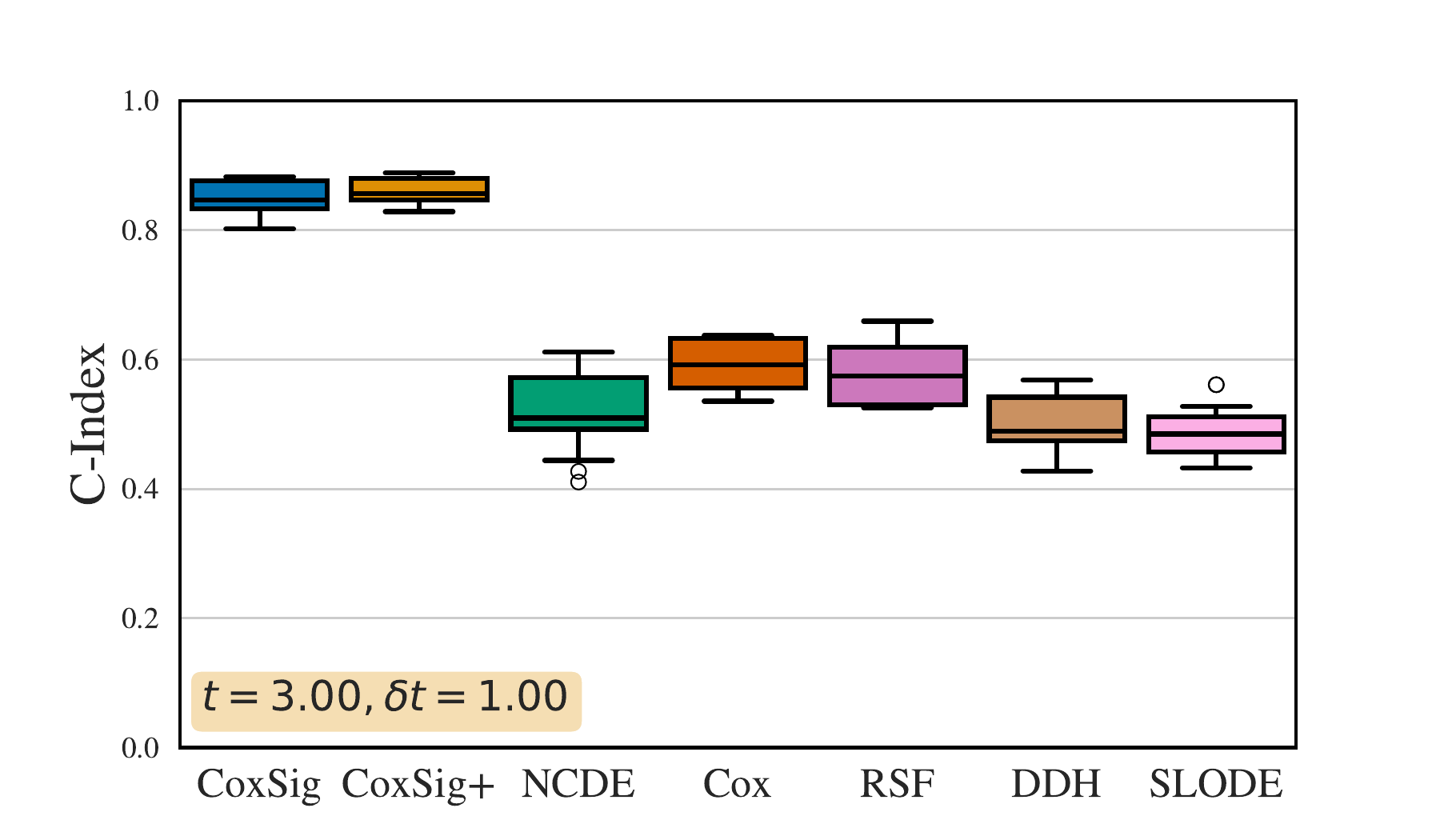}
    \includegraphics[width=0.33\textwidth]{figures/updated_figures/churn/Califrais_C-Index_pt=3.00,dt=2.00.pdf}
    \includegraphics[width=0.33\textwidth]{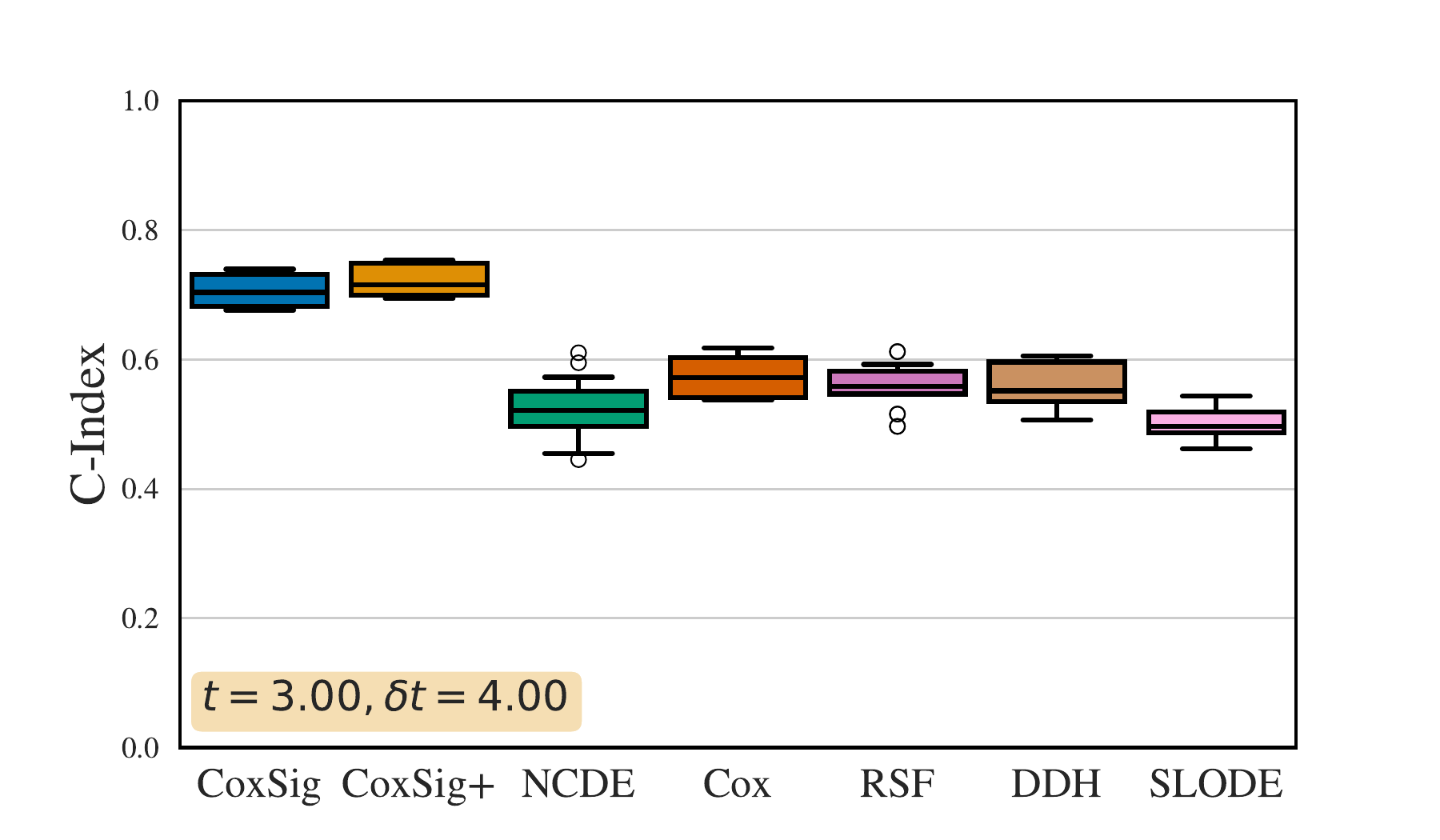}
    \includegraphics[width=0.33\textwidth]{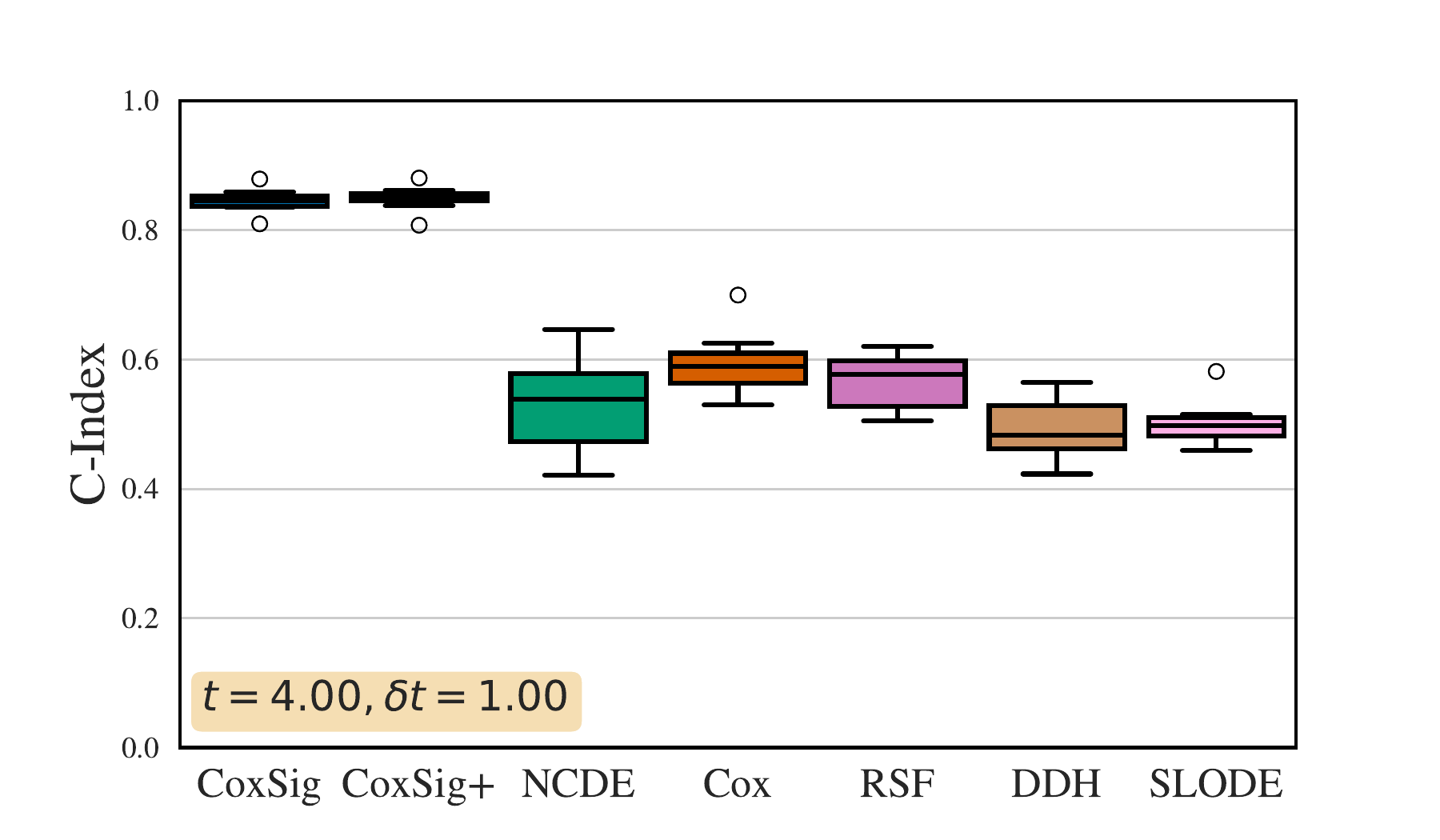}
    \includegraphics[width=0.33\textwidth]{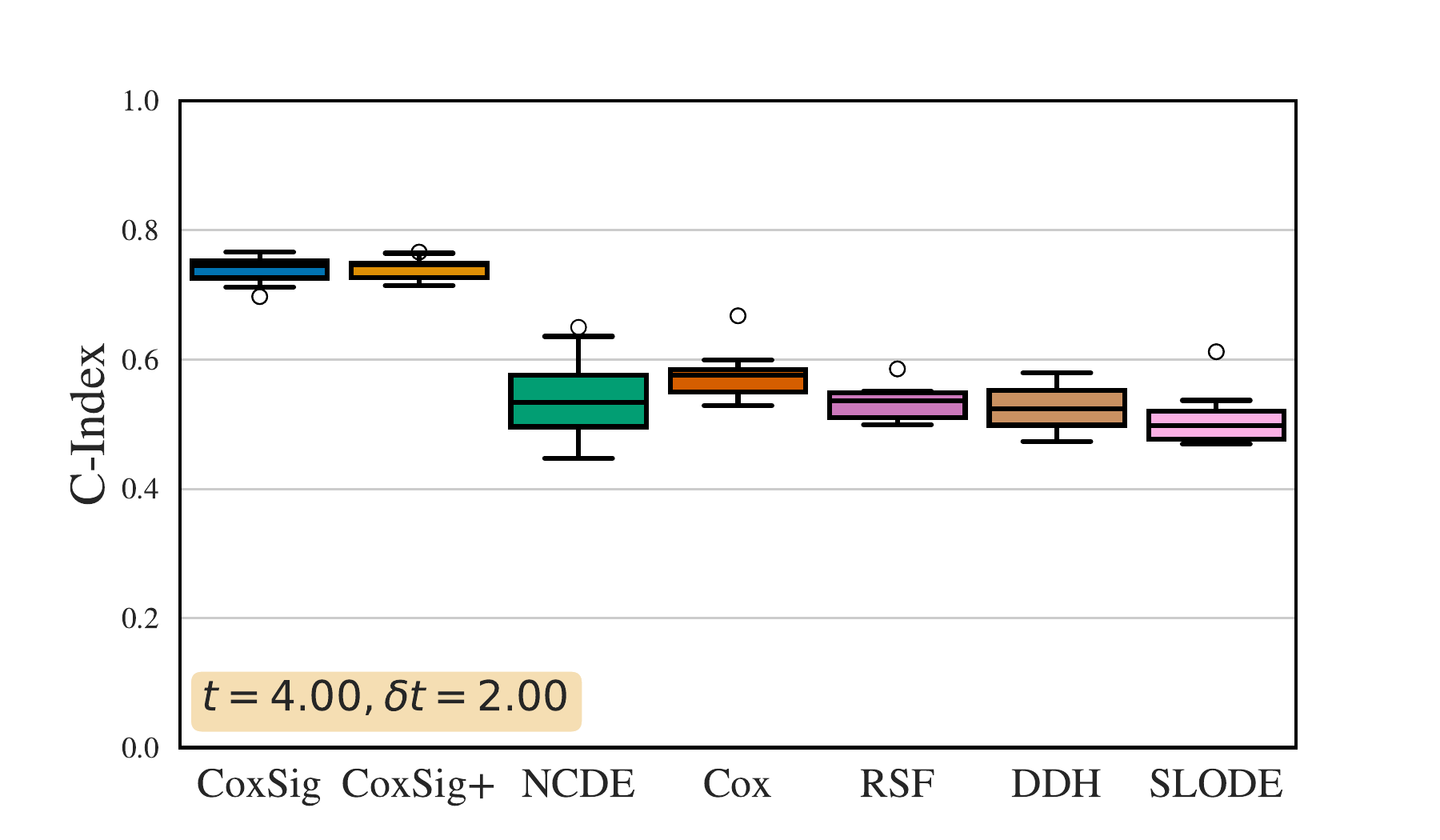}
    \includegraphics[width=0.33\textwidth]{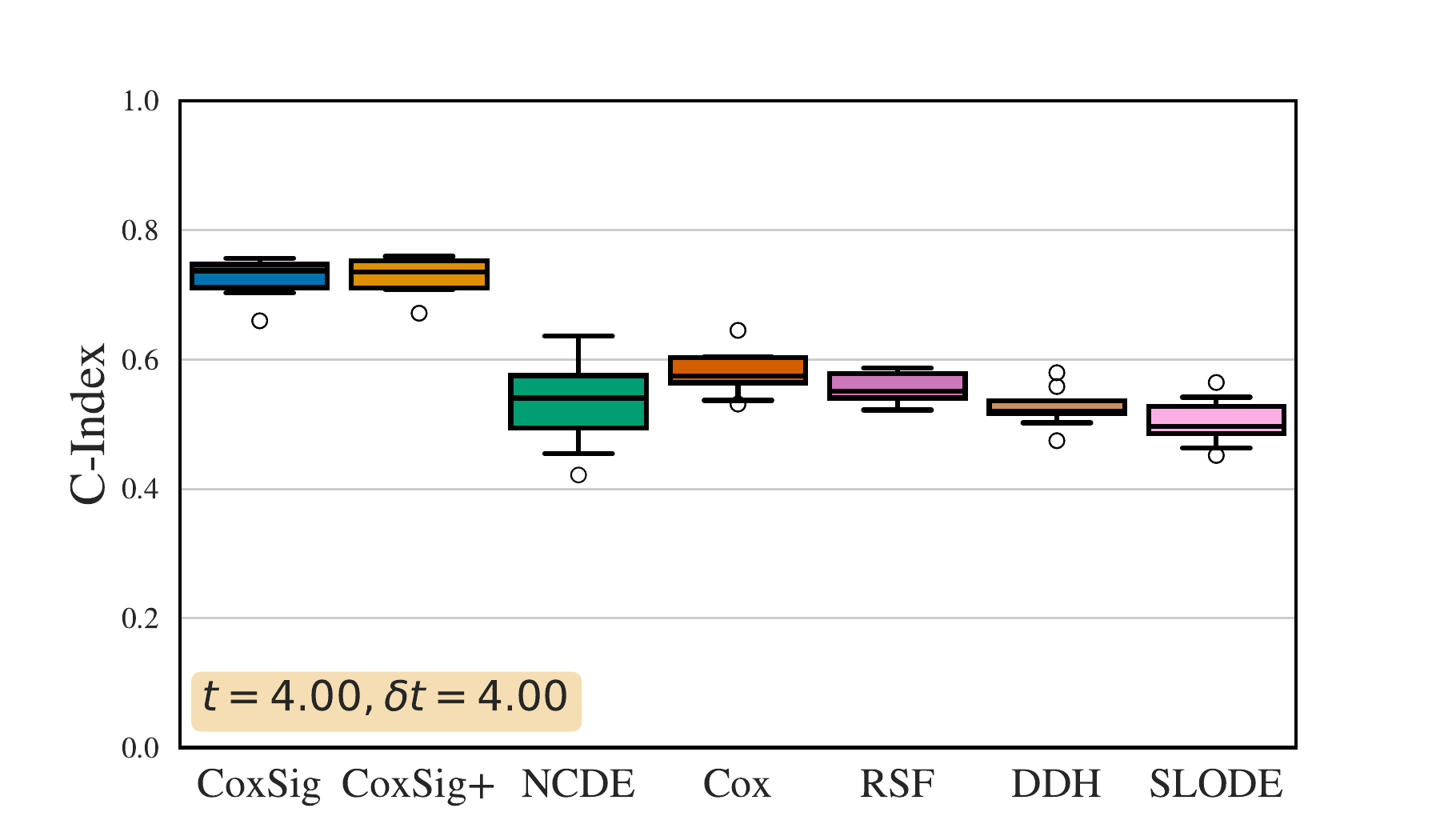}
    \caption{\footnotesize C-Index (\textit{higher} is better) for \textbf{churn prediction} for numerous points $(t,\delta t)$.}
    \label{fig:c_index_churn}
\end{figure*}

\begin{figure*}
    \centering
    \includegraphics[width=0.33\textwidth]{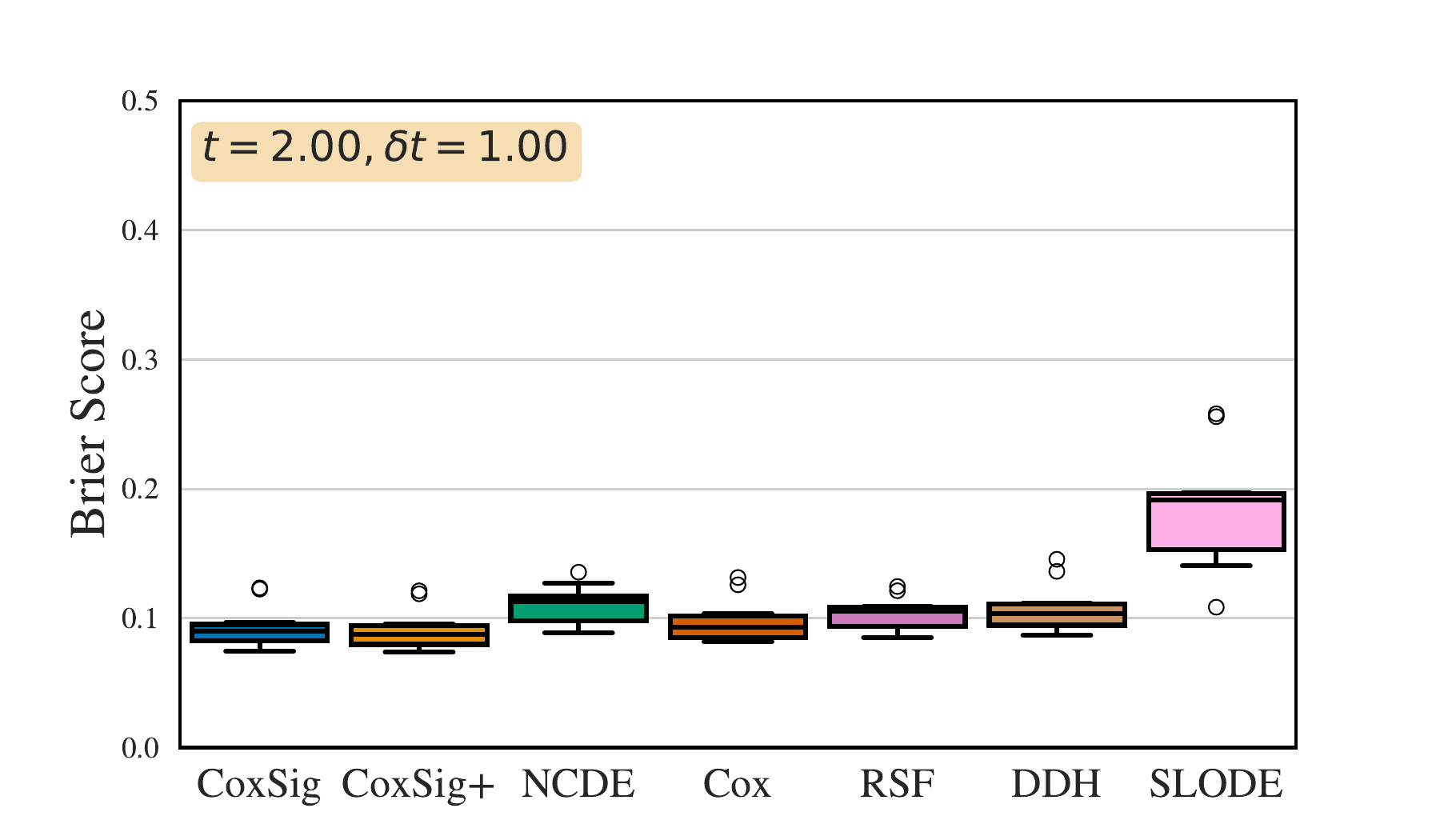}
    \includegraphics[width=0.33\textwidth]{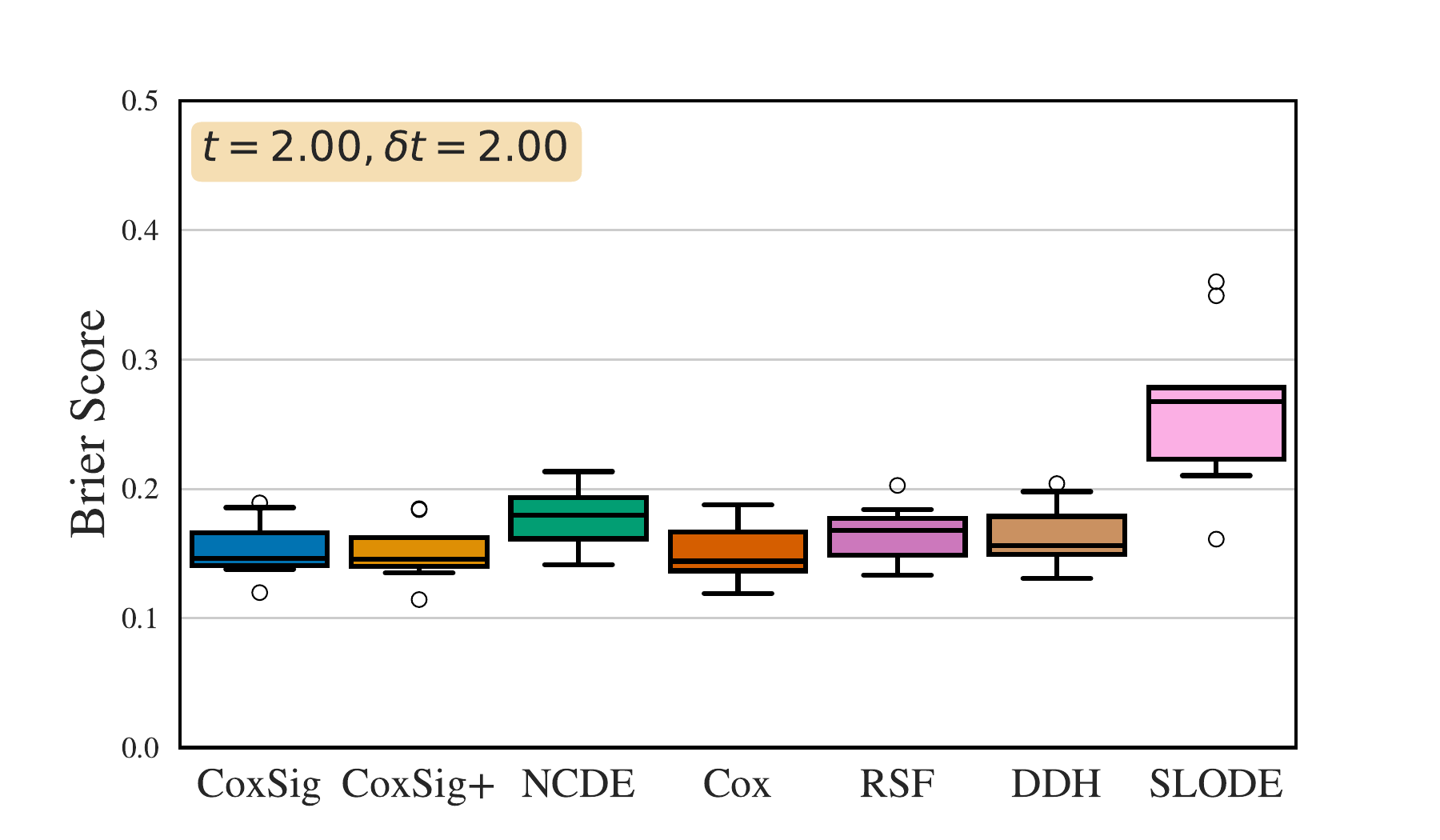}
    \includegraphics[width=0.33\textwidth]{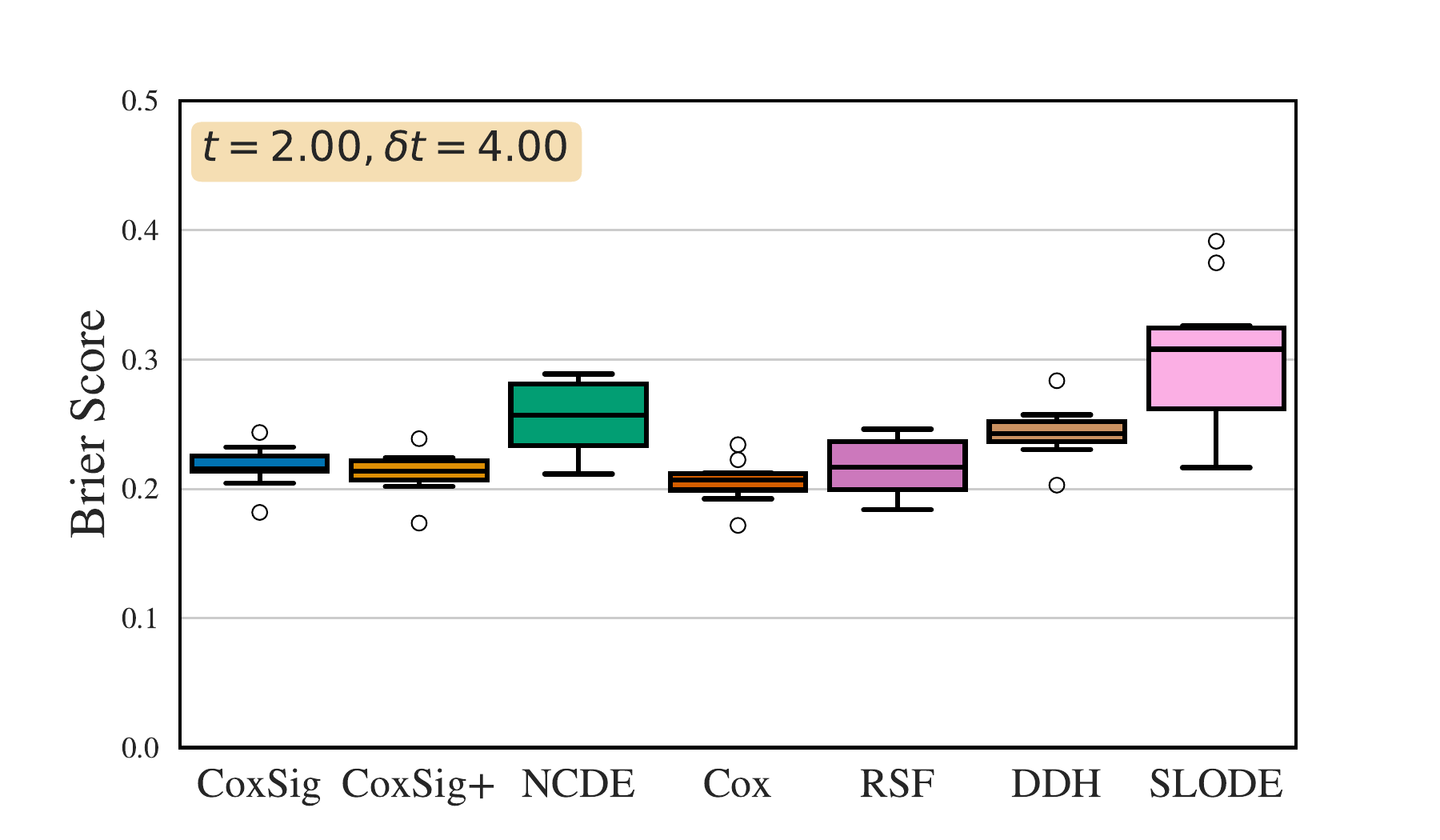}
    \includegraphics[width=0.33\textwidth]{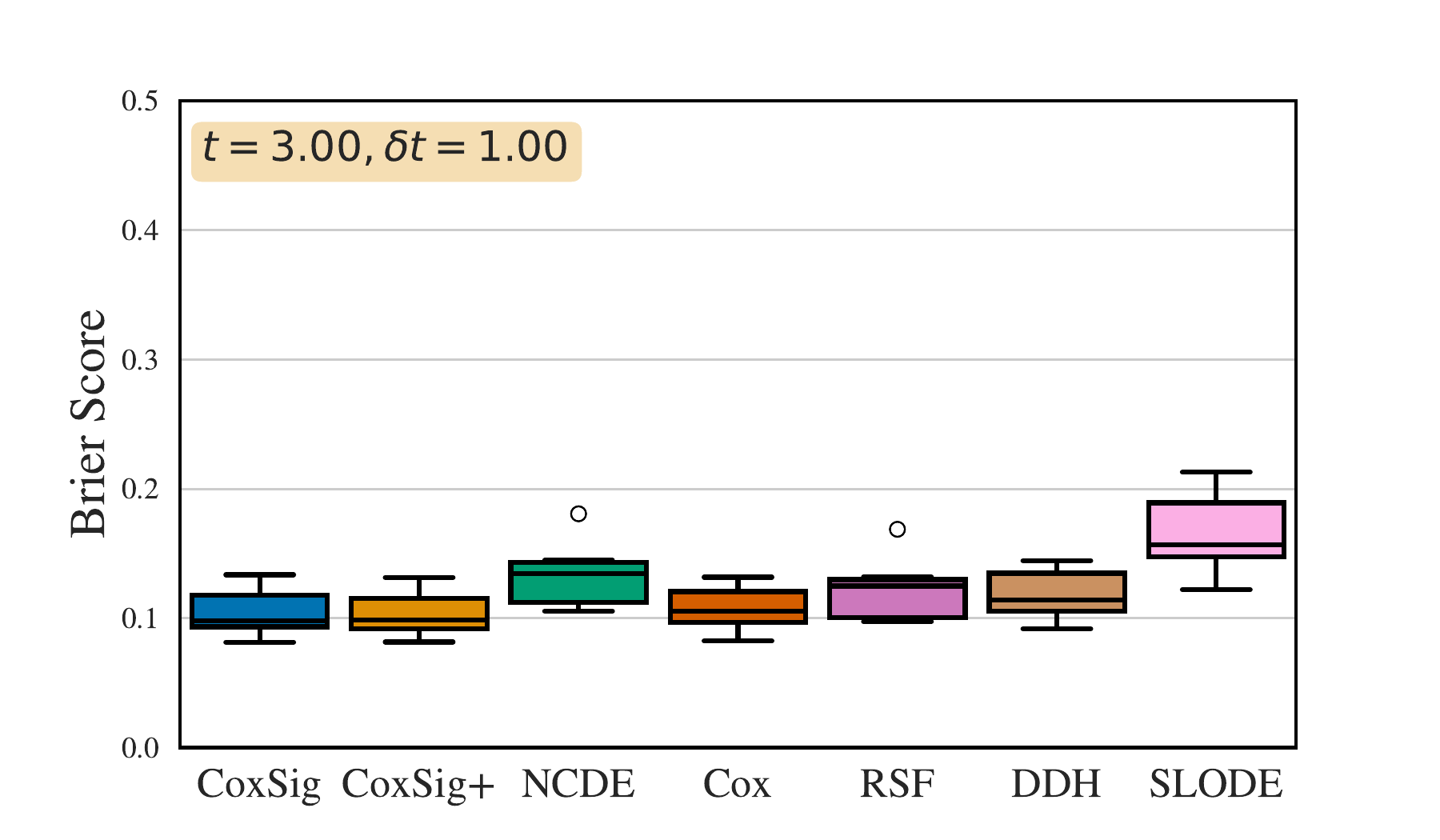}
    \includegraphics[width=0.33\textwidth]{figures/updated_figures/churn/Califrais_BS_pt=3.00,dt=2.00.pdf}
    \includegraphics[width=0.33\textwidth]{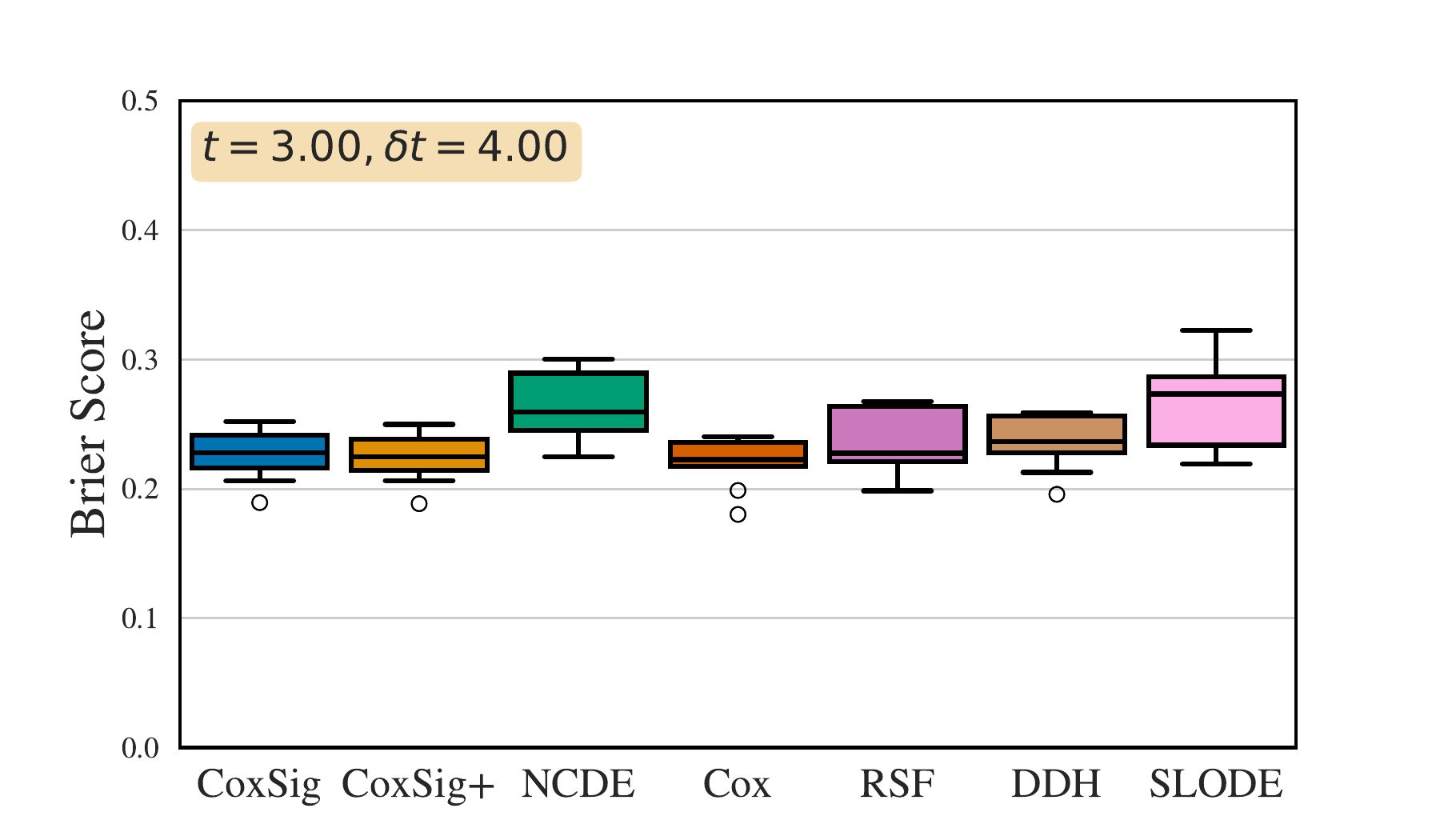}
    \includegraphics[width=0.33\textwidth]{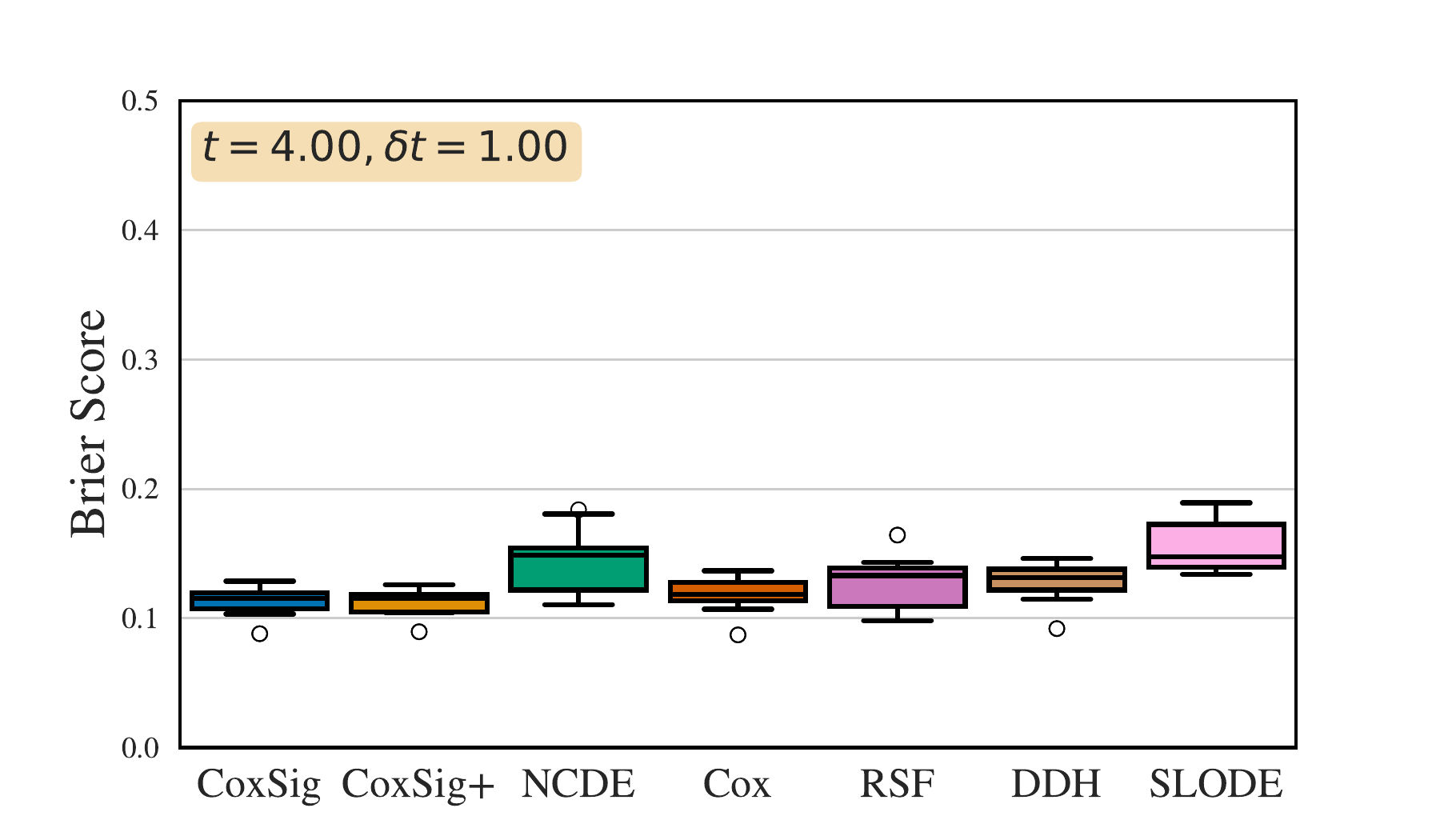}
    \includegraphics[width=0.33\textwidth]{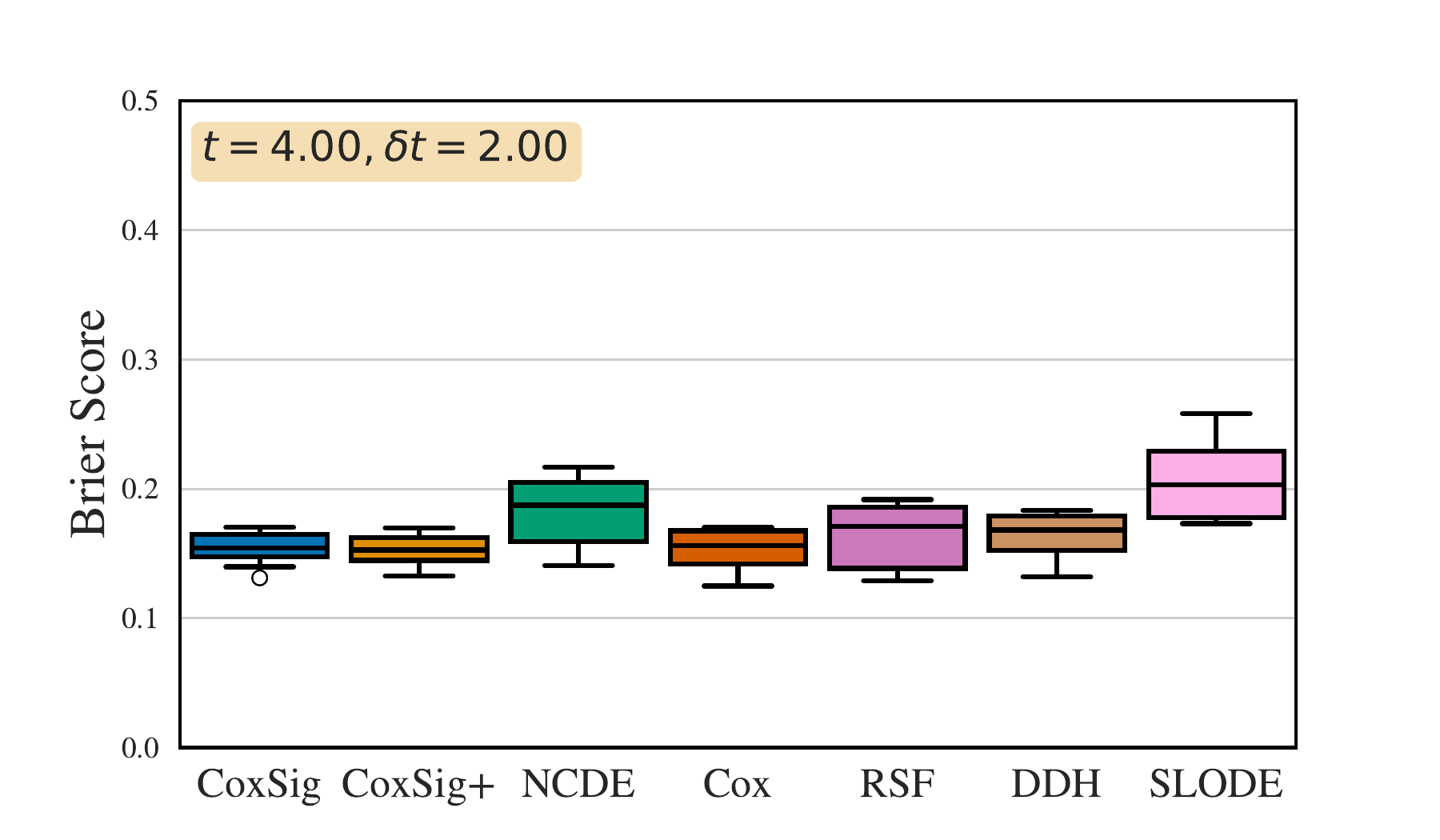}
    \includegraphics[width=0.33\textwidth]{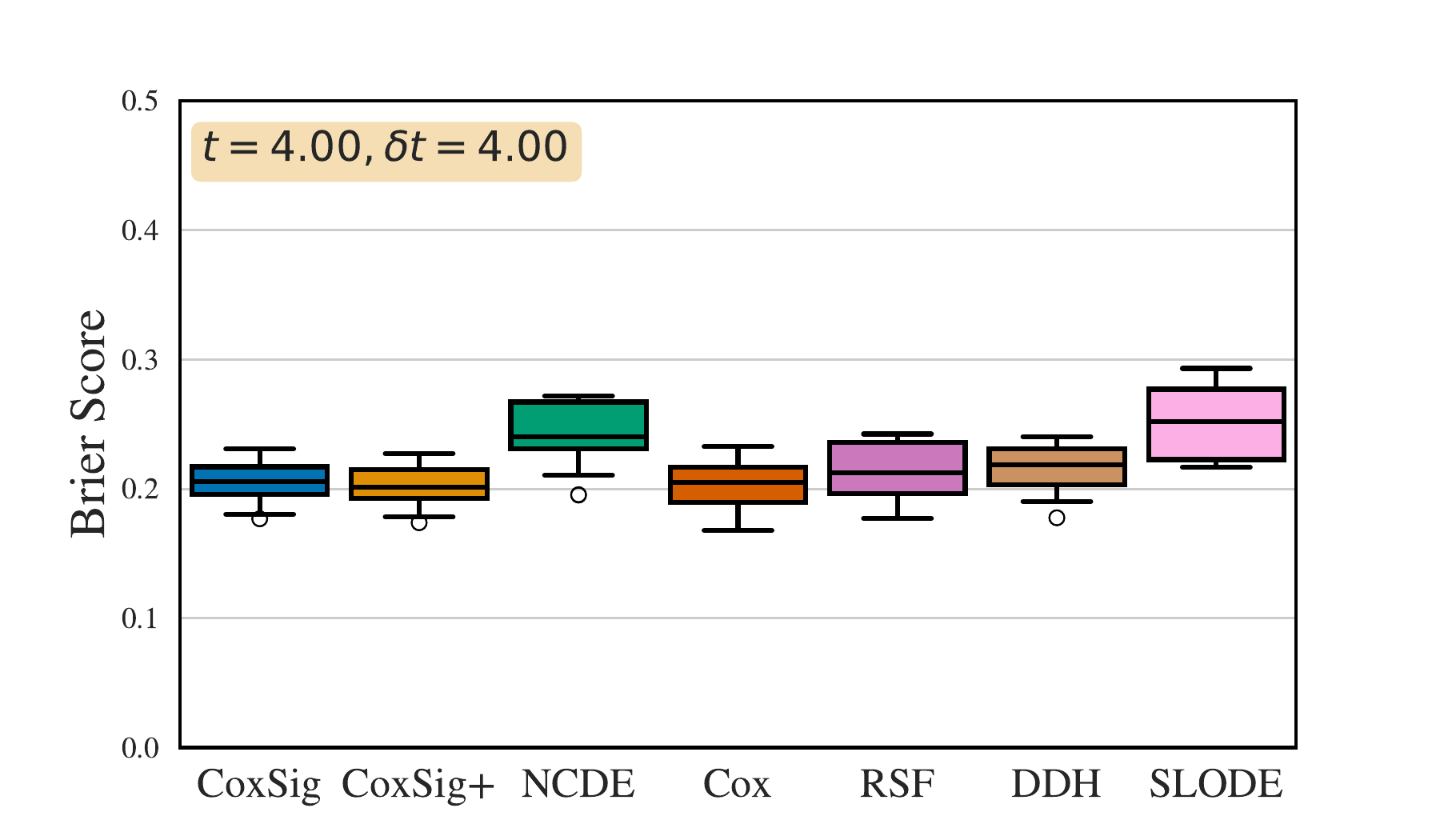}
    \caption{\footnotesize Brier score (\textit{lower} is better) for \textbf{churn prediction} for numerous points $(t,\delta t)$.}
    \label{fig:bs_churn}
\end{figure*}

\begin{figure*}
    \centering
    \includegraphics[width=0.33\textwidth]{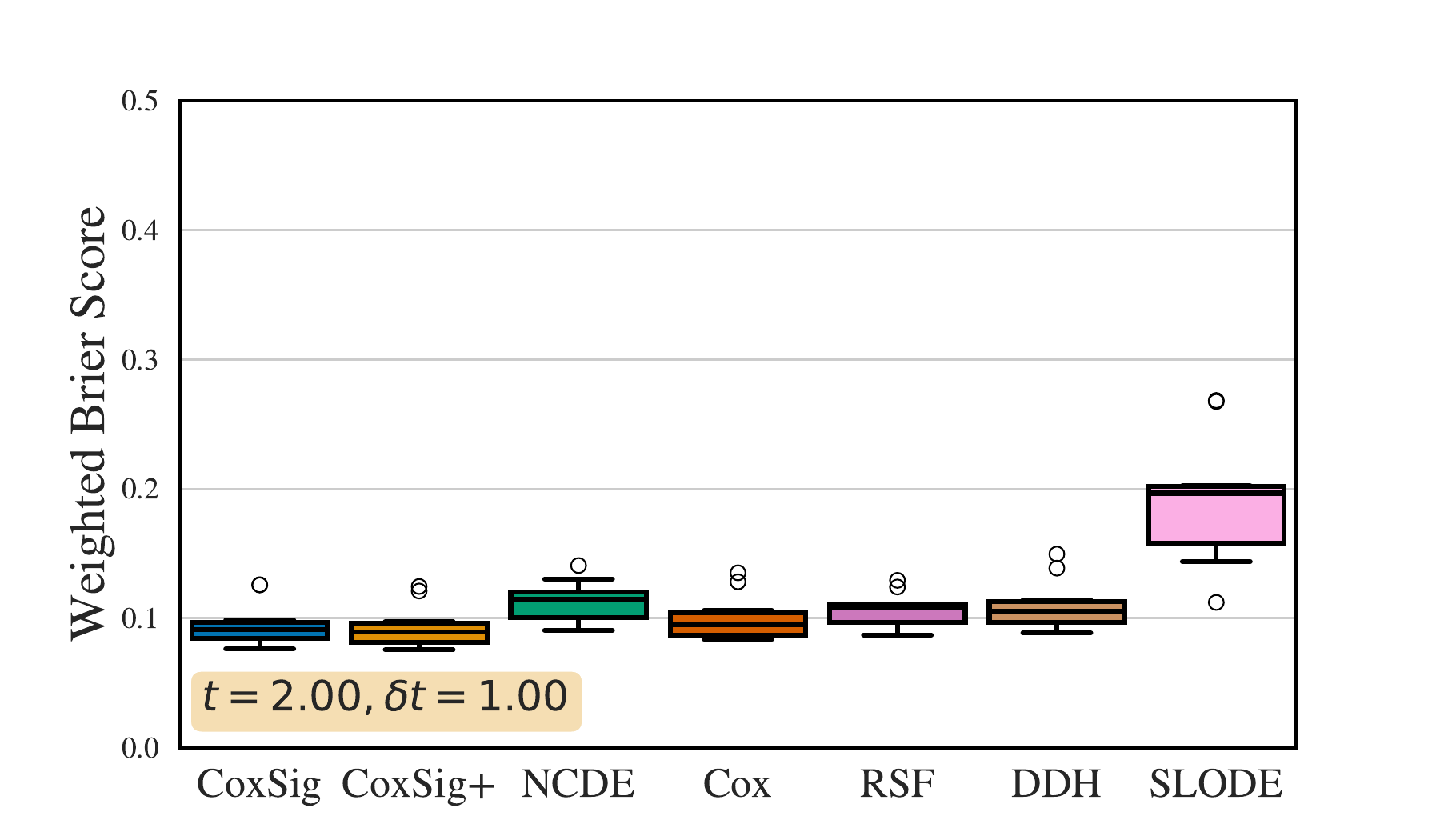}
    \includegraphics[width=0.33\textwidth]{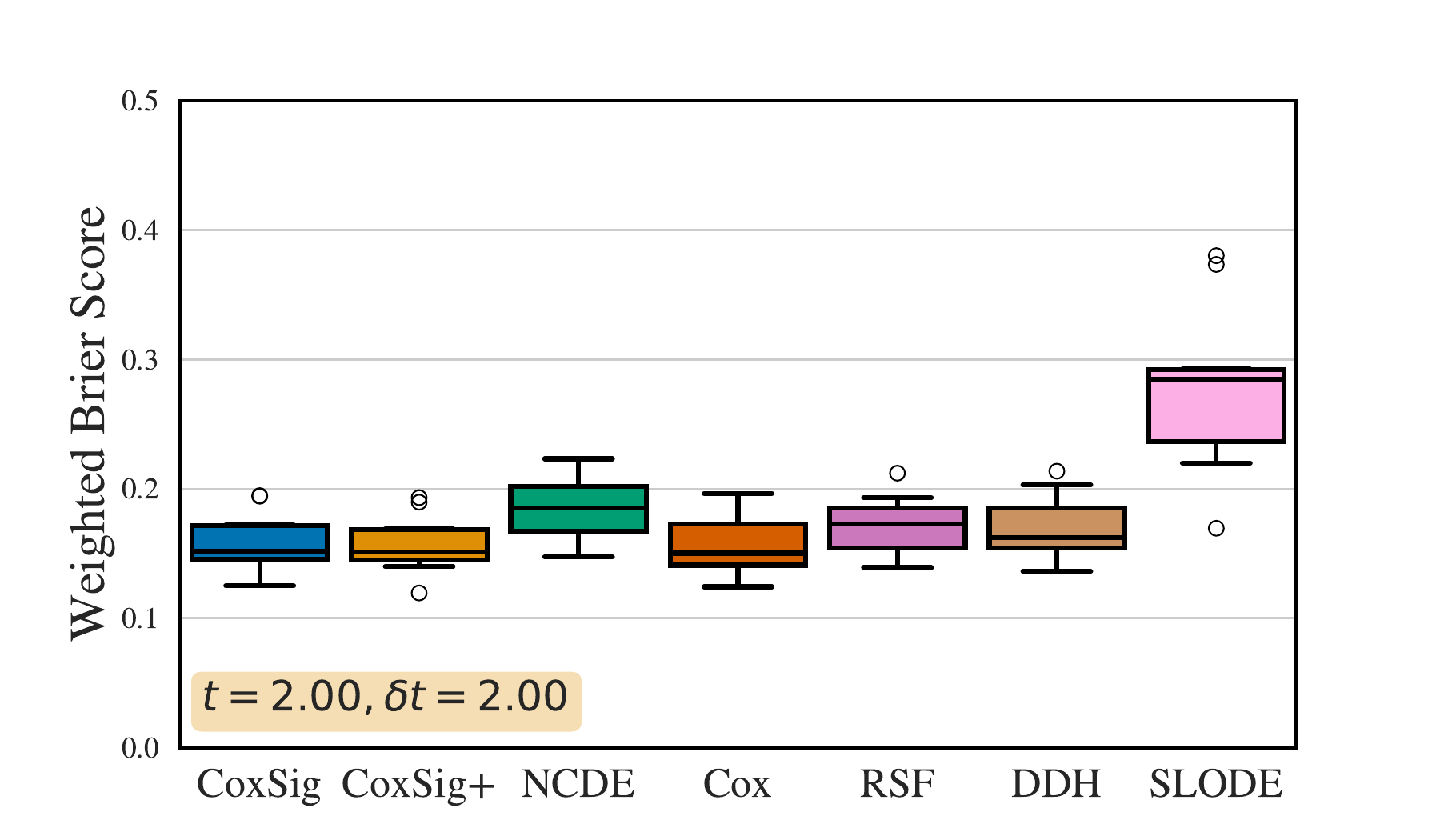}
    \includegraphics[width=0.33\textwidth]{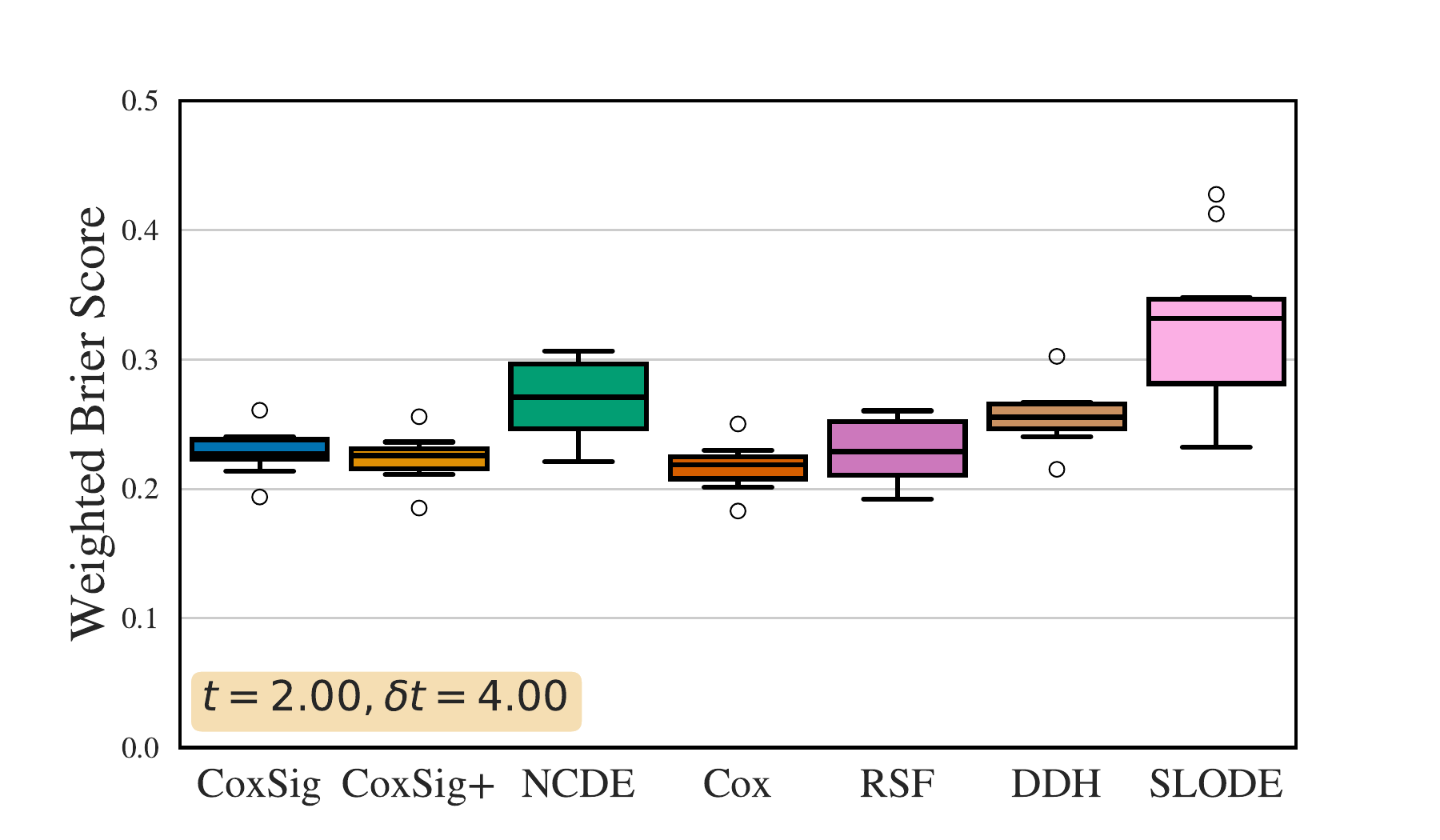}
    \includegraphics[width=0.33\textwidth]{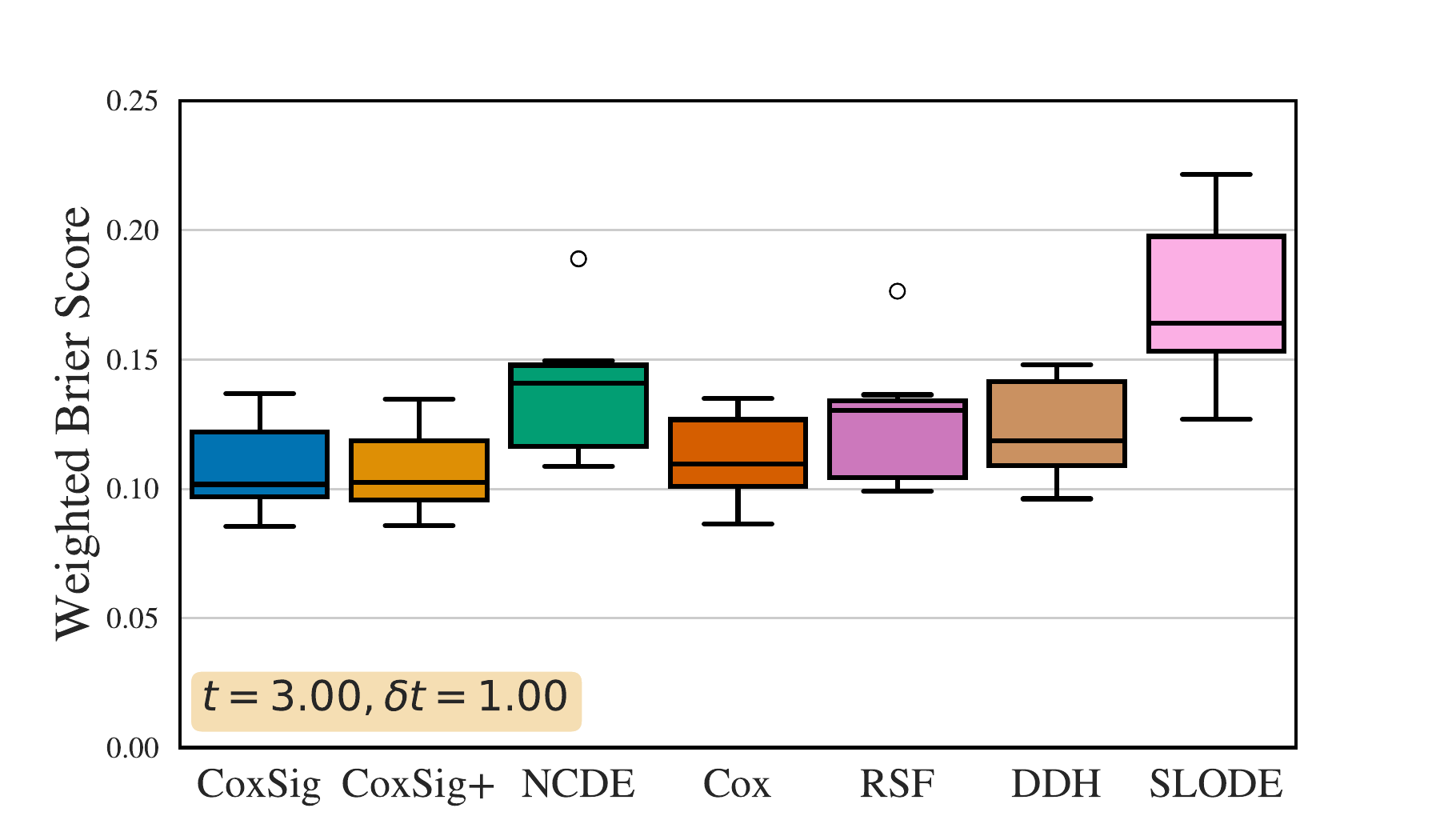}
    \includegraphics[width=0.33\textwidth]{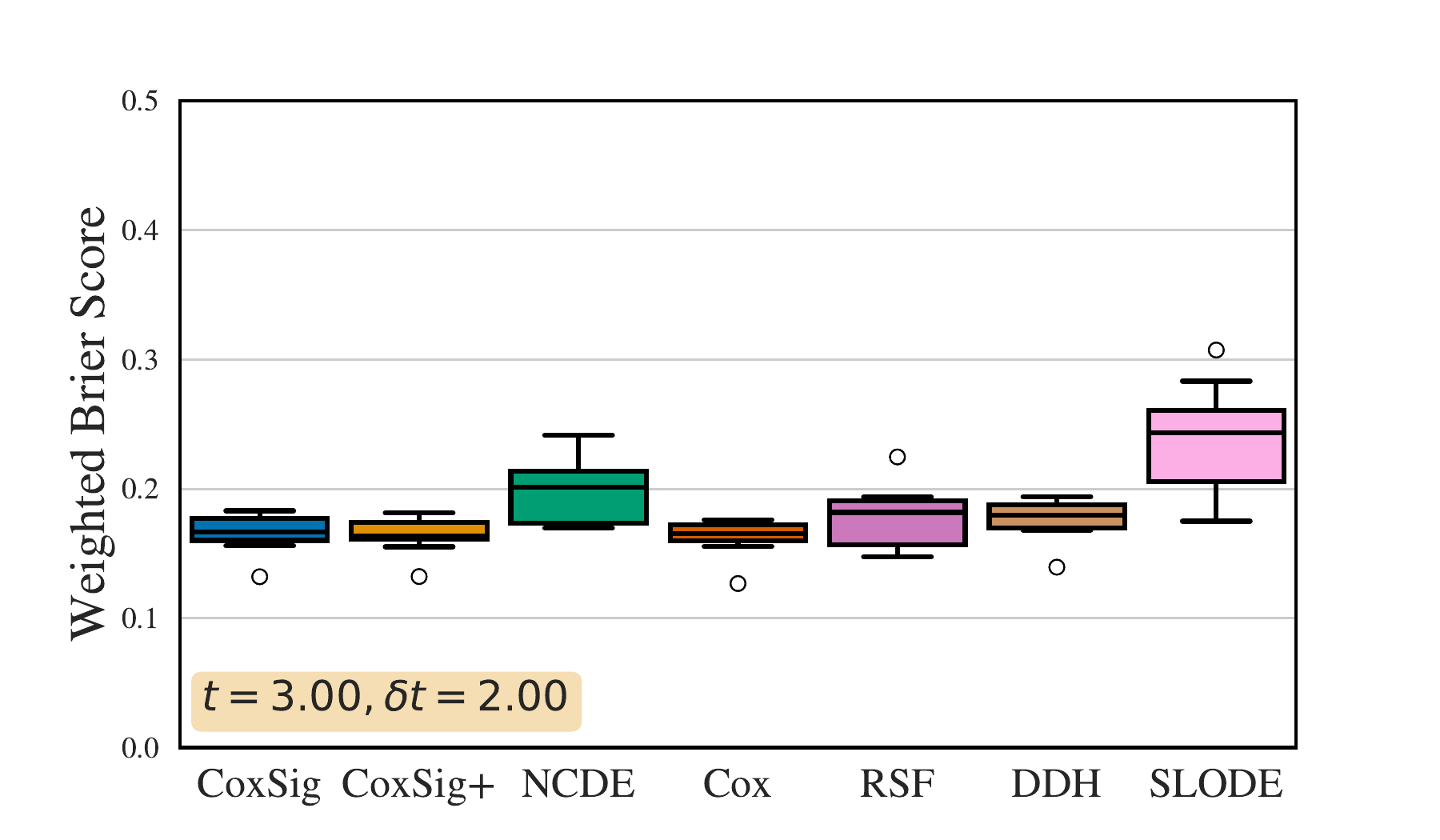}
    \includegraphics[width=0.33\textwidth]{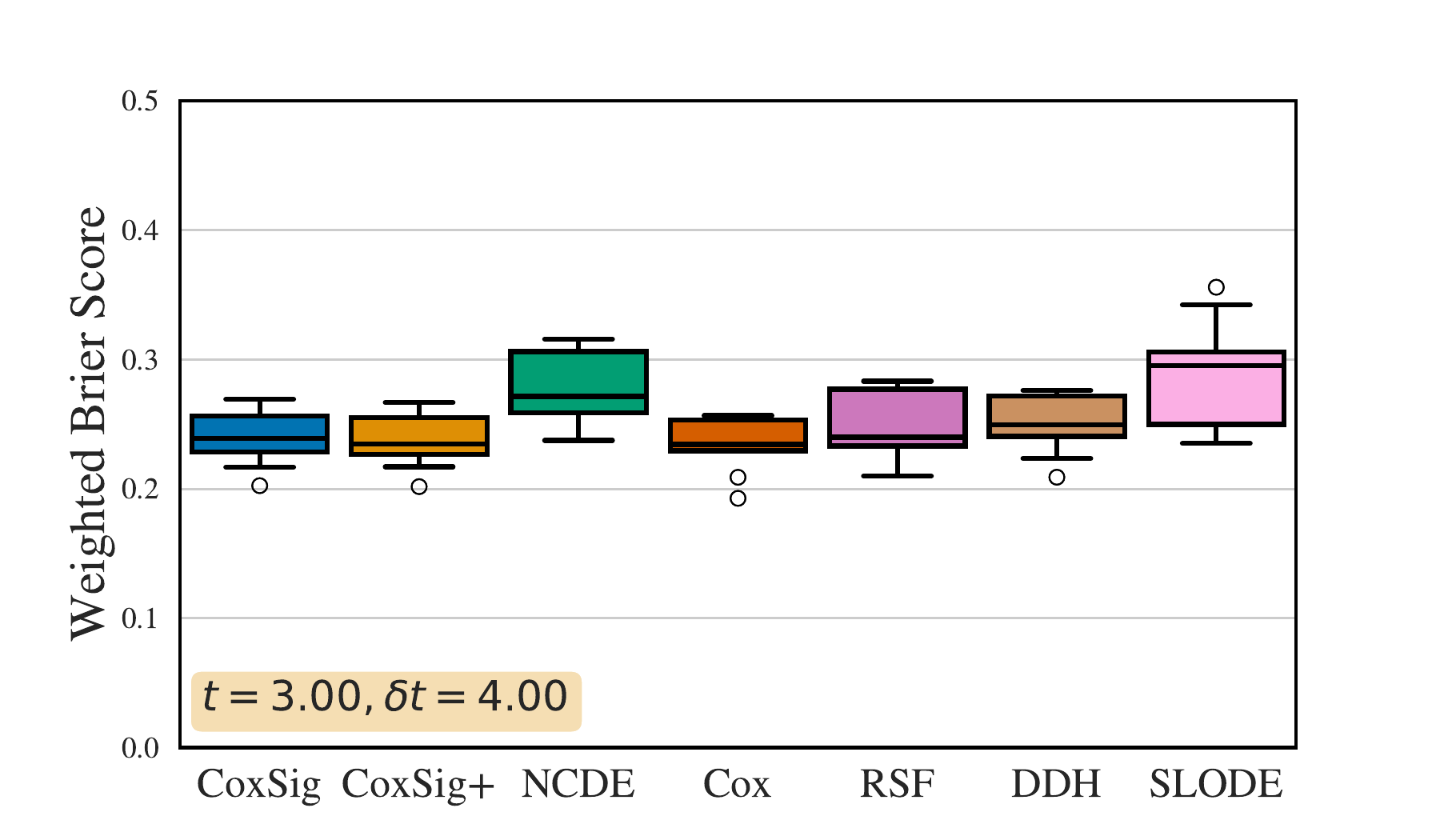}
    \includegraphics[width=0.33\textwidth]{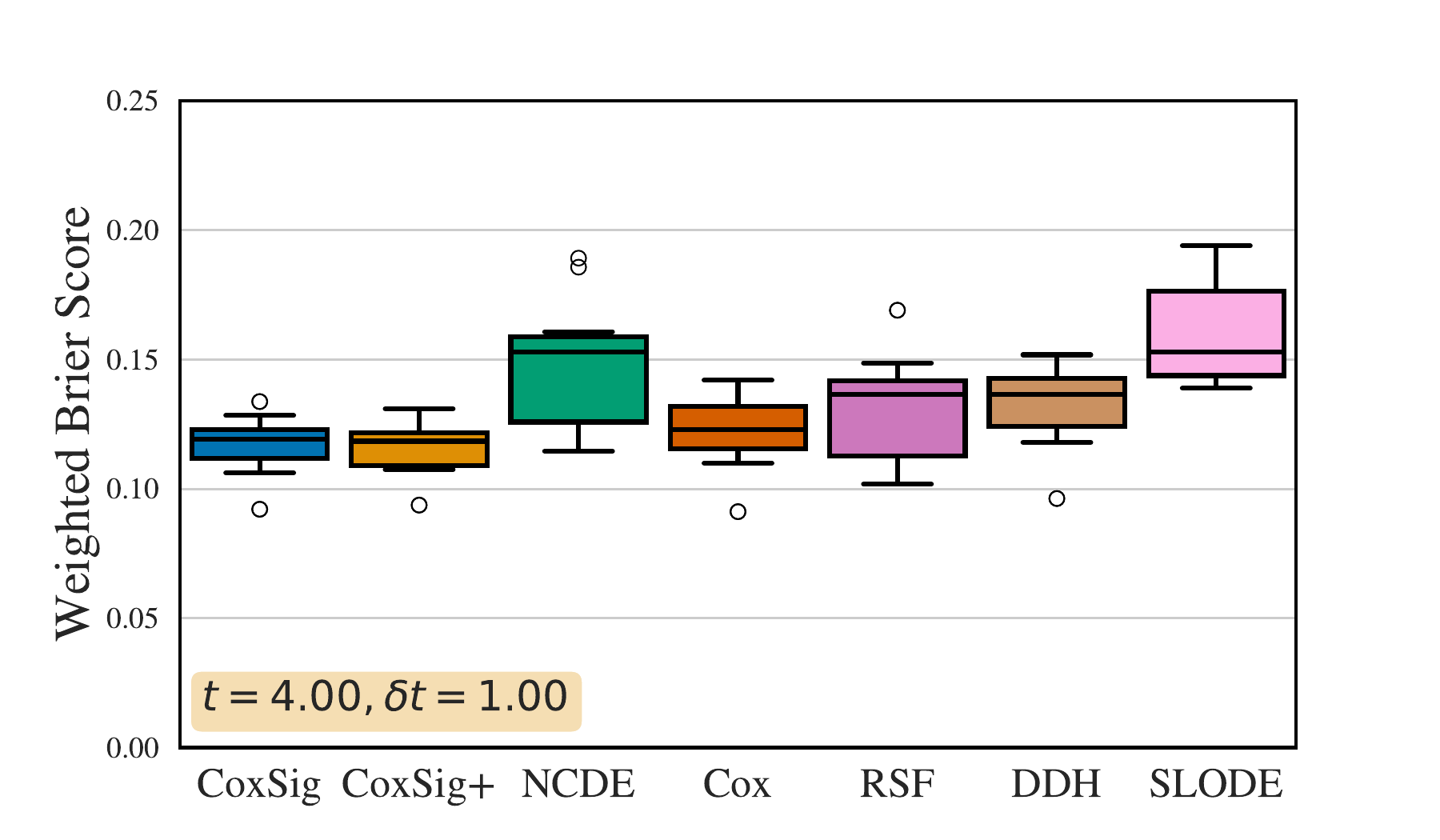}
    \includegraphics[width=0.33\textwidth]{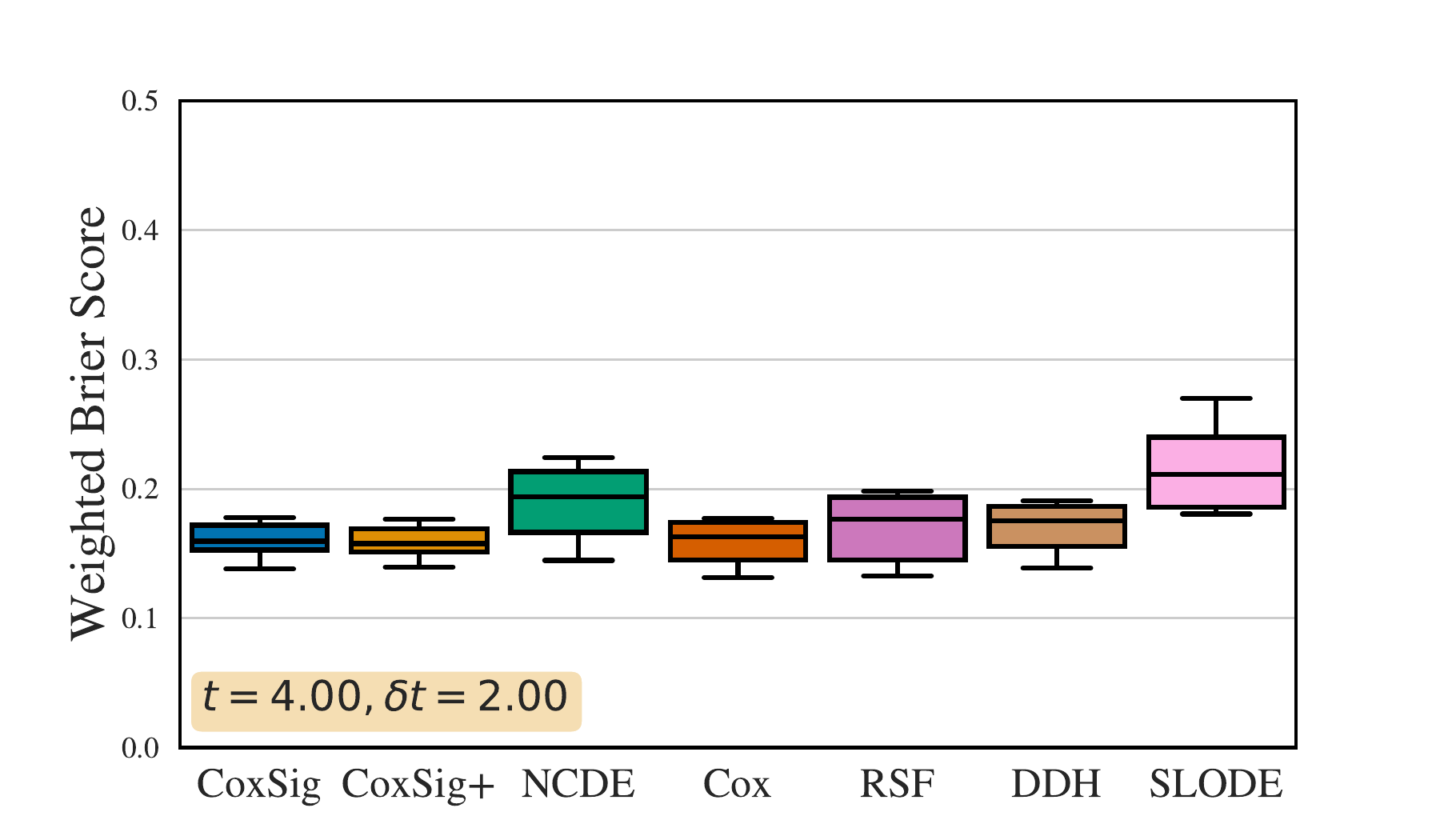}
    \includegraphics[width=0.33\textwidth]{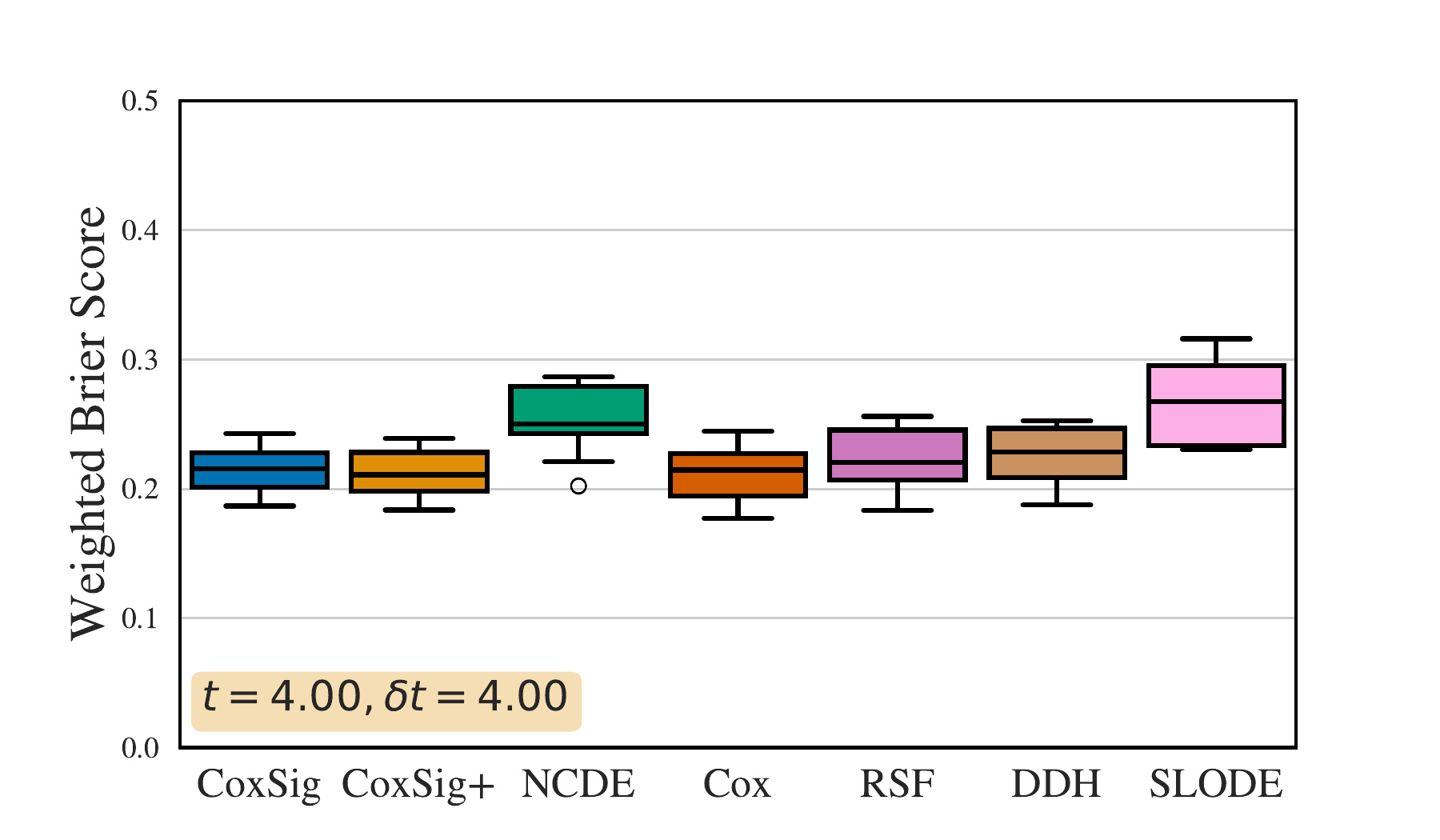}
    \caption{\footnotesize Weighted Brier score (\textit{lower} is better) for \textbf{churn prediction} for numerous points $(t,\delta t)$.}
    \label{fig:wbs_churn}
\end{figure*}

\begin{figure*}
    \centering
    \includegraphics[width=0.33\textwidth]{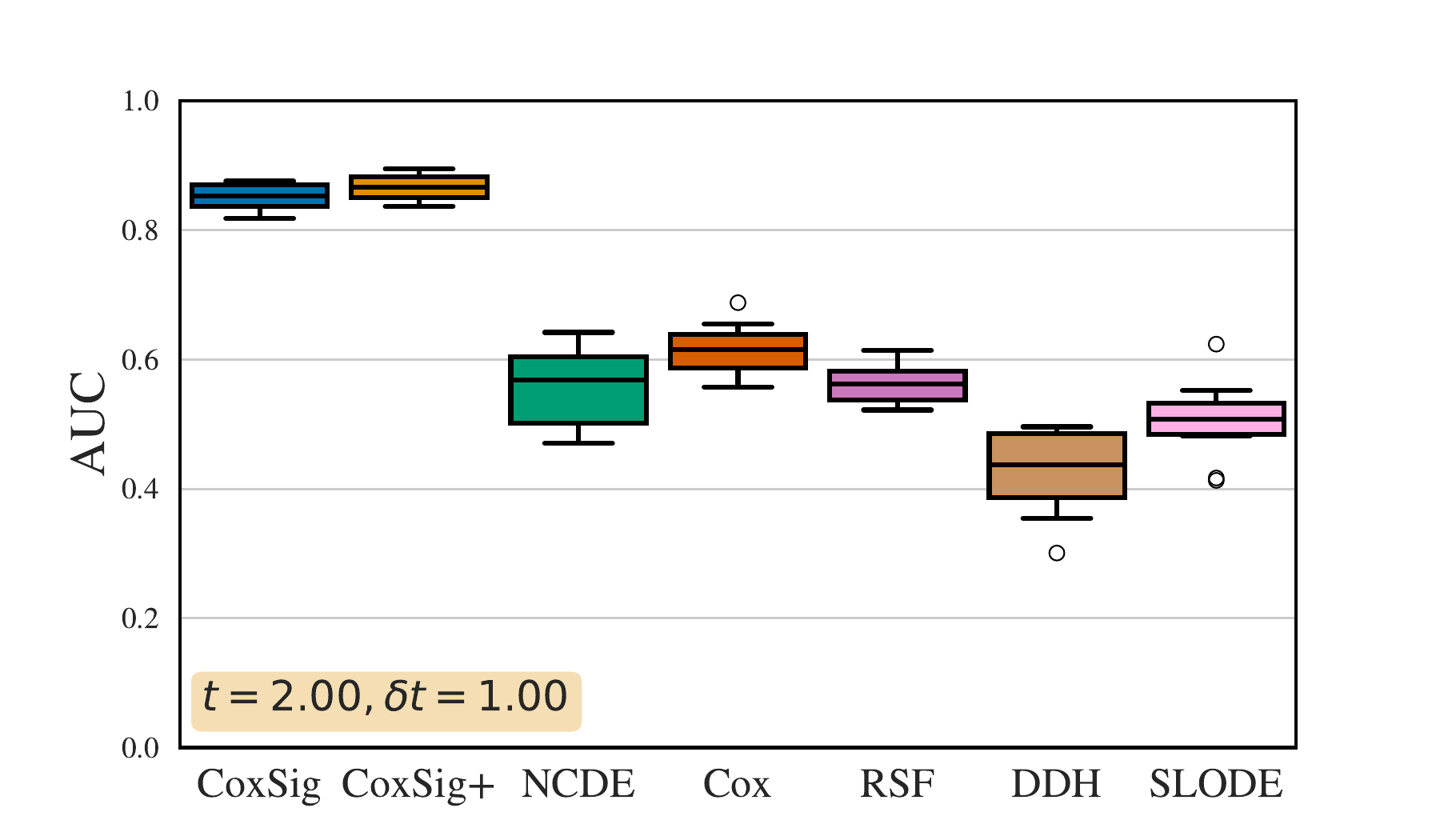}
    \includegraphics[width=0.33\textwidth]{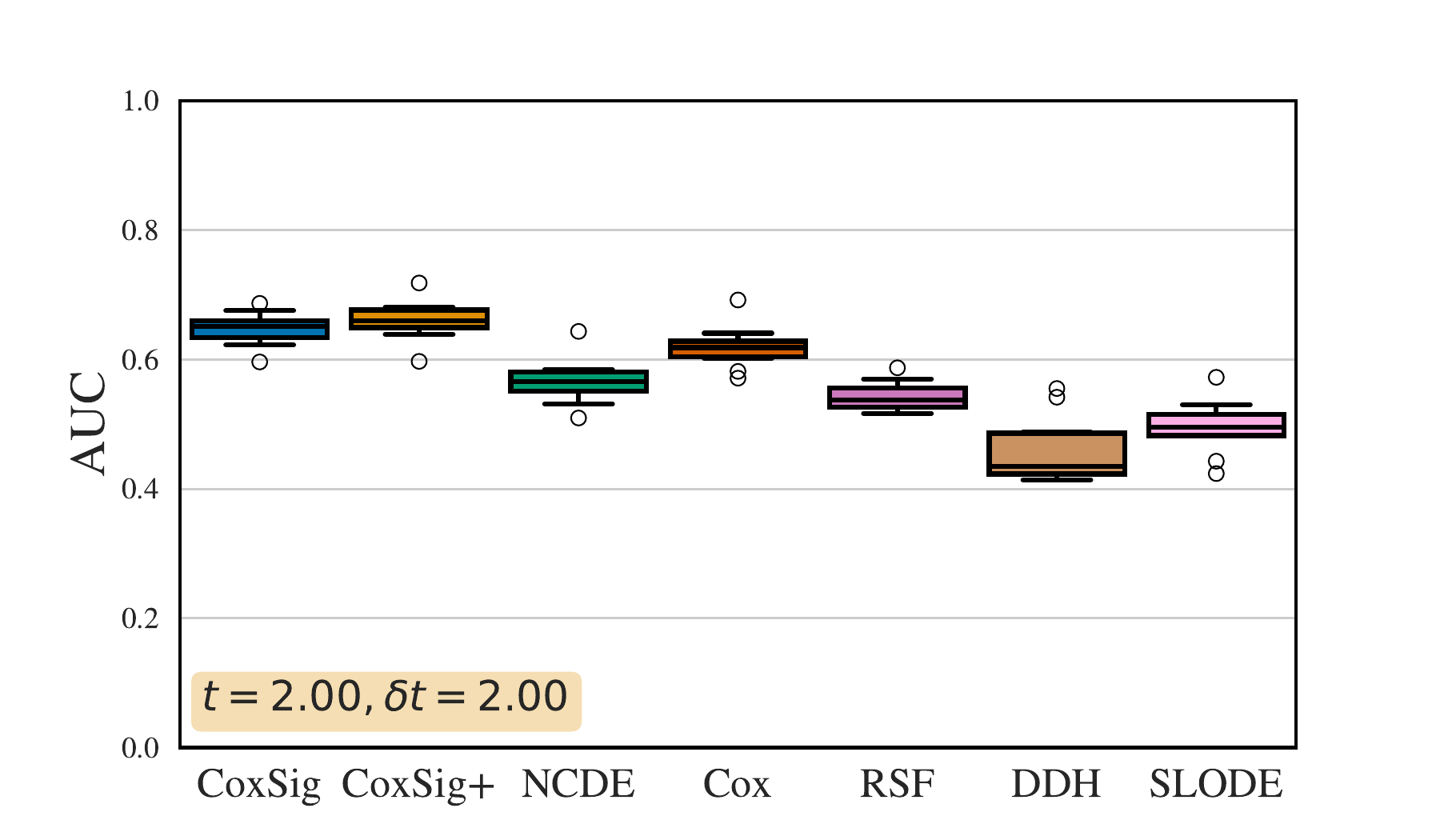}
    \includegraphics[width=0.33\textwidth]{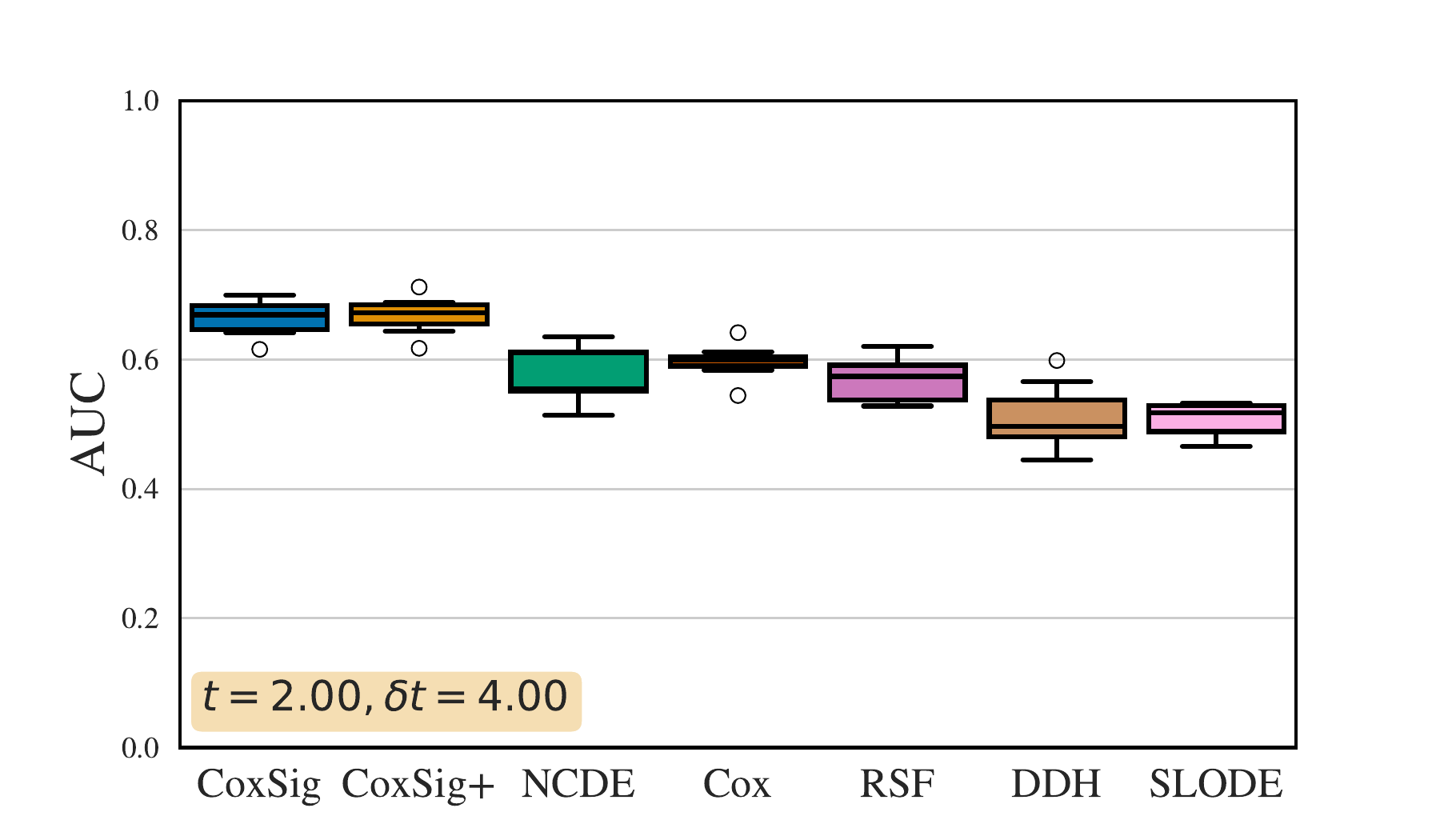}
    \includegraphics[width=0.33\textwidth]{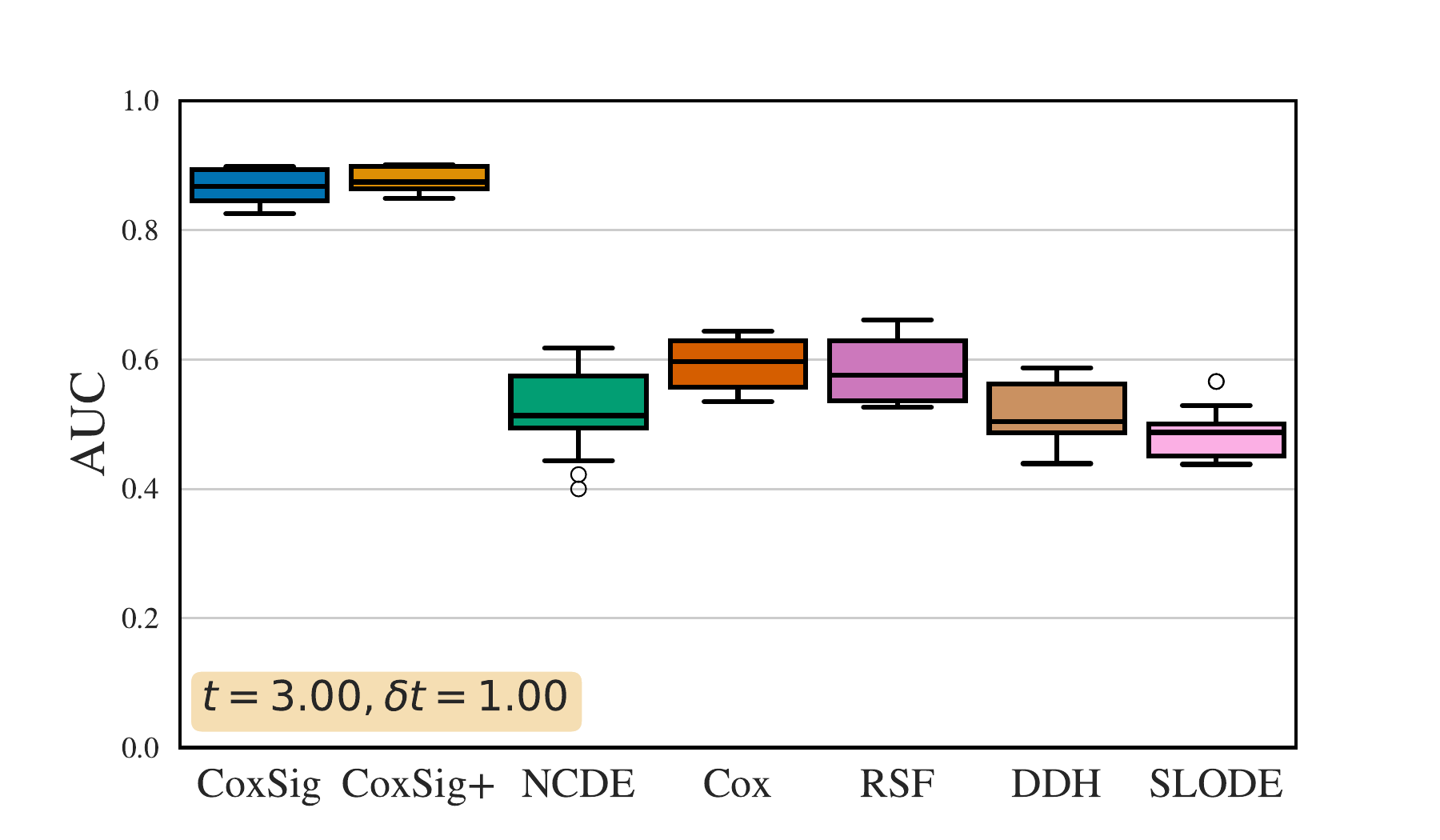}
    \includegraphics[width=0.33\textwidth]{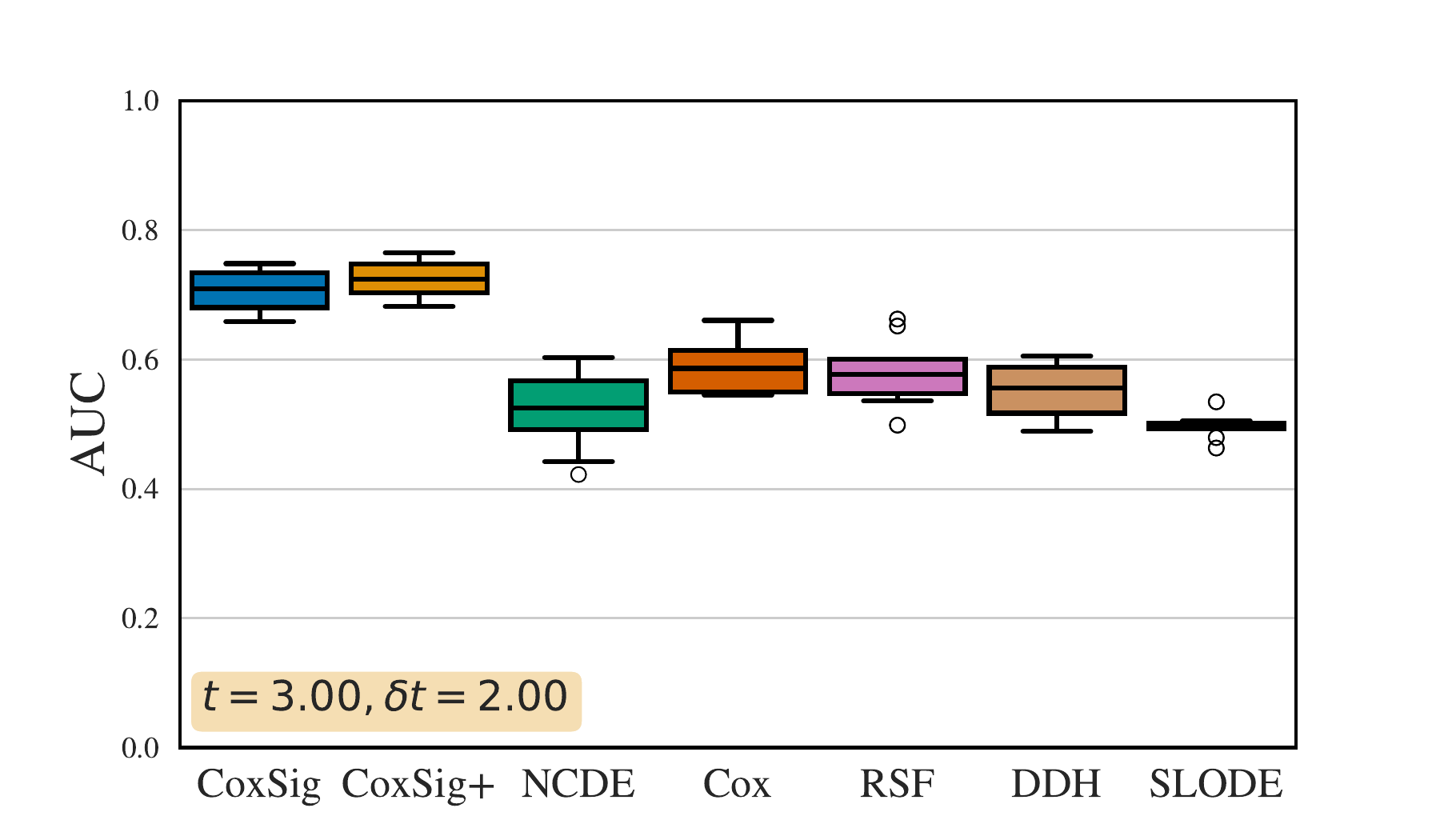}
    \includegraphics[width=0.33\textwidth]{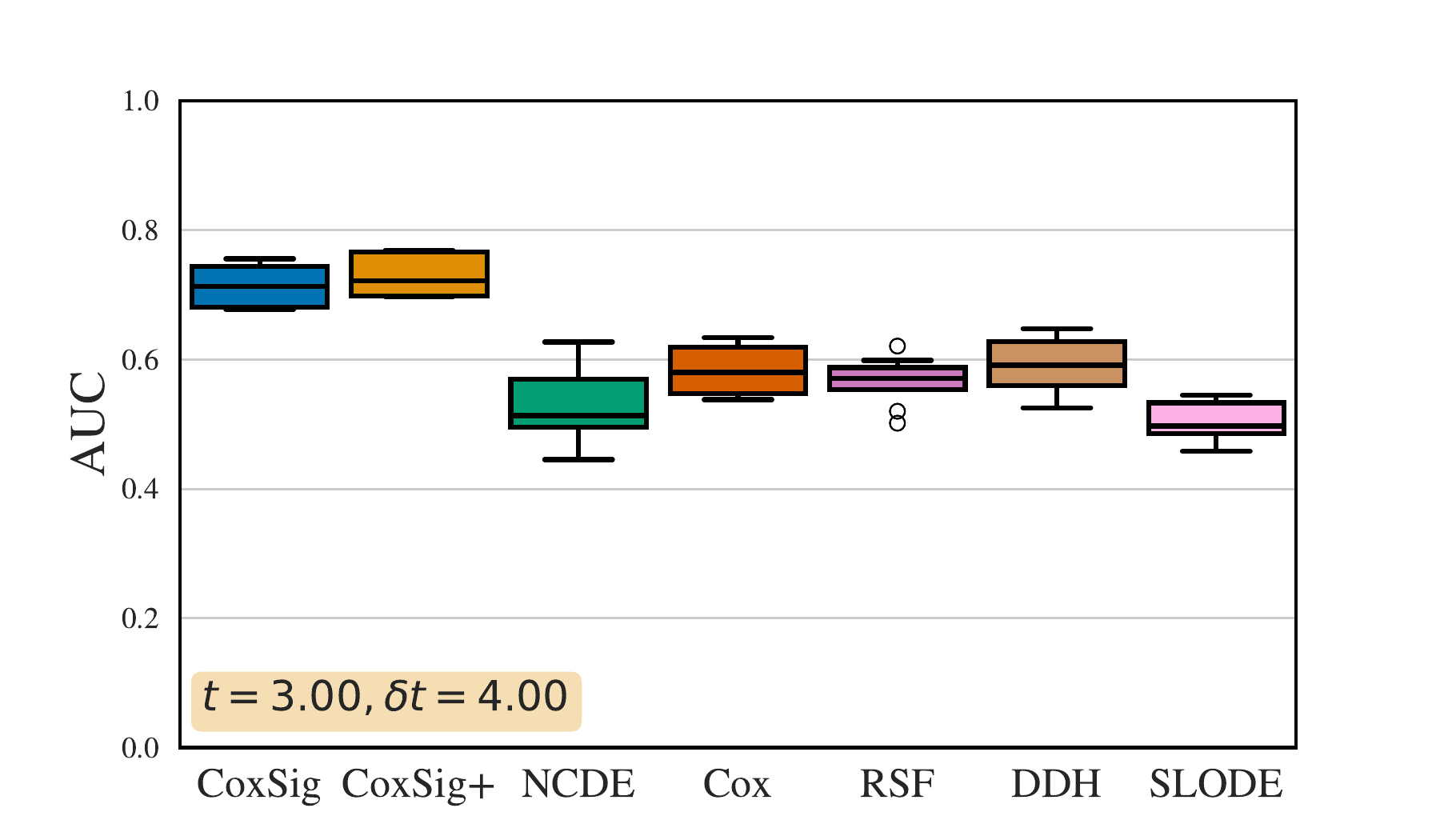}
    \includegraphics[width=0.33\textwidth]{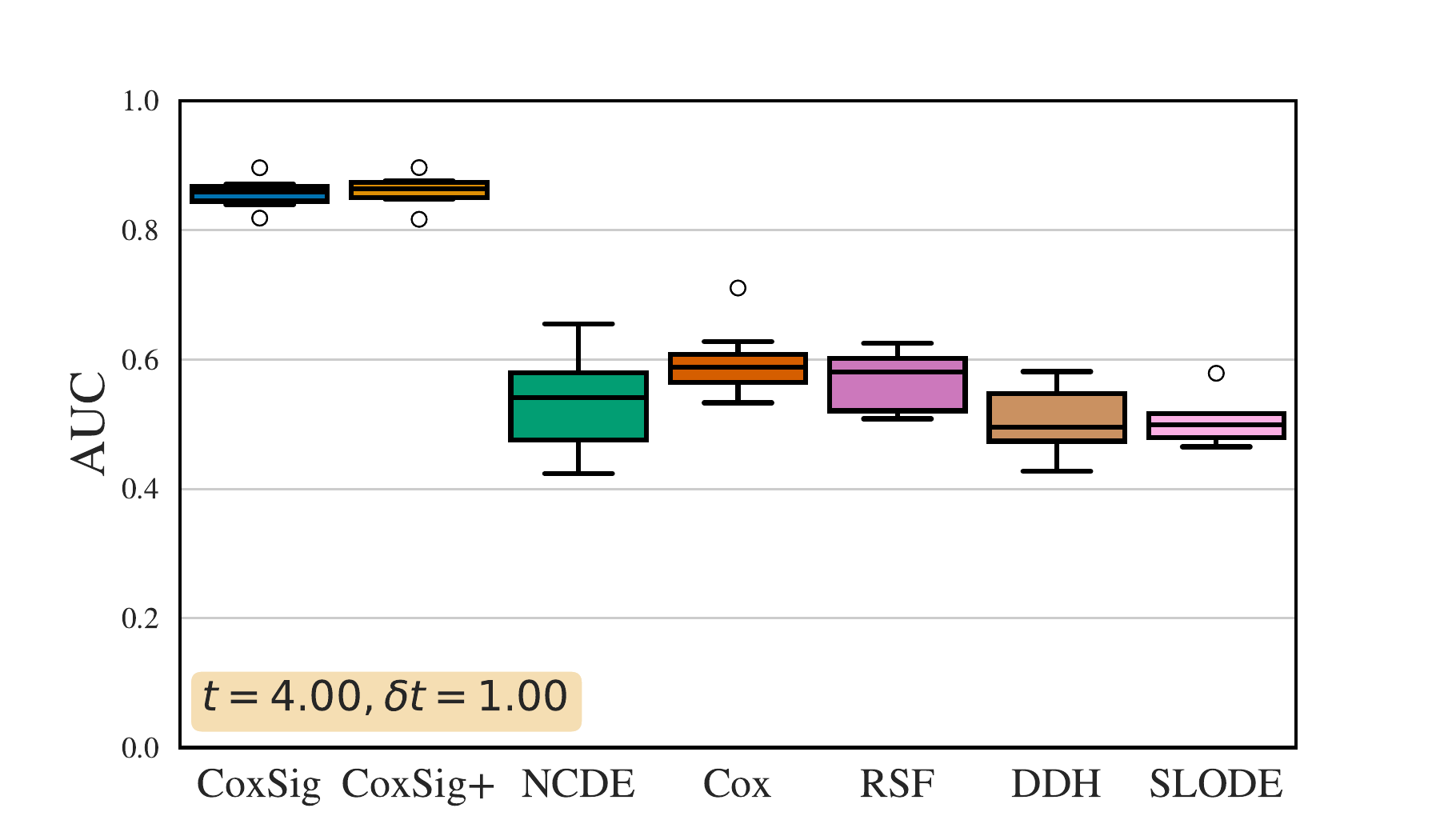}
    \includegraphics[width=0.33\textwidth]{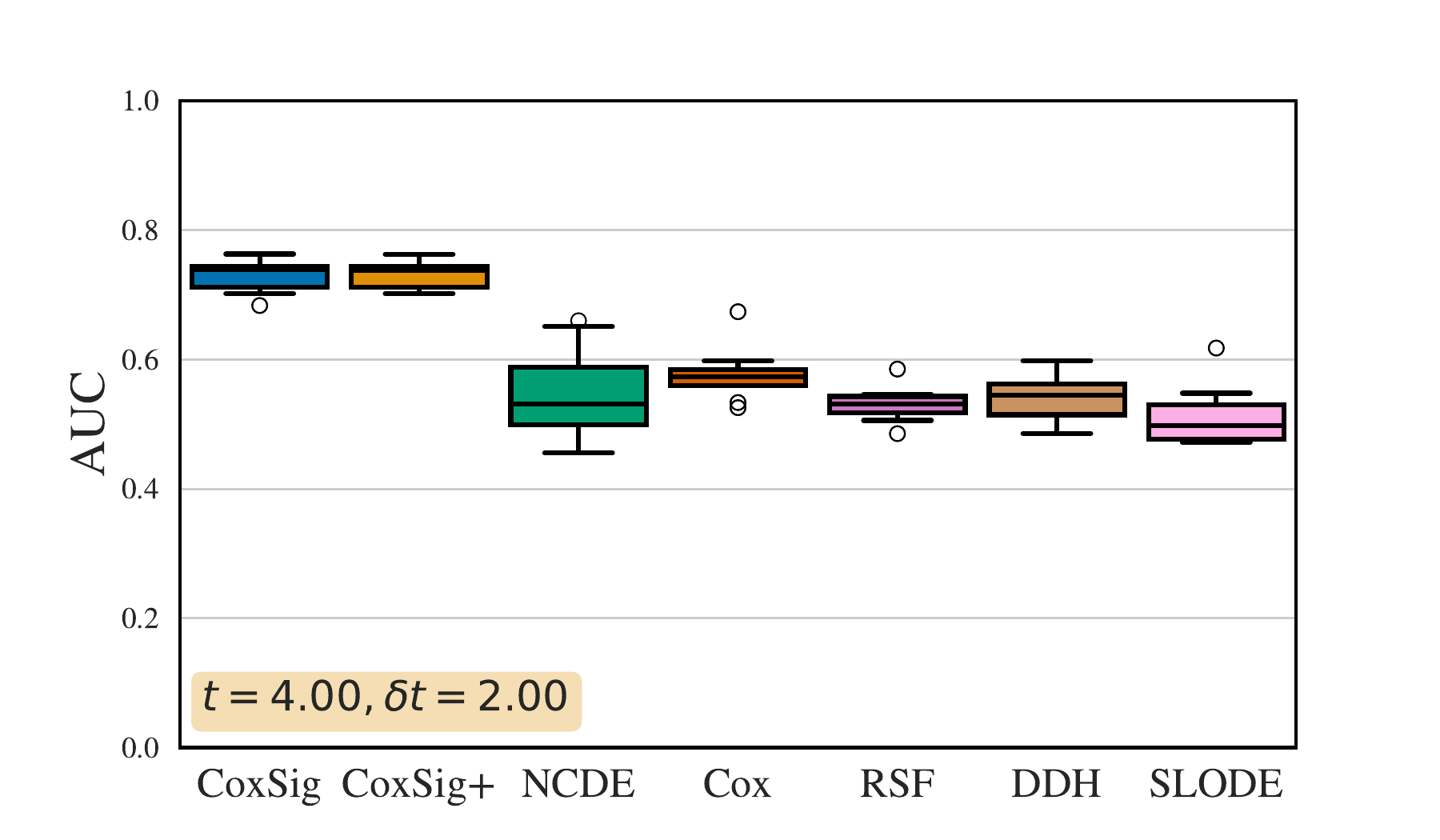}
    \includegraphics[width=0.33\textwidth]{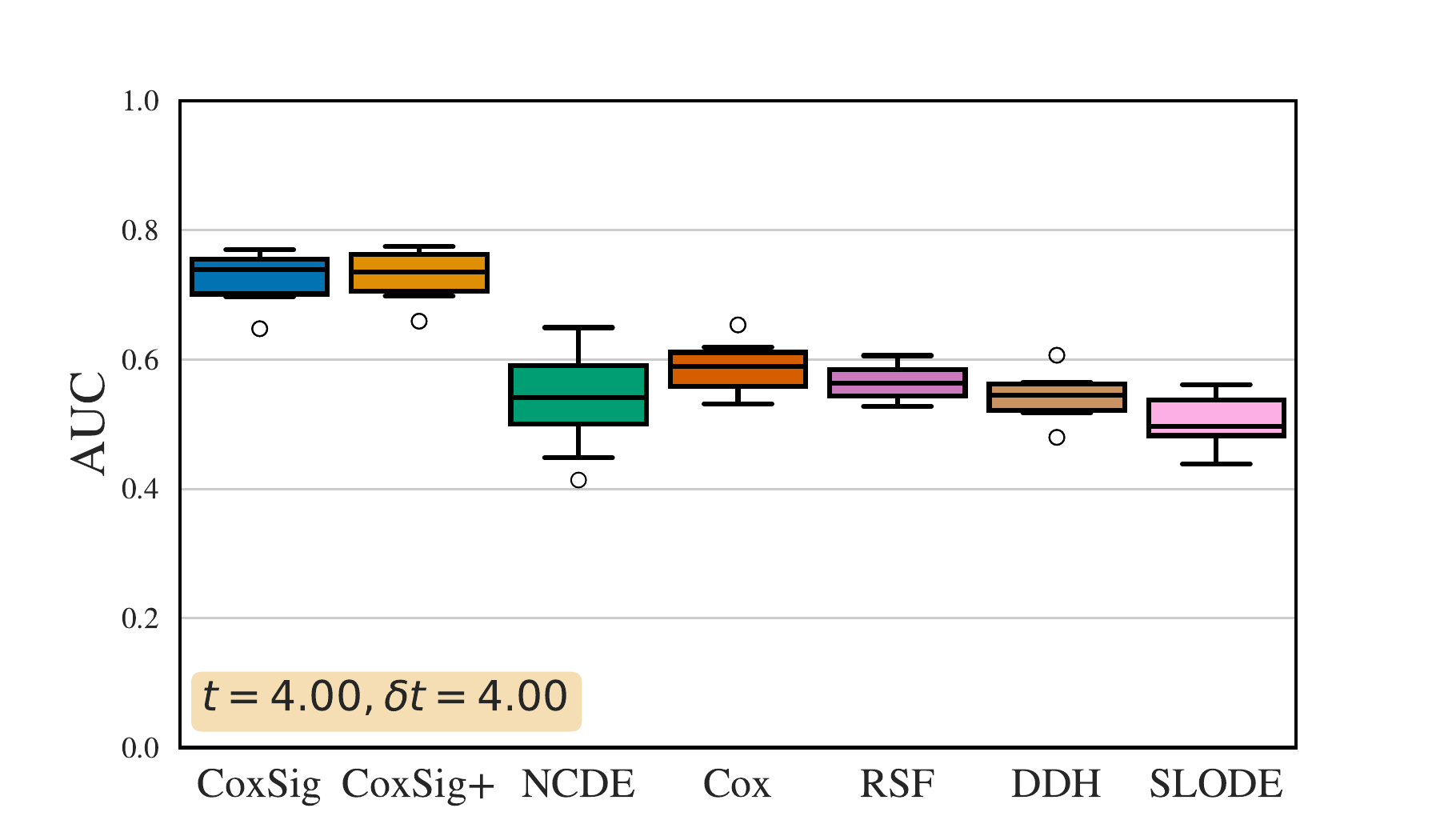}
    \caption{\footnotesize AUC (\textit{higher} is better) for \textbf{churn prediction} for numerous points $(t,\delta t)$.}
    \label{fig:auc_churn}
\end{figure*}

\end{document}